%% file: _main.tex
\documentclass{article}

 \usepackage[preprint]{neurips_2026}

\usepackage[utf8]{inputenc} 
\usepackage[T1]{fontenc}    
\usepackage{hyperref}       
\usepackage{url}            
\usepackage{booktabs}       
\usepackage{amsfonts}       
\usepackage{nicefrac}       
\usepackage{microtype}      
\usepackage{xcolor}         

\usepackage{float}
\usepackage{amsmath}
\usepackage{amssymb}
\usepackage{mathtools}
\usepackage{amsthm}
\usepackage{graphicx}
\usepackage{subcaption}
\usepackage{multirow}
\usepackage{comment}
\usepackage{dsfont}

\newcommand{\figref}[1]{Figure~\ref{#1}}

\theoremstyle{plain}
\newtheorem{theorem}{Theorem}[section]
\newtheorem{proposition}[theorem]{Proposition}

\newtheorem{corollary}[theorem]{Corollary}
\theoremstyle{definition}
\newtheorem{definition}[theorem]{Definition}
\newtheorem{assumption}[theorem]{Assumption}
\theoremstyle{remark}
\newtheorem{remark}[theorem]{Remark}

\title{Feedback Manipulation Regularization: Enabling Offline Agent Alignment for Imitation Learning}

\author{%
  Benjamin D. Poole\\
  Department of Computer Science\\
  University of North Carolina at Charlotte\\
  \texttt{bpoole16@charlotte.edu} \\
  \And
  Minwoo Lee\\
  Department of Computer Science\\
  University of North Carolina at Charlotte\\
  \texttt{minwoo.lee@charlotte.edu} \\
}
\begin{document}

\maketitle

\begin{abstract}
Reinforcement learning (RL) research has increasingly shifted focus towards alignment, ensuring agents learn behaviors adhering to human values. While human demonstrations and feedback have proven crucial for alignment, existing approaches predominantly combine these signals using multi-stage pipelines designed for the contextual bandit framing of language generation. Yet little work explores how these complementary inputs can serve as a richer, interconnected signal for single-stage offline training in fully sequential decision-making environments. We propose Feedback Manipulation Regularization (FMR), an algorithm-agnostic method that harnesses evaluative feedback as a corrective signal to improve the alignment of imitation learning policies. We adapt Safety Gymnasium environments to be a principled testbed for alignment evaluation, demonstrating improved aptitude and up to a 98\% reduction in misalignment across a range of imitation learning algorithms. FMR remains robust in limited data regimes, even when learning from scarce aligned and uninformative noisy demonstrations.
\end{abstract}

\section{Introduction}\label{sec:intro}

Reinforcement learning (RL) has traditionally faced two major challenges: aptitude and alignment. While early RL focused on aptitude, improving an agent's ability to learn and master a problem, recent concerns have shifted toward alignment, ensuring agents learn in ways that adhere to human values, intentions, and preferences \citep{ji_align_2025}. These alignment challenges are expected to intensify as agents become more capable \citep{leike_alignment_2018, amodei_aisafty_2016, dulacarnold_challenges_2019}. One prominent solution is Interactive RL (Int-RL), which integrates human input directly into the RL framework \citep{ li_irlsurvey_2019, arzate_irlsurvey_2020, najar_irlsurvey_2021}. Demonstrations and feedback are of particular interest as intuitive mediums for communication. Demonstrations allow humans to explicitly depict desired behavior, while feedback (e.g., preferences) excels at refining behavior by indicating what is good, bad, or simply preferred \citep{zare_ilsurvey_2024, najar_irlsurvey_2021, christiano_preferences_2017, macglashan_coach_2017}.

Recently, these two modalities have seen significant use in fine-tuning Large Language Models via RL to better align model responses with human expectations \citep{ouyang_training_2022}. These approaches commonly employ Learning from Preferences (LfP) or Reinforcement Learning from Human Feedback (RLHF), yet rely on sequential, multi-stage pipelines that can require environment access or an additional supervised trained policy \citep{ouyang_training_2022, lui_on2off_2025, rafailov_dpo_2023}. Only recent works have begun to explore combining demonstrations with feedback in a single-stage offline scenario \citep{meng_simpo_2024, hong_orpo_2024, hejna_cpl_2024, hejna_ipl_2023}. 

Despite this progress, these methods remain primarily designed within the contextual bandit framing of language generation, which does not readily translate to traditional, fully sequential decision-making environments. Consequently, only limited work has attempted to apply preference-based learning to such settings \citep{hejna_cpl_2024, hejna_ipl_2023}. Yet, in terms of alignment, preference-based feedback carries several notable limitations in these fully sequential settings. Firstly, since preference is a relative judgment rather than an absolute declaration of alignment, a preferred behavior is not necessarily an aligned one. Secondly, this ambiguity is further compounded when comparison pairs are generated or sampled randomly, making it difficult to target specific behaviors without explicitly curating aligned and misaligned pairs. Therefore, preferences are fundamentally insufficient for aligning sequential decision-making behavior in RL settings where behavioral precision is essential.

To address the limitations of preference based approaches, we turn to evaluative feedback. Inspired by TAMER and COACH \cite{knox_interactively_2009, macglashan_coach_2017}, evaluative feedback operates at the state-action pair level, asking "is this behavior good or bad?" rather than "which of the two is better?". This removes the relational ambiguity of preference feedback and enables direct targeting of specific behaviors. This feedback structure motivates a reframing of how demonstrations are incorporated, which we formalize through imitation learning. Pairing evaluative feedback with demonstrations allows aligned sub-trajectories to be encouraged and misaligned behaviors to be suppressed. This formulation can further be extended to Learning from Noisy Demonstrations (LfND), where demonstrations of high alignment quality are treated as expert data and demonstrations containing noisy mixtures of aligned and misaligned behaviors are treated as imperfect data. This formulation enables the combination of demonstrations and evaluative feedback in a purely offline, single-stage pipeline for fully sequential decision-making environments.

To incorporate evaluative feedback into the imitation learning framework, we introduce Feedback Manipulation Regularization (FMR), a framework that is agnostic to the choice of imitation learning algorithm. FMR uses demonstrations to learn the state space and leverages evaluative feedback to directly manipulate the policy toward human aligned behavior via temperature scaling. Adapting various Safety Gymnasium environments, we evaluate alignment by harnessing cost signals to indicate deviation from aligned behavior, providing a principled measure of misalignment \citep{ji_safety-gymnasium_2023}. We show that FMR improves alignment across a wide variety of algorithms in a limited data regime (\figref{fig:hook}). Importantly, as the expert-to-imperfect data ratio shrinks, FMR helps maintain performance in both high and low data overlap scenarios where the imperfect demonstrations contain little to no helpful behaviors. Finally, FMR outperforms alternative approaches adapted to utilize evaluative feedback for reward or as a scoring signal to determine preferences between pairwise comparisons.

\begin{figure}[t]
    \centering
    \begin{subfigure}[b]{\textwidth}
        \hspace{0.6cm}
        \includegraphics[width=.85\textwidth,trim=0 665 0 0,clip]{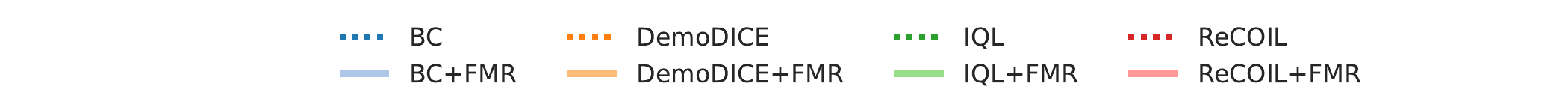}
    \end{subfigure}
    
    \vspace{0.5em}
    
    \begin{subfigure}[b]{0.50\textwidth}
        \centering
        \includegraphics[width=\textwidth,trim=0 0 0 82,clip]{images/main/PathM/misalignment_10-50.pdf}
    \end{subfigure}
    
    \caption{Imitation learning algorithms suffer from high misalignment. Our proposed method, FMR (solid lines) is able to greatly reduce misalignment for each corresponding baseline (dotted lines) without sacrificing aptitude performance. Results depict misalignment learning curves for the PathM navigation policy using a 10-50 data ratio. The shaded region represents the standard error over 5 seeds.}
    \vspace{-0.5cm}
    \label{fig:hook}
\end{figure}

\section{Related Work}

\paragraph{Learning for Noisy Demonstrations} 
Common imitation learning methods include Behavioral Cloning (BC), which maps observations to actions via supervised learning, and IQ-Learn, which frames imitation as inverse soft-RL without explicit reward models \citep{garg_iq-learn_2021}. Meanwhile, Learning from Noisy Demonstrations (LfND) addresses the realistic mixed-quality demonstration setting by learning from distinct expert and imperfect datasets \citep{zare_ilsurvey_2024}. Notable LfND methods include DemoDICE, which matches occupancy to expert data while regularizing against suboptimal data \citep{kim_demodice_2021}, and ReCOIL, which relaxes occupancy-matching to enable imitation from arbitrary off-policy data \citep{sikchi_dual_2024}. Despite outperforming standard imitation learning baselines, LfND methods frequently yield misaligned policies even when expert demonstrations consist entirely of aligned behaviors, a limitation that worsens as the expert-to-imperfect demonstration ratio decreases.

\paragraph{Interactive Reinforcement Learning}
Int-RL integrates human input directly into RL, typically online, in the form of advice (e.g., feedback) or demonstrations \citep{najar_irlsurvey_2021, zare_ilsurvey_2024}. Classical methods like TAMER \citep{knox_interactively_2009} and COACH \citep{macglashan_coach_2017} apply evaluative feedback at the state-action level to label behaviors as ``good'' or ``bad'', with credit assignment enabling targeted sub-trajectory labeling. In contrast, Learning from Preferences (LfP) methods use comparative feedback on trajectory pairs to infer a reward function or policy directly \citep{christiano_preferences_2017, ouyang_training_2022, rafailov_dpo_2023}. While works such as \citep{hejna_ipl_2023} and \citep{hejna_cpl_2024} learn policies directly from offline preferences, this feedback is fundamentally relative rather than absolute. This means a preferred behavior is not necessarily an aligned one. When both behaviors in a comparison pair are misaligned, the least misaligned may still be preferred. Furthermore, randomly sampled comparison pairs limit the ability to target specific aligned or misaligned behaviors, meaning preference feedback provides only a coarse, relative signal ill-suited for the behavioral precision required for alignment in sequential decision-making settings.

\paragraph{Alignment Evaluation} 
Historically, RL performance has focused on total return, leaving the open question of how to properly quantify agent alignment. Return-based alignment metrics \citep{lui_on2off_2025} fail when misaligned behaviors achieve similar returns, while reward model evaluations do not generalize to methods that forgo modeling reward functions \citep{malik_rewardbench_2025}. Meanwhile, SafeRL has long employed Constrained Markov Decision Processes, where environment-elicited costs quantify constraint violations for evaluating the ``safety'' of behaviors \citep{garcia_safesurvey_2015, ji_safety-gymnasium_2023}. Generalizing this idea, cost serves as a natural evaluative metric for alignment, not as a training objective, directly detecting when a policy exhibits misaligned behavior.

\paragraph{Temperature Scaling}  
Temperature scaling divides logits by a temperature $\tau > 0$ before softmax, producing class probabilities $p_i = \frac{e^{z_i / \tau}}{\sum_{j} e^{z_j / \tau}}$, and has seen wide adoption in knowledge distillation \citep{hinton_distilling_2015}, calibration \citep{guo_calibration_2017, xie_calibrating_2024}, and contrastive learning \citep{wang_understanding_2021}. Two extensions are particularly relevant: adaptive temperature scaling predicts a per-input temperature for flexible adjustment \citep{xie_calibrating_2024}, while vector temperature scaling assigns a distinct temperature to each class \citep{guo_calibration_2017}. FMR takes inspiration from both, using feedback to selectively redistribute probability mass toward more aligned behaviors, but does so in probability space rather than in logit space.
 
\section{Preliminaries}
We model the environment as a modified Markov Decision Process (MDP) to account for human feedback $\mathcal{M} = (\mathcal{S}, \mathcal{A}, p, r, h, \gamma)$ with state space $\mathcal{S}$, action space $\mathcal{A}$, transition function $p: \mathcal{S} \times \mathcal{A} \rightarrow \Delta(\mathcal{S})$ over a distribution of states, reward function $r: \mathcal{S} \times \mathcal{A} \rightarrow \mathbb{R}$, human feedback function $h: \mathcal{S} \times \mathcal{A} \rightarrow \{-1, 0, 1\}$, and discount factor $\gamma \in [0, 1]$. The goal of imitation learning and RL is to learn a policy $\pi : \mathcal{S} \rightarrow \Delta(\mathcal{A})$ mapping states to distributions over actions. In RL, this is done by maximizing the return, $G = \sum_{t=0}^{\infty} \gamma^t r(s_t, a_t)$ discounting the future by $\gamma$. In imitation learning, $r$ is unknown, which means it is often estimated, explicitly or implicitly. We define $D^U = D^E \cup D^I$ as the union of the expert and imperfect datasets containing tuples without reward $\mathcal{M} \backslash \{r\}$. 

Classical approaches to imitation learning such as BC aim to learn a policy $\pi$ that maps state $s$ to action $a$ via supervised learning by minimizing the negative log-likelihood:
\begin{equation}
    \min\limits_\pi \mathcal{L}_{\text{BC}}(\pi) = \min\limits_\pi -\frac{1}{| D^U|}\sum_{(s,a) \in D^U} \log \pi(a|s).
\end{equation}
Policy learning methods in LfND build upon this core idea to improve robustness against noise in $D^I$ \citep{kim_demodice_2021, sikchi_dual_2024}.

Additionally, many imitation works aim to estimate the value of a state, where $V_\pi: \mathcal{S} \rightarrow \mathbb{R}$ denotes the state value function of $\pi$. The state value function estimates the expected return $V_\pi(s) = \mathbb{E}_\pi \big [ G | s_0 = s\big ]$ when starting at $s$ and thereafter following $\pi$. Likewise, $Q_\pi:\mathcal{S} \times \mathcal{A} \rightarrow \mathbb{R}$ is the state-action value function of $\pi$ where $Q_\pi(s, a) = \mathbb{E}_\pi \big [ G | s_0 = s, a_0 = a\big ]$. Although traditional value function estimation relies on rewards $r$, imitation learning algorithms have developed alternative techniques for estimating value functions without explicit reward signals \citep{kim_demodice_2021, sikchi_dual_2024, garg_iq-learn_2021}. 

\section{Feedback Manipulation Regularization}\label{sec:fmr}
In this section, we introduce our proposed method of feedback manipulation regularization (FMR) for harnessing both demonstrations and feedback as an interconnected source of information. For this, the goal of FMR is to enable humans to refine an agent's policy by providing evaluative feedback $h$ on demonstrations.

Let $D^U$ denote an imbalanced dataset of demonstrations composed of aligned (expert) and imperfect trajectories where $|D^E| \ll |D^I|$. We make the following assumption about the learned policy $\pi_\theta$ and latent human-aligned policy $\pi_h$.

\begin{assumption}[Coverage]\label{ass:coverage}
For any $s\in\mathcal{S}$, if $\pi_h(a\mid s)>0$ then $\pi_\theta(a\mid s)>0$.
\end{assumption}

Since $\pi_h$ is unknown, we approximate it using $\pi_\theta$, where feedback acts as a corrective signal that redistributes probability mass towards the underlying aligned policy. To achieve this, we take inspiration from temperature scaling for rescaling distributions.

\begin{definition}[Feedback and Temperature]\label{def:feedback-temperature}
Human feedback $h(s,a_j)$ induces a \textit{per-action} temperature $\tau(s, a_j)$ (for simplicity, denoted by $\tau_j$ hereafter) via a strictly decreasing map. Hence, positive $h$ implies $\tau_j < 1$ (\textit{encourage}), and negative $h$ implies $\tau_j > 1$ (\textit{discourage}).
\end{definition}

\begin{remark}[Feedback-Adaptive Temperature]\label{rem:fat-map}
We propose a simple exponential equation given as:
\begin{equation}\label{eq:tau_j}
    \tau_j = \begin{cases}
\beta^{-h} & \text{if } h < 0 \text{ and } j = \text{selected action}, \\
\beta^{h} & \text{if } h > 0 \text{ and } j \neq \text{selected action}, \\
1 & \text{if } h = 0, \\
\end{cases}
\end{equation}
where $\beta > 1$ is a hyperparameter. In a discrete action space, if negative feedback is received, the action selected receives $\tau > 1$. If positive feedback is received, all other actions\footnote{This is akin to assigning $\tau < 1$ to the selected action, but prevents negative loss due to the logarithm.}, other than the action selected, receive $\tau > 1$. When no feedback is received $\tau_j = 1$, which has no impact on learning. Feedback magnitude can be achieved by using credit assignment \citep{knox_interactively_2009} to map feedback to prior states. Repeatedly elicited feedback given in quick succession will then lead to larger magnitude and a greater redistribution of the probability mass.
\end{remark}

\begin{definition}[Temperature Scaling for Human Aligned Policy Approximation]\label{def:temp-scaled_policy}
Given $\pi_\theta(\cdot\mid s)$ and $\tau(\cdot\mid s)$, define the \emph{approximated human-aligned policy} as 
\[
\pi_{\tau}(a_j\mid s)\;:=\; \frac{\pi_\theta(a_j\mid s)\,\tau_j^{-1}}{Z_\theta(s)}, 
\]
where $Z_\theta(s)=\sum_{a_k\in\mathcal{A}} \pi_\theta(a_k\mid s) \tau_k^{-1}$.
\end{definition}
By operating directly in probability space, Definition~\ref{def:temp-scaled_policy} requires no assumptions on how $\pi_\theta$ is computed. Nevertheless, the temperature scaling intuition is preserved where misaligned behaviors are discouraged through redistribution of probability mass, making aligned demonstrations more probable even without explicit positive reinforcement.

With $\pi_\tau$ defined as a feedback-tempered surrogate for $\pi_h$, we formulate policy alignment as a distribution matching problem via reverse KL divergence, a common approach in imitation learning \citep{kim_demodice_2021, sikchi_dual_2024}.

\begin{definition}[Reverse-KL Formulation]\label{def:reverse-kl}
We align $\pi_\theta$ to a feedback-tempered surrogate of $\pi_h$ by minimizing the reverse KL divergence
\[
\mathcal{J}(\theta;\tau) \;:=\; \mathbb{E}_{(s,a)\sim D^U}\!\Big[ D_{\mathrm{KL}}\!\big(\pi_\theta\,\|\, \pi_{\tau}\big) \Big].
\]
\end{definition}

\begin{proposition}[Reverse-KL Decomposition]\label{prop:rkl-decomp}
For any fixed $s$,
\[
D_{\mathrm{KL}}\!\big(\pi_\theta\,\|\, \pi_{\tau}\big) =
\mathbb{E}_{a_j\sim \pi_\theta(\cdot\mid s)}\!\big[\log \tau(s,a_j)\big] + \log Z_\theta(s).
\]
\begin{proof}
From Definition~\ref{def:temp-scaled_policy}, $\log \pi_\tau(a_j\mid s) = \log \pi_\theta(a_j\mid s) - \log \tau_j - \log Z_\theta(s)$. Substituting into the KL divergence and noting that the $\log \pi_\theta(a_j\mid s)$ terms cancel,
\[
\sum_j \pi_\theta(a_j\mid s) \log \frac{\pi_\theta(a_j\mid s)}{\pi_\tau(a_j\mid s)}
= \sum_j \pi_\theta(a_j\mid s) \log \tau_j + \log Z_\theta(s).
\]
\end{proof}
\end{proposition}

\begin{corollary}[Generalized KL]\label{cor:gen-kl}
Let the unnormalized target measure be $\tilde \pi_\tau(a\mid s)\propto \pi_\theta(a\mid s)\,\tau(s,a)^{-1}$. Then
\[
D_{\mathrm{KL}}\!\big(\pi_\theta(\cdot\mid s)\,\|\, \tilde \pi_\tau(\cdot\mid s)\big)
\;=\;
\mathbb{E}_{a\sim \pi_\theta(\cdot\mid s)}\!\big[\log \tau(s,a)\big] + \mathrm{const}.
\]
\end{corollary}

\begin{remark}[Connection to Entropy Regularization]
Setting $\tau(s,a) = \pi_\theta(a\mid s)$ recovers entropy regularization \citep{williams_function_1991, mnih_a3c_2016}. More generally, since $\tau$ is derived from $h$, feedback modulates the entropy of $\pi_\theta$ where negative feedback increases 
entropy while positive feedback decreases it.
\end{remark}

\begin{definition}[Feedback Manipulation Regularization]\label{def:fmr-reg}
Given $\tau$ induced by feedback via Definition~\ref{def:feedback-temperature}, the surrogate regularizer is
\[
\mathcal{R}_{\mathrm{FMR}}(\theta) := \mathbb{E}_{s\sim D^U}\,\mathbb{E}_{a\sim \pi_\theta(\cdot\mid s)}\big[\log \tau(s,a)\big].
\]

Any imitation learning objective can incorporate FMR as
\begin{equation}
\mathcal{L}(\theta)=\mathcal{L}_{\mathrm{IL}}(\theta) \;+\; \alpha\, \mathcal{R}_{\mathrm{FMR}}(\theta),
\end{equation}
where $\mathcal{L}_{\mathrm{IL}}$ is the imitation learning loss and $\alpha > 0$ controls alignment strength.
\end{definition}

Together, Definitions~\ref{def:feedback-temperature} and~\ref{def:temp-scaled_policy} specify how feedback systematically reshapes $\pi_\theta$, while Definition~\ref{def:fmr-reg} provides a model-agnostic regularization term that enforces this alignment under any imitation learning objective.

\section{Experiments}\label{sec:exps}
This section aims to address the following questions: (1) Can FMR enhance the task alignment performance of imitation learning algorithms? (2) How does FMR compare to alternative approaches adapted to utilize evaluative feedback, and can it outperform methods that substitute feedback for reward or use it for pairwise preference scoring? We evaluate all methods across aligned-to-imperfect demonstration data ratios of 10-50 (1:5), 25-50 (1:2), and 50-50 (1:1) to test algorithmic robustness as aligned data $D^E$ becomes increasingly scarce. Additional details regarding the task, aligned policy, and data collection, along with visualizations, are provided in Appendix \ref{app:task-data-details}. Results for the velocity tasks, additional plots for navigation results, and extensive additional experiments are provided in Appendix \ref{app:add-results}.

\subsection{Experimental Setup}\label{sec:exp-setup}

\paragraph{Environments} We adapt environments from Safety Gymnasium \citep{ji_safety-gymnasium_2023}, built on Gymnasium \citep{towers_gymnasium_2024} and MuJoCo \citep{todorov_mujoco_2012}, using cost as a principled measure of misalignment exclusively for evaluation. We evaluate on a 3D navigation task and three velocity-restricted locomotion tasks (Hopper, Swimmer, Walker2D), each with a discretized action space for human demonstration collection. We adapt the original navigation task so that an agent (red sphere) must reach a fixed goal (green cube) using lidar observations, spawning in a randomized upper-right area. To evaluate alignment of the navigation task, two unobservable hazard variants are introduced: passable floor hazards and blocking hazard walls, each incurring a cost of 1 per violation. For velocity tasks, a cost of 1 is incurred per step exceeding a task-specific x-velocity threshold, inducing velocity-restricted gaits as aligned policies. 

\paragraph{Aligned Policies} We define two aligned policies for the navigation task. PathM requires the agent to move diagonally through the center via a narrow hazard-free corridor, while PathBB requires the agent to move along the right boundary then the bottom boundary, entering the goal from below. Hazards are placed to highlight each aligned path and measure deviation, with hazard walls positioned around the goal to enforce directional entry for PathBB. For each aligned policy, $D^E$ comprises aligned human demonstrations that follow the aligned policy's designated path. All demonstrations within $D^I$ do not reach the goal and circle the goal to obfuscate the objective. Critically, PathM exhibits high overlap between $D^E$ and $D^I$, as imperfect demonstrations frequently traverse the center corridor, while PathBB exhibits low overlap, as imperfect demonstrations rarely follow the boundary bottom path. High overlap helps to evaluate whether aligned sub-trajectories in $D^I$ can be identified from misaligned demonstrations, while low overlap looks at generalization to underrepresented aligned behaviors within $D^I$. Data overlap helps to evaluate how aligned sub-trajectories within otherwise misaligned demonstrations contribute to policy learning, with lower overlap making the problem more challenging due to containing fewer relevant sub-trajectories.

For the velocity tasks, we define one aligned policy per environment by imposing an x-velocity threshold, producing a distinctive gait that contrasts with the optimal high-speed policy. SlowHop ($v \leq 0.74$) performs slow forward hops with brief pauses, SlowSwim ($v \leq 0.75$) shortens the peaks of the optimal sinusoidal gait into a slower serpentine motion, and SlowWalk ($v \leq 1.50$) produces a measured bipedal walk in contrast to the optimal running gait. For each velocity task, $D^E$ comprises aligned demonstrations that follow the aligned policy's velocity-restricted gait. Swimmer data has high overlap as $D^E$ contains similar sinusoidal gait to optimal trajectories present in $D^I$ but with reduced amplitude. Hopper, by contrast, has low overlap as $D^I$ predominantly contains demonstrations exceeding the velocity threshold. Walker2D falls between the two, having medium overlap, where the real difficulty comes with the large action space ($|\mathcal{A}|= 3^6$). 

\paragraph{Data and Feedback Collection} For navigation data, $D^E$ comprises 50 aligned human demonstrations collected per aligned policy, and $D^I$ consists of 50 imperfect human demonstrations shared across policies. For velocity data, human proxy demonstrations are collected, where $D^E$ comprises 50 aligned demonstrations from a cost-augmented SAC-Discrete agent \citep{haarnoja_soft_2018, christodoulou_soft_2019}, and $D^I$ comprises 50 demonstrations sampled across standard SAC-Discrete training, yielding variation in quality and dataset overlap. Navigation feedback is collected from a real human evaluator on $D^I$ via a visual replay interface where a 600ms credit assignment window is used to map feedback to prior states \citep{knox_interactively_2009}. For velocity data, proxy feedback is generated by converting cost violations into negative feedback only. Swimmer and Walker2D use dense feedback per violation, while Hopper uses sparser feedback given at violation onset and for 50 subsequent steps. Walker2D additionally receives negative feedback for the 15 steps preceding failure. 

\paragraph{Baselines} To assess the effectiveness of FMR in enhancing the alignment of baseline imitation learning methods trained on $D^U$, we evaluate both traditional imitation learning and robust LfND algorithms such as BC, IQL \citep{garg_iq-learn_2021}, DemoDICE \citep{kim_demodice_2021}, and ReCOIL \citep{sikchi_dual_2024}, with and without FMR augmentation. These represent diverse algorithmic paradigms: actor-critic (DemoDICE, ReCOIL), soft Q-learning (IQL), and supervised learning (BC). Notably, BC and IQL assume minimal imperfect demonstrations, making LfND challenging for them, while DemoDICE and ReCOIL are explicitly designed for this setting. To examine alternative approaches to harnessing evaluative feedback, we use Dual Value Learning (DVL) \citep{sikchi_dual_2024}, an offline RL variant of ReCOIL, where we substitute evaluative feedback for reward. Likewise, we employ Contrastive Preference Learning (CPL) \citep{hejna_cpl_2024}, a single-stage offline learning from preferences (LfP) algorithm, in which evaluative feedback serves as the comparison score for determining preference. This is achieved by calculating the total feedback accumulated per trajectory segment.

\paragraph{Metrics} To assess the aptitude performance for the navigation policies, we report \textit{success rate}, defined as the ratio of episodes that reach the goal to total evaluation episodes. Alternatively, for the velocity policies, we report \textit{total normalized return} where return is normalized based on the performance of $D^E$ (Appendix \ref{app:vel-details}). For all tasks, alignment performance is measured using \textit{misalignment}, defined as the ratio between the number of steps that incur a cost $c_t$ and the total number of steps in the episode $T$ as $\frac{1}{T}\sum_{t=1}^T c_t$.

\paragraph{Implementation Details} All algorithms trained for 1 million batches with evaluations every 10k iterations (50 episodes each). To evaluate the final average performance and standard deviation for each metric, we compute the mean over the last 10 evaluations across all 5 random seeds. When decreasing the ratio between $D^E$ and $D^I$, we find oversampling $D^E$ helps to mitigate lose in performance due to dataset imbalance across all algorithms. For the navigation data, we remove duplicate states occurring at episode initialization. For FMR, we use $\beta = 10$ and $\alpha = 1$ for all algorithms unless stated otherwise. Additionally, for DemoDICE+FMR, ReCOIL+FMR and DVL, we found that they benefit from cosine learning rate decay when oversampling, while the other algorithms showed no significant benefit. For additional implementation details see Appendix \ref{app:imp-details}.

\begin{figure}[t]
    \centering
    
    \begin{subfigure}[b]{\textwidth}
        \centering
        \includegraphics[width=0.65\textwidth,trim=0 275 1 0,clip]{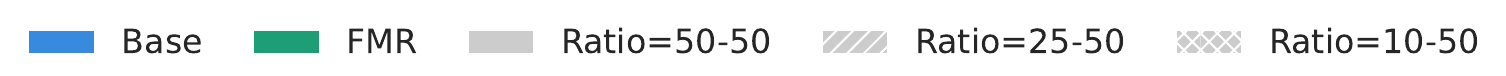}
    \end{subfigure}
    
    \begin{minipage}[t]{0.48\textwidth}
        \centering
        \textbf{PathM}
    \end{minipage}
    \hfill
    \begin{minipage}[t]{0.48\textwidth}
        \centering
        \textbf{PathBB}
    \end{minipage}
    
    \begin{subfigure}[b]{0.45\textwidth}
        \centering
        \includegraphics[width=\textwidth,trim=0 0 0 35,clip]{images/main/PathM/fmr_success_compare.pdf}
    \end{subfigure}
    \hfill
    \begin{subfigure}[b]{0.45\textwidth}
        \centering
        \includegraphics[width=\textwidth,trim=0 0 0 35,clip]{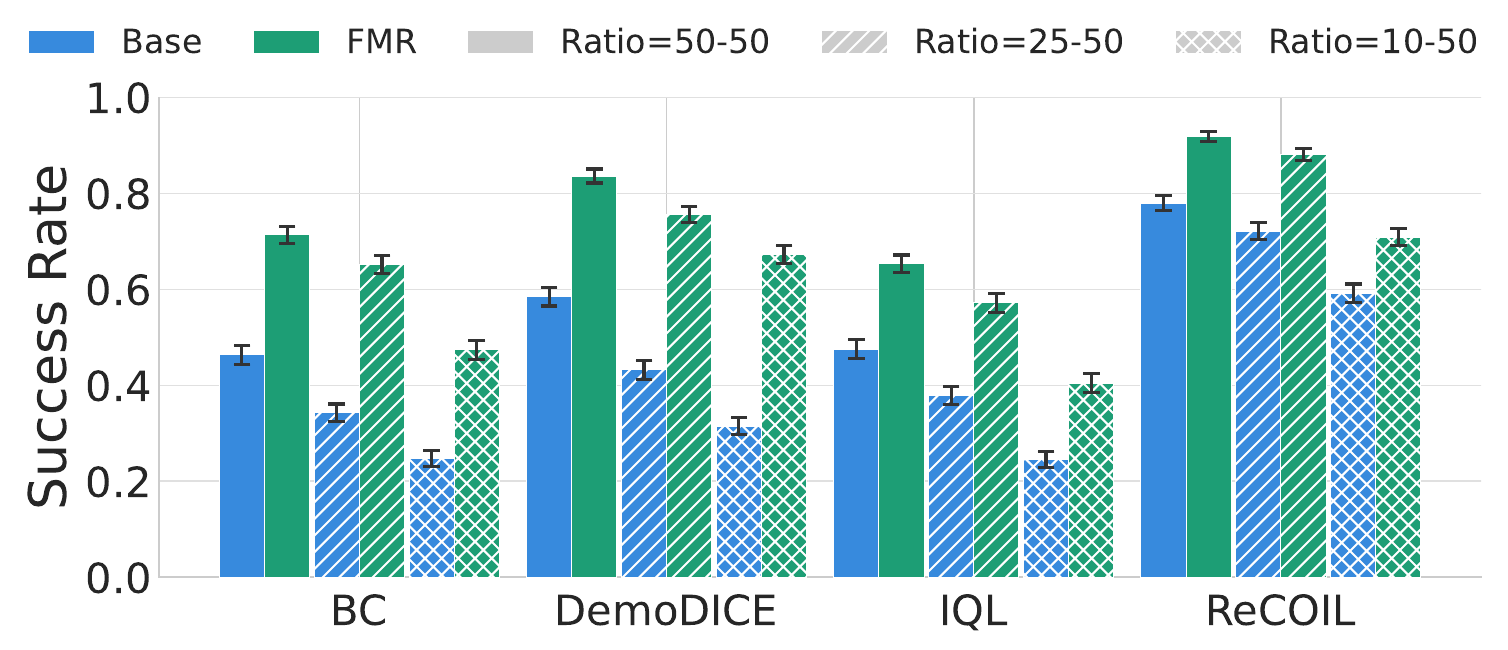}
    \end{subfigure}
    
    \begin{subfigure}[b]{0.45\textwidth}
        \centering
        \includegraphics[width=\textwidth,trim=0 0 0 35,clip]{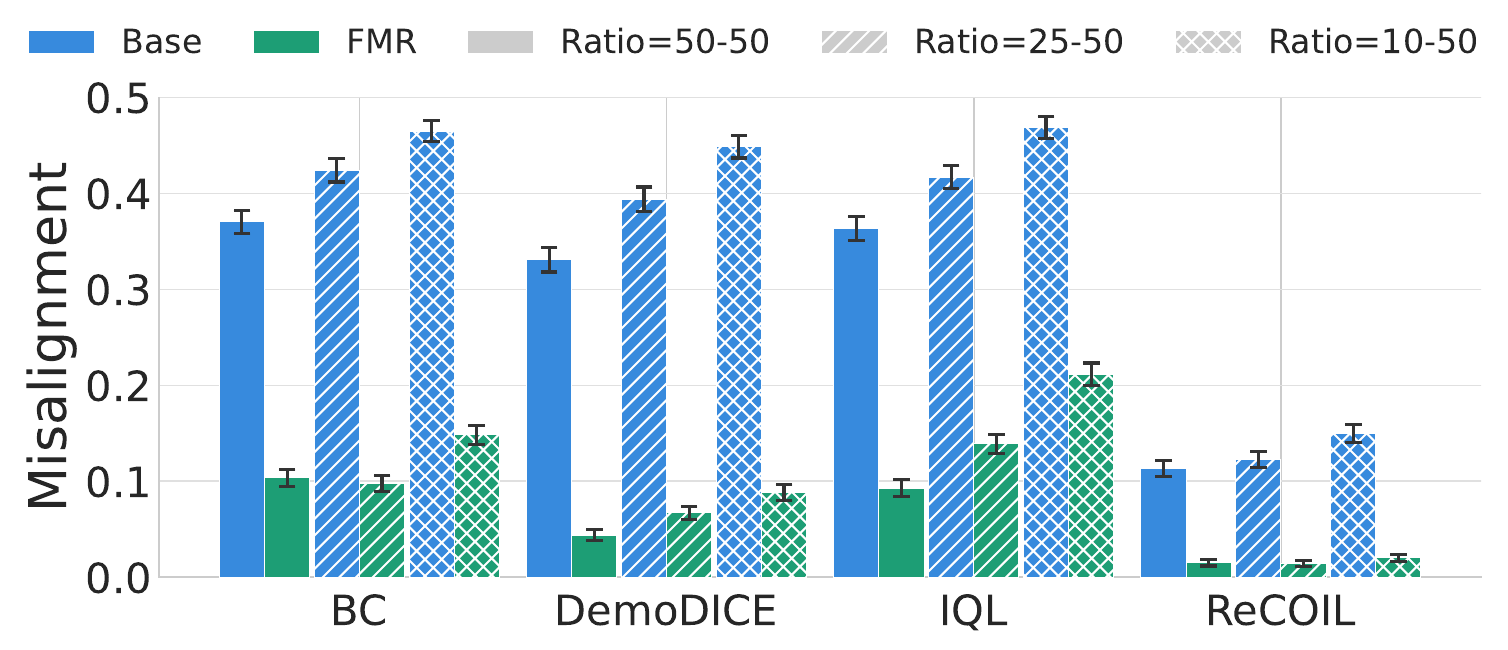}
    \end{subfigure}
    \hfill
    \begin{subfigure}[b]{0.45\textwidth}
        \centering
        \includegraphics[width=\textwidth,trim=0 0 0 35,clip]{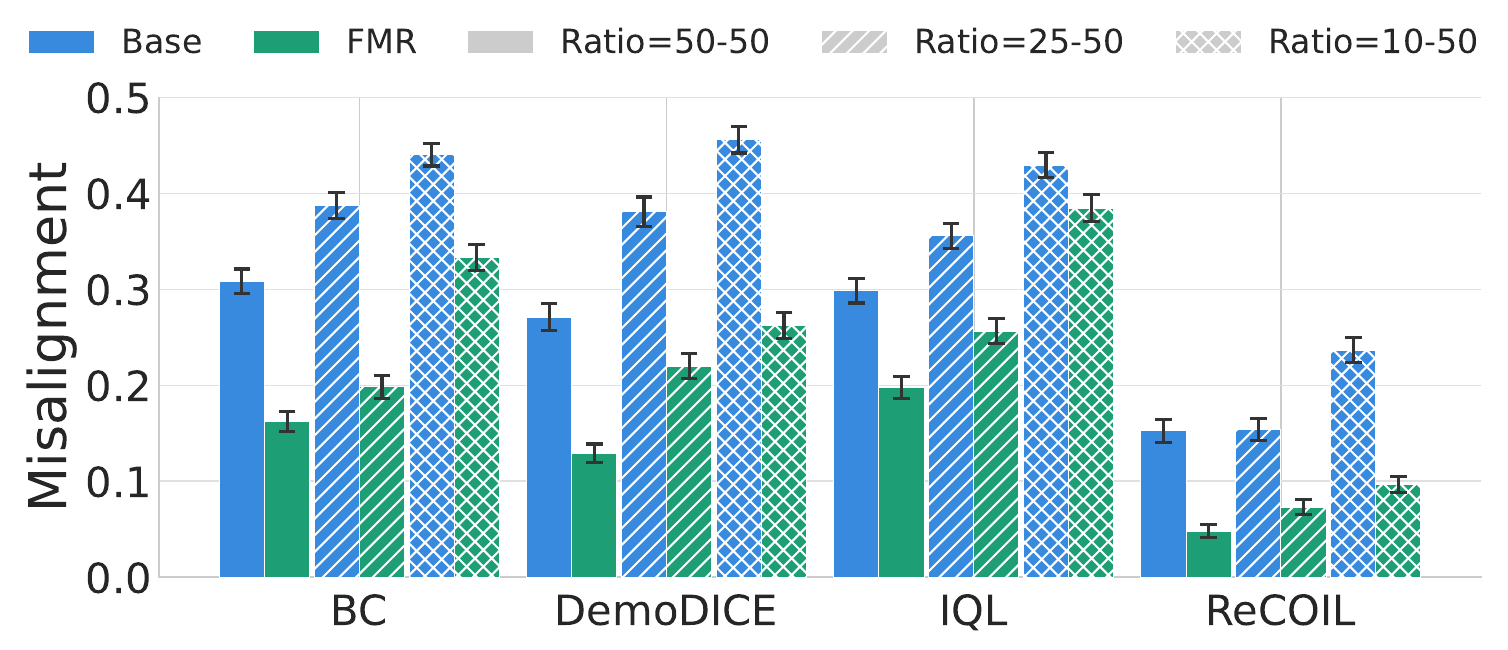}
    \end{subfigure}
    
    \caption{Navigation results across all baseline algorithms and data ratios. Results show mean success rate and misalignment scores with 95\% confidence intervals, averaged over the last 10 evaluations across 5 seeds.}
    \label{fig:nav-results}
\end{figure}


\subsection{Can FMR enhance alignment performance?}\label{sec:fmr-performance}
\figref{fig:nav-results} depicts the aptitude (i.e., success rate) and alignment (i.e., misalignment) performance results for the navigation policies PathM and PathBB across all data ratios. PathM exhibits higher baseline performance than PathBB, likely due to higher overlap between $D^E$ and $D^I$. For both policies, as data ratio decrease, we observe substantial drops in success rate and corresponding increases in misalignment across all baselines. The notable exception is ReCOIL, which demonstrates the strongest baseline performance, particularly on PathM, though still at the cost of higher misalignment. FMR yields dramatic improvements across all baselines, consistently improving success rates and reducing misalignment, even as aligned demonstrations are reduced. ReCOIL+FMR achieves the strongest overall performance, combining ReCOIL's higher success rates with drastically reduced misalignment, achieving up to 92\% reduction on PathM and 67\% on PathBB.

Among the velocity policies, SlowSwim (high overlap data) achieves the strongest baseline performance while SlowHop and SlowWalker prove considerably more challenging (see Appendix \ref{app:vel-results}). For SlowSwim, ReCOIL achieves a return close to $D^E$ with low misalignment, while other baselines reach comparable returns but with significantly higher misalignment that grows as the data ratio decreases. In contrast, all FMR variants demonstrate remarkable stability across data ratios with returns consistently near SlowSwim and up to 98\% reduction in misalignment, with ReCOIL+FMR approaching near-zero misalignment. SlowHop baselines exhibit more varied performance, with BC and IQL significantly exceeding $D^E$ performance and producing much larger misalignment scores, while misalignment worsens as the data ratio decreases across all baselines. FMR variants substantially reduce misalignment by 75\% across all ratios and algorithms, with ReCOIL+FMR maintaining the best alignment. 

SlowWalk follows a similar but more pronounced degradation trend, with baseline misalignment remaining persistently high and returns declining at lower data ratios. For this task, adding positive feedback to all state-action pairs in $D^E$ and increasing $\alpha=2$ for FMR improved performance, an effect attributed to the challenges posed by the large action space. FMR variants consistently reduce misalignment across all algorithms and ratios with ReCOIL+FMR achieving the best trade-off between return and alignment, reducing misalignment by up to 76\% overall. DemoDICE+FMR at the 10-50 ratio is a notable outlier showing little reduction, likely due to the large action space and potentially addressable through better hyperparameter tuning of the base algorithm.

\begin{figure}[t]
    \centering
    
    \begin{subfigure}[b]{\textwidth}
        \centering
        \includegraphics[width=0.65\textwidth,trim=0 275 1 0,clip]{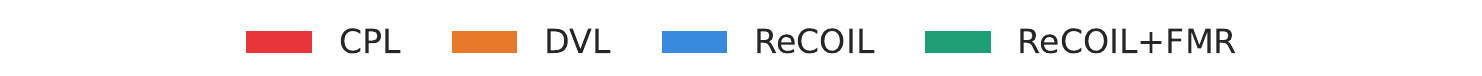}
    \end{subfigure}
    
    \begin{minipage}[t]{0.48\textwidth}
        \centering
        \textbf{PathM}
    \end{minipage}
    \hfill
    \begin{minipage}[t]{0.48\textwidth}
        \centering
        \textbf{PathBB}
    \end{minipage}
    
    \begin{subfigure}[b]{0.40\textwidth}
        \centering
        \includegraphics[width=\textwidth,trim=0 0 0 35,clip]{images/cpl-dvl/PathM/success_compare.pdf}
    \end{subfigure}
    \hfill
    \begin{subfigure}[b]{0.40\textwidth}
        \centering
        \includegraphics[width=\textwidth,trim=0 0 0 35,clip]{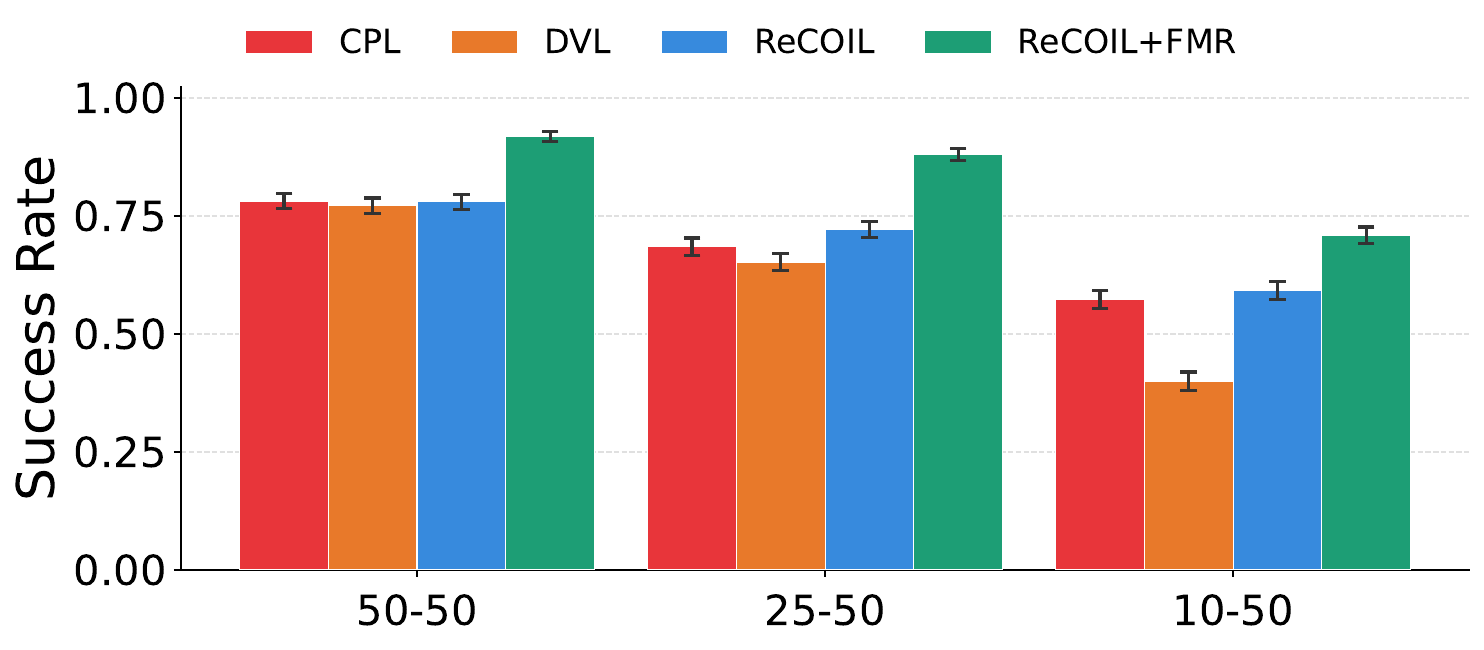}
    \end{subfigure}
    
    \begin{subfigure}[b]{0.40\textwidth}
        \centering
        \includegraphics[width=\textwidth,trim=0 0 0 35,clip]{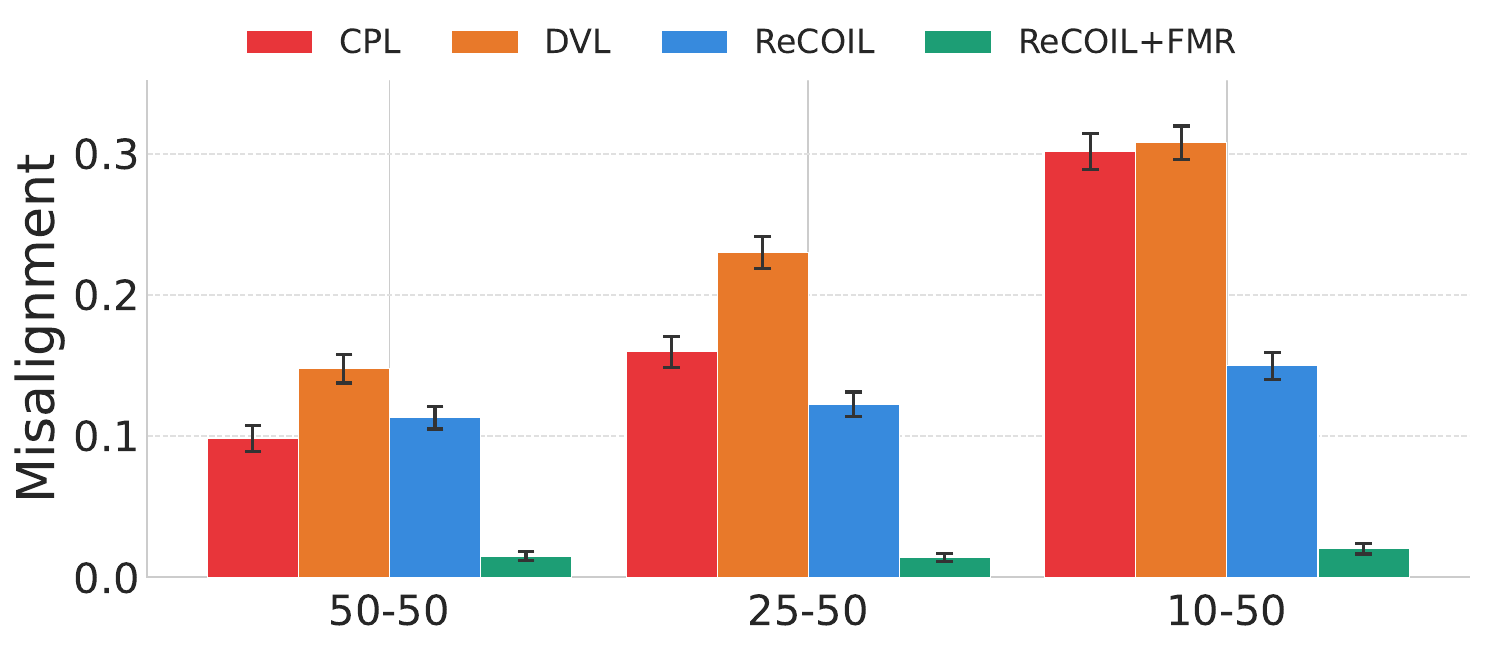}
    \end{subfigure}
    \hfill
    \begin{subfigure}[b]{0.40\textwidth}
        \centering
        \includegraphics[width=\textwidth,trim=0 0 0 35,clip]{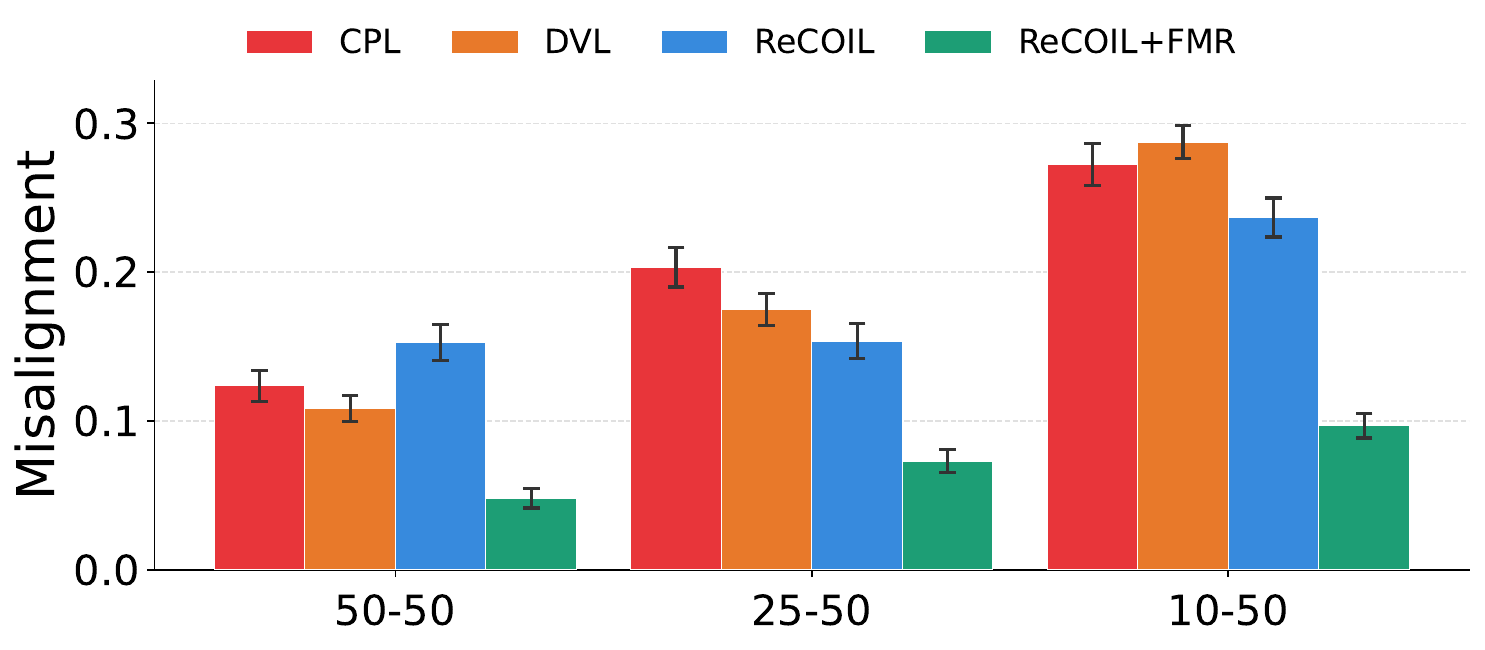}
    \end{subfigure}
    
    \caption{Navigation policies data ratio comparison between FMR, DVL, and CPL with ReCOIL baseline. Results show mean success rate and misalignment scores with 95\% confidence intervals, averaged over the last 10 evaluations across 5 seeds.}
    \label{fig:nav-cpl-dvl-results}
\end{figure}

\subsection{How Does FMR Compare to Alternative Uses of Evaluative Feedback?}\label{sec:alt-methods}
To evaluate FMR's effectiveness at harnessing evaluative feedback, we adapt two popular alternative methods for comparison. Specifically, we compare FMR against DVL \citep{sikchi_dual_2024}, the offline RL version of ReCOIL adapted by replacing its reward signal with evaluative feedback, and CPL \citep{hejna_cpl_2024}, an offline single-stage LfP algorithm adapted by summing evaluative feedback over trajectory segments to derive preference scores between pairs. For both DVL and CPL, in addition to the evaluative feedback for $D^I$, we add positive feedback for all state-action pairs in $D^E$. Without this additional positive feedback, both algorithms can exhibit significant performance degradation.

Results for PathM and PathBB are presented in \figref{fig:nav-cpl-dvl-results}. On PathM, ReCOIL+FMR outperforms all alternatives across all data ratios, achieving the highest success rates with substantially lower misalignment. While CPL and DVL achieve moderate success rates at higher data ratios, both degrade considerably as the ratio decreases, with misalignment rising sharply at the 10-50 ratio. PathBB follows a similar trend, with all methods showing degraded performance at lower data ratios. ReCOIL+FMR again achieves the strongest overall performance, maintaining the highest success rates and lowest misalignment across all ratios. CPL performs comparably to ReCOIL at higher ratios but degrades more sharply as the data ratio decreases, while DVL consistently underperforms both, particularly at the 10-50 ratio.

Across the velocity policies, DVL consistently exhibits poor misalignment despite achieving returns comparable to $D^E$, worsening as the data ratio decreases. Notably, on SlowWalker DVL achieves better performance than both ReCOIL and CPL, though it remains substantially higher than ReCOIL+FMR. CPL performs comparably to ReCOIL+FMR on SlowSwim and SlowHop, achieving competitive returns with low misalignment, but degrades substantially on SlowWalk likely due to the higher-dimensional action space. ReCOIL+FMR demonstrates the most robust performance across all tasks and data ratios, consistently achieving the best trade-off between return and misalignment.

Overall, these results demonstrate that FMR more effectively harnesses evaluative feedback than DVL or CPL. These results suggest DVL cannot properly utilize feedback, consistent with prior work indicating that naively interpreting feedback as reward is ineffective \citep{macglashan_coach_2017, thomaz_reinforcement_2006}. Moreover, CPL falls short even when using evaluative feedback to alleviate the limitations of typical preference feedback.

\section{Discussion}\label{sec:discussion}

\paragraph{Only Negative Feedback} Results for the velocity policies demonstrate that FMR learns effectively from negative feedback alone. Removing positive feedback from aligned navigation policies (Appendix \ref{app:only-neg-results}) yields negligible changes in PathM and PathBB performance, confirming negative feedback as the primary driver of FMR's advantage. We hypothesize this is particularly critical at deviation points, where suboptimal trajectories diverge from the aligned policy, supported by SlowHop's performance as it receives feedback exclusively at such points. Positive feedback for $D^E$ does, however, boost Walker2D results, suggesting it may play a role in high-dimensional action spaces.

\paragraph{Impact of Imperfect Demonstrations} Both navigation and velocity results reveal a consistent performance gap between high overlap policies (PathM, SlowSwim) and low overlap policies (PathBB, SlowHop) with $D^I$. We attribute this to noisy demonstrations in $D^I$, which provide limited learning signal that feedback alone cannot fully correct. For instance, in PathBB, the majority of $D^I$ trajectories are misaligned, with only a small subset reflecting the aligned path. This underscores that imperfect demonstrations must retain some relevance to aligned behavior for FMR to be effective. While FMR mitigates noise through behavior reweighting, its effectiveness is ultimately constrained by data quality. When $D^I$ consists predominantly of irrelevant sub-trajectories, feedback on those segments becomes uninformative, explaining the observed performance decline as $D^E$ shrinks in low overlap tasks, where the agent receives fewer aligned demonstrations with little meaningful feedback from $D^I$ to compensate.

We also examine whether FMR can improve performance using imperfect demonstrations when $D^E$ and $D^I$ have high overlap. Comparing ReCOIL+FMR trained on $D^U$ with a 10-50 demonstration split against BC and IQL trained exclusively on $D^E$ (Appendix \ref{app:imp-demo-results}), ReCOIL+FMR consistently outperforms or matches both baselines across high overlap tasks. This demonstrates that feedback enables learners to benefit from imperfect data when it is structurally aligned with the target behavior. Results on low overlap tasks further corroborate that FMR is most effective when imperfect demonstrations meaningfully intersect with the aligned behavior.

\paragraph{Feedback Scalability} To assess feedback scalability, we reduce the proportion of demonstrations in $D^I$ receiving feedback, under a 50-50 data ratio, by randomly selecting a subset of trajectories and dropping all its feedback (see Appendix~\ref{app:fb-scale-results}). For navigation policies, ReCOIL+FMR exhibits a gradual performance decay as feedback coverage decreases, bounded below by baseline ReCOIL. At 20\% feedback coverage performance marginally exceeds ReCOIL, while 50\% coverage yields a slight degradation relative to full feedback. This decay is likely amplified by random feedback selection, as dropping all feedback for underrepresented or misaligned trajectories risks reinforcing poor behaviors, further motivating targeted feedback at points of deviation. Among velocity policies, only SlowHop follows this trend, while SlowSwim and SlowWalk appear unaffected. Although SlowHop already has strong baseline performance, almost all trajectories in $D^I$ follow the same initial path, meaning that feedback must be provided in sufficient quantity to override the dominant behavioral prior in the data. A promising direction for future work is the identification of similar state-action pairs to generalize feedback across regions of the state-action space, reducing the  feedback burden \citep{du_behavior_2023}.



\paragraph{Generalized Feedback-Adaptive Temperature} All evaluated environments employ discretized action spaces to facilitate the collection of human demonstrations. We therefore consider a natural extension of FMR to continuous action spaces, generalizing feedback-adaptive temperature (Eq.~(\ref{eq:tau_j})) to $\tau^{gen}(s, a) = \beta^{-h}$, hereafter referred to as \textit{generalized feedback-adaptive temperature}. In this formulation, when positive feedback is provided, the temperature adjustment is applied exclusively to the selected action. Positive feedback then yields $\tau^{gen}(s, a) < 1$, resulting in a negative loss that reduces interpretability. Preliminary experiments in discrete action spaces reveal negligible performance differences across all tasks (see Appendix~\ref{app:atau-results}). As FMR operates in probability space and is therefore model-agnostic, these findings suggest that $\tau^{gen}$ facilitates straightforward extension to continuous action spaces.


\section{Conclusion}
We propose Feedback Manipulation Regularization (FMR), a method that leverages the rich, interconnected information contained in demonstration feedback to improve alignment performance and, in many cases, aptitude performance in an offline setting. FMR is compatible with any imitation learning algorithm, functioning analogously to human controlled entropy regularization. Specifically, feedback modulates the entropy of state–action pairs, whereby negative feedback increases entropy to discourage undesirable actions, while positive feedback decreases entropy to reinforce preferred behaviors. Through extensive experiments, we demonstrate that FMR consistently enhances a variety of imitation learning baselines in a limited data regime, yielding policies that better align with human intentions. 

\paragraph{Limitations and Future Work}\label{sec:limits}
FMR's primary limitation is its scalability, though it already demonstrates improved scalability over alternative methods via its superior performance. This scalability nonetheless remains dependent on task complexity and the dynamics between $D^E$ and $D^I$. Natural future directions include generalizing feedback across similar state-action pairs to reduce annotation burden \citep{du_behavior_2023}, and providing targeted feedback for $D^I$ at points which deviate from the aligned policy to improve efficacy. Additionally, while this work is limited to discrete action spaces, the $\tau^{gen}$ formulation suggests a principled extension to continuous action spaces.


\bibliographystyle{unsrt}   
\bibliography{_references}      

\appendix
\input{_appendix}



\end{document}

%% file: _appendix.tex
\appendix
\onecolumn
\clearpage
\section{Implementation Details}\label{app:imp-details}
\input{appendix/implementation.tex}

\clearpage
\section{Task and Data Collection Details}\label{app:task-data-details}
\input{appendix/task-details}

\clearpage
\section{Extended Results}\label{app:add-results}

This section provides a collection of additional experimental results mentioned throughout the paper. All learning curves results are averaged over 5 seeds and smoothed using a moving average of 10. The shaded region represents the standard error.

\subsection{Navigation Task}\label{app:nav-results}
\input{appendix/main-nav}

\clearpage
\subsection{Velocity Tasks}\label{app:vel-results}
\input{appendix/main-vel}

\clearpage
\subsection{Feedback as Preferences (CPL) and Reward (DVL)}\label{app:dvl-results}
\input{appendix/cpl-dvl}

\clearpage
\subsection{Only Negative Feedback}\label{app:only-neg-results}
\input{appendix/no-pos}

\clearpage
\subsection{Usefulness of Imperfect Demonstrations}\label{app:imp-demo-results}
\input{appendix/imp-demos}

\clearpage
\subsection{Feedback Scale}\label{app:fb-scale-results}
\input{appendix/fb-scale}

\clearpage
\subsection{Generalized Feedback-Adaptive Temperature}\label{app:atau-results}
\input{appendix/gen-tau}

\clearpage
\subsection{Hyperparameters}\label{app:hyper-results}
The following subsections conduct various hyperparameter experiments for FMR.
\input{appendix/hyperparams}

\clearpage
\subsection{FMR Induces Feedback-Modulated Entropy}
This subsection visualizes state-space entropy for the aligned navigation policies using ReCOIL models. \figref{fig:pathm-entropy} and \figref{fig:pathbb-entropy} depict the heatmaps for normalized entropy. To generate these figures, normalized entropy is computed using the final actor model from training, with the expert and imperfect trajectories as input. The average over 5 seeds is then computed. FMR increases entropy for positive feedback while decreasing it for negative feedback. Expert demonstrations receive no feedback and therefore converge to high entropy by default. These entropy heatmaps correspond closely with the original feedback heatmaps given in \figref{fig:pathm-feedback} and \figref{fig:pathbb-feedback}.
\input{appendix/entropy}

%% file: appendix/implementation.tex
\subsection{Data Preprocessing} For all tasks, we oversample $D^E$ as the number of demonstrations decreases. Oversampling is performed so that the number of samples in $D^E$ is close to, or slightly lower than, that in $D^I$. Oversampling is done by repeating entire trajectories, and we typically oversample $D^E$ to the size of $D^I$ (i.e., 50 demonstrations). The exceptions are SlowHop and SlowWalk, where the oversample target is 40 demonstrations to prevent $D^E$ from exceeding the number of samples in $D^I$. Additionally, for the navigation task, we removed duplicate starting states. These duplicates arise from a delay in the human beginning their demonstration, which creates repeated states within the first 120 steps.

\subsection{Algorithms} All algorithm networks are implemented in PyTorch using two-layer neural networks with 256 hidden units and ReLU activation. During training, algorithms are trained for 1 million batches (i.e., training iterations) using a learning rate of $3 \times 10^{-4}$ and a batch size of 128. Evaluations are conducted every 10k iterations and consist of 50 episodes each. To evaluate the final average performance and standard deviation for each metric, we compute the mean over the last 10 evaluations from training, across all 5 random seeds. Algorithm hyperparameters are selected based on author recommendations or limited searches. Algorithm-specific hyperparameters and details are provided below.

\paragraph{FMR} Appendix \ref{app:hyper-results} depicts the search over hyperparameters using PathBB ratio 10–50: $\beta \in \{1.5, 10, 100, 1000\}$, $\alpha \in \{0.1, 1, 10, 100\}$, and credit assignment $\in \{1.5, 10, 100, 1000\}$. For all tasks, unless stated otherwise, we set the hyperparameters as follows: $\beta = 10$ for the feedback-adaptive temperature, $\alpha = 1$ for the regularization strength (for both actor and critic), and the credit assignment window is set to 600 ms. When oversampling we found using cosine learning rate decay to be helpful for DemoDICE+FMR and ReCOIL+FMR. For Walker2D, we found that using positive feedback for all expert state-action pairs in $D^E$ and a slightly higher $\alpha = 2$ led to better performance due to the large action space. Although feedback is not typically provided for $D^E$, we observed that positive feedback for Walker2D helped elevate aligned actions. While adding such feedback is trivial, it was generally unnecessary in spaces with smaller action dimensions.

For critic-based algorithms (e.g., IQL, DemoDICE, ReCOIL), FMR can be applied to the critic under the assumption of a discrete action space. This requires transforming the value function into probability space using a function such as softmax. However, the softmax function is highly sensitive to input scale. When value function outputs have large magnitudes, the softmax produces overly peaked distributions that concentrate probability on a single action, likely limiting the ability to modulate entropy. For actor-critic algorithms, we apply FMR to both the actor and the critic.

\paragraph{IQL} Following \citep{garg_iq-learn_2021}, our implementation of IQL uses soft Q-learning to test FMR's ability to manipulate a critic-only algorithm. We implement IQL as a critic-only algorithm to verify FMR's ability to work with critic-only algorithms in discrete spaces. A hyperparameter search was performed over the reward loss $\{ 0.5, 10 \}$ using PathM and all ratios. We set the following hyperparameters for all tasks: $\chi^2$ regularization for the reward loss is $10$, and the soft value temperature is $1$. 

\paragraph{DemoDICE} Our implementation of DemoDICE follows \citep{kim_demodice_2021}. We set the following hyperparameters for all tasks: the discount factor is $0.99$, and the gradient penalty coefficients for the discriminator (i.e., cost) and critic are $0.1$ and $1 \times 10^{-4}$, respectively.

\paragraph{ReCOIL} Our implementation of ReCOIL follows \citep{sikchi_dual_2024} using actor-critic with target networks. We set the following hyperparameters for all tasks: the discount factor is $0.99$, actor temperature is $0.1$, value temperature is $1$, max clip is $7$, reward gap is $2$, target Q gap is $200$, and the target network is updated using a value of $5 \times 10^{-3}$.

\paragraph{DVL} Our implementation of DVL follows \citep{sikchi_dual_2024} using actor-critic with target networks. As DVL is the offline RL variant of ReCOIL that additionally accounts for reward, we adopt largely the same hyperparameters used for ReCOIL across all tasks: a $\chi^2$-based value loss, discount factor of $0.99$, actor temperature of $0.1$, value temperature of $1$, max clip of $5$, and a target network update rate of $5 \times 10^{-3}$. To improve performance across all tasks, we set positive feedback for all expert state-action pairs in $D^E$. When oversampling we found using cosine learning rate decay to be helpful.

\paragraph{CPL} Our implementation of CPL follows \citep{hejna_cpl_2024}. To select hyperparameters, we performed a grid search over the following values: segment length $\{16, 32, 64\}$, segment count $\{20\text{k}, 40\text{k}\}$, contrastive bias $\{0.25, 0.5, 0.75, 0.85, 0.95\}$, and temperature $\{0.1, 0.3, 0.5, 1\}$. Based on performance across the PathM, PathBB, and Walker2D tasks across various ratios, we selected a segment length of $64$, a segment count of $20\text{k}$, a contrastive bias of $0.85$, and a temperature of $0.1$. We use a discount factor of $1$ and dense preference comparisons as recommended in \citep{hejna_cpl_2024}. The preference score is computed using the sum of evaluative feedback per trajectory segment. To improve performance across all tasks, we set positive feedback for all expert state-action pairs in $D^E$.

\subsection{Computing Resources}\label{app:compute}
All experiments are run in a SLURM cluster environment across five different types of GPU nodes: NVIDIA A40 with an Intel Xeon Gold 6326 CPU, NVIDIA A100 with an Intel Xeon Gold 6326 CPU, NVIDIA A100 with an AMD EPYC 7502 CPU, NVIDIA L40s with an Intel Xeon Gold 6338 CPU or Intel Xeon Platinum 8362 CPU, and NVIDIA H200 NVL with a dual 64-core AMD EPYC 9555 CPU. All nodes used 40 GB of RAM.

%% file: appendix/task-details.tex
This section provides various visualizations of the different tasks, along with their corresponding trajectories and feedback. 

\subsection{Navigation Policies}\label{app:nav-details}

\figref{fig:nav-tasks} depicts the navigation task environment for the aligned PathM and PathBB policies. Blue areas represent hazards, the red sphere with a blue box (indicating the front) represents the agent, and the green cube represents the goal. The yellow walls mark the environment boundaries. These are observable to the agent and do not allow the agent to pass through them. When the agent intersects with a hazard, a red sphere will encompass the agent, indicating it is violating a constraint. This is purely a visual effect for observers.

\paragraph{Data Collection} For the navigation policies, human demonstrations are collected for both $D^E$ and $D^I$. For each aligned policy, 50 expert demonstrations are collected for $D^E$ with zero cost violations and 100\% success rate, while 50 imperfect demonstrations are collected for $D^I$ by having humans explore the state space without reaching the goal. \figref{fig:nav-trajs} visualizes a 2D slice of the environment and trajectories for both aligned navigation policies. The green lines indicate the expert trajectories for each respective policy. The gray lines (shared across policies) represent the imperfect trajectories. The majority of $D^I$ demonstrations move diagonally through the center (similar to PathM), while a minority follow the top or right boundary before moving toward the goal (similar to PathBB), with all demonstrations circling the goal to obfuscate the objective. 

Human feedback is collected by replaying demonstrations through a visual interface, where negative or positive feedback can be provided via keyboard input. When a feedback key is pressed, the feedback is mapped to the nearest state. Feedback is collected only for $D^I$, as it contains demonstrations requiring correction. Policy-specific feedback is collected for PathM and PathBB. The feedback methodology is defined such that the human evaluator provides positive feedback when the agent follows or moves toward the aligned policy, and negative feedback when it diverges or appears to be diverging. \figref{fig:nav-feedbacks} displays a heat map of the human feedback for $D^I$, provided with respect to each aligned navigation policy. Notice that feedback for \figref{fig:pathm-feedback} encourages the agent to move diagonally towards the goal or back towards the middle of the environment (indicated by green positive feedback). Likewise, \figref{fig:pathbb-feedback} encourages the agent to move towards the bottom path, although there are few trajectories that do so.

\clearpage
\begin{figure}[h]
    \centering
    \begin{subfigure}[b]{0.48\textwidth}
        \centering
        \includegraphics[width=\textwidth]{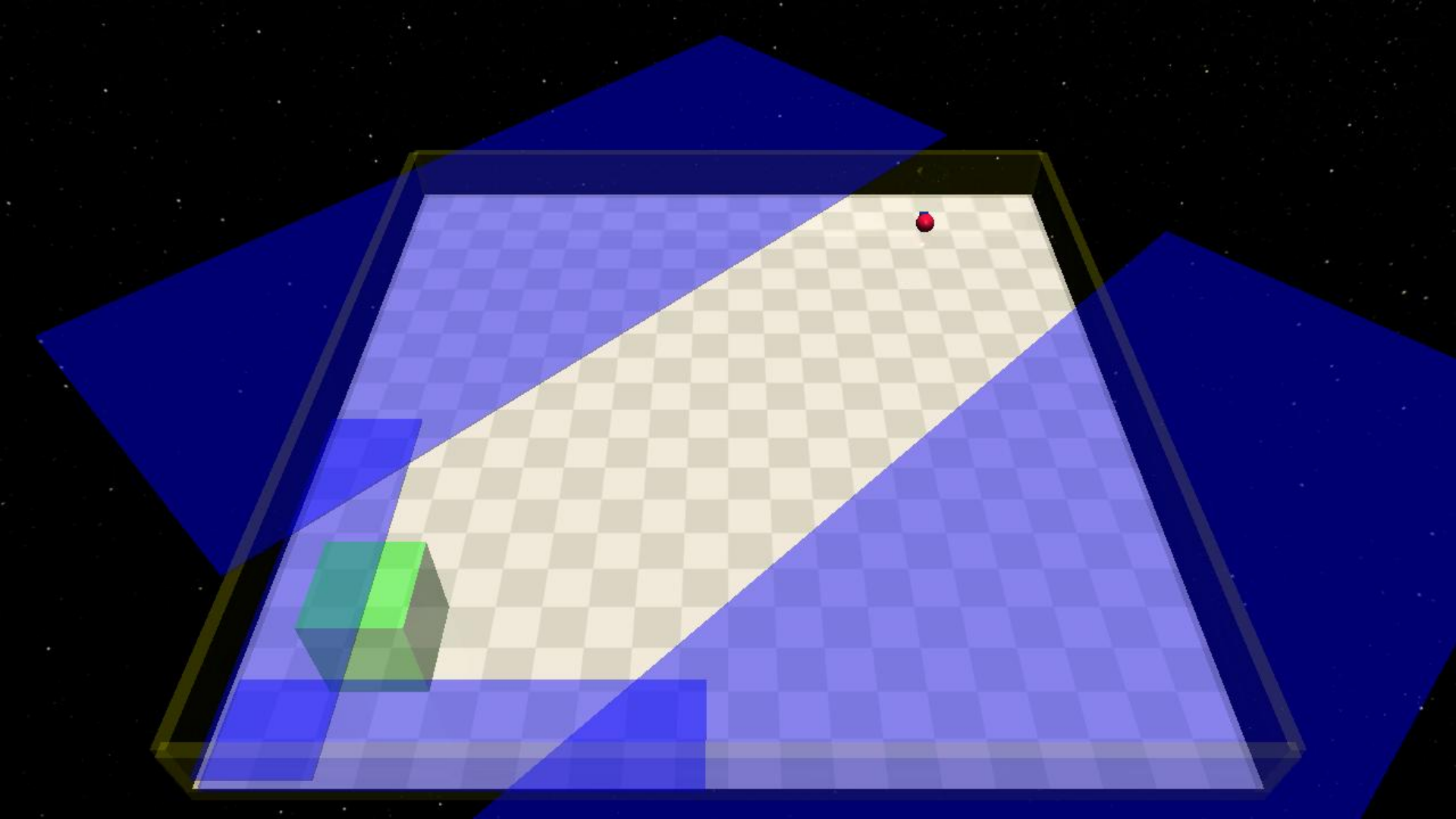}
        \caption{PathM}
    \end{subfigure}
    \hfill
    \begin{subfigure}[b]{0.48\textwidth}
        \centering
        \includegraphics[width=\textwidth]{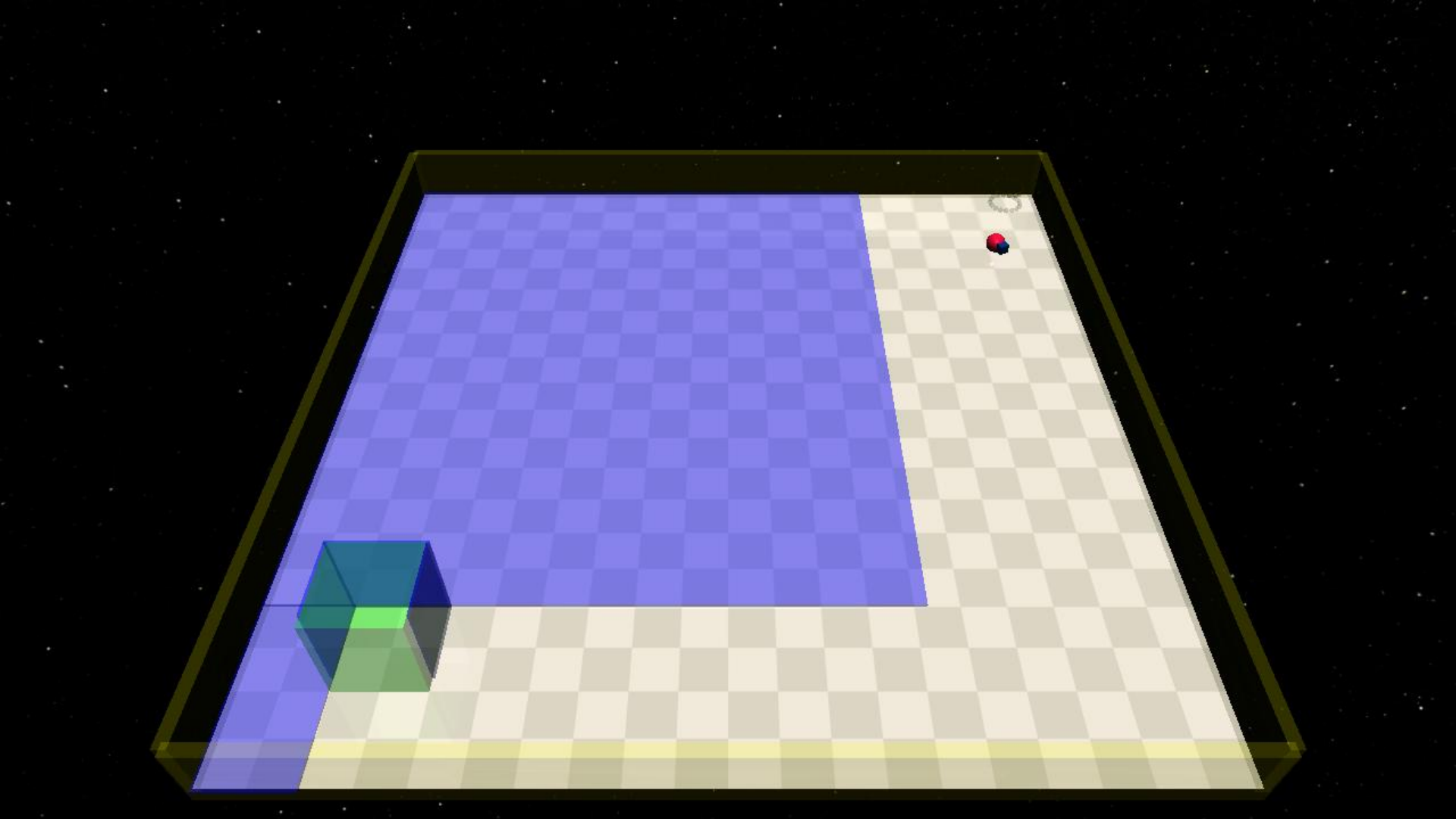}
        \caption{PathBB}
    \end{subfigure}
    \caption{Example of navigation task environment setups.}
    \label{fig:nav-tasks}
\end{figure}

\begin{figure}[H]
    \centering
    \begin{subfigure}[b]{0.48\textwidth}
        \centering
        \includegraphics[width=\textwidth]{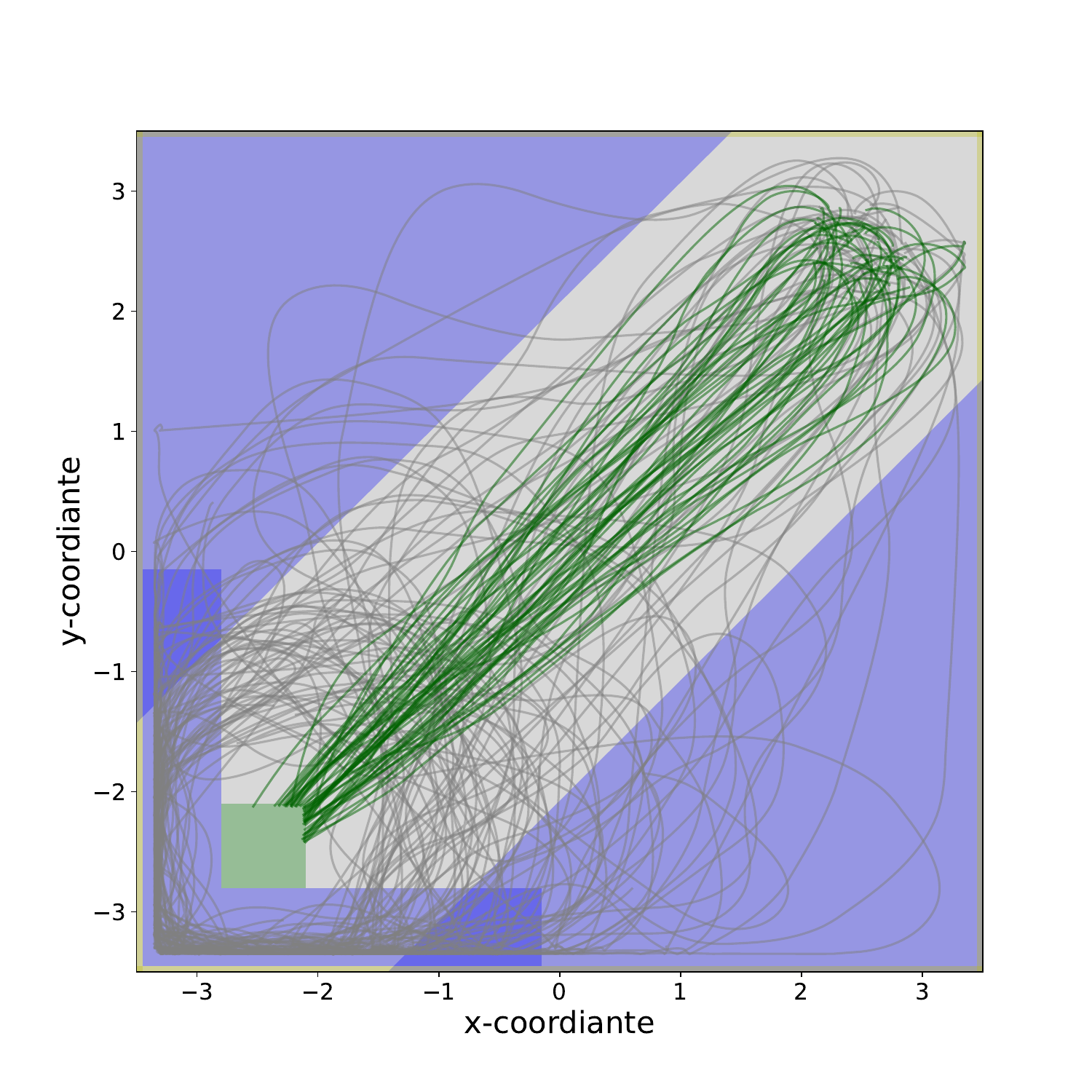}
        \caption{PathM}
    \end{subfigure}
    \hfill
    \begin{subfigure}[b]{0.48\textwidth}
        \centering
        \includegraphics[width=\textwidth]{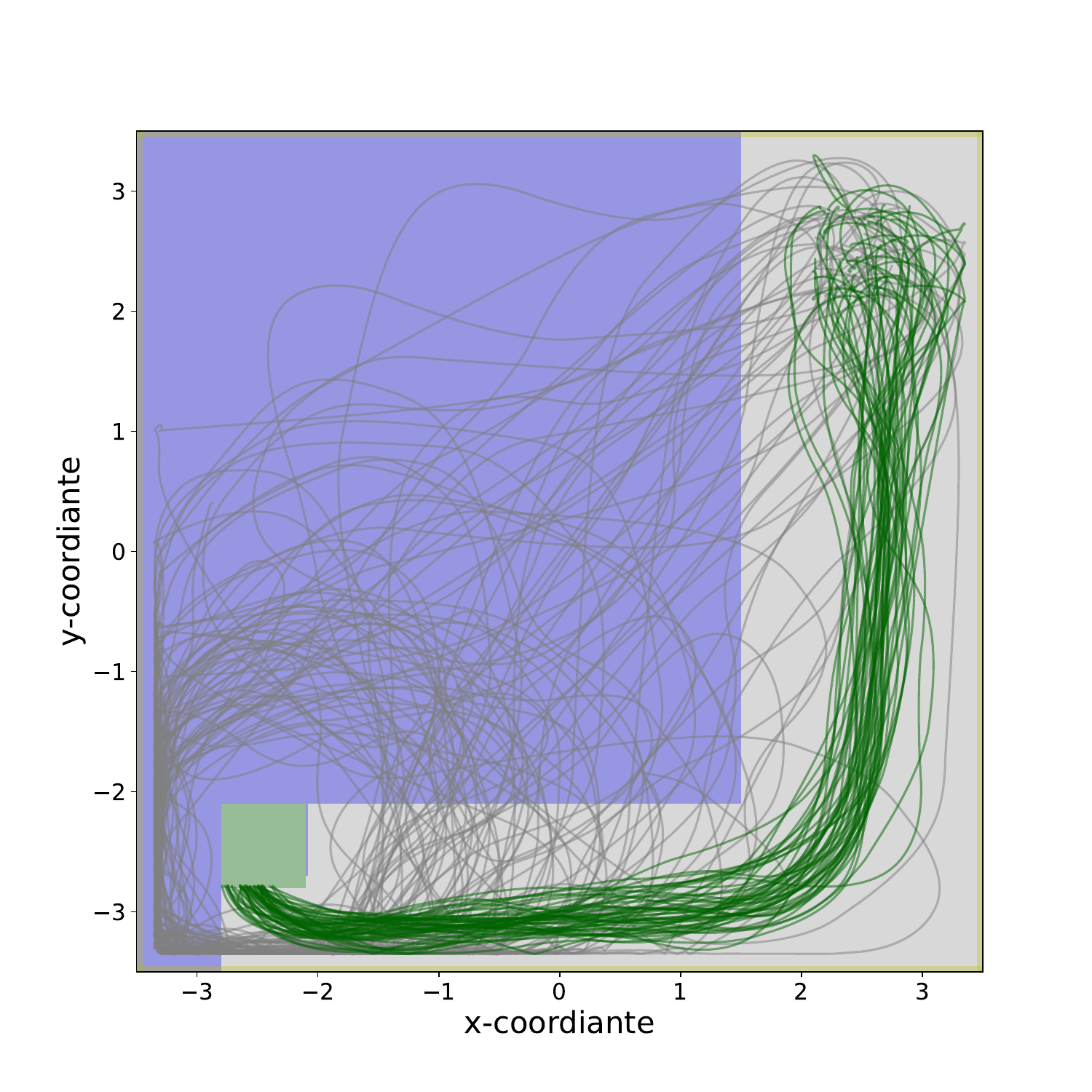}
        \caption{PathBB}
    \end{subfigure}
    \caption{Trajectories for the navigation aligned policies depicted as a 2D slice of the environment. The green lines represent expert trajectories, the gray lines represent imperfect trajectories, and the blue areas represent hazards.}
    \label{fig:nav-trajs}
\end{figure}

\begin{figure}[h]
    \centering
    \begin{subfigure}[b]{0.48\textwidth}
        \centering
        \includegraphics[width=\textwidth]{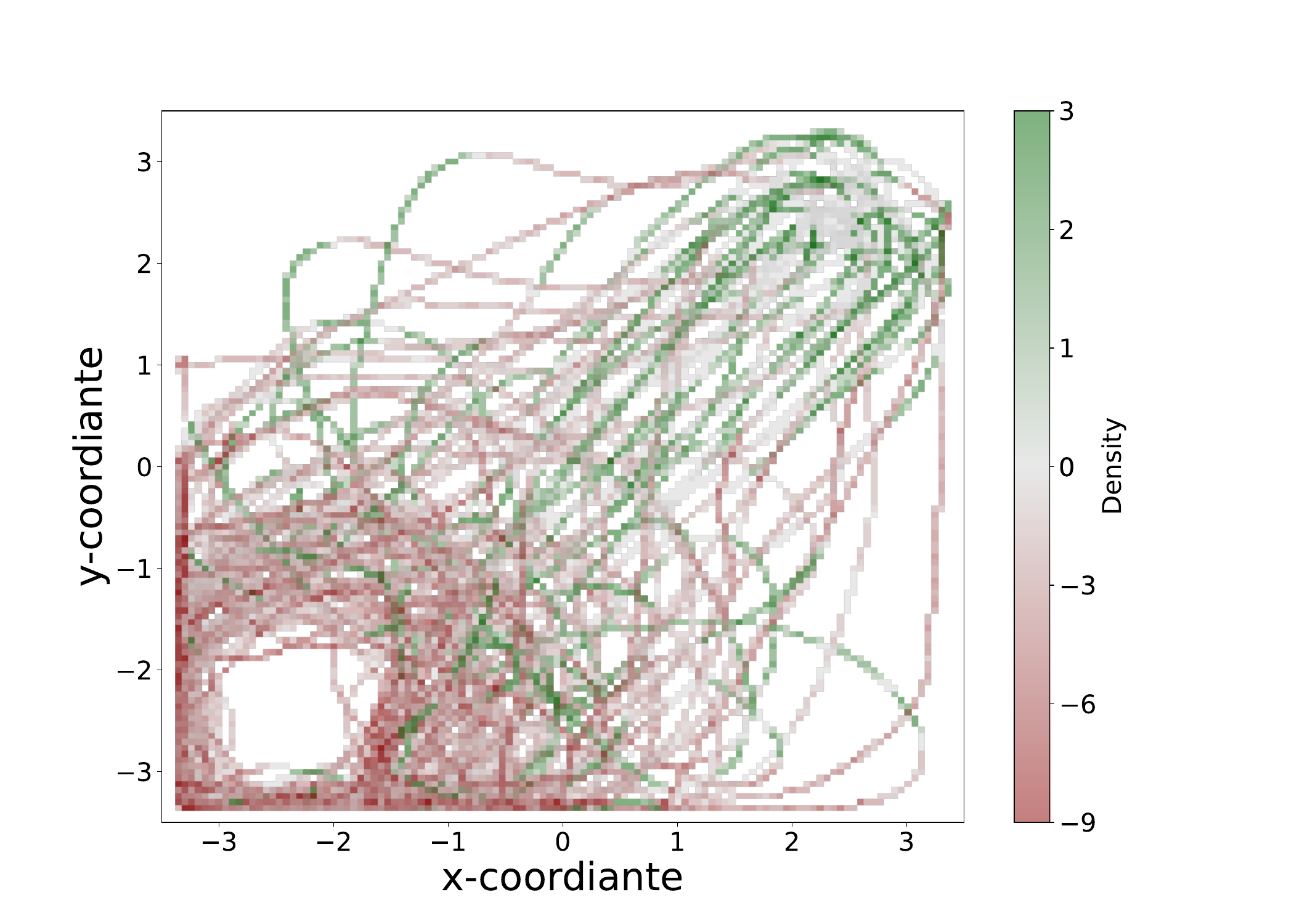}
        \caption{PathM}
        \label{fig:pathm-feedback}
    \end{subfigure}
    \hfill
    \begin{subfigure}[b]{0.48\textwidth}
        \centering
        \includegraphics[width=\textwidth]{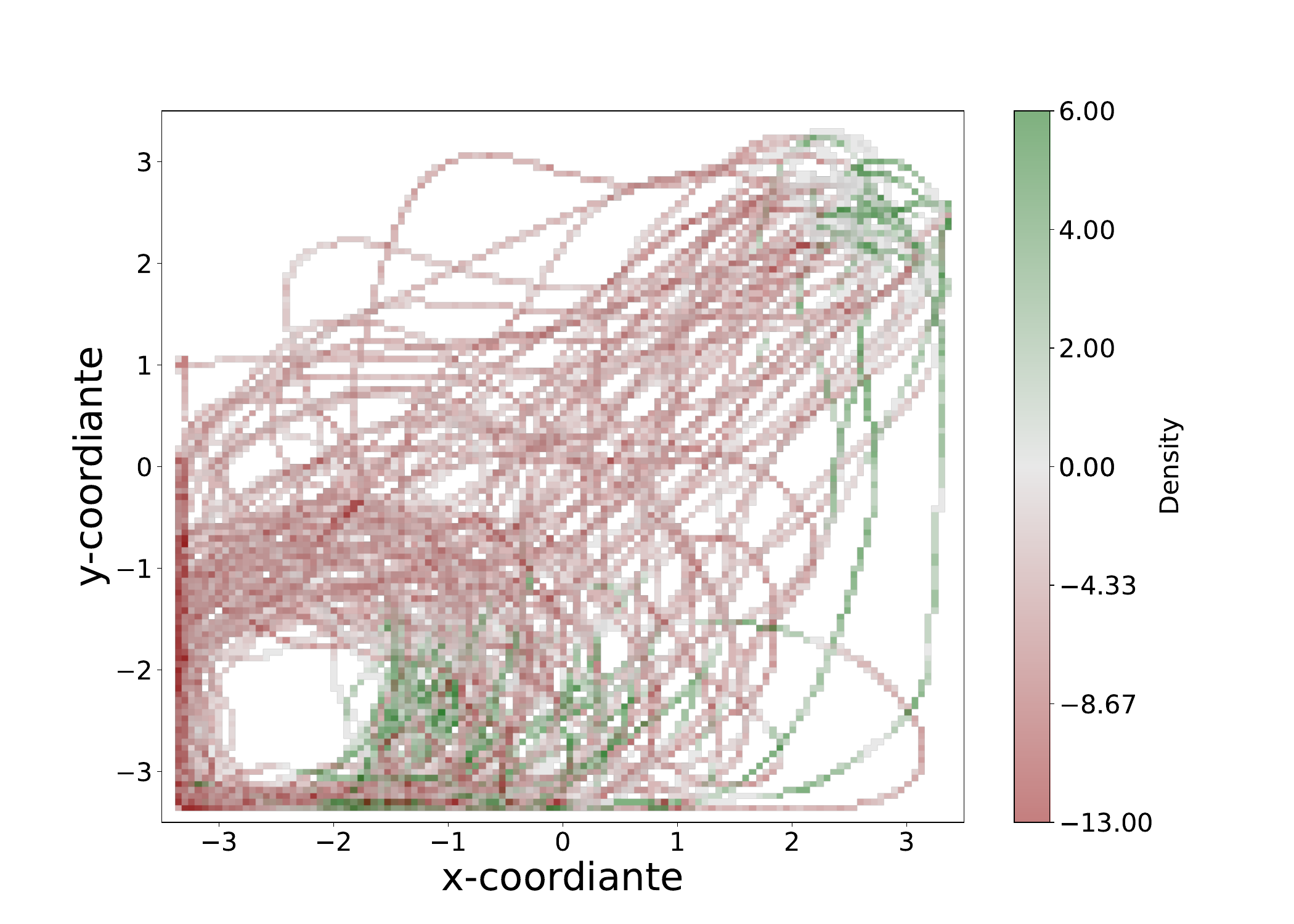}
        \caption{PathBB}
        \label{fig:pathbb-feedback}
    \end{subfigure}
    \caption{Feedback heatmaps of $D^I$ demonstrations for the aligned navigation policies with credit assignment. Color is determined by the feedback density in a particular area in the state space.}
    \label{fig:nav-feedbacks}
\end{figure}

\clearpage
\subsection{Velocity Policies} \label{app:vel-details}
\figref{fig:vel-tasks} depicts the velocity task environment setups for Swimmer, Hopper, and Walker2D. The sphere above each agent represents whether the agent is exceeding the velocity threshold (this is for visualization purposes only). If the sphere is green, the agent is under the threshold. If the sphere is red, the agent is exceeding the threshold. 

\paragraph{Data Collection} For the velocity policies, proxy demonstrations are collected for both $D^E$ and $D^I$ using SAC-Discrete. Our implementation of SAC-Discrete follows \citep{christodoulou_soft_2019}. For all velocity tasks, we used a two-layer neural networks with 256 hidden units and ReLU activation for the actor and critic networks. Additionally, we used a learning rate of $3 \times 10^{-4}$ for the actor and critic and $1 \times 10^{-3}$ for entropy learning. For Swimmer, we used a discount factor of $0.995$ to reach near optimal performance. The SAC-Discrete agent for Swimmer and Hopper is trained for 5 millions steps while Walk2D is trained for 10 million due to the large action space slowly down learning. For Hopper and Swimmer, samples for $D^I$ are collected every 1 million steps while for Walker2D samples are collected at steps 2, 3, 4, 5 and 8 million. The sampling for Walker2D is done to ensure a $D^I$ contains samples that have lower overlap with $D^E$ but are of higher quality.

\figref{fig:swim-trajs} depicts the expert and the imperfect demonstrations for SlowSwim where the blue dotted line indicates the velocity threshold. \figref{fig:hop-trajs} depicts the expert and the imperfect demonstrations for SlowHop where the imperfect demonstrations only cross the threshold. \figref{fig:walk-trajs} depicts the expert and the imperfect demonstrations for SlowWalk which demonstrations act as a middle ground between SlopSwim and SlowHop. 

For each velocity aligned policy, demonstrations in $D^E$ do contain some cost violations. The performance of the expert demonstrations $D^E$ for SlowSwim has a return of $283$ with a misalignment score of $8\times 10^{-5}$. SlowHop has a return of $1451$ with a misalignment score of $0.01$. SlowWalk has a return of $2147$ with a misalignment score of $8 \times 10^{-3}$. Each policy's $D^E$ return is normalized based on these returns.

Feedback is generated using an oracle instead of a human. Oracle feedback is generated by converting the cost into negative feedback, meaning no positive feedback is provided by the oracle. Oracle feedback is used for two reasons: first, to evaluate the viability of synthetic feedback that might be required when human feedback is difficult to provide; second, to evaluate the effectiveness of using only negative feedback. 


\clearpage
\begin{figure}[h]
    \centering
    \begin{subfigure}[b]{0.32\textwidth}
        \centering
        \includegraphics[width=\textwidth]{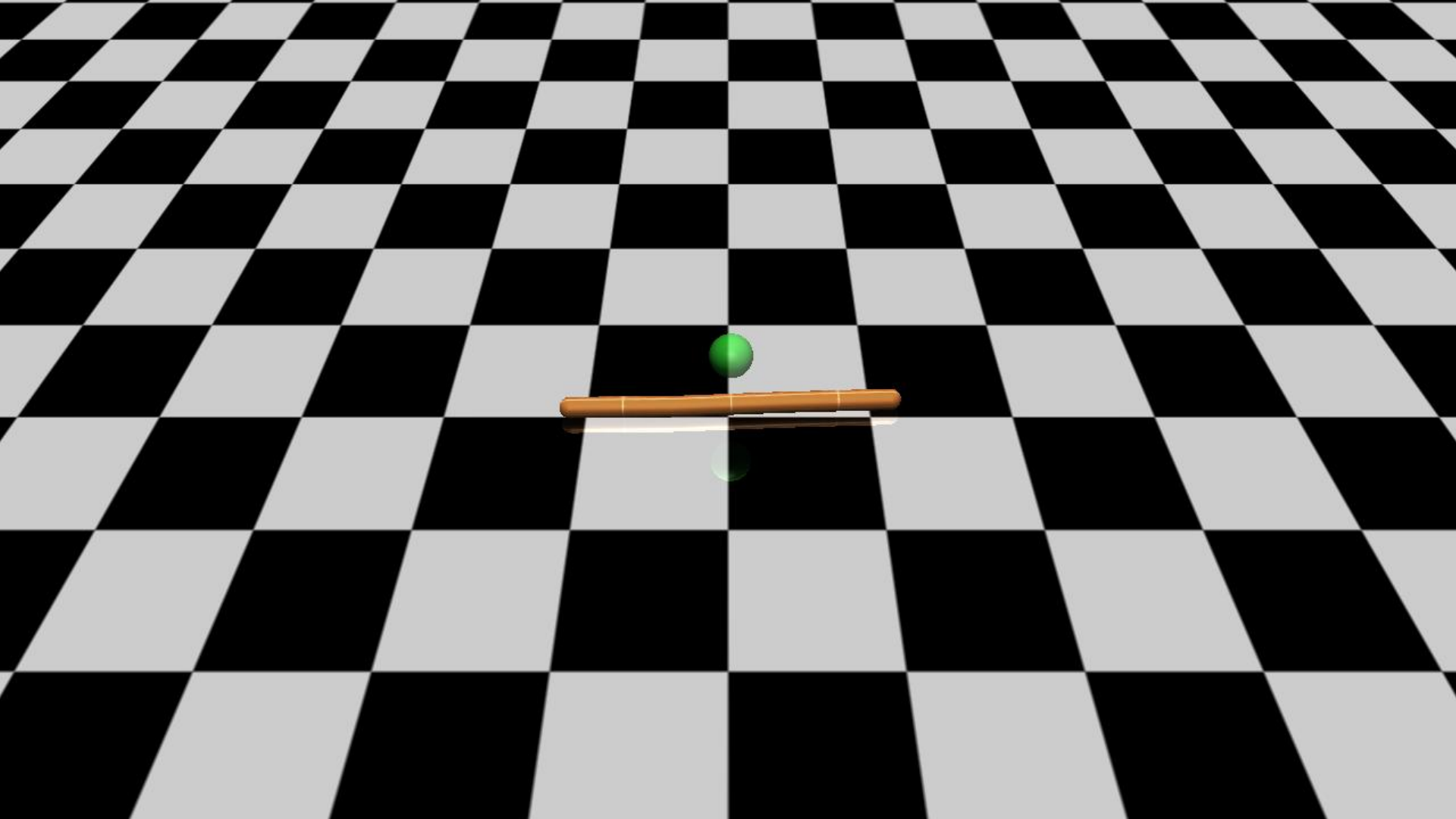}
        \caption{Swimmer}
    \end{subfigure}
    \hfill
    \begin{subfigure}[b]{0.32\textwidth}
        \centering
        \includegraphics[width=\textwidth]{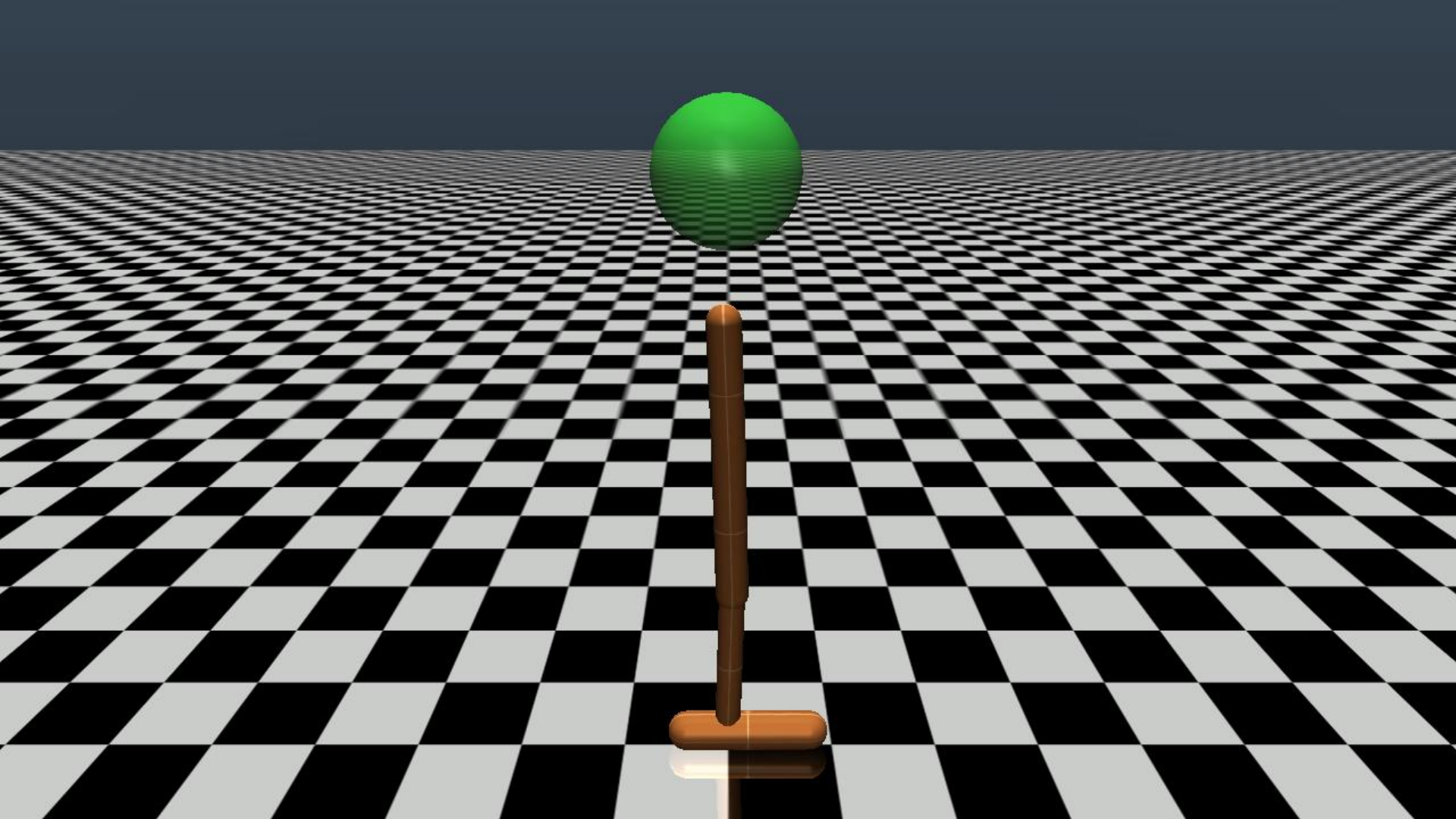}
        \caption{Hopper}
    \end{subfigure}
    \begin{subfigure}[b]{0.32\textwidth}
        \centering
        \includegraphics[width=\textwidth]{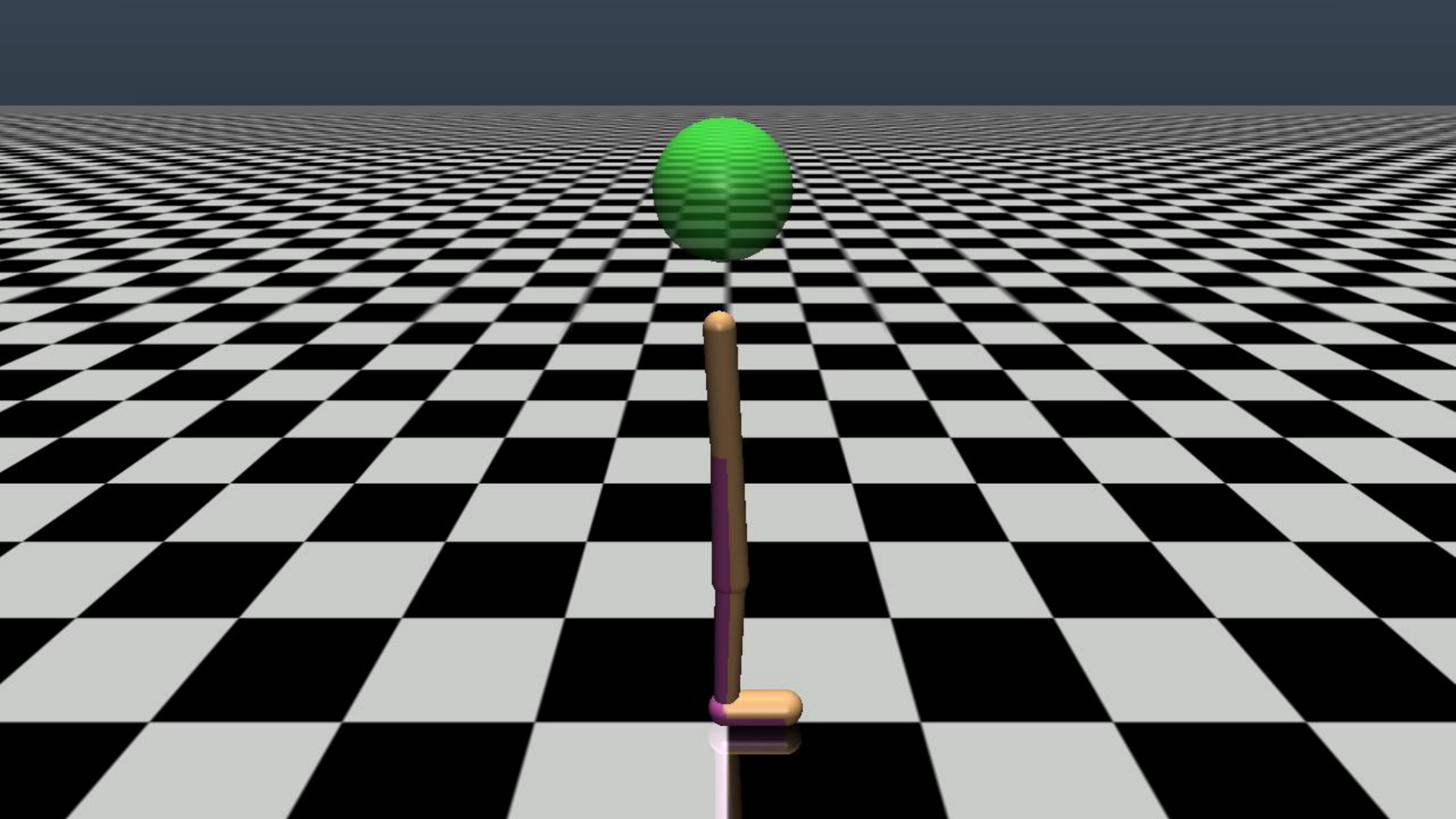}
        \caption{Walker2D}
    \end{subfigure}
    \caption{Example of velocity task environment setups.}
    \label{fig:vel-tasks}
\end{figure}

\begin{figure}[h]
    \centering
    \begin{subfigure}[b]{0.48\textwidth}
        \centering
        \includegraphics[width=\textwidth]{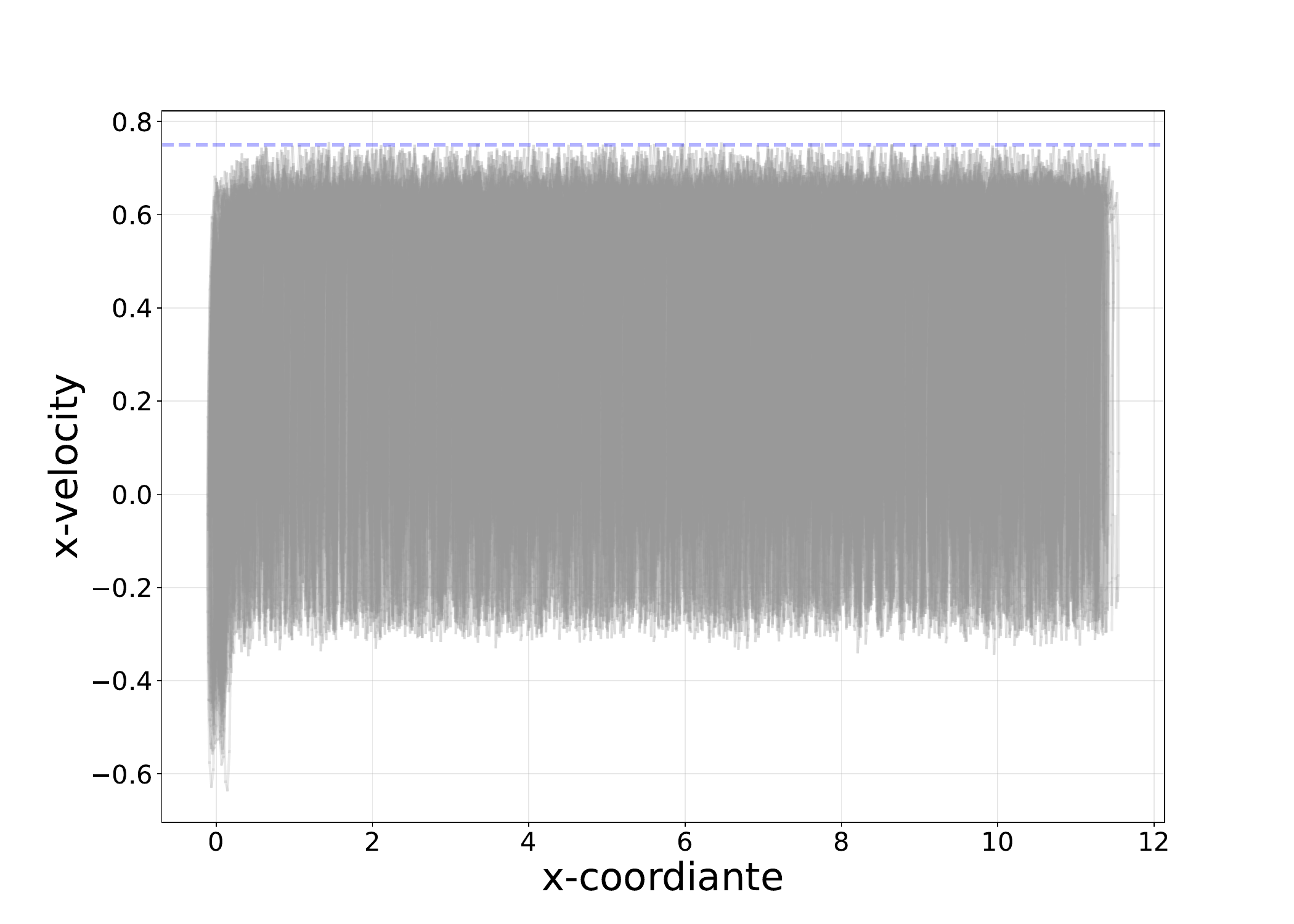}
        \caption{Expert}
    \end{subfigure}
    \hfill
    \begin{subfigure}[b]{0.48\textwidth}
        \centering
        \includegraphics[width=\textwidth]{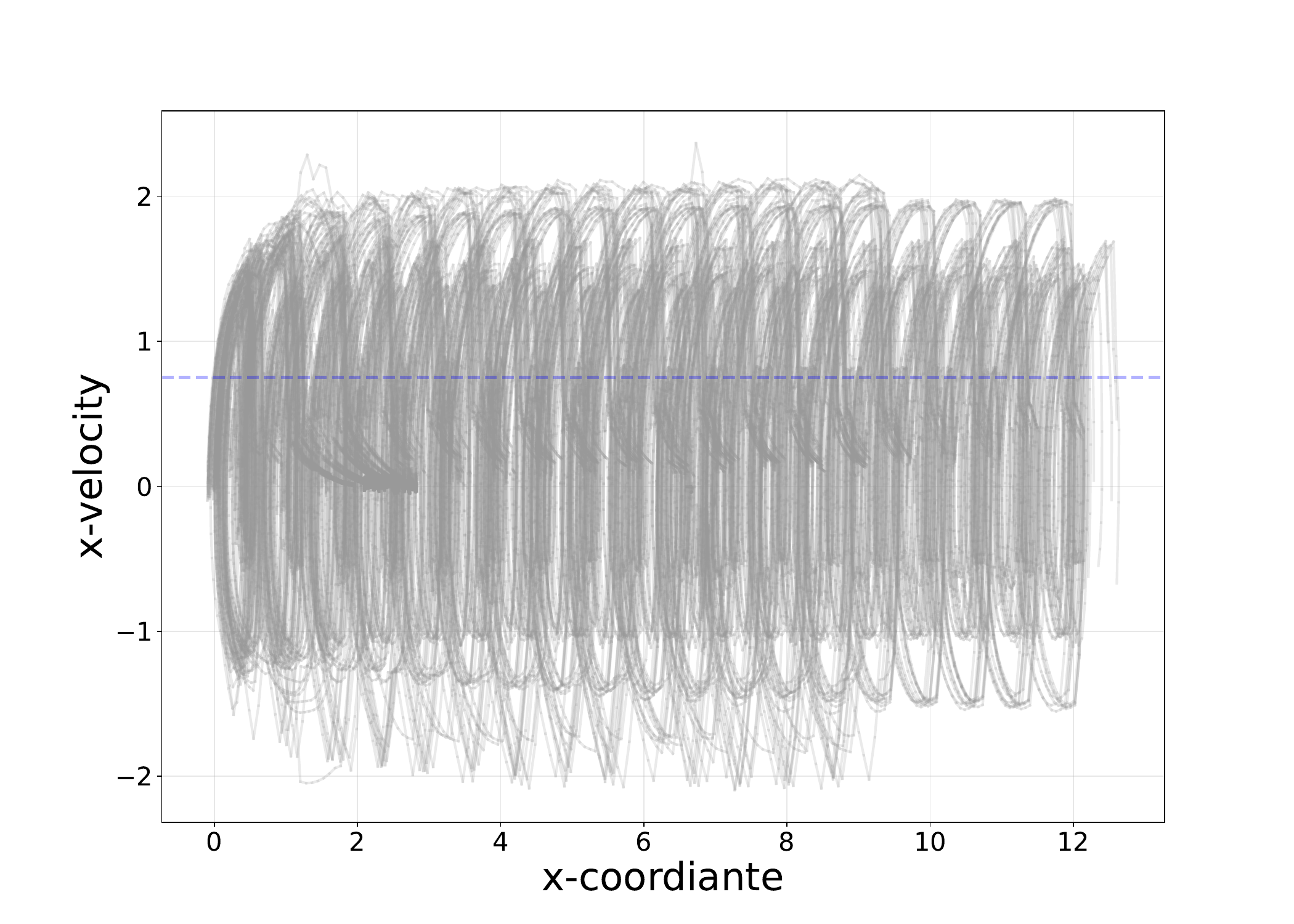}
        \caption{Imperfect}
    \end{subfigure}
    \caption{Trajectories for Swimmer plotted using coordinate and x-velocity. The blue dotted line represents the 0.74 velocity threshold.}
    \label{fig:swim-trajs}
\end{figure}

\begin{figure}[h]
    \centering
    \begin{subfigure}[b]{0.48\textwidth}
        \centering
        \includegraphics[width=\textwidth]{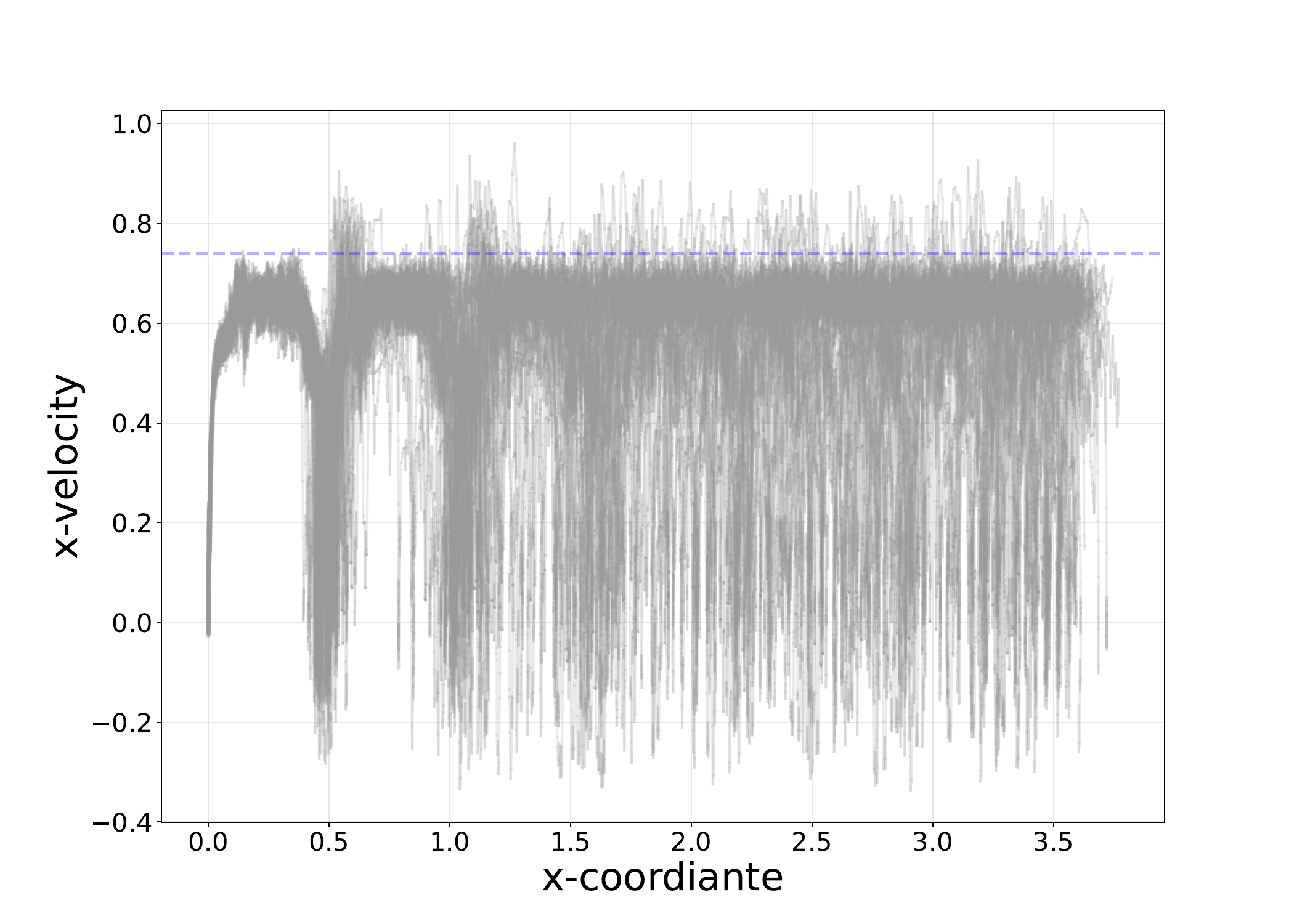}
        \caption{Expert}
    \end{subfigure}
    \hfill
    \begin{subfigure}[b]{0.48\textwidth}
        \centering
        \includegraphics[width=\textwidth]{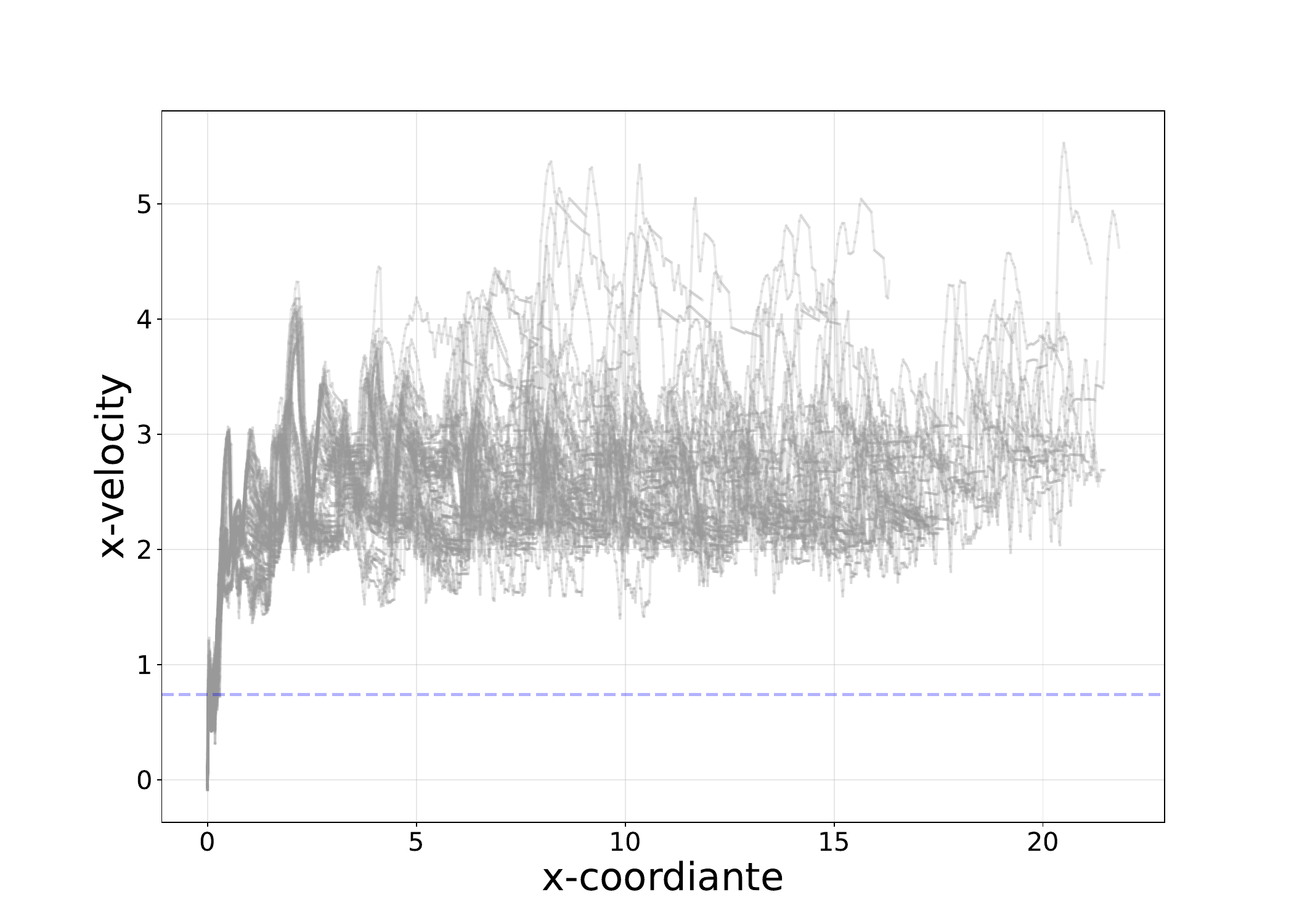}
        \caption{Imperfect}
    \end{subfigure}
    \caption{Trajectories for Hopper plotted using coordinate and x-velocity. The blue dotted line represents the 0.75 velocity threshold.}
    \label{fig:hop-trajs}
\end{figure}

\begin{figure}[h]
    \centering
    \begin{subfigure}[b]{0.48\textwidth}
        \centering
        \includegraphics[width=\textwidth]{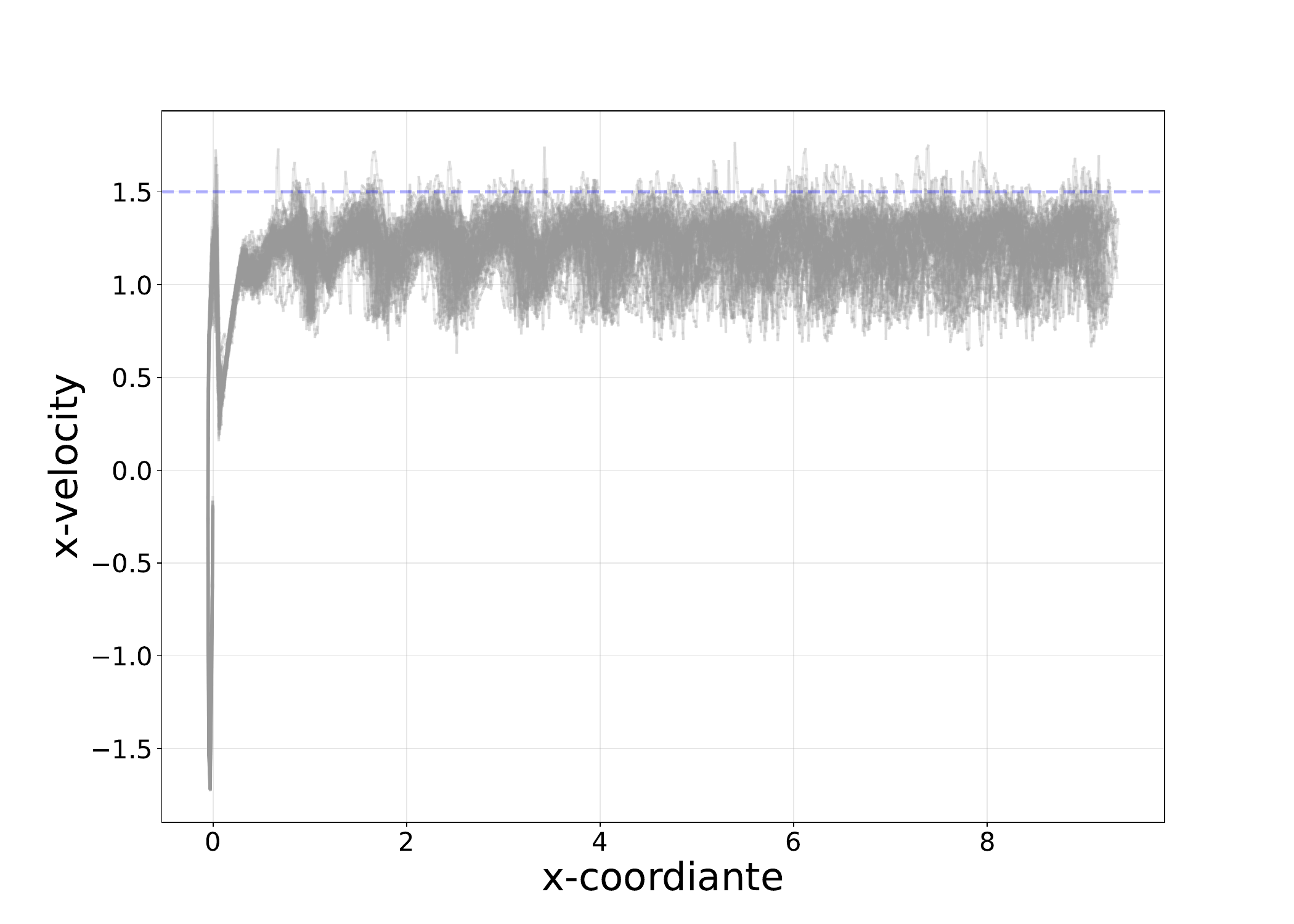}
        \caption{Expert}
    \end{subfigure}
    \hfill
    \begin{subfigure}[b]{0.48\textwidth}
        \centering
        \includegraphics[width=\textwidth]{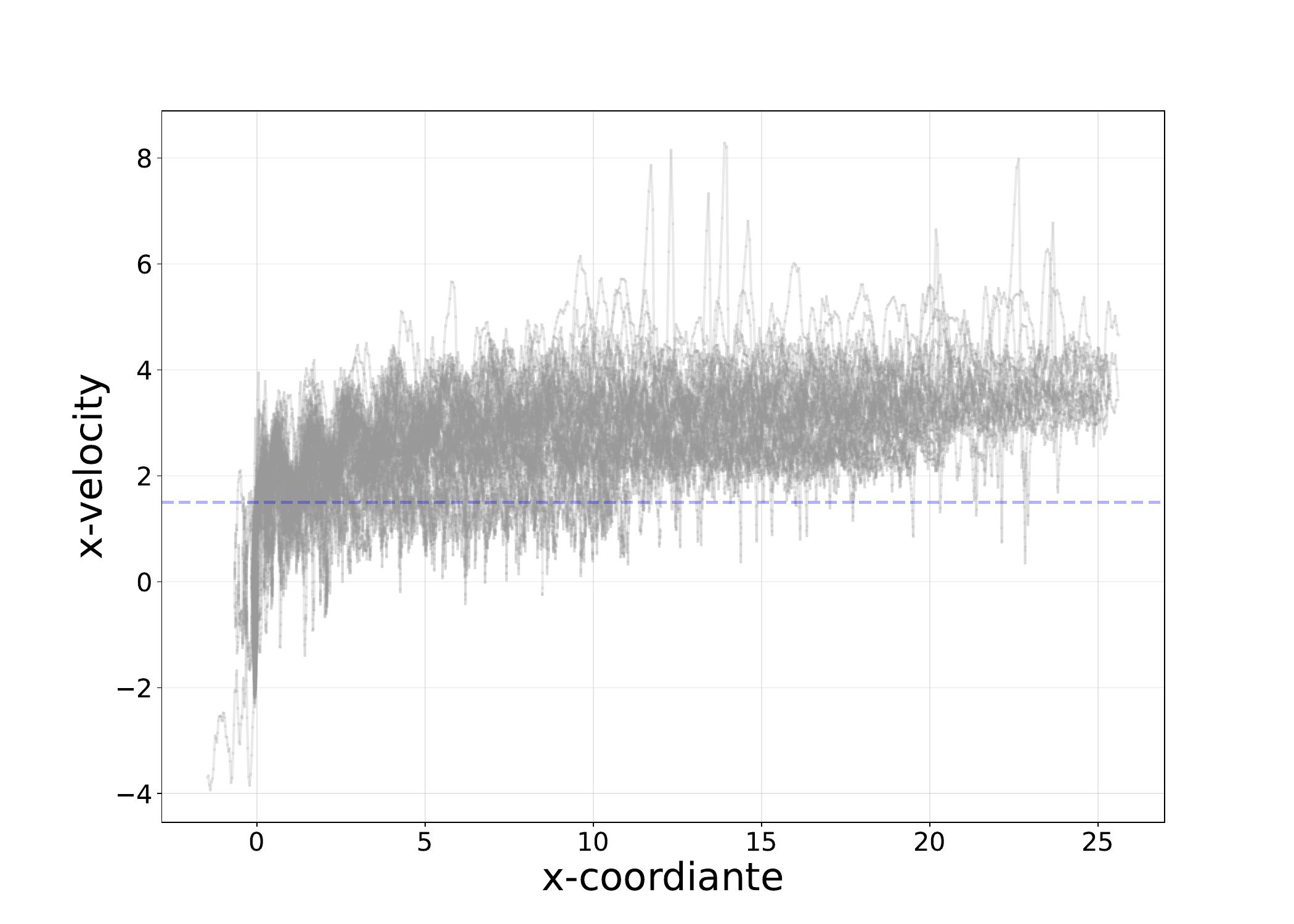}
        \caption{Imperfect}
    \end{subfigure}
    \caption{Trajectories for Walker2D plotted using coordinate and x-velocity. The blue dotted line represents the 1.5 velocity threshold.}
    \label{fig:walk-trajs}
\end{figure}

%% file: appendix/main-nav.tex
The following are the results for the navigation tasks.

\begin{table*}[h]
\centering
\caption{PathM performance comparison across algorithms and data ratios. Results show mean ± std for the last 10 evaluations, over 5 seeds.}
\label{tab:pathm-main}
\footnotesize
\setlength{\tabcolsep}{3.5pt}
\begin{tabular}{lccccccccc}
\toprule
& & \multicolumn{2}{c}{\textbf{BC}} & \multicolumn{2}{c}{\textbf{IQL}} & \multicolumn{2}{c}{\textbf{DemoDICE}} & \multicolumn{2}{c}{\textbf{ReCOIL}} \\
\textbf{Variant} & \textbf{Ratio} & Return & Mis. & Return & Mis. & Return & Mis. & Return & Mis. \\
\midrule
\multirow{3}{*}{Base} & 50-50 & 0.45 ± 0.50 & 0.37 ± 0.31 & 0.46 ± 0.50 & 0.36 ± 0.31 & 0.54 ± 0.50 & 0.33 ± 0.33 & 0.95 ± 0.22 & 0.11 ± 0.21 \\
& 25-50 & 0.37 ± 0.48 & 0.42 ± 0.31 & 0.39 ± 0.49 & 0.42 ± 0.31 & 0.43 ± 0.50 & 0.39 ± 0.33 & 0.92 ± 0.28 & 0.12 ± 0.22 \\
& 10-50 & 0.31 ± 0.46 & 0.46 ± 0.29 & 0.30 ± 0.46 & 0.47 ± 0.29 & 0.33 ± 0.47 & 0.45 ± 0.30 & 0.86 ± 0.35 & 0.15 ± 0.24 \\
\midrule
\multirow{3}{*}{FMR} & 50-50 & 0.87 ± 0.34 & 0.10 ± 0.23 & 0.87 ± 0.33 & 0.09 ± 0.22 & 0.94 ± 0.24 & 0.04 ± 0.15 & 0.99 ± 0.11 & 0.01 ± 0.08 \\
& 25-50 & 0.88 ± 0.32 & 0.10 ± 0.22 & 0.81 ± 0.39 & 0.14 ± 0.26 & 0.87 ± 0.33 & 0.07 ± 0.18 & 0.97 ± 0.16 & 0.01 ± 0.08 \\
& 10-50 & 0.81 ± 0.39 & 0.15 ± 0.25 & 0.70 ± 0.46 & 0.21 ± 0.30 & 0.87 ± 0.34 & 0.09 ± 0.21 & 0.95 ± 0.22 & 0.02 ± 0.09 \\
\bottomrule
\end{tabular}
\end{table*}

\begin{table*}[h]
\centering
\caption{PathBB performance comparison across algorithms and data ratios. Results show mean ± std for the last 10 evaluations, over 5 seeds.}
\label{tab:pathbb-main}
\footnotesize
\setlength{\tabcolsep}{3.5pt}
\begin{tabular}{lccccccccc}
\toprule
& & \multicolumn{2}{c}{\textbf{BC}} & \multicolumn{2}{c}{\textbf{IQL}} & \multicolumn{2}{c}{\textbf{DemoDICE}} & \multicolumn{2}{c}{\textbf{ReCOIL}} \\
\textbf{Variant} & \textbf{Ratio} & Return & Mis. & Return & Mis. & Return & Mis. & Return & Mis. \\
\midrule
\multirow{3}{*}{Base} & 50-50 & 0.46 ± 0.50 & 0.31 ± 0.32 & 0.48 ± 0.50 & 0.30 ± 0.33 & 0.58 ± 0.49 & 0.27 ± 0.35 & 0.78 ± 0.41 & 0.15 ± 0.31 \\
 & 25-50 & 0.34 ± 0.47 & 0.39 ± 0.34 & 0.38 ± 0.48 & 0.36 ± 0.33 & 0.43 ± 0.50 & 0.38 ± 0.39 & 0.72 ± 0.45 & 0.15 ± 0.30 \\
 & 10-50 & 0.25 ± 0.43 & 0.44 ± 0.30 & 0.25 ± 0.43 & 0.43 ± 0.33 & 0.31 ± 0.46 & 0.46 ± 0.36 & 0.59 ± 0.49 & 0.24 ± 0.33 \\
\midrule
\multirow{3}{*}{FMR}  & 50-50 & 0.71 ± 0.45 & 0.16 ± 0.27 & 0.65 ± 0.48 & 0.20 ± 0.29 & 0.84 ± 0.37 & 0.13 ± 0.25 & 0.92 ± 0.27 & 0.05 ± 0.17 \\
 & 25-50 & 0.65 ± 0.48 & 0.20 ± 0.31 & 0.57 ± 0.49 & 0.26 ± 0.33 & 0.76 ± 0.43 & 0.22 ± 0.33 & 0.88 ± 0.32 & 0.07 ± 0.19 \\
 & 10-50 & 0.47 ± 0.50 & 0.33 ± 0.34 & 0.40 ± 0.49 & 0.38 ± 0.36 & 0.67 ± 0.47 & 0.26 ± 0.35 & 0.71 ± 0.45 & 0.10 ± 0.21 \\
\bottomrule
\end{tabular}
\end{table*}

\begin{figure}[H]
    \centering
    
    \begin{subfigure}[b]{\textwidth}
        \centering
        \includegraphics[width=0.7\textwidth,trim=0 665 1 0,clip]{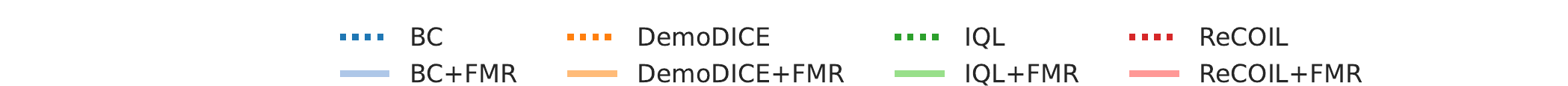}
    \end{subfigure}
    
    \vspace{0.5em}
    
    \begin{subfigure}[b]{0.32\textwidth}
        \centering
        \includegraphics[width=\textwidth,trim=0 44 0 55,clip]{images/main/PathM/success_50-50.pdf}
    \end{subfigure}
    \hfill
    \begin{subfigure}[b]{0.32\textwidth}
        \centering
        \includegraphics[width=\textwidth,trim=0 44 0 55,clip]{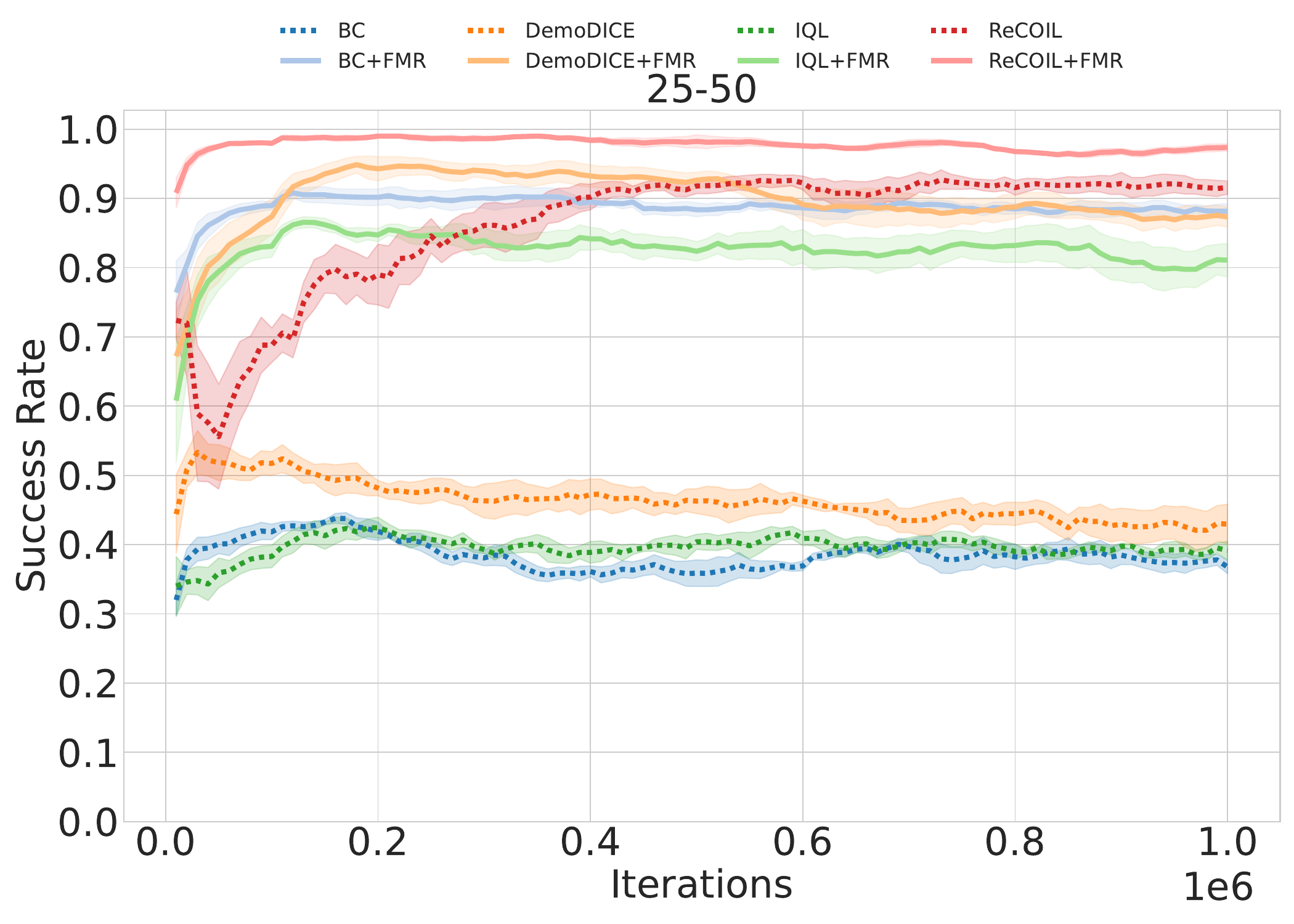}
    \end{subfigure}
    \hfill
    \begin{subfigure}[b]{0.32\textwidth}
        \centering
        \includegraphics[width=\textwidth,trim=0 44 0 55,clip]{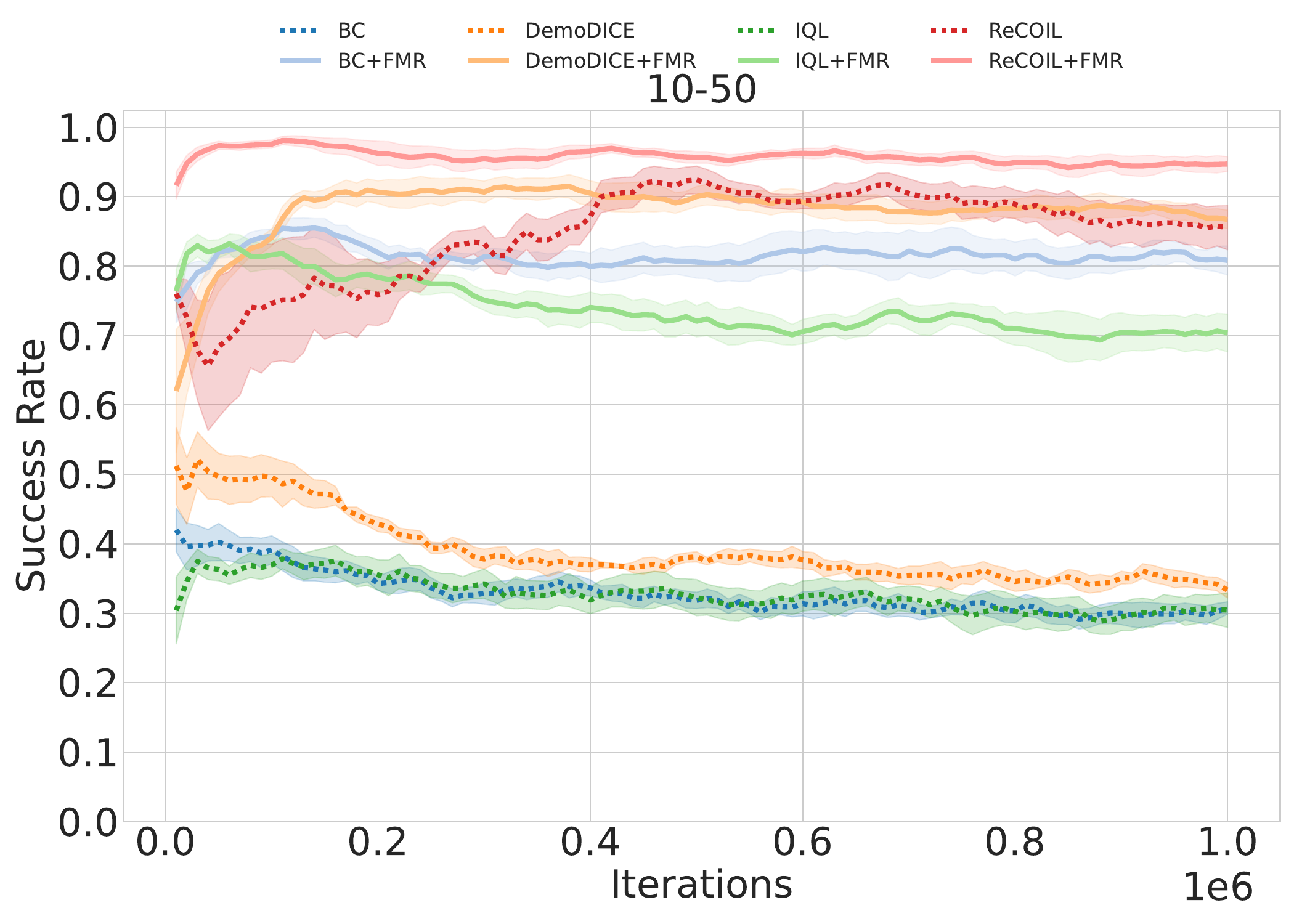}
    \end{subfigure}
    
    \begin{subfigure}[b]{0.32\textwidth}
        \centering
        \includegraphics[width=\textwidth,trim=0 0 0 81,clip]{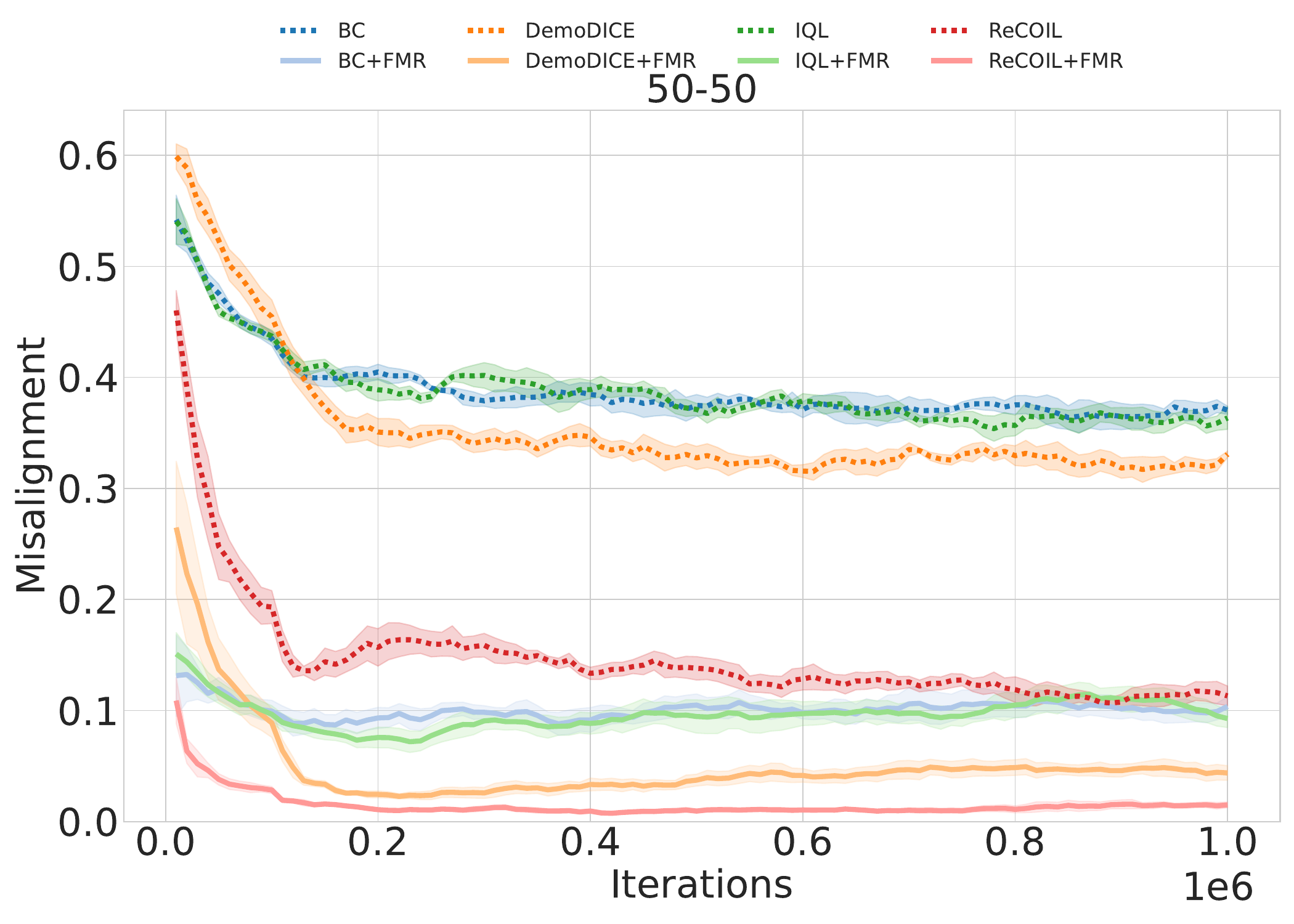}
    \end{subfigure}
    \hfill
    \begin{subfigure}[b]{0.32\textwidth}
        \centering
        \includegraphics[width=\textwidth,trim=0 0 0 81,clip]{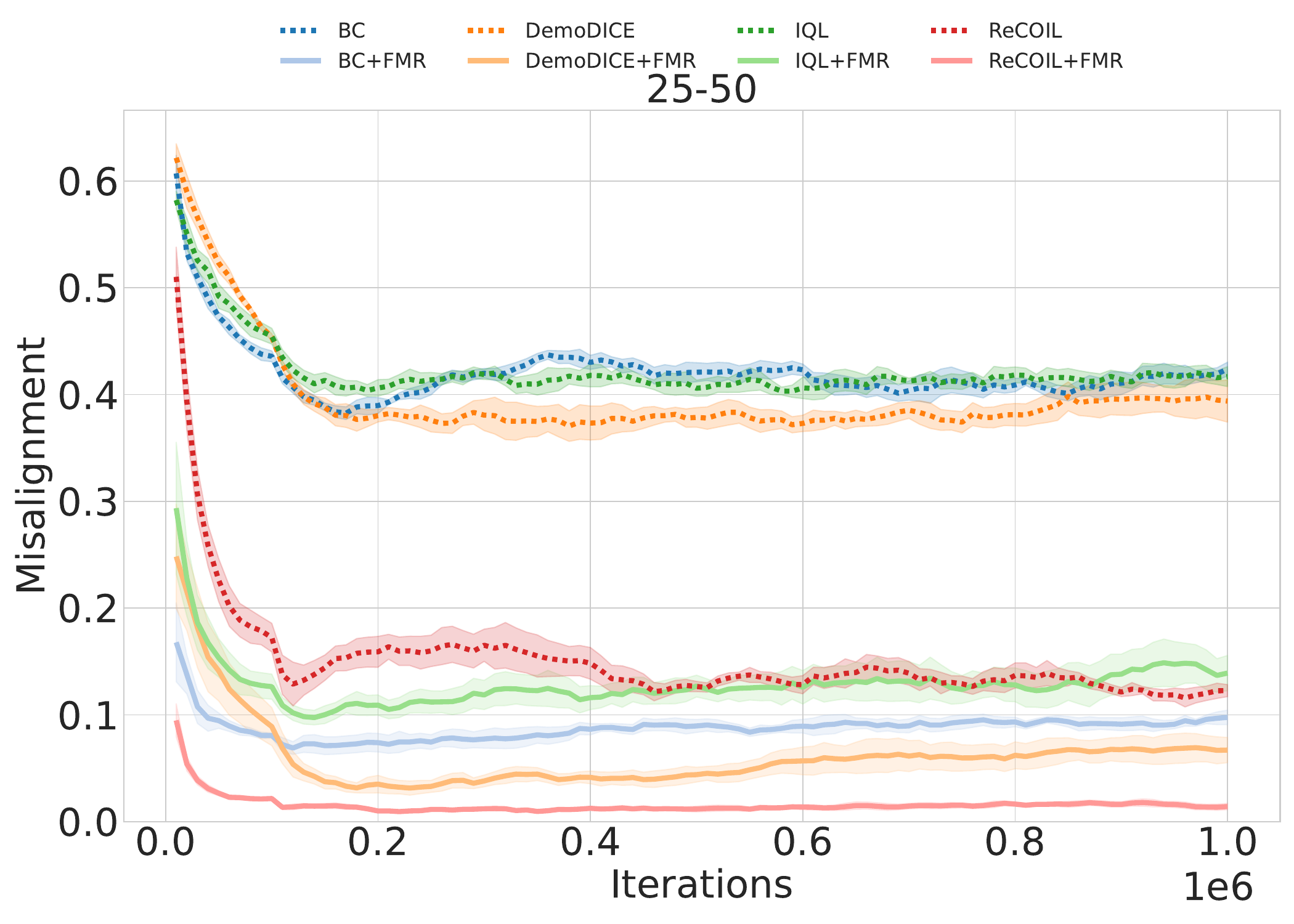}
    \end{subfigure}
    \hfill
    \begin{subfigure}[b]{0.32\textwidth}
        \centering
        \includegraphics[width=\textwidth,trim=0 0 0 81,clip]{images/main/PathM/misalignment_10-50.pdf}
    \end{subfigure}
    
    \caption{PathM learning curves for baselines and FMR. The shaded region represents the standard error.}
    \label{fig:pathm-main}
\end{figure}

\begin{figure}[H]
    \centering
    
    \begin{subfigure}[b]{\textwidth}
        \centering
        \includegraphics[width=0.7\textwidth,trim=0 665 0 0,clip]{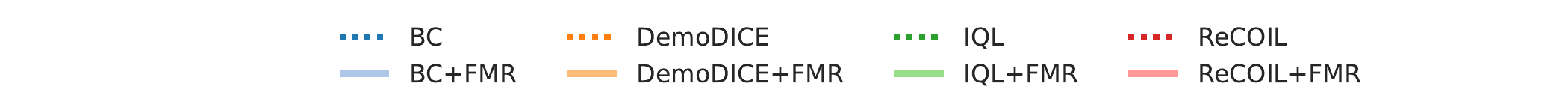}
    \end{subfigure}
    
    \vspace{0.5em}
    
    \begin{subfigure}[b]{0.32\textwidth}
        \centering
        \includegraphics[width=\textwidth,trim=0 44 0 55,clip]{images/main/PathBB/success_50-50.pdf}
    \end{subfigure}
    \hfill
    \begin{subfigure}[b]{0.32\textwidth}
        \centering
        \includegraphics[width=\textwidth,trim=0 44 0 55,clip]{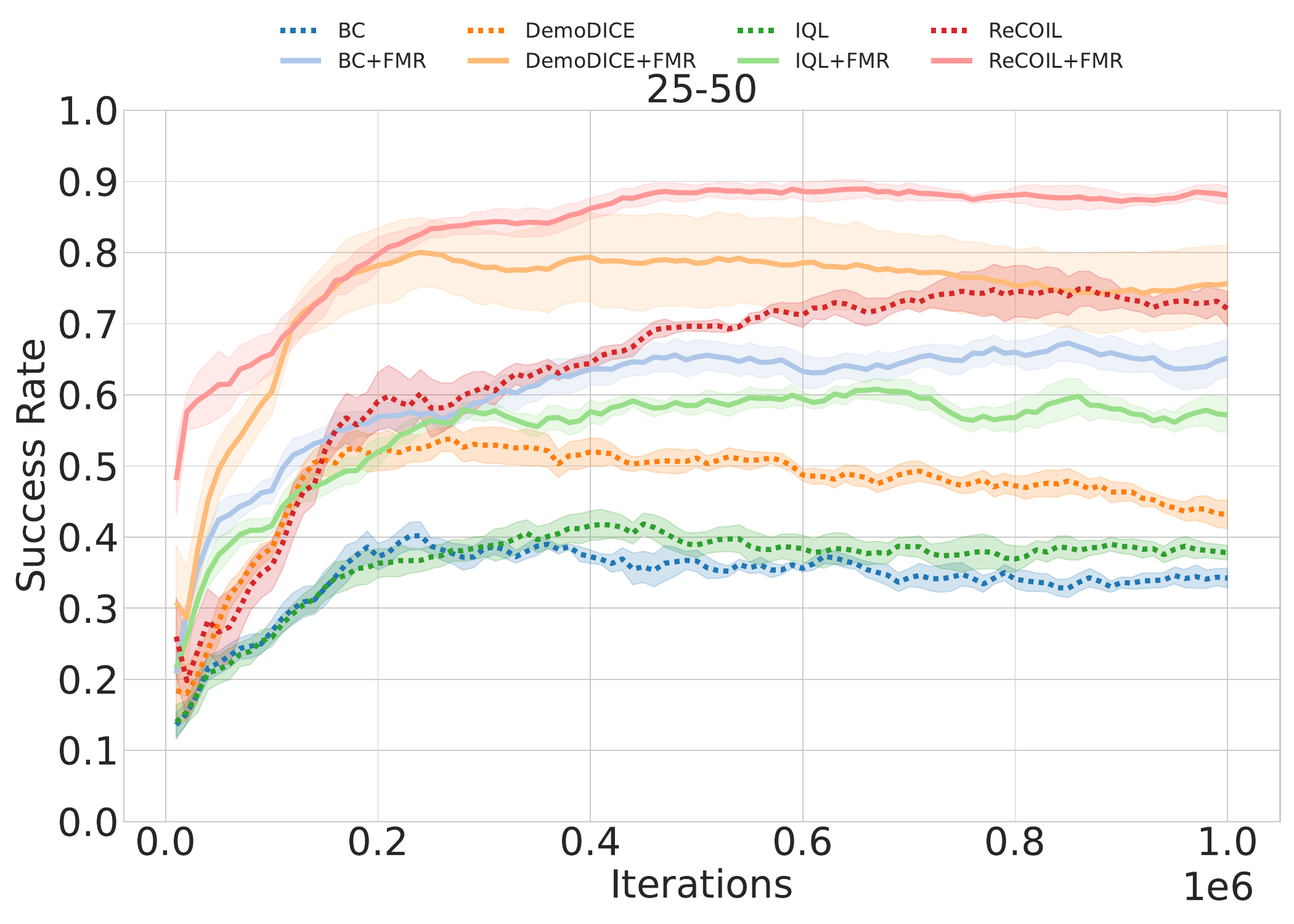}
    \end{subfigure}
    \hfill
    \begin{subfigure}[b]{0.32\textwidth}
        \centering
        \includegraphics[width=\textwidth,trim=0 44 0 55,clip]{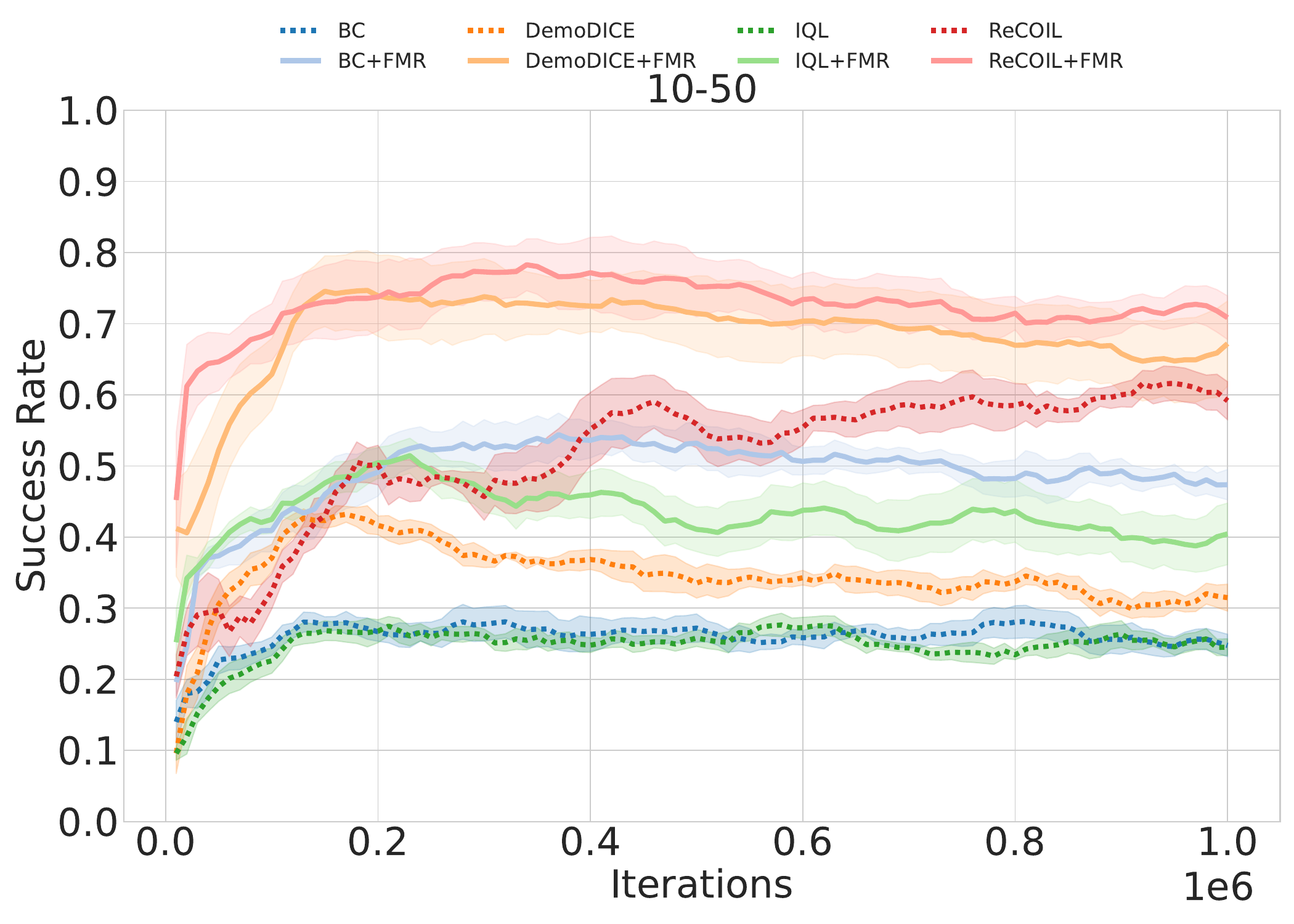}
    \end{subfigure}
    
    \begin{subfigure}[b]{0.32\textwidth}
        \centering
        \includegraphics[width=\textwidth,trim=0 0 0 81,clip]{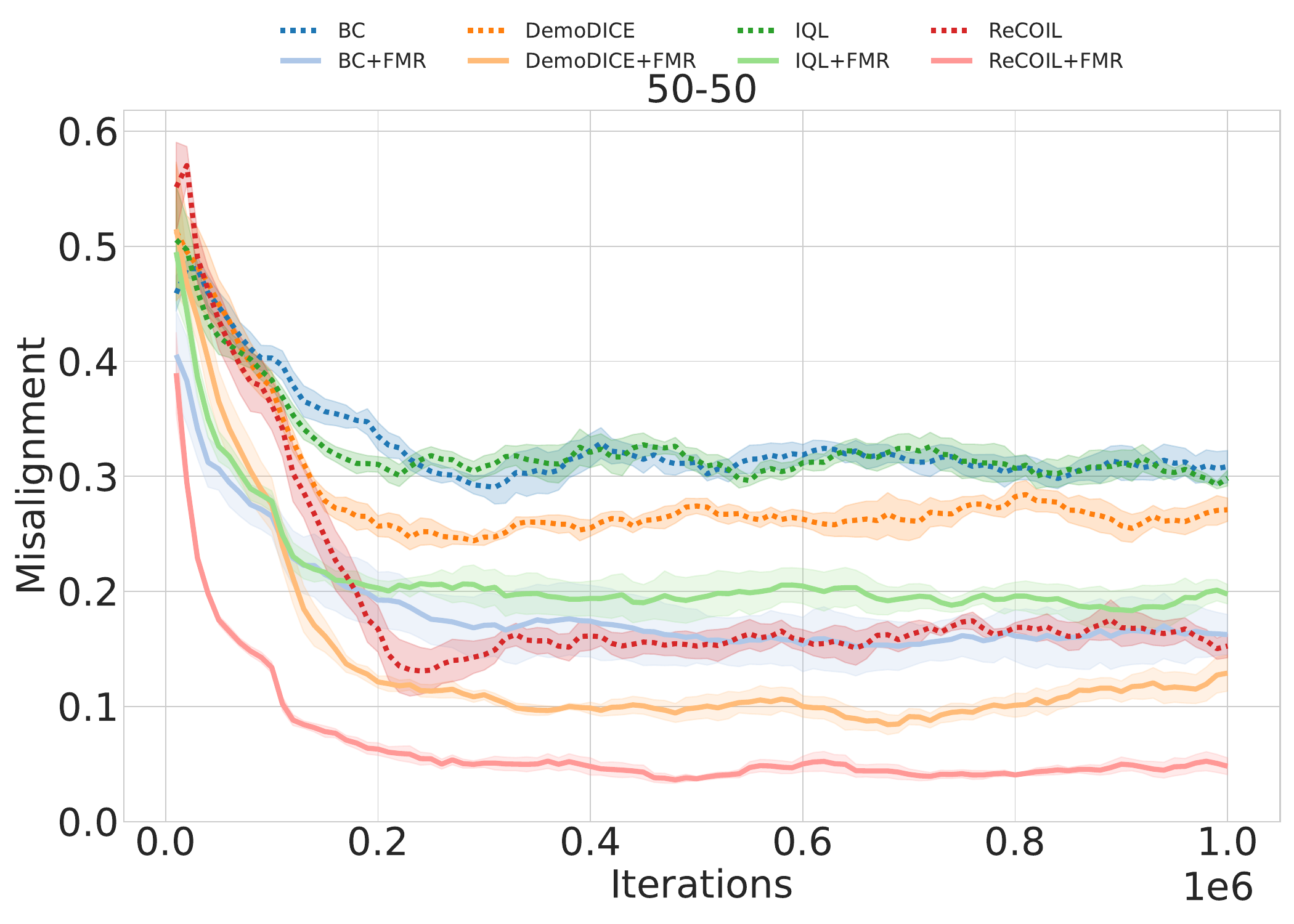}
    \end{subfigure}
    \hfill
    \begin{subfigure}[b]{0.32\textwidth}
        \centering
        \includegraphics[width=\textwidth,trim=0 0 0 81,clip]{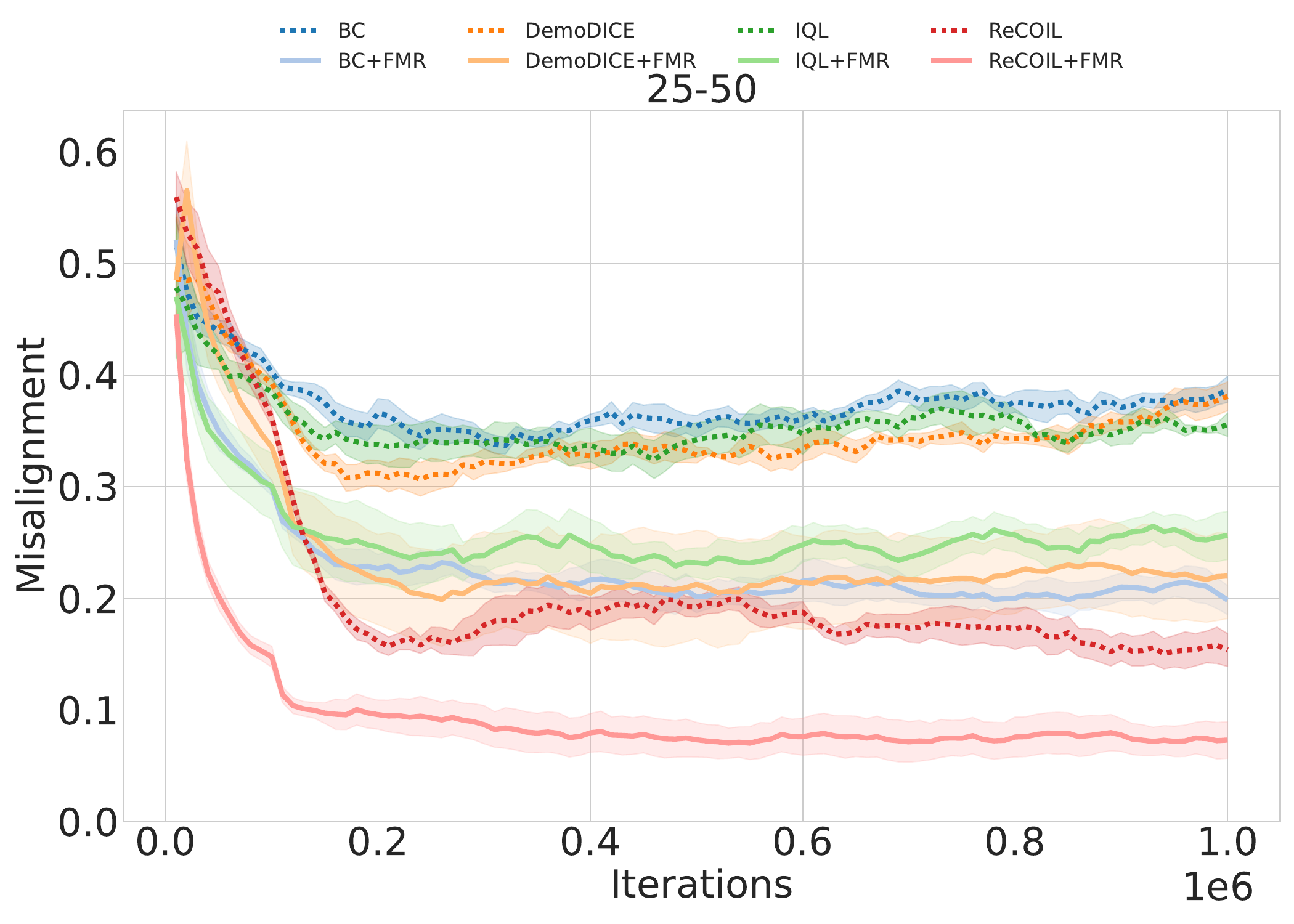}
    \end{subfigure}
    \hfill
    \begin{subfigure}[b]{0.32\textwidth}
        \centering
        \includegraphics[width=\textwidth,trim=0 0 0 81,clip]{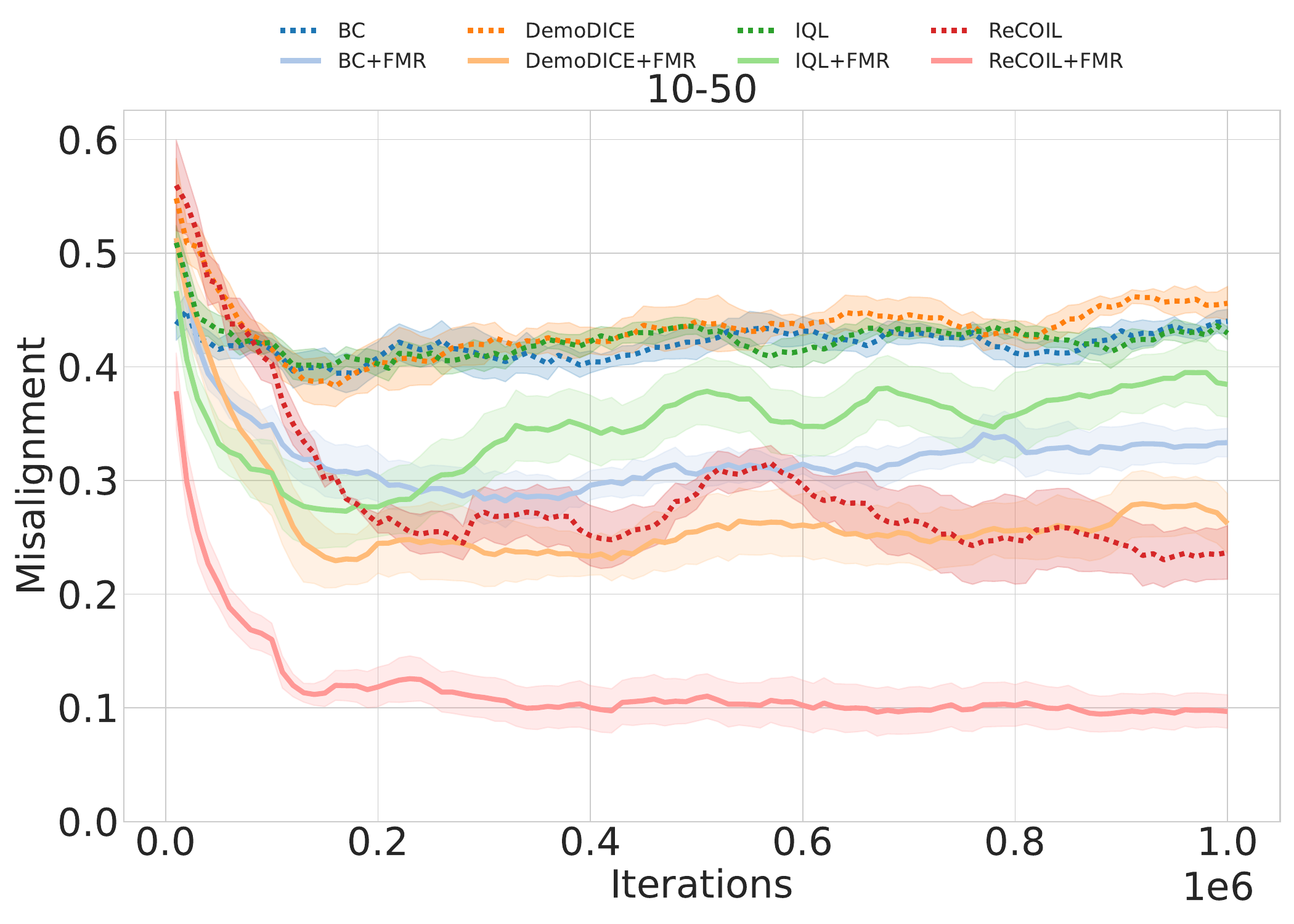}
    \end{subfigure}
    
    \caption{PathBB learning curves for baselines and FMR. The shaded region represents the standard error.}
    \label{fig:pathbb-main}
\end{figure}

%% file: appendix/main-vel.tex
The following are the results for the velocity tasks where return is normalized based on $D^E$ performance.

\begin{figure}[h]
    \centering
    
    \begin{subfigure}[b]{\textwidth}
        \centering
        \includegraphics[width=0.7\textwidth,trim=0 275 1 0,clip]{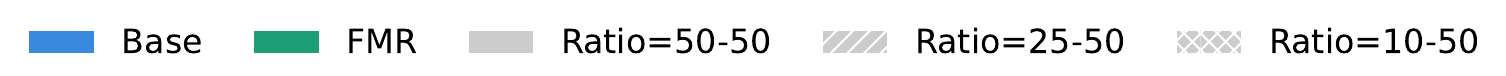}
    \end{subfigure}
    
    \begin{minipage}[t]{0.32\textwidth}
        \centering
        \textbf{Swimmer}
    \end{minipage}
    \hfill
    \begin{minipage}[t]{0.32\textwidth}
        \centering
        \textbf{Hopper}
    \end{minipage}
    \begin{minipage}[t]{0.32\textwidth}
        \centering
        \textbf{Walker2D}
    \end{minipage}
    
    \begin{subfigure}[b]{0.32\textwidth}
        \centering
        \includegraphics[width=\textwidth,trim=0 0 0 35,clip]{images/main/SlowSwim/fmr_return_compare.pdf}
    \end{subfigure}
    \hfill
    \begin{subfigure}[b]{0.32\textwidth}
        \centering
        \includegraphics[width=\textwidth,trim=0 0 0 35,clip]{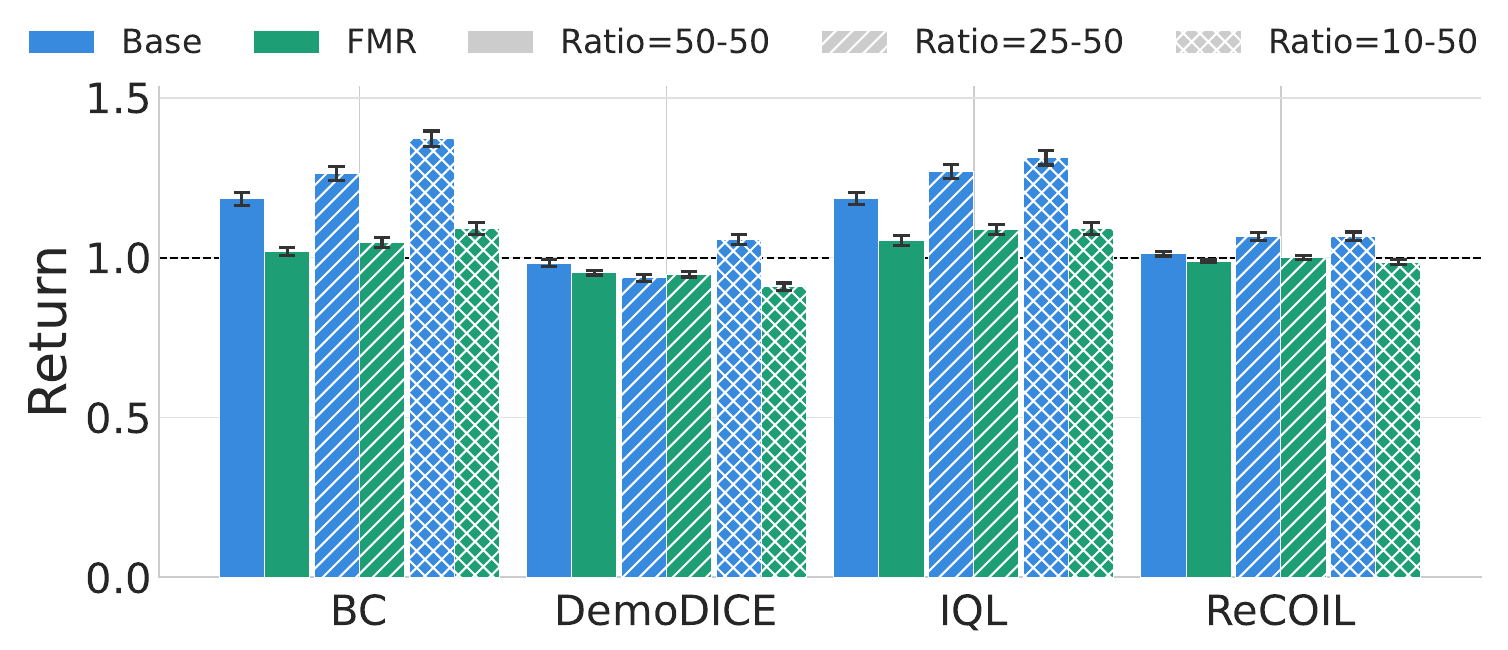}
    \end{subfigure}    
    \hfill
    \begin{subfigure}[b]{0.32\textwidth}
        \centering
        \includegraphics[width=\textwidth,trim=0 0 0 35,clip]{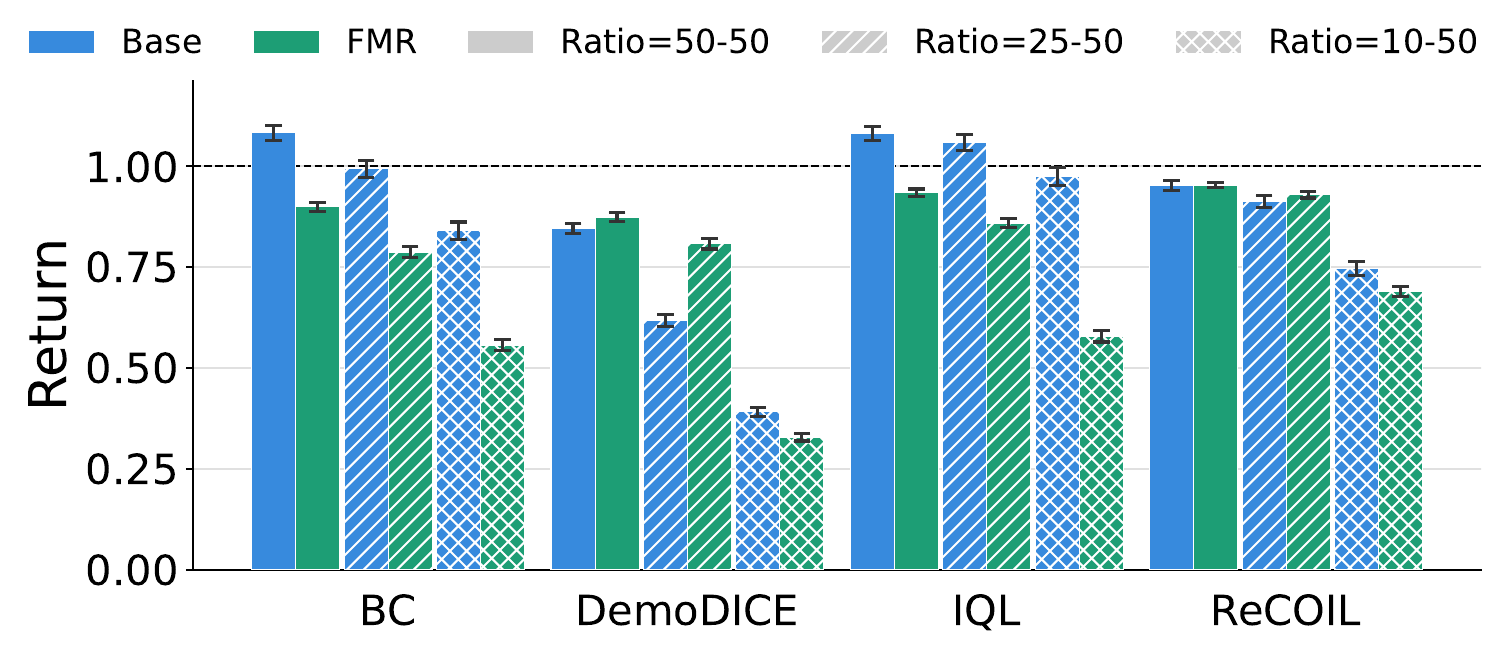}
    \end{subfigure}
    
    \begin{subfigure}[b]{0.32\textwidth}
        \centering
        \includegraphics[width=\textwidth,trim=0 0 0 35,clip]{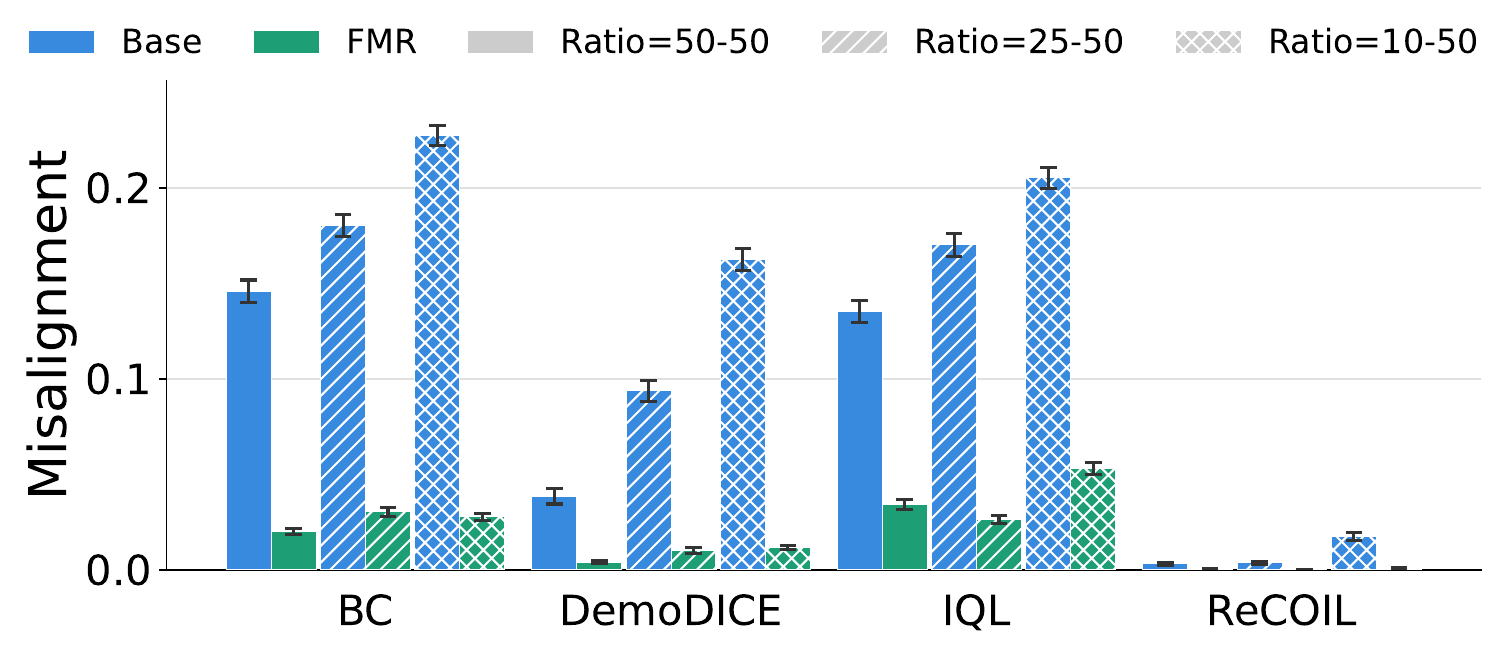}
    \end{subfigure}
    \hfill
    \begin{subfigure}[b]{0.32\textwidth}
        \centering
        \includegraphics[width=\textwidth,trim=0 0 0 35,clip]{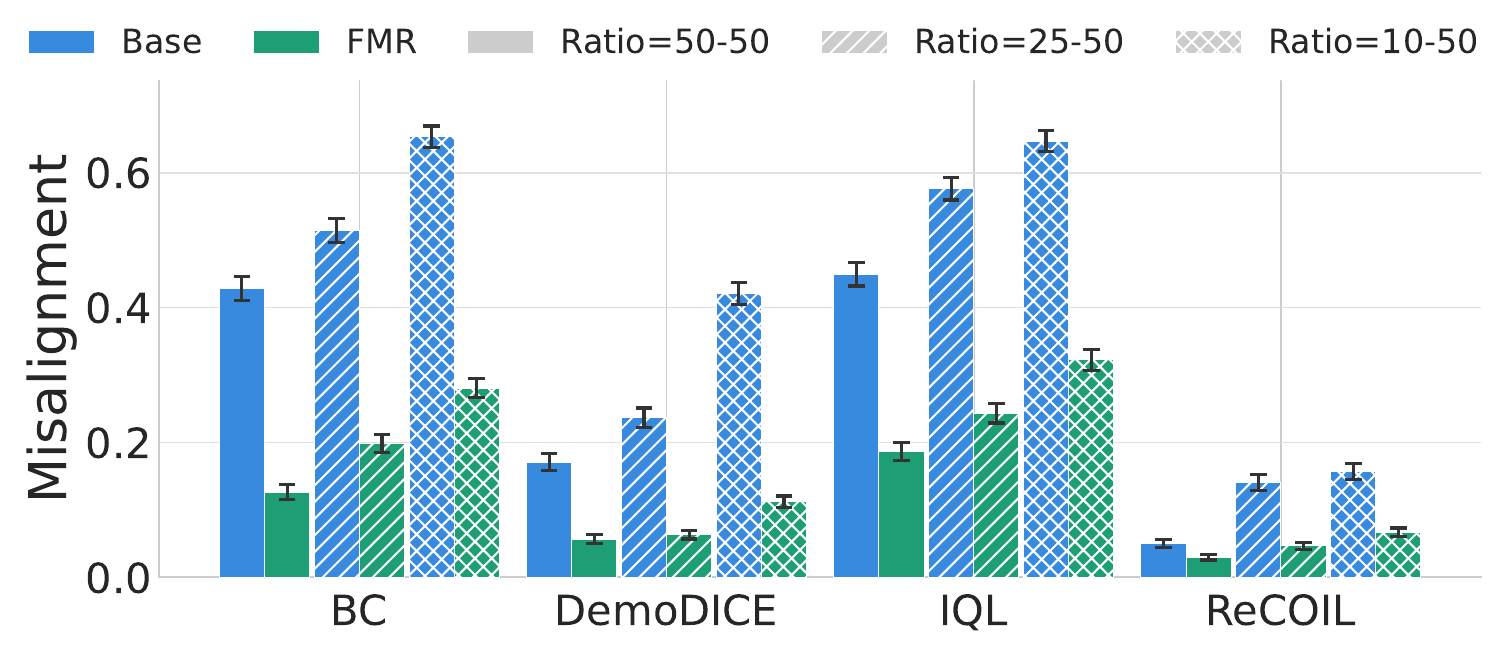}
    \end{subfigure}    
    \hfill
    \begin{subfigure}[b]{0.32\textwidth}
        \centering
        \includegraphics[width=\textwidth,trim=0 0 0 35,clip]{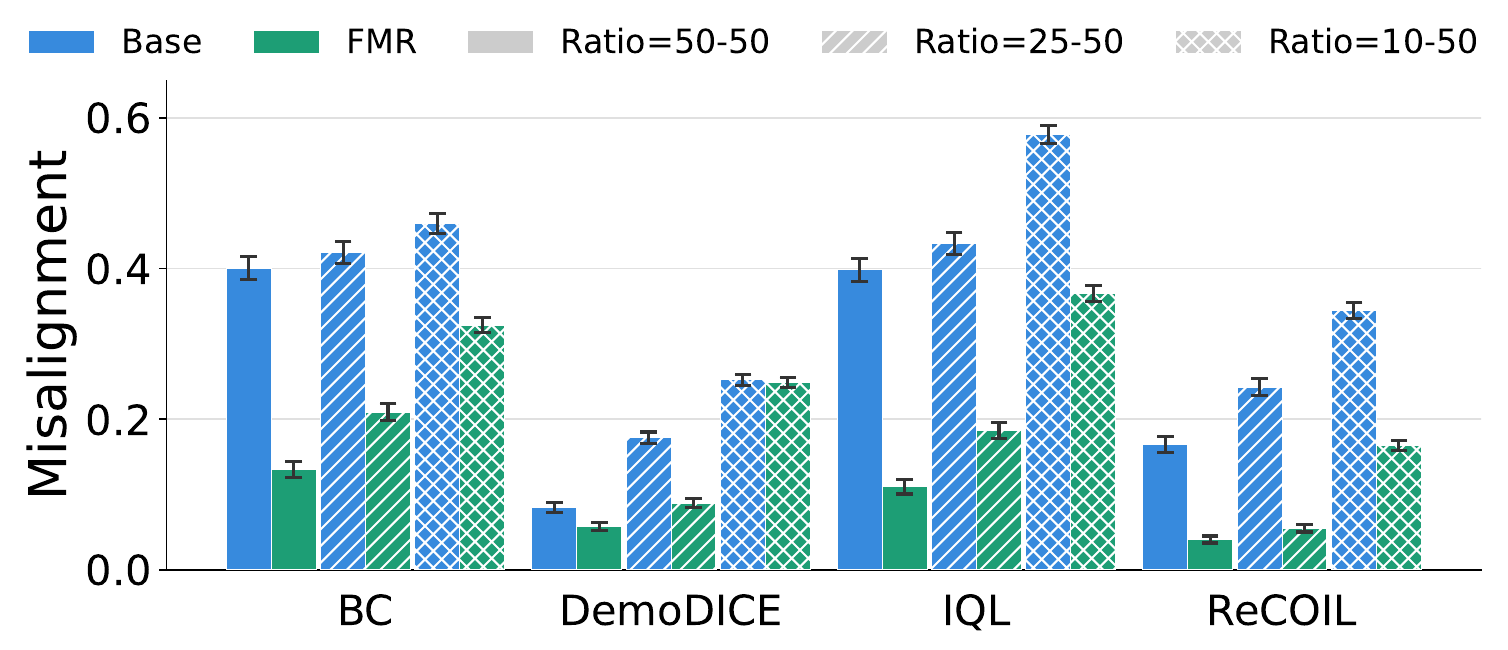}
    \end{subfigure}
    
    \caption{Velocity results across all baseline algorithms and data ratios. Results show mean success rate and misalignment scores with 95\% confidence intervals, averaged over the last 10 evaluations across 5 seeds}
    \label{fig:vel-results}
\end{figure}

\begin{table*}[h]
\centering
\caption{SlowSwim performance comparison across algorithms and data ratios. Results show mean ± std for the last 10 evaluations, over 5 seeds. The $^*$ symbol indicates a very small non-zero number.} 
\label{tab:swim-main}
\footnotesize
\setlength{\tabcolsep}{3.5pt}
\begin{tabular}{lccccccccc}
\toprule
& & \multicolumn{2}{c}{\textbf{BC}} & \multicolumn{2}{c}{\textbf{IQL}} & \multicolumn{2}{c}{\textbf{DemoDICE}} & \multicolumn{2}{c}{\textbf{ReCOIL}} \\
\textbf{Variant} & \textbf{Ratio} & Return & Mis. & Return & Mis. & Return & Mis. & Return & Mis. \\
\midrule
\multirow{3}{*}{Base} & 50-50 & 0.96 ± 0.18 & 0.15 ± 0.15 & 0.97 ± 0.15 & 0.14 ± 0.15 & 1.00 ± 0.05 & 0.04 ± 0.10 & 0.99 ± 0.07 & $0.00^\ast$ ± 0.02 \\
& 25-50 & 0.95 ± 0.20 & 0.18 ± 0.15 & 0.95 ± 0.19 & 0.17 ± 0.15 & 1.00 ± 0.08 & 0.09 ± 0.14 & 0.99 ± 0.08 & $0.00^\ast$ ± 0.02 \\
& 10-50 & 0.93 ± 0.23 & 0.23 ± 0.14 & 0.96 ± 0.17 & 0.21 ± 0.14 & 0.98 ± 0.14 & 0.16 ± 0.15 & 0.97 ± 0.15 & 0.02 ± 0.05 \\
\midrule
\multirow{3}{*}{FMR}  & 50-50 & 0.97 ± 0.13 & 0.02 ± 0.04 & 0.97 ± 0.09 & 0.03 ± 0.07 & 0.99 ± 0.08 & $0.00^\ast$ ± 0.02 & 1.00 ± 0.01 & $0.00^\ast$ ± $0.00^\ast$ \\
 & 25-50 & 0.95 ± 0.14 & 0.03 ± 0.06 & 0.99 ± 0.07 & 0.03 ± 0.05 & 1.00 ± 0.03 & 0.01 ± 0.04 & 1.00 ± 0.01 & $0.00^\ast$ ± $0.00^\ast$ \\
 & 10-50 & 0.99 ± 0.07 & 0.03 ± 0.05 & 0.95 ± 0.13 & 0.05 ± 0.08 & 1.00 ± 0.06 & 0.01 ± 0.03 & 1.00 ± 0.01 & $0.00^\ast$ ± $0.00^\ast$ \\
\bottomrule
\end{tabular}
\end{table*}

\begin{table*}[h]
\centering
\caption{SlowHop performance comparison across algorithms and data ratios. Results show mean ± std for the last 10 evaluations, over 5 seeds.}
\label{tab:hop-main}
\footnotesize
\setlength{\tabcolsep}{3.5pt}
\begin{tabular}{lccccccccc}
\toprule
& & \multicolumn{2}{c}{\textbf{BC}} & \multicolumn{2}{c}{\textbf{IQL}} & \multicolumn{2}{c}{\textbf{DemoDICE}} & \multicolumn{2}{c}{\textbf{ReCOIL}} \\
\textbf{Variant} & \textbf{Ratio} & Return & Mis. & Return & Mis. & Return & Mis. & Return & Mis. \\
\midrule
\multirow{3}{*}{Base}  & 50-50 & 1.18 ± 0.51 & 0.43 ± 0.45 & 1.18 ± 0.49 & 0.45 ± 0.45 & 0.98 ± 0.28 & 0.17 ± 0.32 & 1.01 ± 0.18 & 0.05 ± 0.16 \\
 & 25-50 & 1.26 ± 0.55 & 0.51 ± 0.45 & 1.27 ± 0.55 & 0.58 ± 0.43 & 0.94 ± 0.28 & 0.24 ± 0.37 & 1.07 ± 0.31 & 0.14 ± 0.30 \\
 & 10-50 & 1.37 ± 0.61 & 0.65 ± 0.40 & 1.31 ± 0.59 & 0.65 ± 0.40 & 1.06 ± 0.43 & 0.42 ± 0.41 & 1.07 ± 0.34 & 0.16 ± 0.30 \\
\midrule
\multirow{3}{*}{FMR}  & 50-50 & 1.02 ± 0.31 & 0.13 ± 0.28 & 1.05 ± 0.38 & 0.19 ± 0.34 & 0.95 ± 0.19 & 0.06 ± 0.16 & 0.99 ± 0.13 & 0.03 ± 0.10 \\
& 25-50 & 1.05 ± 0.40 & 0.20 ± 0.34 & 1.09 ± 0.41 & 0.24 ± 0.37 & 0.95 ± 0.21 & 0.06 ± 0.17 & 1.00 ± 0.18 & 0.05 ± 0.13 \\
& 10-50 & 1.09 ± 0.48 & 0.28 ± 0.38 & 1.09 ± 0.48 & 0.32 ± 0.39 & 0.91 ± 0.29 & 0.11 ± 0.22 & 0.99 ± 0.23 & 0.07 ± 0.17 \\
\bottomrule
\end{tabular}
\end{table*}

\begin{table*}[h]
\centering
\caption{SlowWalk performance comparison across algorithms and data ratios. Results show mean ± std for the last 10 evaluations, over 5 seeds.}
\label{tab:walk-main}
\footnotesize
\setlength{\tabcolsep}{3.5pt}
\begin{tabular}{lccccccccc}
\toprule
& & \multicolumn{2}{c}{\textbf{BC}} & \multicolumn{2}{c}{\textbf{IQL}} & \multicolumn{2}{c}{\textbf{DemoDICE}} & \multicolumn{2}{c}{\textbf{ReCOIL}} \\
\textbf{Variant} & \textbf{Ratio} & Return & Mis. & Return & Mis. & Return & Mis. & Return & Mis. \\
\midrule
\multirow{3}{*}{Base} & 50-50 & 1.08 ± 0.47 & 0.40 ± 0.39 & 1.08 ± 0.45 & 0.40 ± 0.39 & 0.85 ± 0.32 & 0.08 ± 0.17 & 0.95 ± 0.32 & 0.17 ± 0.27 \\
 & 25-50 & 0.99 ± 0.53 & 0.42 ± 0.38 & 1.06 ± 0.49 & 0.43 ± 0.39 & 0.62 ± 0.38 & 0.18 ± 0.20 & 0.91 ± 0.38 & 0.24 ± 0.30 \\
 & 10-50 & 0.84 ± 0.57 & 0.46 ± 0.33 & 0.98 ± 0.57 & 0.58 ± 0.32 & 0.39 ± 0.29 & 0.25 ± 0.18 & 0.75 ± 0.46 & 0.34 ± 0.27 \\
\midrule
\multirow{3}{*}{FMR} & 50-50 & 0.90 ± 0.27 & 0.13 ± 0.26 & 0.93 ± 0.23 & 0.11 ± 0.25 & 0.87 ± 0.29 & 0.06 ± 0.13 & 0.95 ± 0.17 & 0.04 ± 0.12 \\
 & 25-50 & 0.79 ± 0.36 & 0.21 ± 0.29 & 0.86 ± 0.30 & 0.18 ± 0.28 & 0.81 ± 0.32 & 0.09 ± 0.15 & 0.93 ± 0.21 & 0.06 ± 0.12 \\
 & 10-50 & 0.56 ± 0.35 & 0.33 ± 0.27 & 0.58 ± 0.36 & 0.37 ± 0.27 & 0.33 ± 0.25 & 0.25 ± 0.17 & 0.69 ± 0.32 & 0.17 ± 0.16 \\
\bottomrule
\end{tabular}
\end{table*}

\begin{figure}[h]
    \centering
    
    \begin{subfigure}[b]{\textwidth}
        \centering
        \includegraphics[width=0.7\textwidth,trim=0 665 0 0,clip]{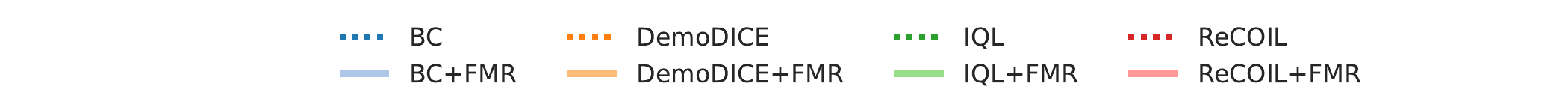}
    \end{subfigure}
    
    \vspace{0.5em}
    
    \begin{subfigure}[b]{0.32\textwidth}
        \centering
        \includegraphics[width=\textwidth,trim=0 44 0 55,clip]{images/main/SlowSwim/return_50-50.pdf}
    \end{subfigure}
    \hfill
    \begin{subfigure}[b]{0.32\textwidth}
        \centering
        \includegraphics[width=\textwidth,trim=0 44 0 55,clip]{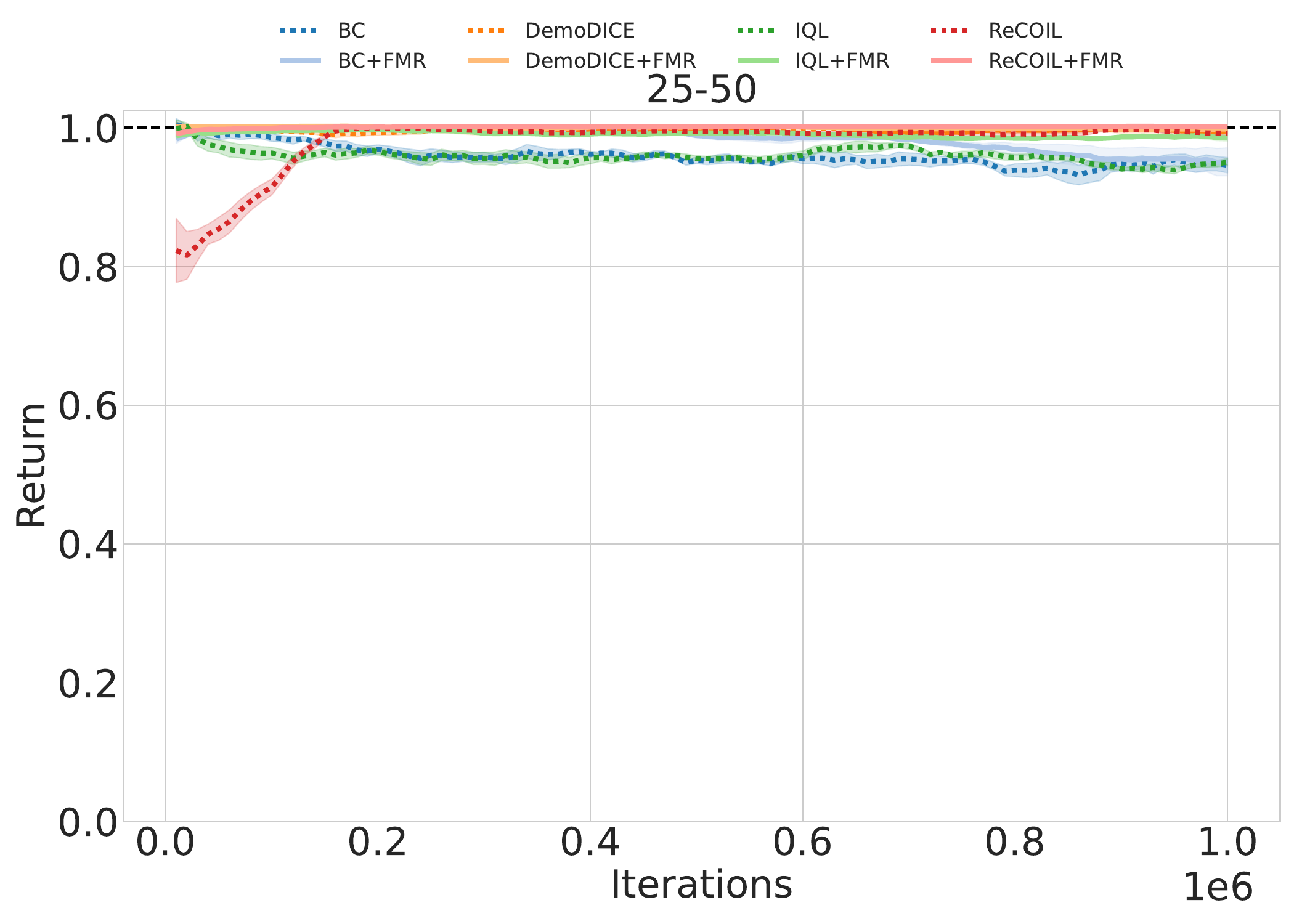}
    \end{subfigure}
    \hfill
    \begin{subfigure}[b]{0.32\textwidth}
        \centering
        \includegraphics[width=\textwidth,trim=0 44 0 55,clip]{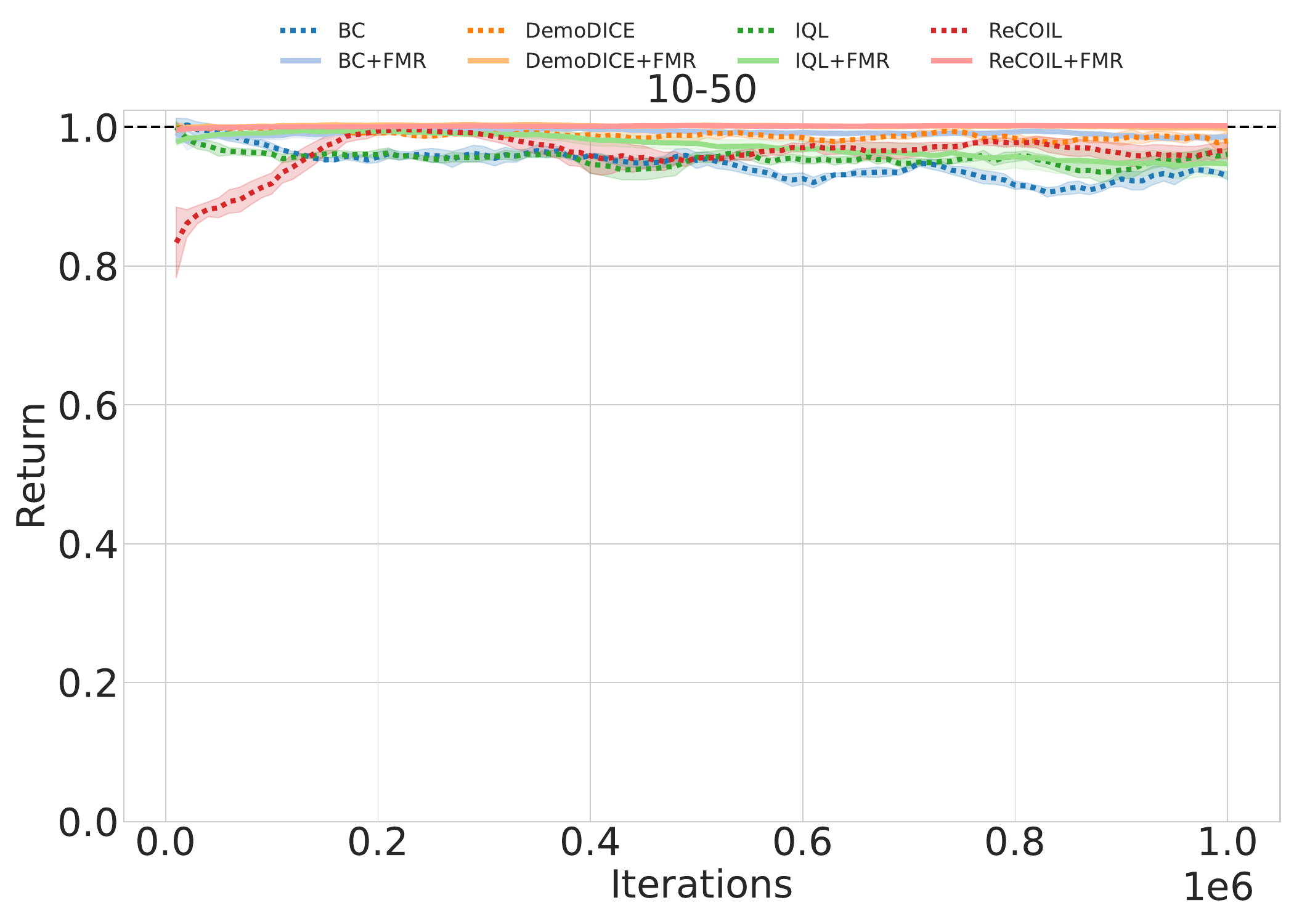}
    \end{subfigure}
    
    \begin{subfigure}[b]{0.32\textwidth}
        \centering
        \includegraphics[width=\textwidth,trim=0 0 0 81,clip]{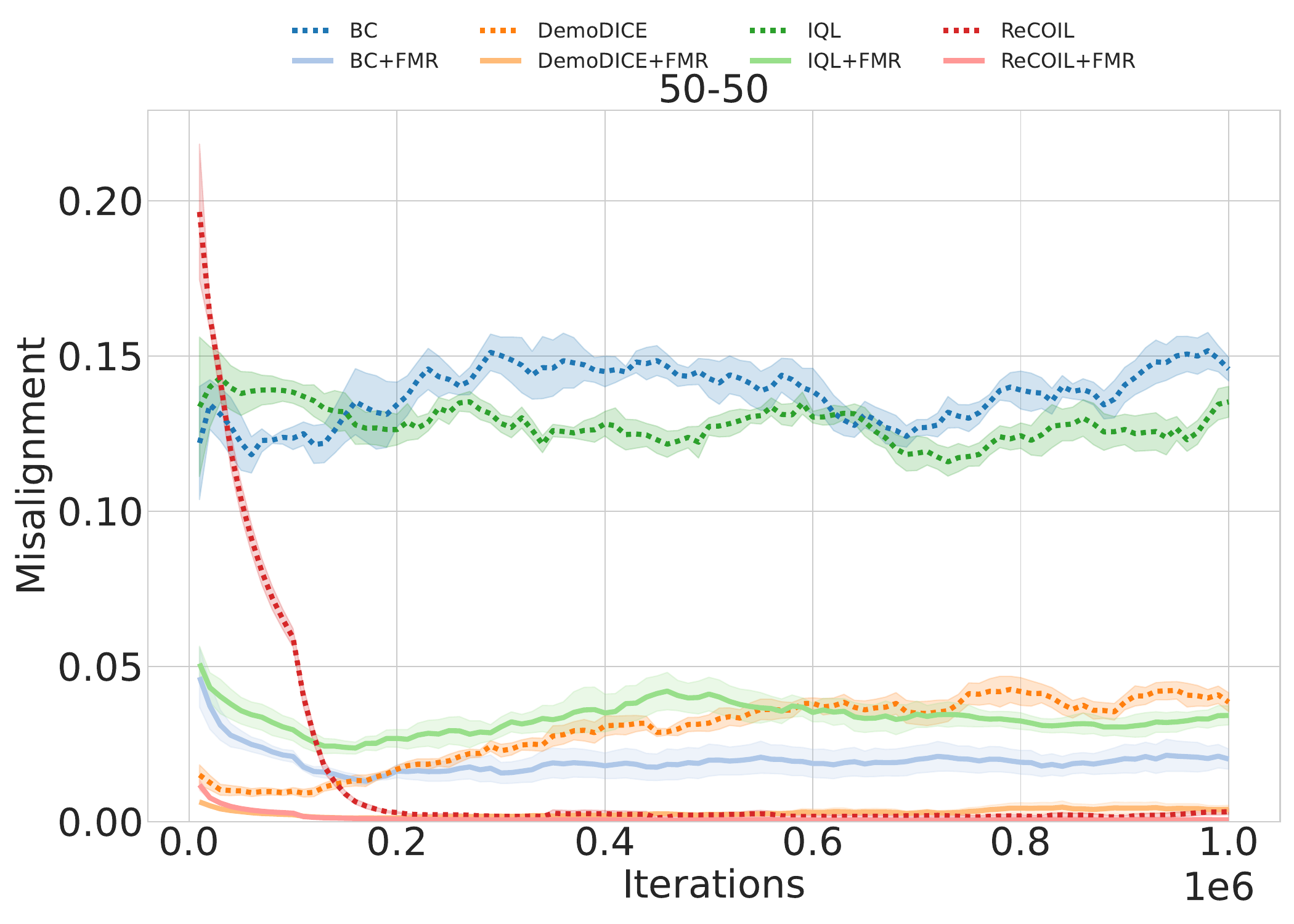}
    \end{subfigure}
    \hfill
    \begin{subfigure}[b]{0.32\textwidth}
        \centering
        \includegraphics[width=\textwidth,trim=0 0 0 81,clip]{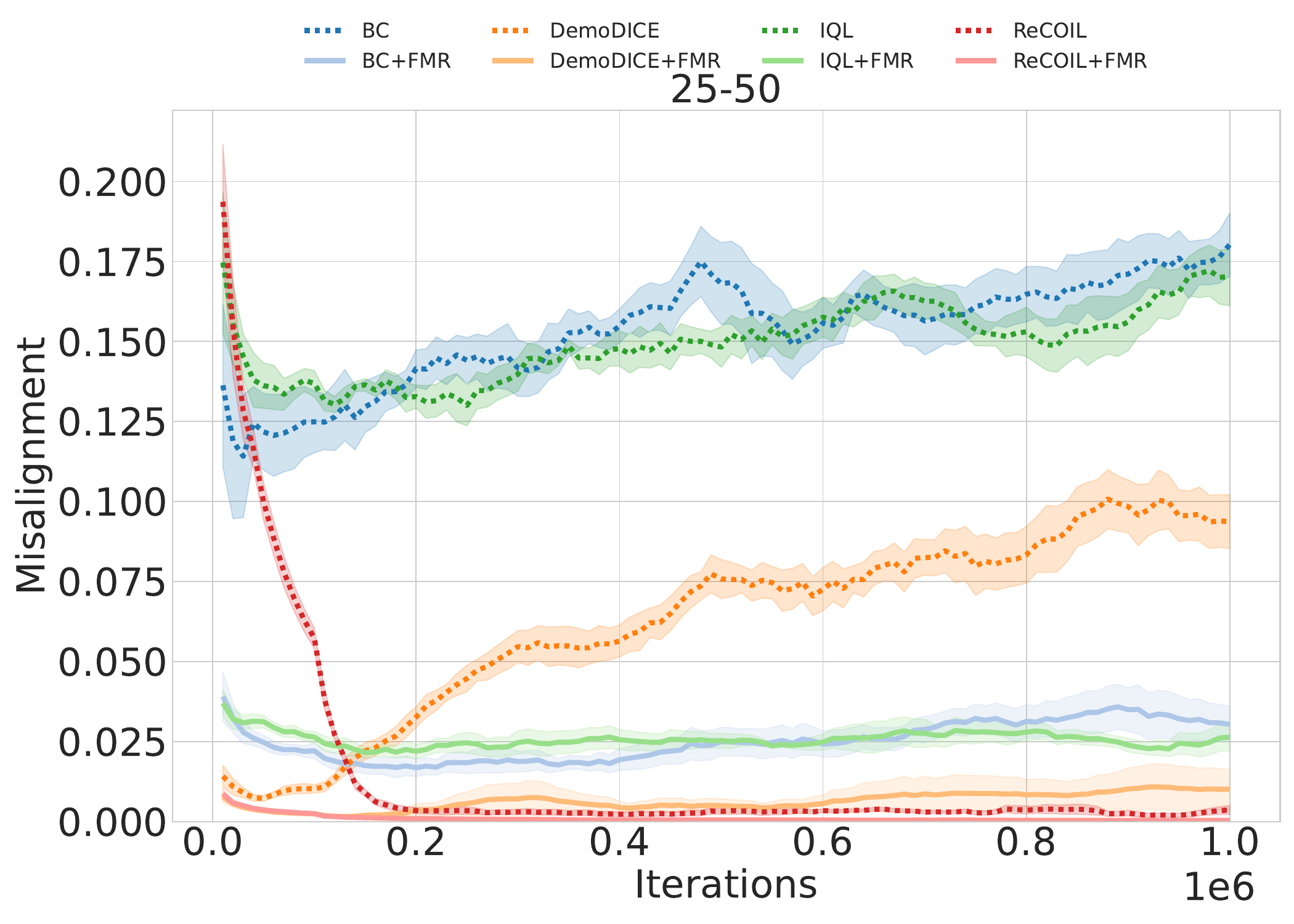}
    \end{subfigure}
    \hfill
    \begin{subfigure}[b]{0.32\textwidth}
        \centering
        \includegraphics[width=\textwidth,trim=0 0 0 81,clip]{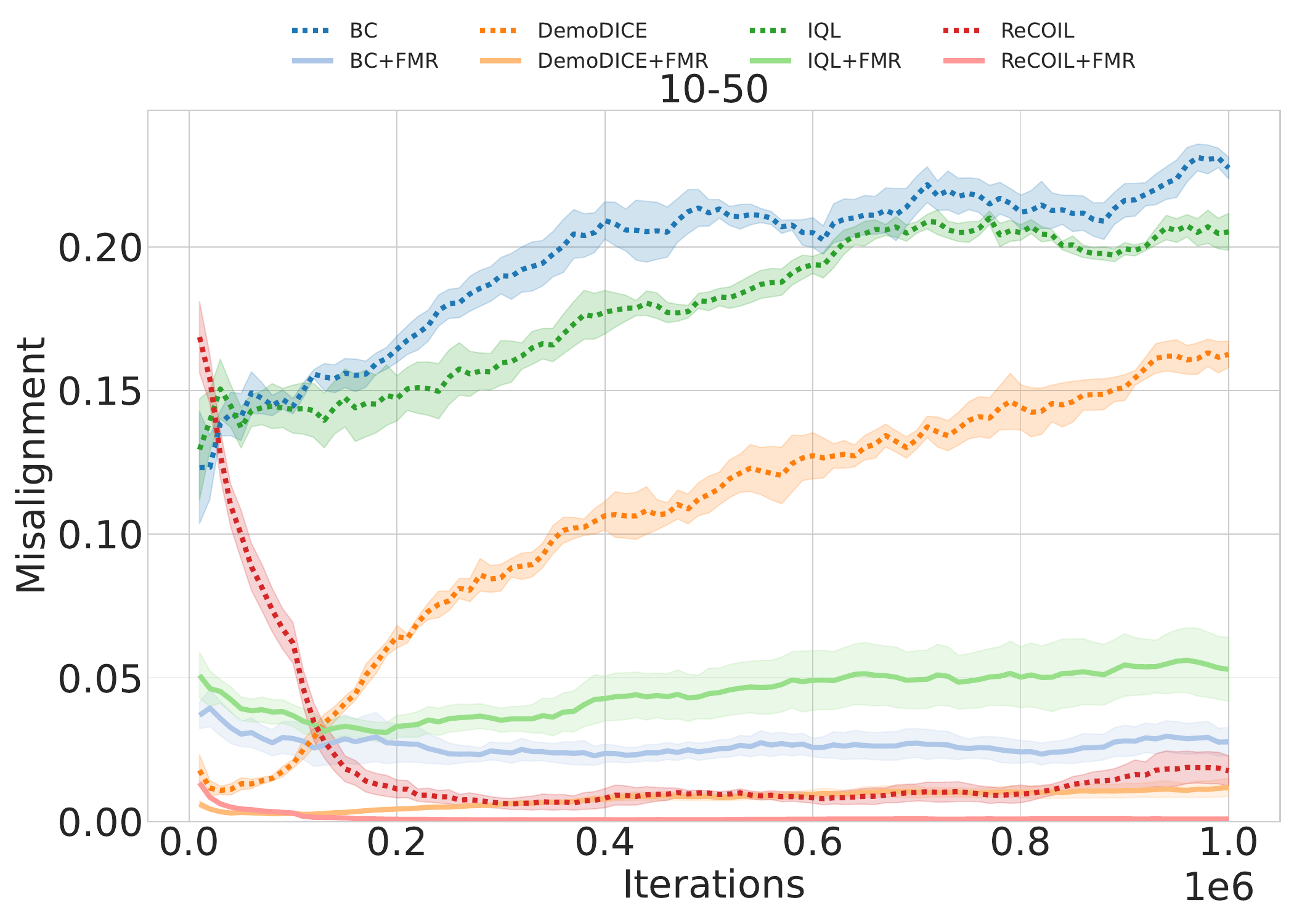}
    \end{subfigure}
    
    \caption{SlowSwim learning curves for baselines and FMR. The shaded region represents the standard error.}
    \label{fig:swim-main}
\end{figure}

\begin{figure}[h]
    \centering
    
    \begin{subfigure}[b]{\textwidth}
        \centering
        \includegraphics[width=0.7\textwidth,trim=0 665 0 0,clip]{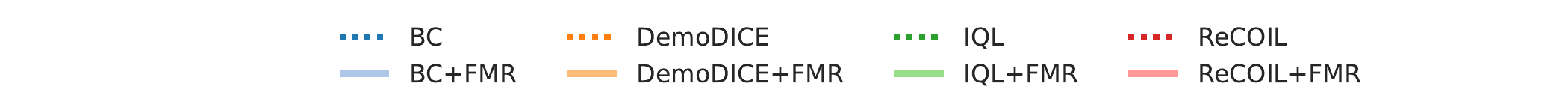}
    \end{subfigure}
    
    \vspace{0.5em}
    
    \begin{subfigure}[b]{0.32\textwidth}
        \centering
        \includegraphics[width=\textwidth,trim=0 44 0 55,clip]{images/main/SlowHop/return_50-50.pdf}
    \end{subfigure}
    \hfill
    \begin{subfigure}[b]{0.32\textwidth}
        \centering
        \includegraphics[width=\textwidth,trim=0 44 0 55,clip]{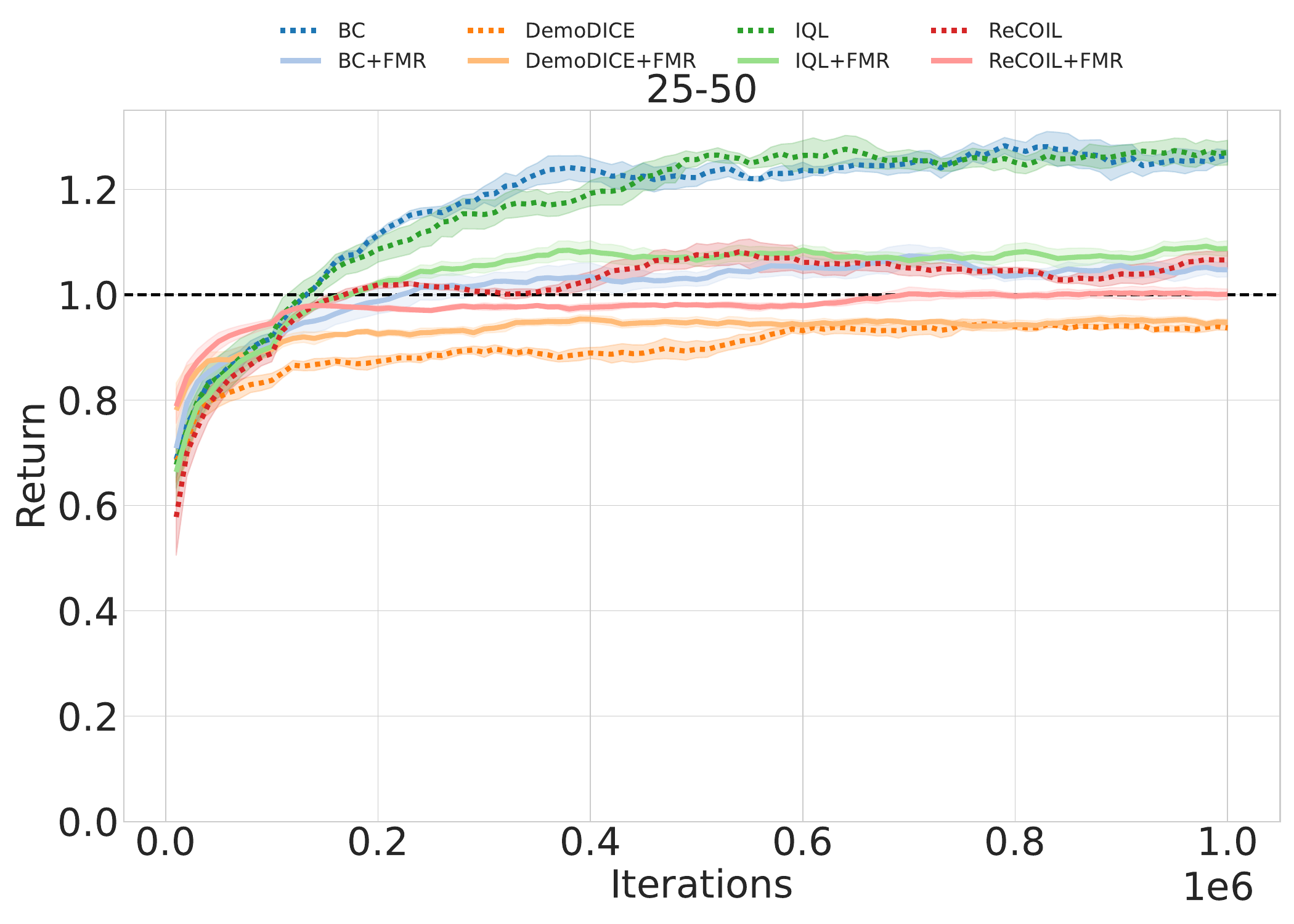}
    \end{subfigure}
    \hfill
    \begin{subfigure}[b]{0.32\textwidth}
        \centering
        \includegraphics[width=\textwidth,trim=0 44 0 55,clip]{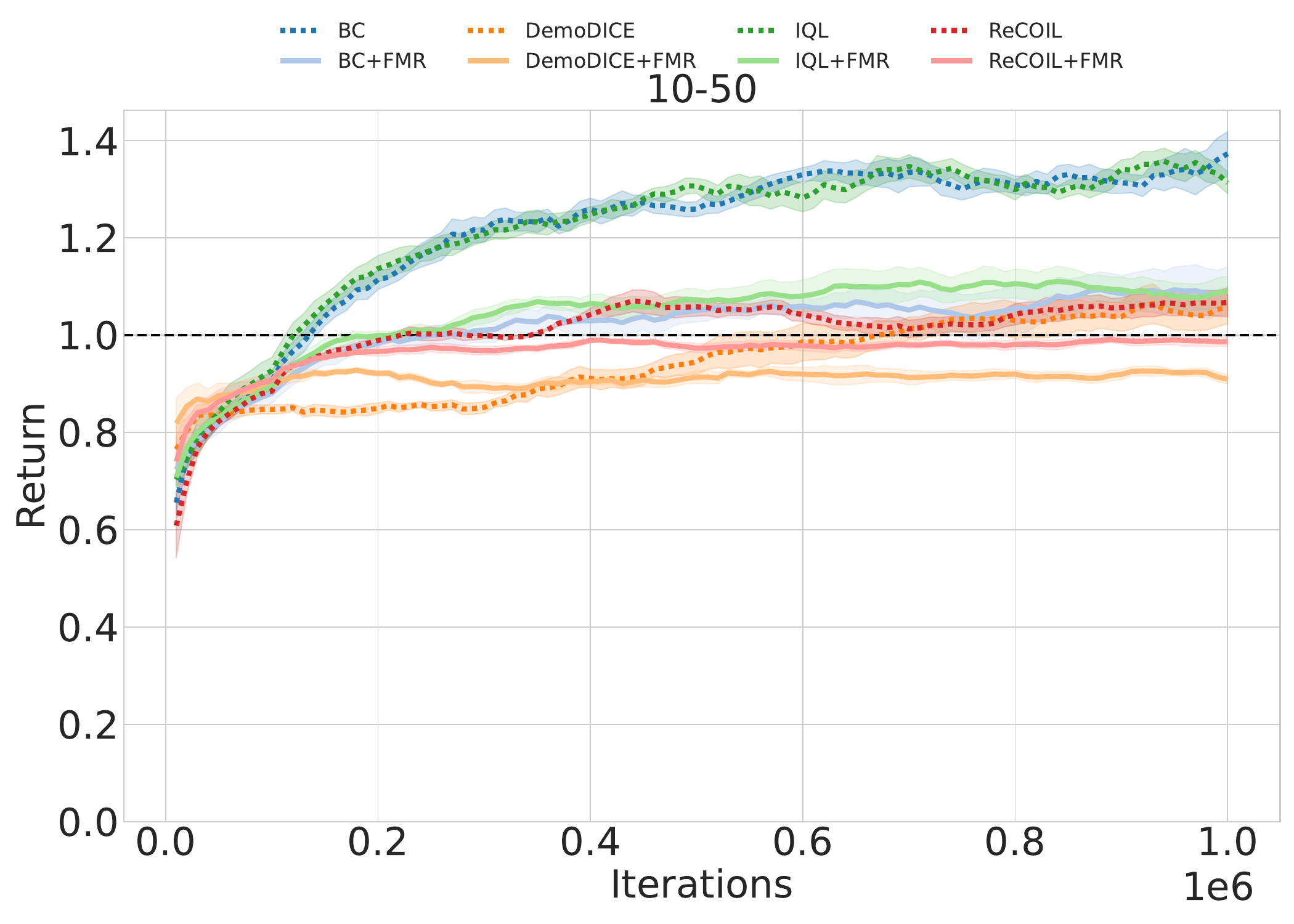}
    \end{subfigure}
    
    \begin{subfigure}[b]{0.32\textwidth}
        \centering
        \includegraphics[width=\textwidth,trim=0 0 0 81,clip]{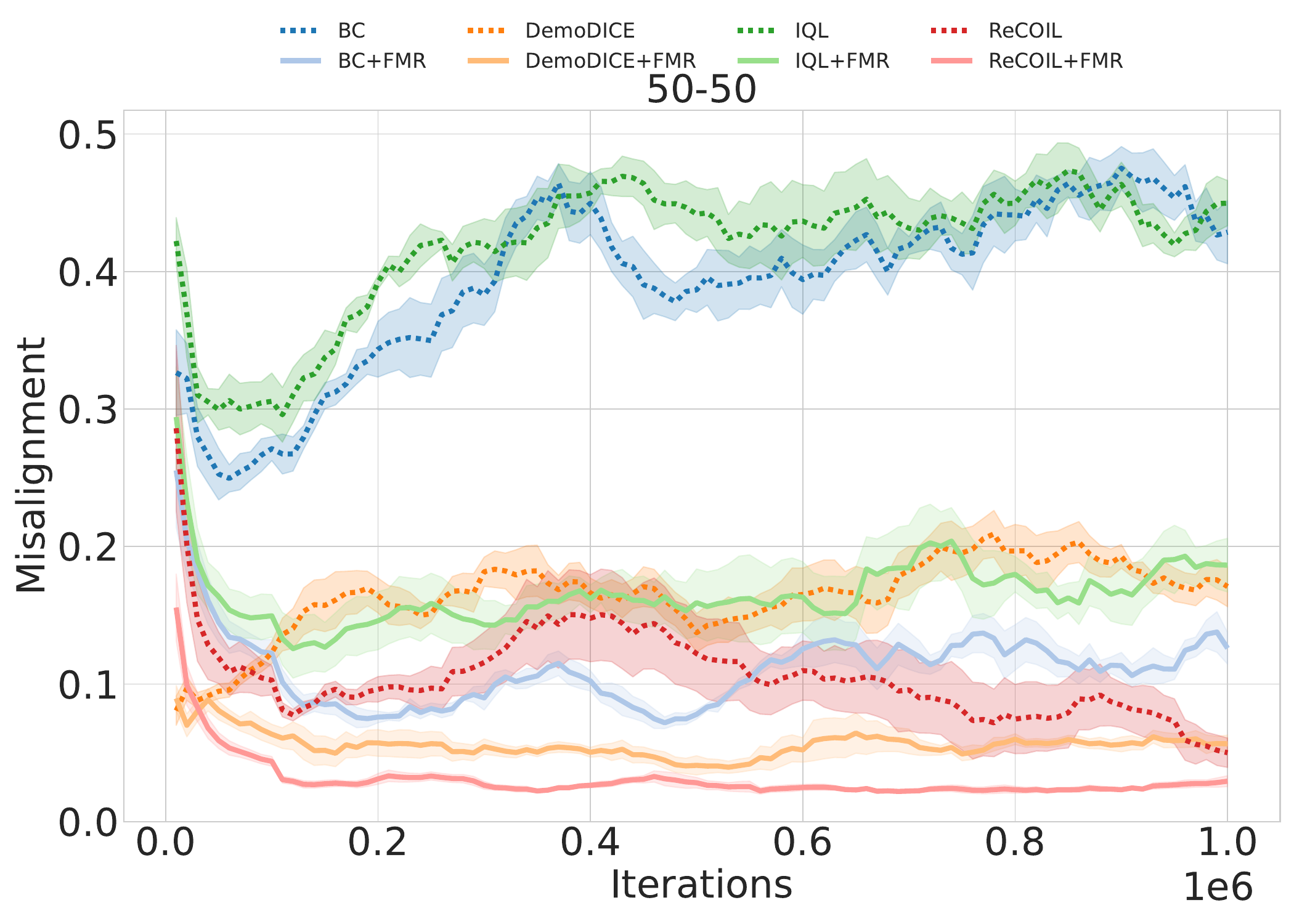}
    \end{subfigure}
    \hfill
    \begin{subfigure}[b]{0.32\textwidth}
        \centering
        \includegraphics[width=\textwidth,trim=0 0 0 81,clip]{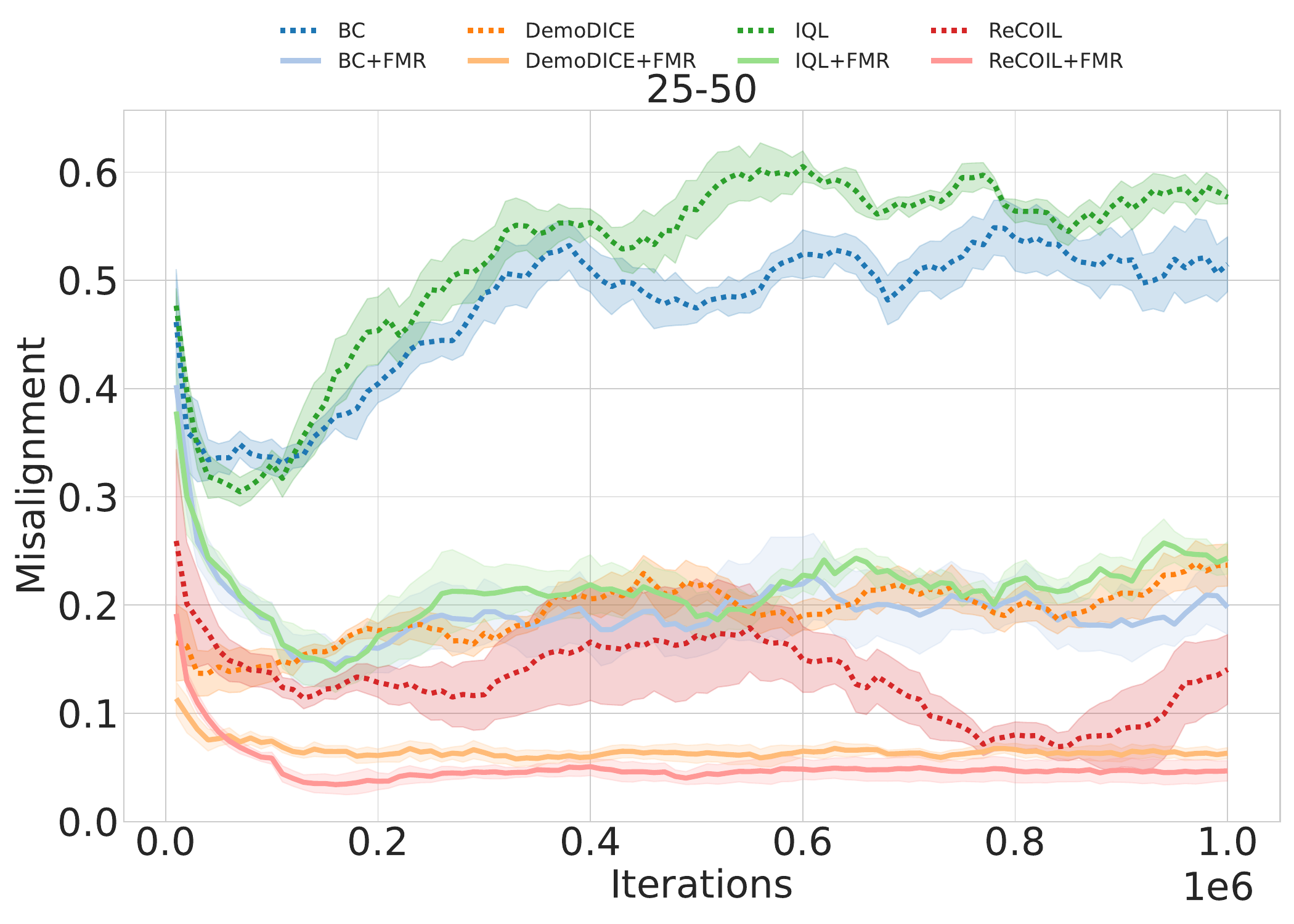}
    \end{subfigure}
    \hfill
    \begin{subfigure}[b]{0.32\textwidth}
        \centering
        \includegraphics[width=\textwidth,trim=0 0 0 81,clip]{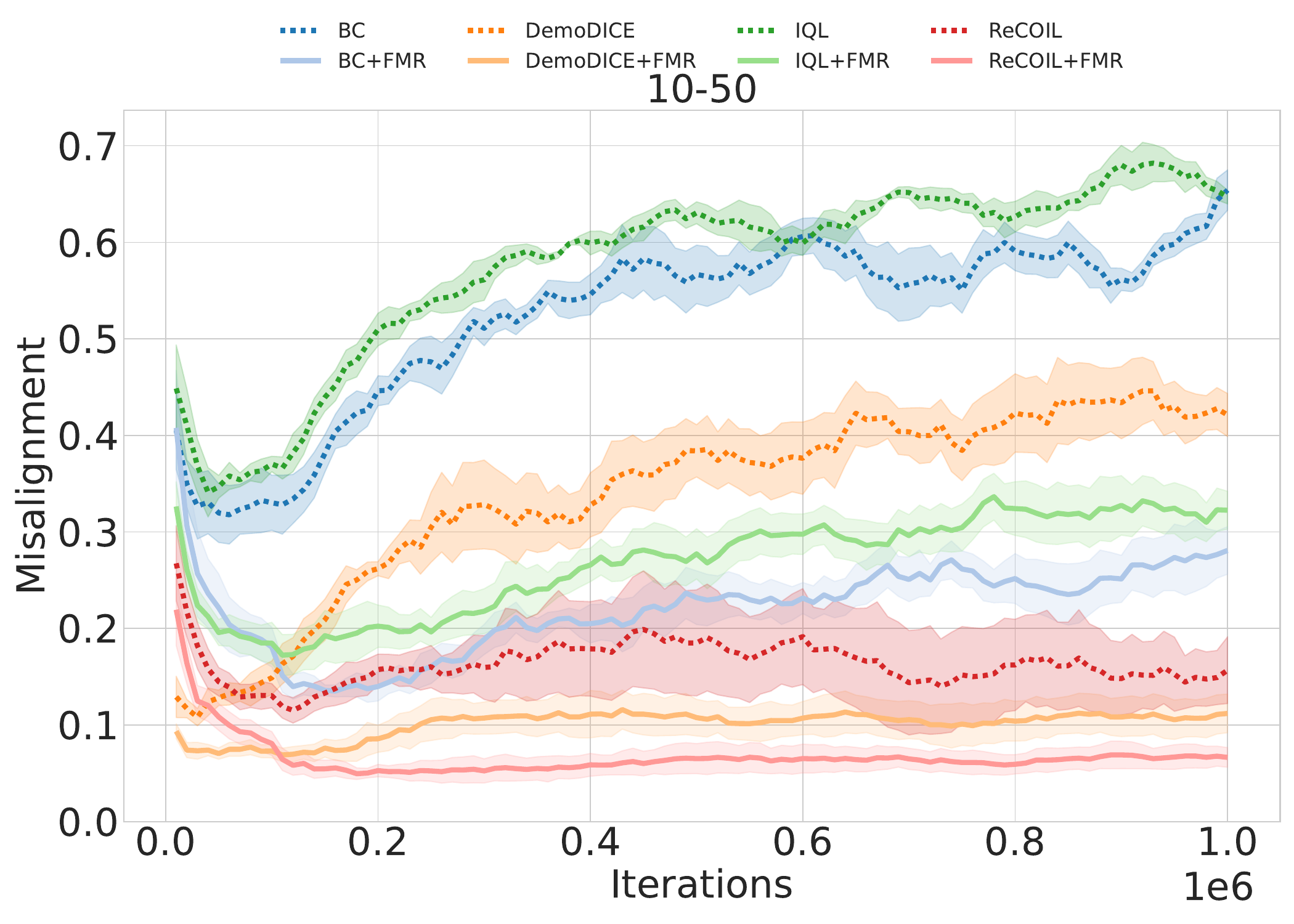}
    \end{subfigure}
    
    \caption{SlowHop learning curves for baselines and FMR. The shaded region represents the standard error.}
    \label{fig:hop-main}
\end{figure}

\begin{figure}[h]
    \centering
    
    \begin{subfigure}[b]{\textwidth}
        \centering
        \includegraphics[width=0.7\textwidth,trim=0 665 0 0,clip]{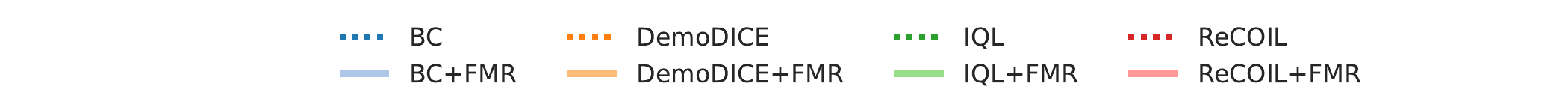}
    \end{subfigure}
    
    \vspace{0.5em}
    
    \begin{subfigure}[b]{0.32\textwidth}
        \centering
        \includegraphics[width=\textwidth,trim=0 44 0 55,clip]{images/main/SlowWalk/return_50-50.pdf}
    \end{subfigure}
    \hfill
    \begin{subfigure}[b]{0.32\textwidth}
        \centering
        \includegraphics[width=\textwidth,trim=0 44 0 55,clip]{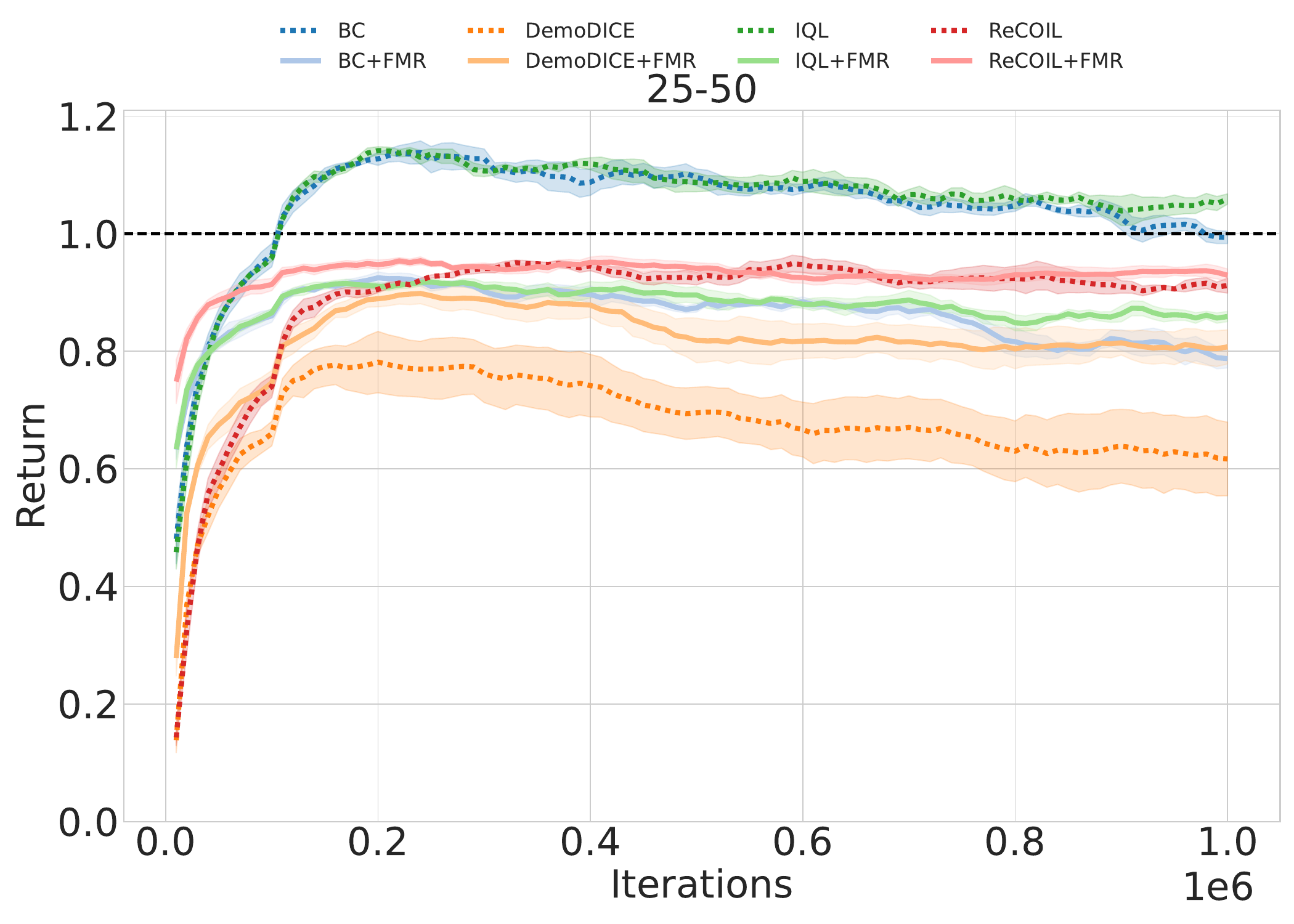}
    \end{subfigure}
    \hfill
    \begin{subfigure}[b]{0.32\textwidth}
        \centering
        \includegraphics[width=\textwidth,trim=0 44 0 55,clip]{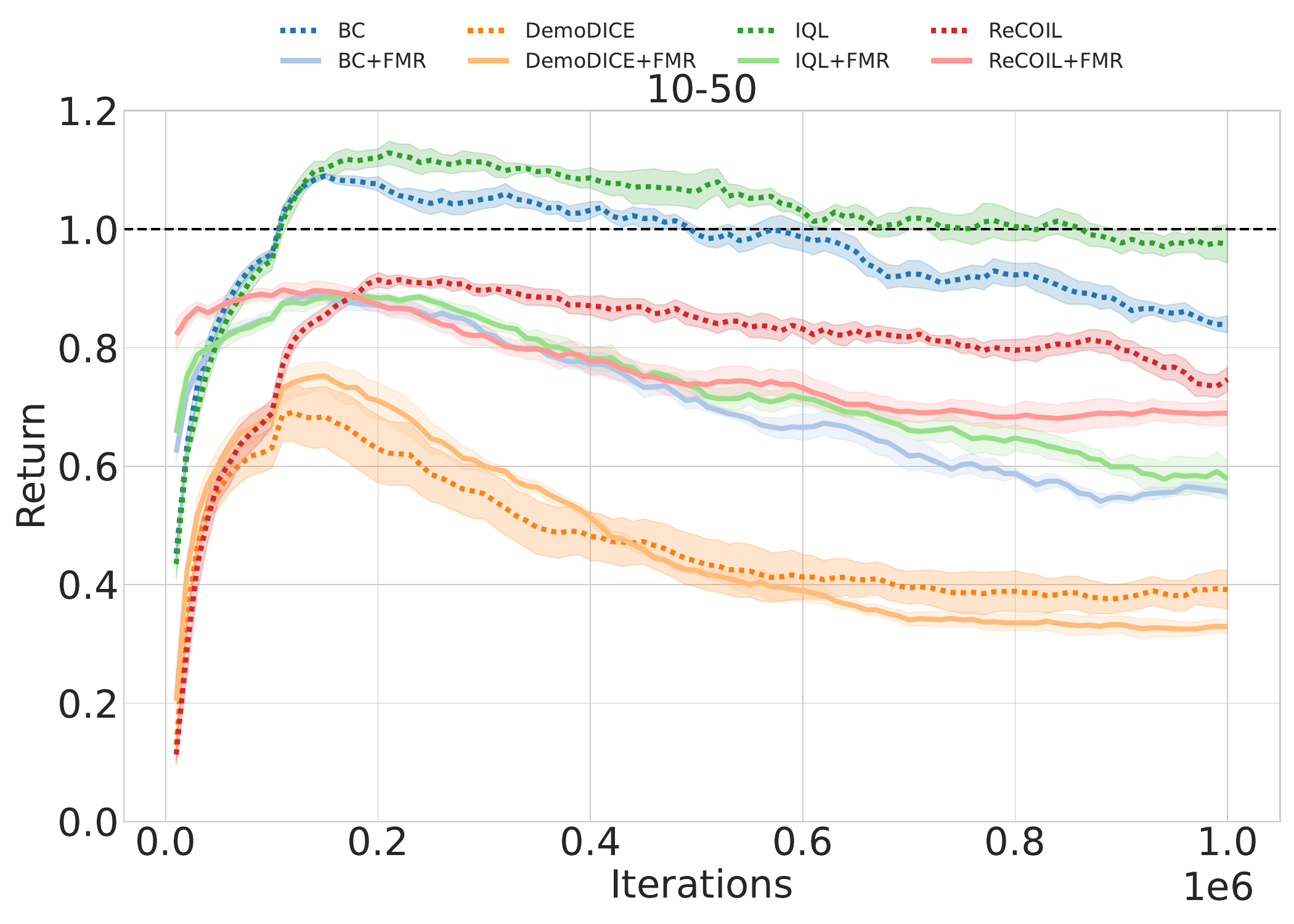}
    \end{subfigure}
    
    \begin{subfigure}[b]{0.32\textwidth}
        \centering
        \includegraphics[width=\textwidth,trim=0 0 0 81,clip]{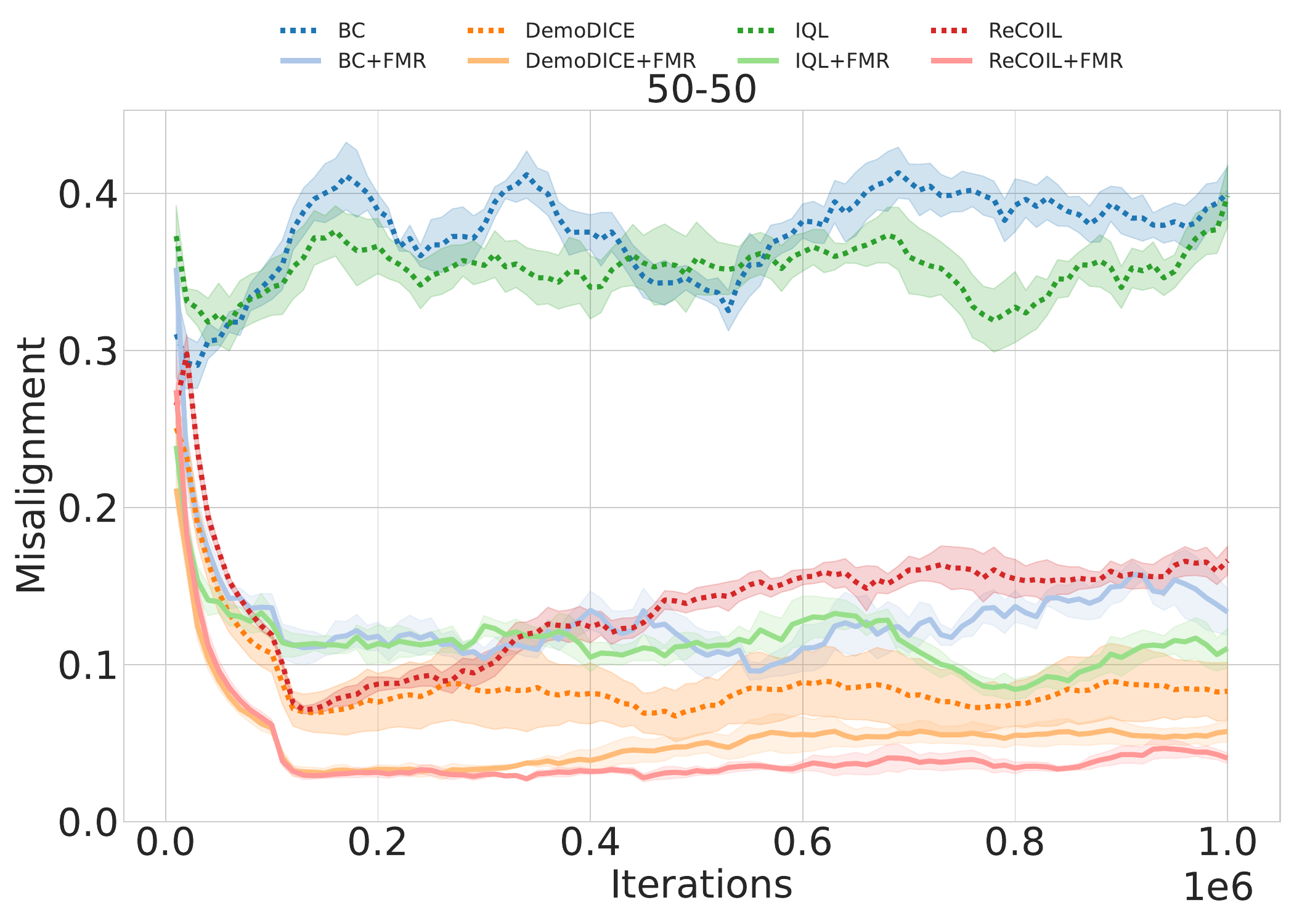}
    \end{subfigure}
    \hfill
    \begin{subfigure}[b]{0.32\textwidth}
        \centering
        \includegraphics[width=\textwidth,trim=0 0 0 81,clip]{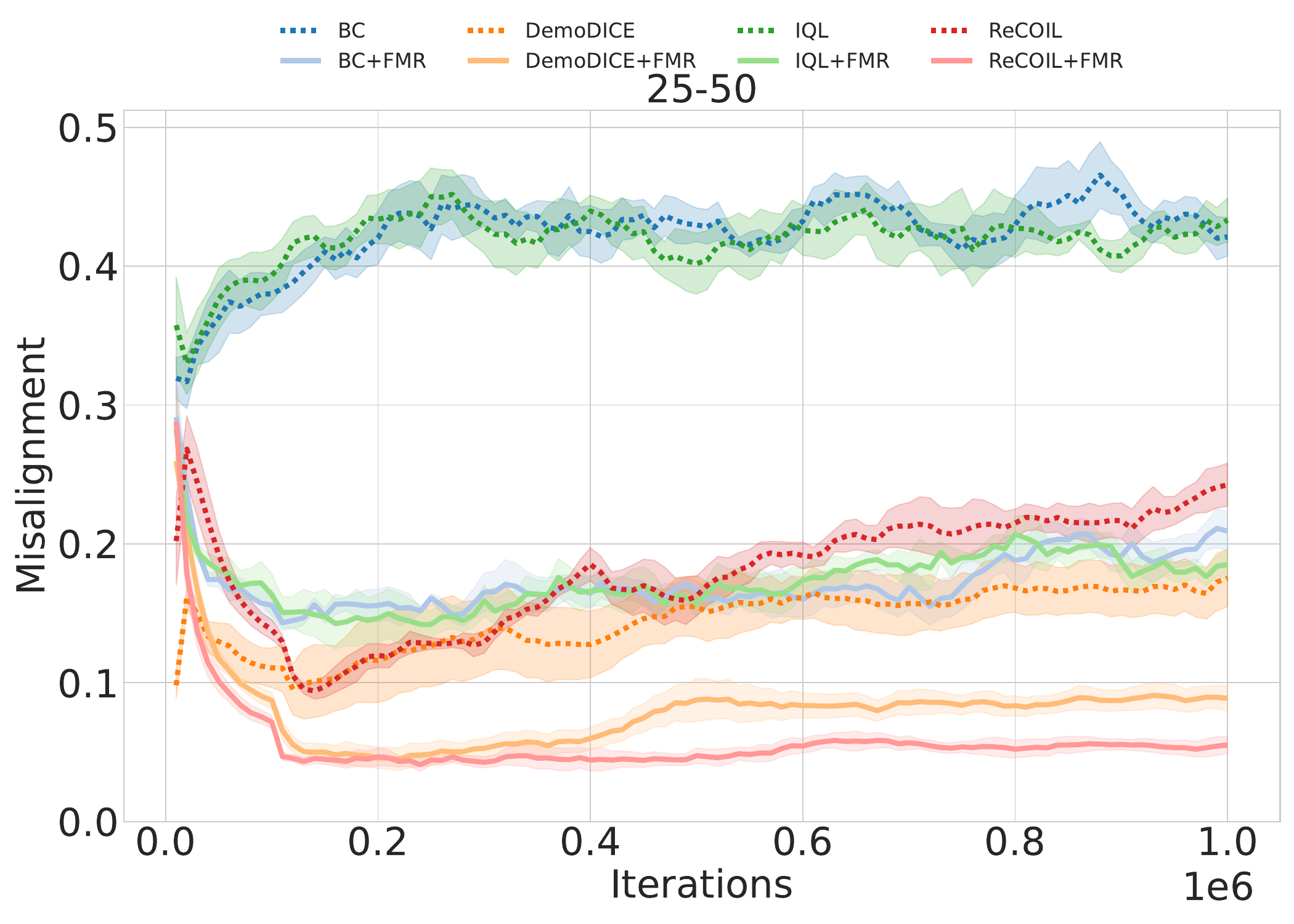}
    \end{subfigure}
    \hfill
    \begin{subfigure}[b]{0.32\textwidth}
        \centering
        \includegraphics[width=\textwidth,trim=0 0 0 81,clip]{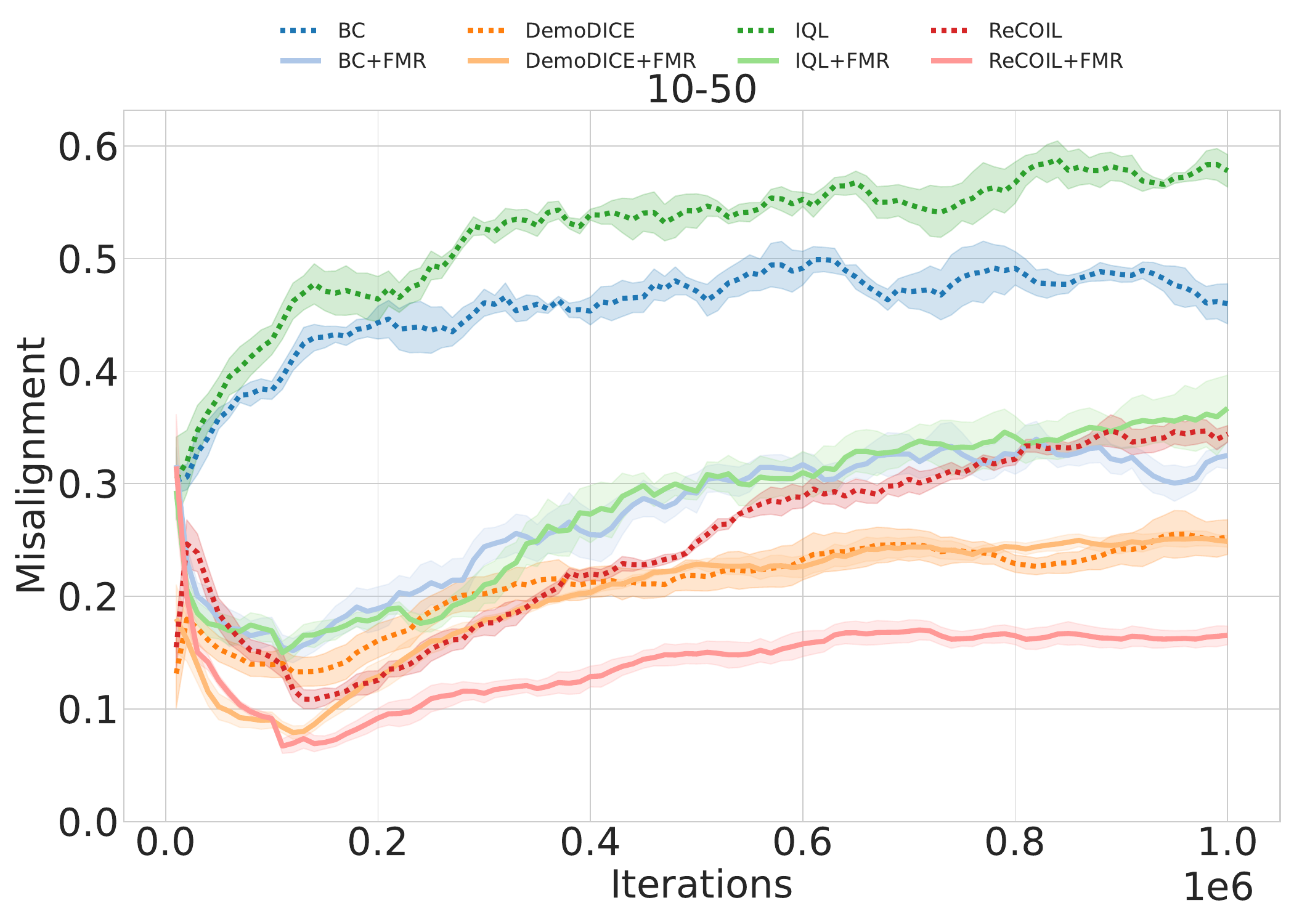}
    \end{subfigure}
    
    \caption{SlowWalk learning curves for baselines and FMR. The shaded region represents the standard error.}
    \label{fig:walker-main}
\end{figure}

%% file: appendix/cpl-dvl.tex
The following results compare FMR with CPL when using the sum feedback to score sub-trajectory preference and  DVL when using feedback to replace reward. Results for velocity tasks and additional learning curves are provided.

\begin{figure}[h]
    \centering
    
    \begin{subfigure}[b]{\textwidth}
        \centering
        \includegraphics[width=0.7\textwidth,trim=0 275 1 0,clip]{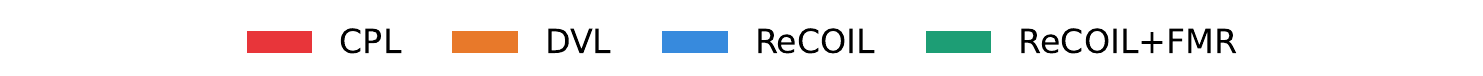}
    \end{subfigure}
    
    \begin{minipage}[t]{0.32\textwidth}
        \centering
        \textbf{Swimmer}
    \end{minipage}
    \hfill
    \begin{minipage}[t]{0.32\textwidth}
        \centering
        \textbf{Hopper}
    \end{minipage}
    \begin{minipage}[t]{0.32\textwidth}
        \centering
        \textbf{Walker2D}
    \end{minipage}
    
    \begin{subfigure}[b]{0.32\textwidth}
        \centering
        \includegraphics[width=\textwidth,trim=0 0 0 35,clip]{images/cpl-dvl/SlowSwim/return_compare.pdf}
    \end{subfigure}
    \hfill
    \begin{subfigure}[b]{0.32\textwidth}
        \centering
        \includegraphics[width=\textwidth,trim=0 0 0 35,clip]{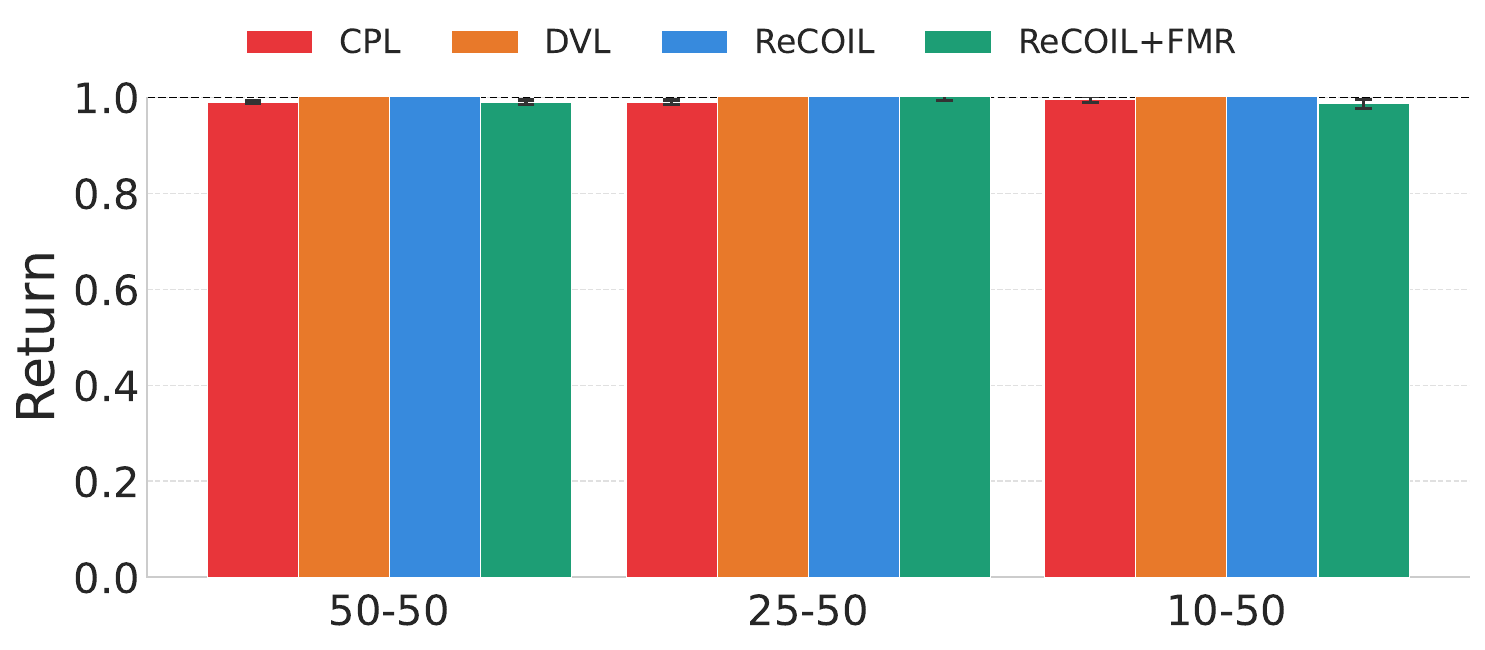}
    \end{subfigure}    
    \hfill
    \begin{subfigure}[b]{0.32\textwidth}
        \centering
        \includegraphics[width=\textwidth,trim=0 0 0 35,clip]{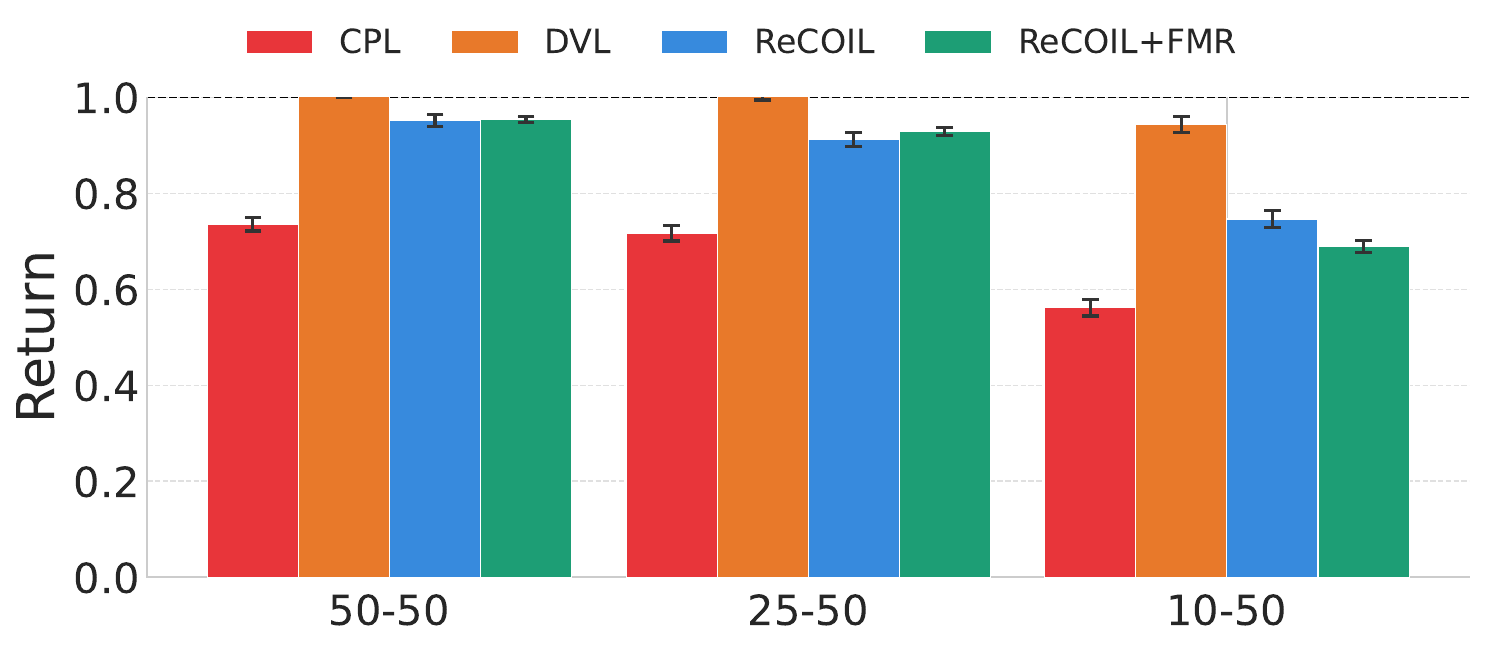}
    \end{subfigure}
    
    \begin{subfigure}[b]{0.32\textwidth}
        \centering
        \includegraphics[width=\textwidth,trim=0 0 0 35,clip]{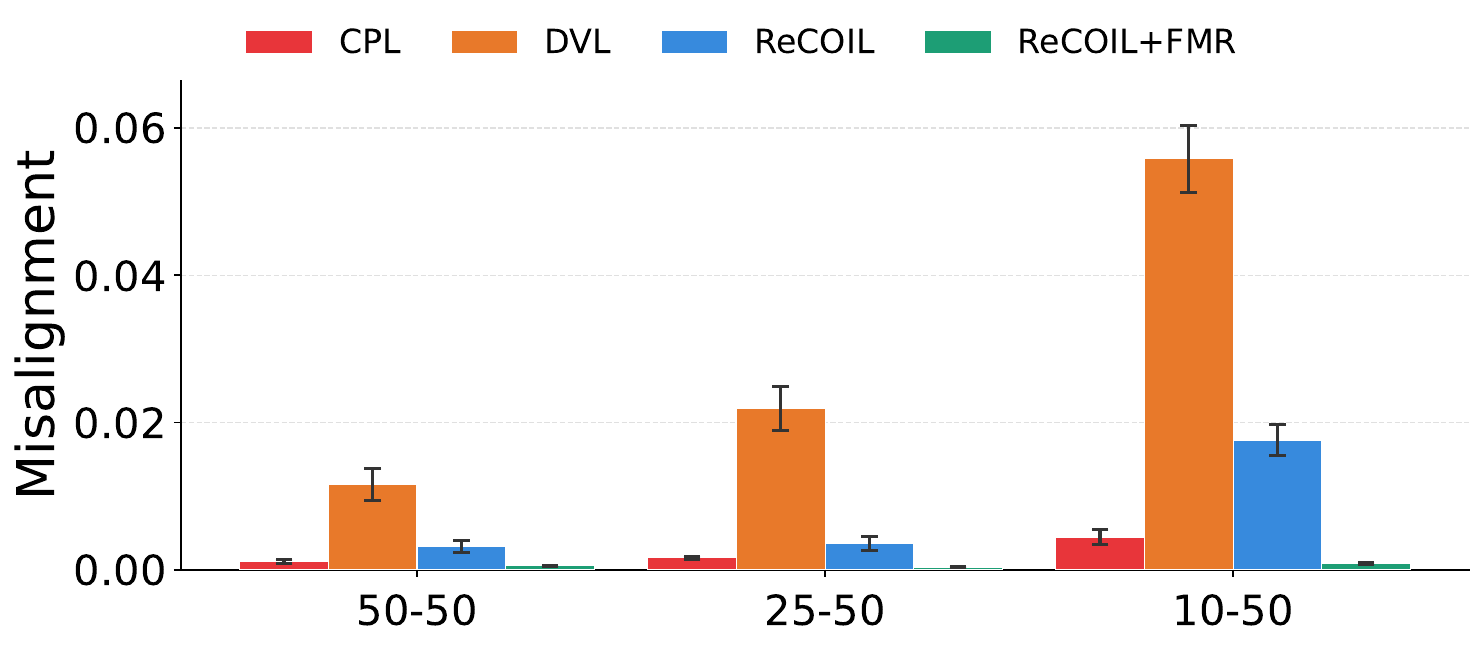}
    \end{subfigure}
    \hfill
    \begin{subfigure}[b]{0.32\textwidth}
        \centering
        \includegraphics[width=\textwidth,trim=0 0 0 35,clip]{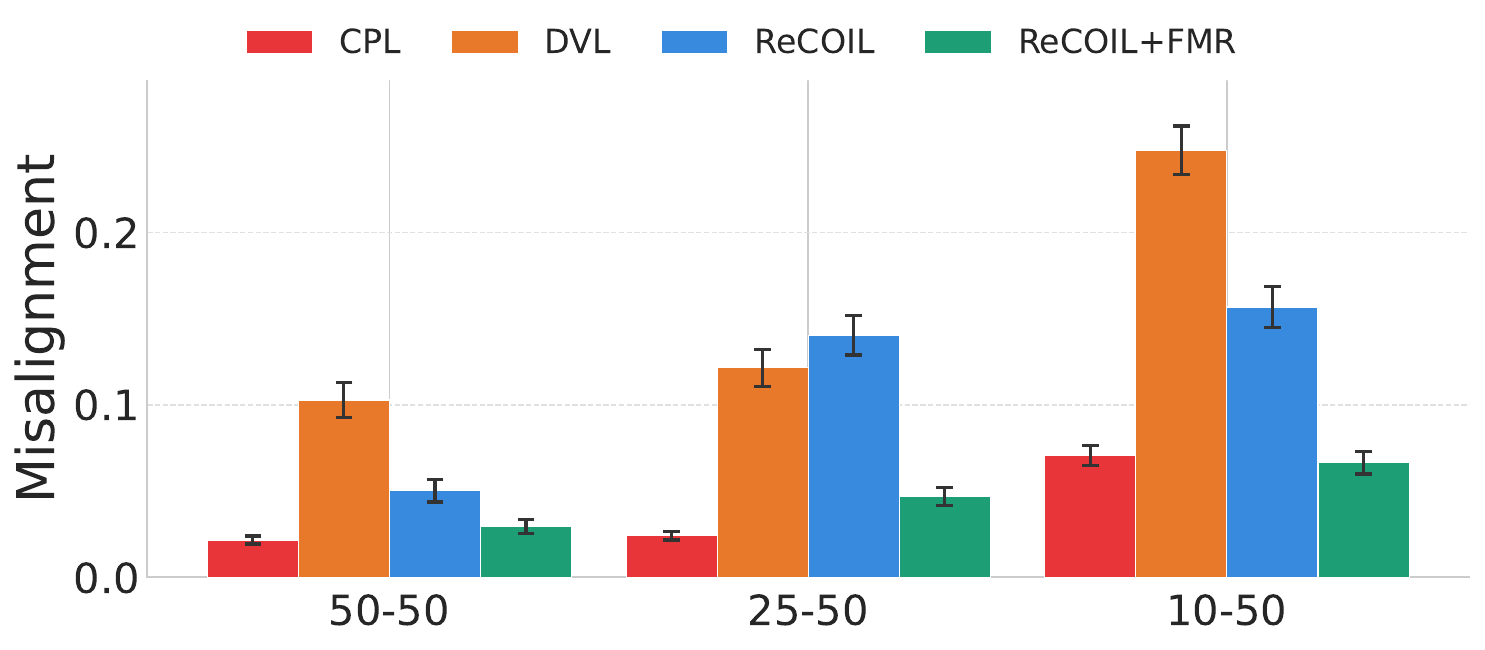}
    \end{subfigure}    
    \hfill
    \begin{subfigure}[b]{0.32\textwidth}
        \centering
        \includegraphics[width=\textwidth,trim=0 0 0 35,clip]{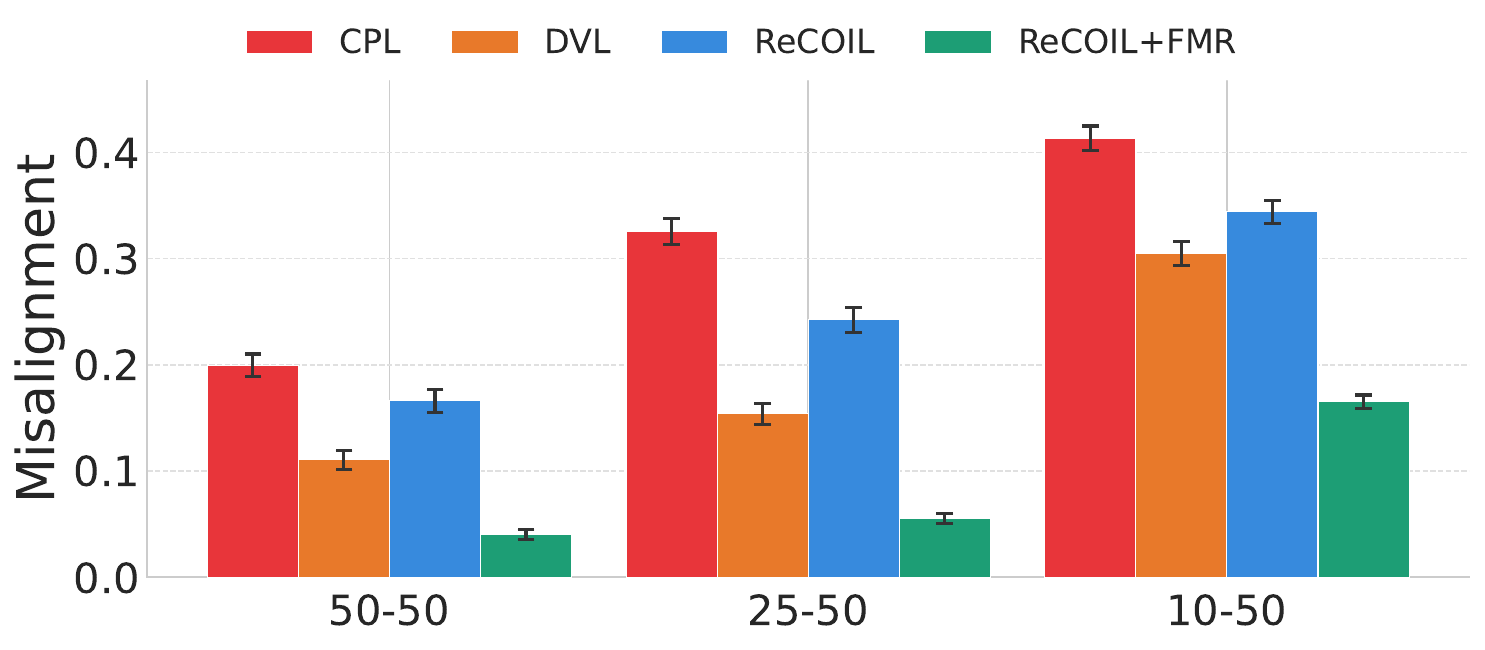}
    \end{subfigure}
    
    \caption{Velocity policies data ratio comparison between FMR, DVL, and CPL with ReCOIL baseline. Results show mean success rate and misalignment scores with 95\% confidence intervals, averaged over the last 10 evaluations across 5 seeds.}
    \label{fig:vel-dvl-cpl-results}
\end{figure}

\begin{table}[h]
\centering
\caption{PathM comparison of FMR against alternative methods for utilizing evaluative feedback. Results show mean ± std for the last 10 evaluations, over 5 seeds.}
\label{tab:pathm-dvl}
\small
\begin{tabular}{llrr}
\toprule
Algorithm & Ratio & Suc. & Mis. \\
\midrule
\multirow[t]{3}{*}{ReCOIL} 
& 50-50 & 0.951 ± 0.215 & 0.113 ± 0.208 \\
& 25-50 & 0.916 ± 0.278 & 0.123 ± 0.222 \\
& 10-50 & 0.855 ± 0.352 & 0.150 ± 0.239 \\
\cline{1-4}
\multirow[t]{3}{*}{CPL} 
& 50-50 & 0.872 ± 0.334 & 0.098 ± 0.237 \\
& 25-50 & 0.796 ± 0.403 & 0.160 ± 0.282 \\
& 10-50 & 0.565 ± 0.496 & 0.302 ± 0.328 \\
\cline{1-4}
\multirow[t]{3}{*}{DVL} 
& 50-50 & 0.773 ± 0.419 & 0.148 ± 0.260 \\
& 25-50 & 0.655 ± 0.476 & 0.230 ± 0.291 \\
& 10-50 & 0.520 ± 0.500 & 0.308 ± 0.302 \\
\cline{1-4}
\multirow[t]{3}{*}{ReCOIL+FMR} 
& 50-50 & 0.988 ± 0.107 & 0.015 ± 0.084 \\
& 25-50 & 0.974 ± 0.160 & 0.014 ± 0.076 \\
& 10-50 & 0.947 ± 0.224 & 0.020 ± 0.092 \\
\bottomrule
\end{tabular}
\end{table}

\begin{table}[h]
\centering
\caption{PathBB comparison of FMR against alternative methods for utilizing evaluative feedback. Results show mean ± std for the last 10 evaluations, over 5 seeds.}
\label{tab:pathbb-dvl}
\small
\begin{tabular}{llrr}
\toprule
Algorithm & Ratio & Suc. & Mis. \\
\midrule
\multirow[t]{3}{*}{ReCOIL} 
& 10-50 & 0.592 ± 0.492 & 0.237 ± 0.334 \\
& 25-50 & 0.721 ± 0.449 & 0.154 ± 0.296 \\
& 50-50 & 0.780 ± 0.415 & 0.153 ± 0.305 \\
\cline{1-4}
\multirow[t]{3}{*}{CPL} 
& 10-50 & 0.573 ± 0.495 & 0.273 ± 0.359 \\
& 25-50 & 0.685 ± 0.465 & 0.203 ± 0.341 \\
& 50-50 & 0.781 ± 0.414 & 0.124 ± 0.267 \\
\cline{1-4}
\multirow[t]{3}{*}{DVL}
& 10-50 & 0.400 ± 0.490 & 0.287 ± 0.286 \\
& 25-50 & 0.652 ± 0.476 & 0.175 ± 0.274 \\
& 50-50 & 0.771 ± 0.420 & 0.108 ± 0.222 \\
\cline{1-4}
\multirow[t]{3}{*}{ReCOIL+FMR} 
& 10-50 & 0.708 ± 0.455 & 0.097 ± 0.211 \\
& 25-50 & 0.880 ± 0.325 & 0.073 ± 0.192 \\
& 50-50 & 0.919 ± 0.273 & 0.048 ± 0.167 \\
\bottomrule
\end{tabular}
\end{table}

\begin{table}[h]
\centering
\caption{SlowSwim comparison of FMR against alternative methods for utilizing evaluative feedback. Results show mean ± std for the last 10 evaluations, over 5 seeds.}
\label{tab:swim-dvl}
\begin{tabular}{ll|rr}
\toprule
Algorithm & Ratio & Return & Mis. \\
\midrule
\multirow[t]{3}{*}{ReCOIL}
& 50-50 & 0.994 ± 0.068 & 0.003 ± 0.022 \\
& 25-50 & 0.992 ± 0.076 & 0.004 ± 0.024 \\
& 10-50 & 0.967 ± 0.145 & 0.018 ± 0.055 \\
\cline{1-4}
\multirow[t]{3}{*}{CPL} 
& 50-50 & 1.001 ± 0.026 & 0.001 ± 0.007 \\
& 25-50 & 1.002 ± 0.010 & 0.002 ± 0.005 \\
& 10-50 & 1.000 ± 0.023 & 0.004 ± 0.026 \\
\cline{1-4}
\multirow[t]{3}{*}{DVL}
& 50-50 & 0.997 ± 0.048 & 0.012 ± 0.055 \\
& 25-50 & 0.998 ± 0.055 & 0.022 ± 0.077 \\
& 10-50 & 0.993 ± 0.083 & 0.056 ± 0.116 \\
\cline{1-4}
\multirow[t]{3}{*}{ReCOIL+FMR} 
& 50-50 & 1.001 ± 0.008 & 0.001 ± 0.003 \\
& 25-50 & 1.002 ± 0.008 & 0.000 ± 0.002 \\
& 10-50 & 1.002 ± 0.008 & 0.001 ± 0.003 \\
\bottomrule
\end{tabular}
\end{table}

\begin{table}[h]
\centering
\caption{SlowHop comparison of FMR against alternative methods for utilizing evaluative feedback. Results show mean ± std for the last 10 evaluations, over 5 seeds.}
\label{tab:hop-dvl}
\begin{tabular}{llrr}
\toprule
Algorithm & Ratio & Return & Mis. \\
\midrule
\multirow[t]{3}{*}{CPL} 
& 50-50 & 0.990 ± 0.089 & 0.022 ± 0.059 \\
& 25-50 & 0.990 ± 0.111 & 0.024 ± 0.059 \\
& 10-50 & 0.995 ± 0.173 & 0.071 ± 0.152 \\
\cline{1-4}
\multirow[t]{3}{*}{DVL} 
& 50-50 & 1.046 ± 0.273 & 0.103 ± 0.262 \\
& 25-50 & 1.076 ± 0.302 & 0.121 ± 0.275 \\
& 10-50 & 1.135 ± 0.418 & 0.248 ± 0.357 \\
\cline{1-4}
\multirow[t]{3}{*}{ReCOIL} 
& 50-50 & 1.012 ± 0.181 & 0.050 ± 0.164 \\
& 25-50 & 1.066 ± 0.310 & 0.140 ± 0.296 \\
& 10-50 & 1.067 ± 0.340 & 0.157 ± 0.304 \\
\cline{1-4}
\multirow[t]{3}{*}{ReCOIL+FMR} 
& 50-50 & 0.990 ± 0.127 & 0.029 ± 0.099 \\
& 25-50 & 1.001 ± 0.179 & 0.047 ± 0.135 \\
& 10-50 & 0.986 ± 0.225 & 0.066 ± 0.168 \\
\bottomrule
\end{tabular}
\end{table}

\begin{table}[h]
\centering
\caption{SlowWalk comparison of FMR against alternative methods for utilizing evaluative feedback. Results show mean ± std for the last 10 evaluations, over 5 seeds.}
\label{tab:walk-dvl}
\begin{tabular}{llrr}
\toprule
Algorithm & Ratio & Return & Mis. \\
\midrule
\multirow[t]{3}{*}{ReCOIL} 
& 50-50 & 0.951 ± 0.317 & 0.166 ± 0.273 \\
& 25-50 & 0.912 ± 0.379 & 0.243 ± 0.297 \\
& 10-50 & 0.746 ± 0.460 & 0.344 ± 0.275 \\
\cline{1-4}
\multirow[t]{3}{*}{CPL} &
50-50 & 0.735 ± 0.352 & 0.200 ± 0.268 \\
& 25-50 & 0.717 ± 0.411 & 0.326 ± 0.313 \\
& 10-50 & 0.562 ± 0.449 & 0.414 ± 0.294 \\
\cline{1-4}
\multirow[t]{3}{*}{DVL} 
& 50-50 & 1.009 ± 0.221 & 0.111 ± 0.225 \\
& 25-50 & 1.005 ± 0.267 & 0.154 ± 0.248 \\
& 10-50 & 0.943 ± 0.436 & 0.305 ± 0.282 \\
\cline{1-4}
\multirow[t]{3}{*}{ReCOIL+FMR} 
& 50-50 & 0.954 ± 0.169 & 0.040 ± 0.119 \\
& 25-50 & 0.929 ± 0.206 & 0.055 ± 0.122 \\
& 10-50 & 0.690 ± 0.323 & 0.165 ± 0.162 \\
\bottomrule
\end{tabular}
\end{table}

\begin{figure}[h]
    \centering
    
    \begin{subfigure}[b]{\textwidth}
        \centering
        \includegraphics[width=0.8\textwidth,trim=0 665 0 0,clip]{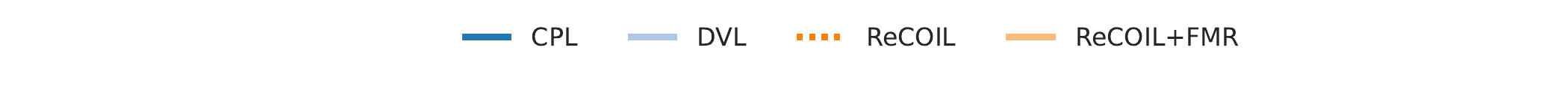}
    \end{subfigure}
    
    \vspace{0.5em}
    
    \begin{subfigure}[b]{0.32\textwidth}
        \centering
        \includegraphics[width=\textwidth,trim=0 44 0 55,clip]{images/cpl-dvl/PathM/success_50-50.pdf}
    \end{subfigure}
    \hfill
    \begin{subfigure}[b]{0.32\textwidth}
        \centering
        \includegraphics[width=\textwidth,trim=0 44 0 55,clip]{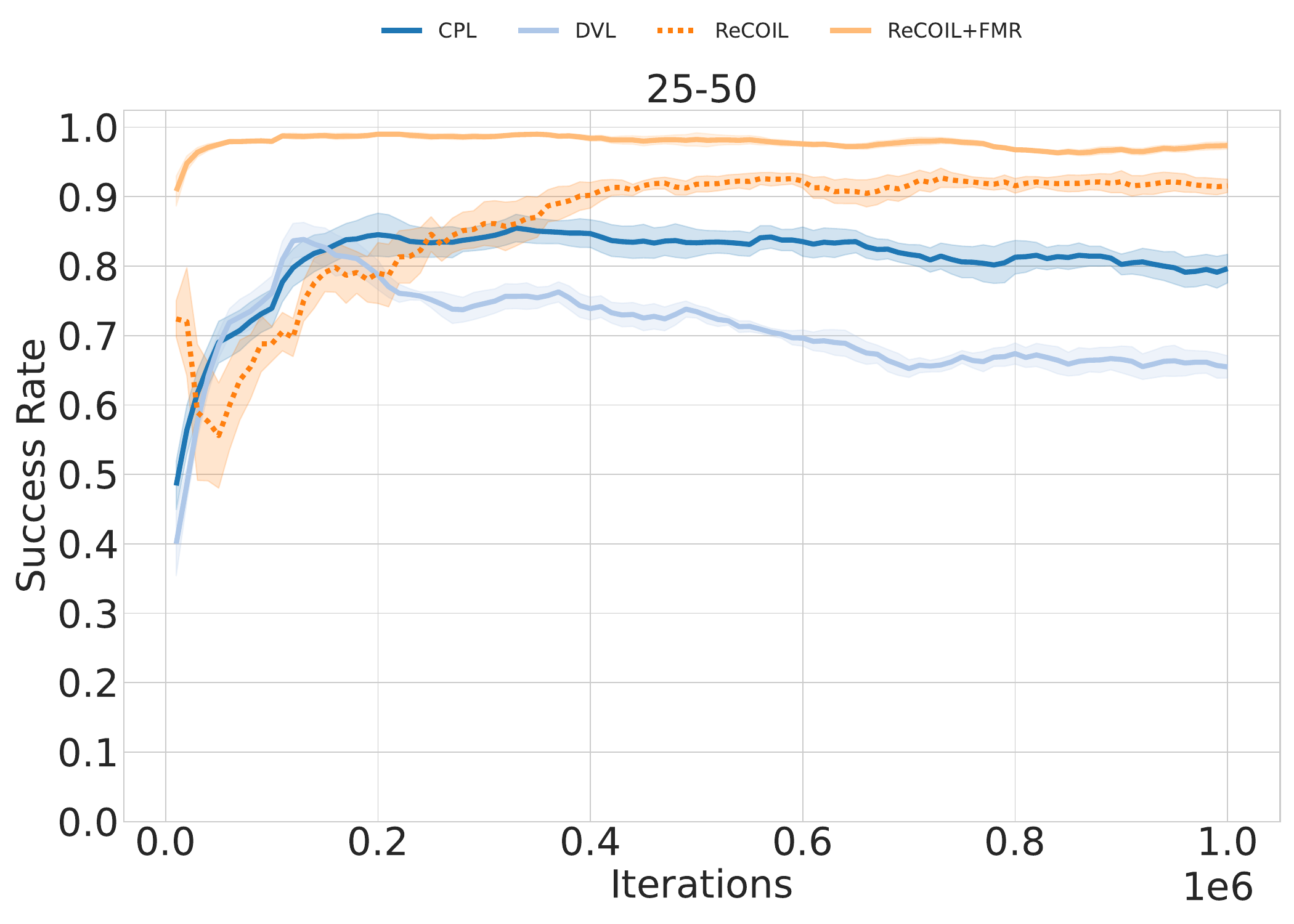}
    \end{subfigure}
    \hfill
    \begin{subfigure}[b]{0.32\textwidth}
        \centering
        \includegraphics[width=\textwidth,trim=0 44 0 55,clip]{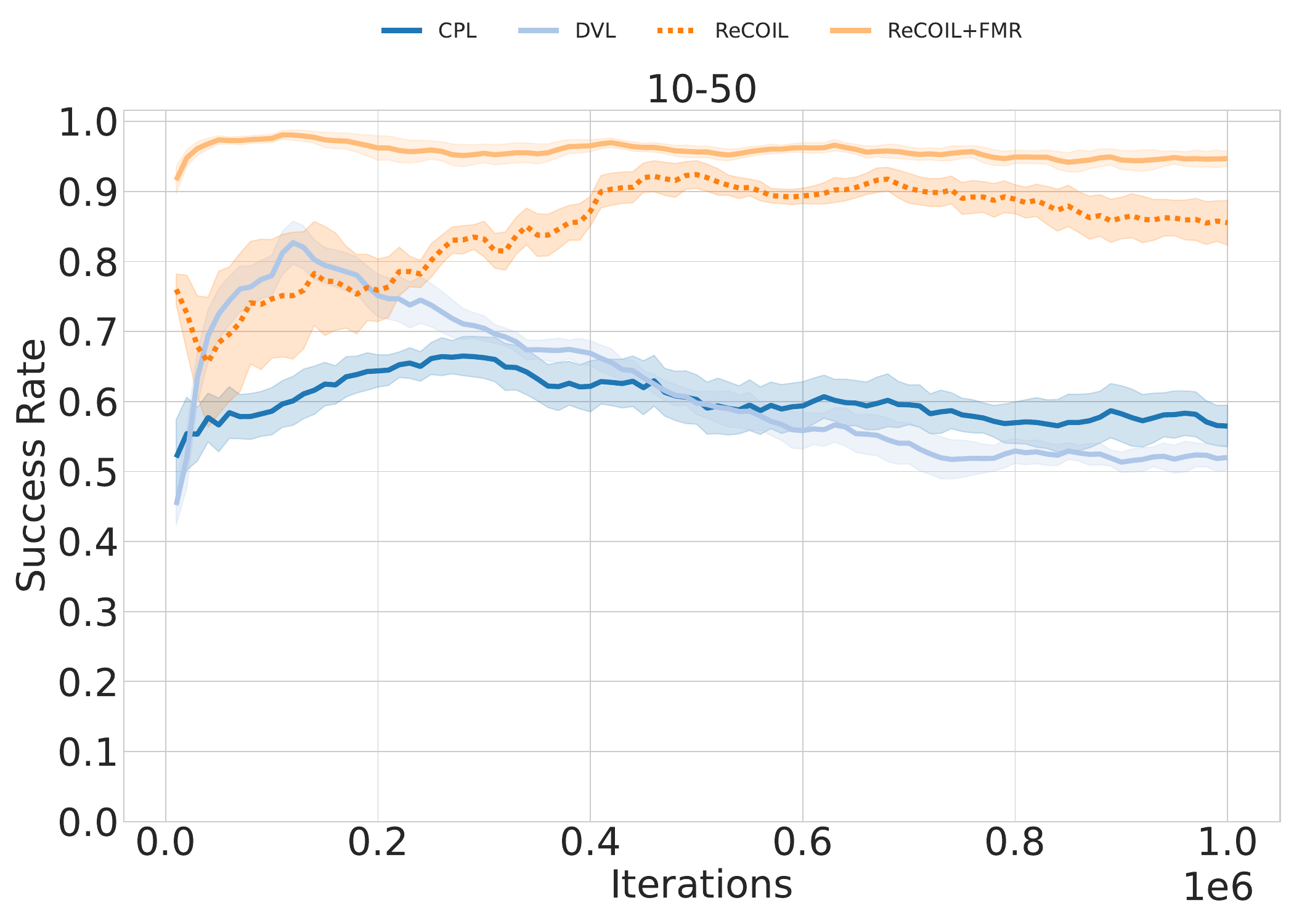}
    \end{subfigure}
    
    \begin{subfigure}[b]{0.32\textwidth}
        \centering
        \includegraphics[width=\textwidth,trim=0 0 0 81,clip]{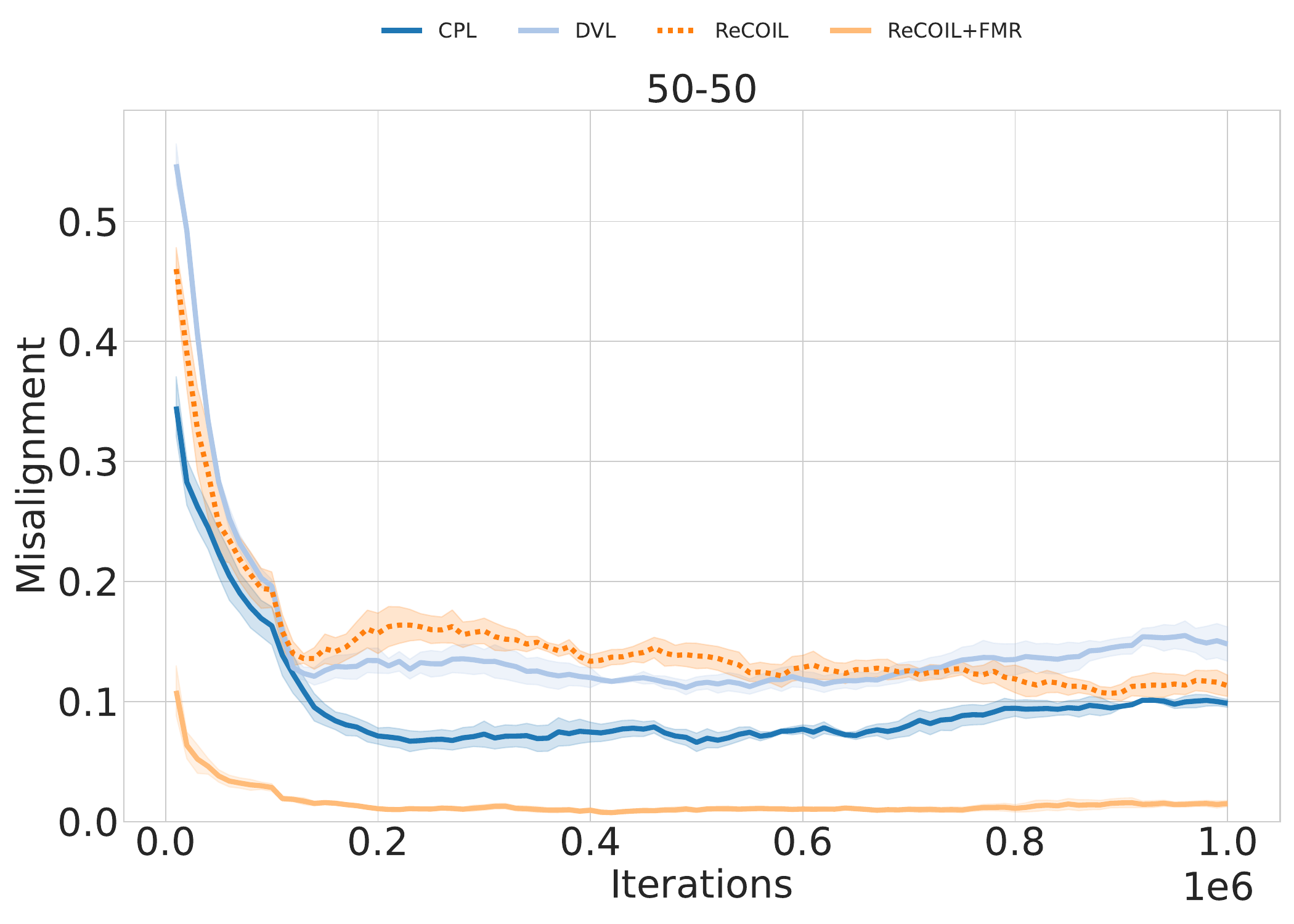}
    \end{subfigure}
    \hfill
    \begin{subfigure}[b]{0.32\textwidth}
        \centering
        \includegraphics[width=\textwidth,trim=0 0 0 81,clip]{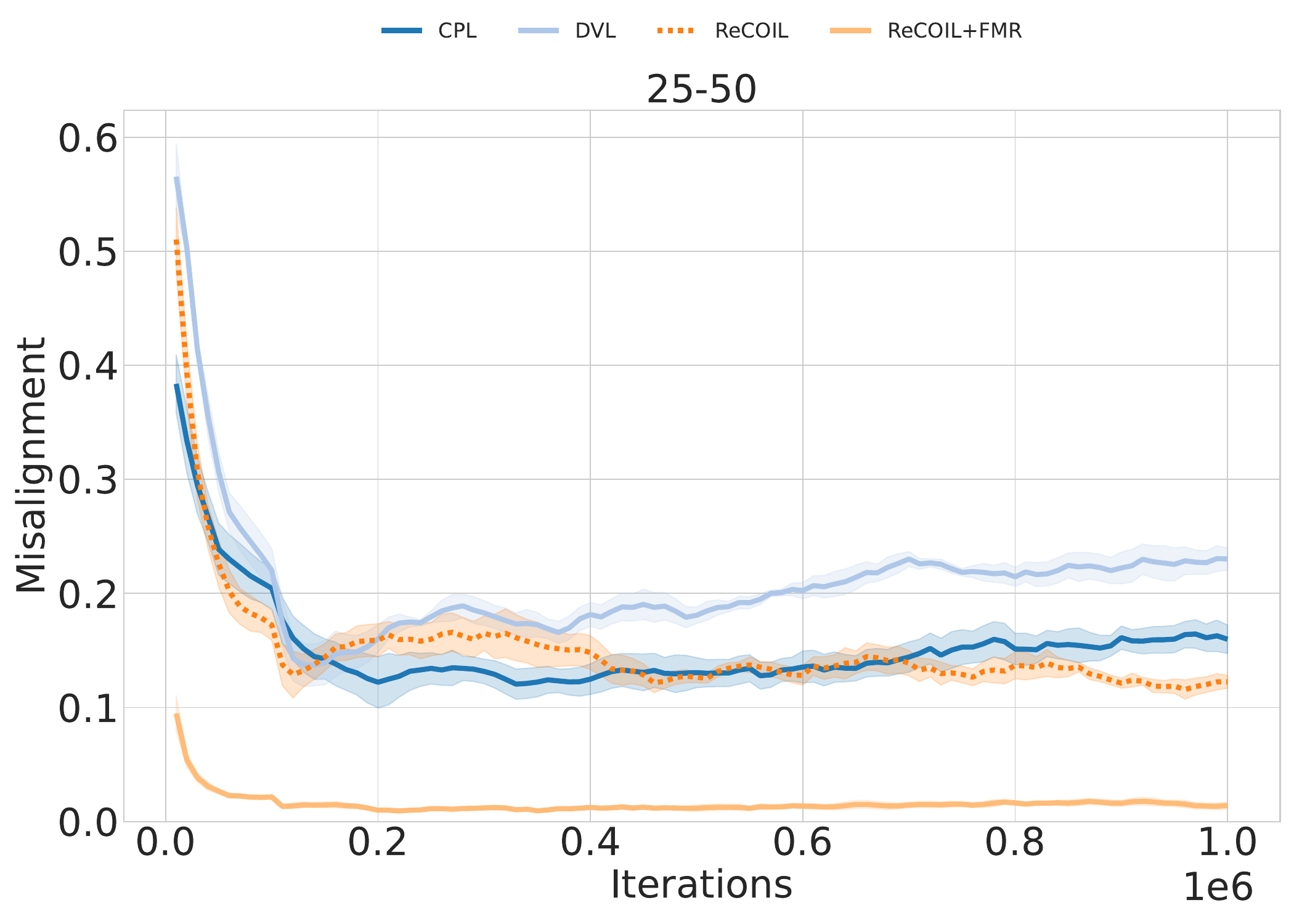}
    \end{subfigure}
    \hfill
    \begin{subfigure}[b]{0.32\textwidth}
        \centering
        \includegraphics[width=\textwidth,trim=0 0 0 81,clip]{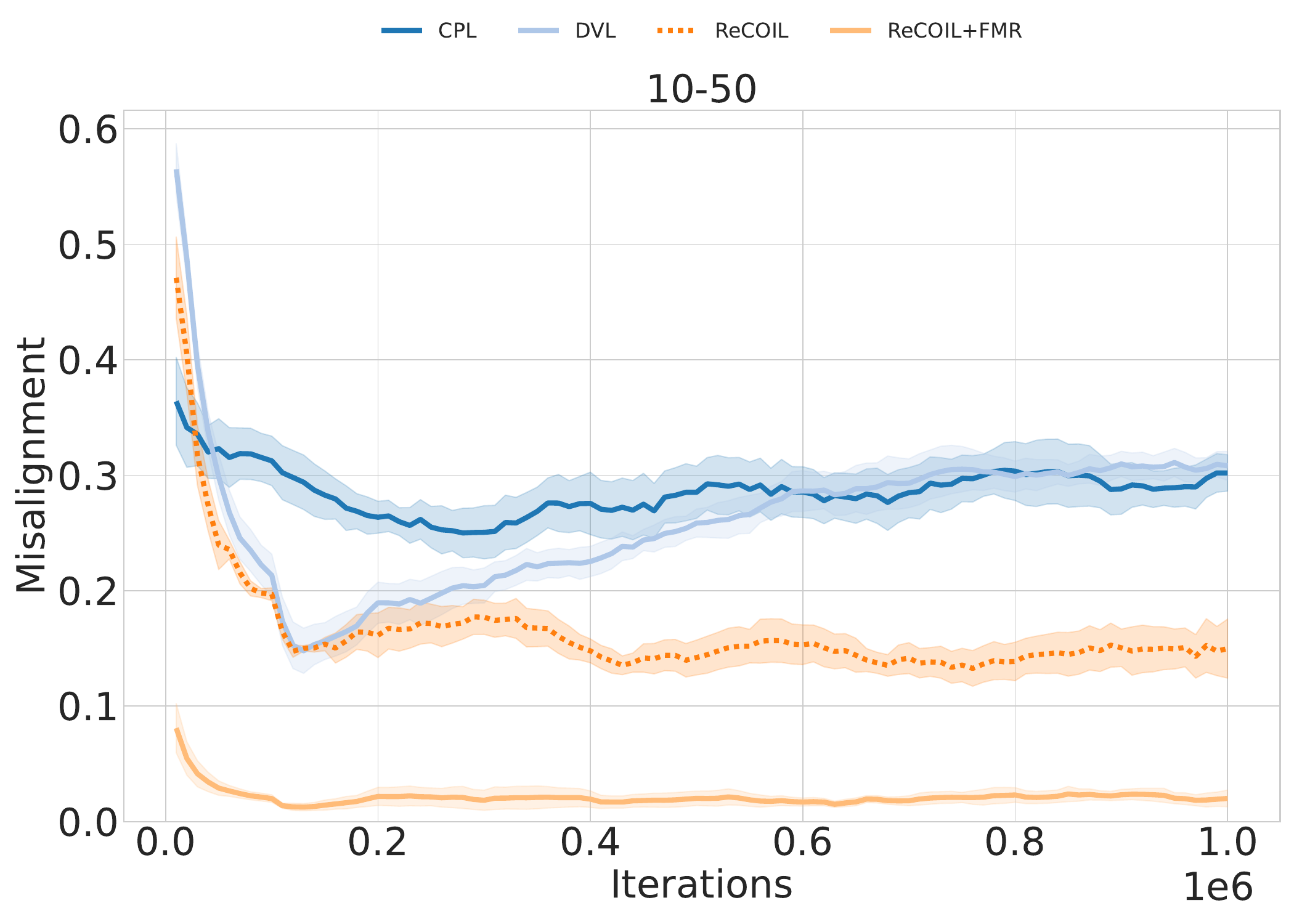}
    \end{subfigure}
    
    \caption{PathM learning curves comparing FMR to alternative methods for utilizing evaluative feedback. The shaded region represents the standard error.}
    \label{fig:pathm-cpl-dvl}
\end{figure}

\begin{figure}[h]
    \centering
    
    \begin{subfigure}[b]{\textwidth}
        \centering
        \includegraphics[width=0.8\textwidth,trim=0 665 0 0,clip]{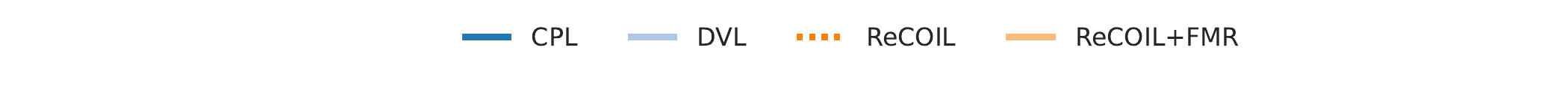}
    \end{subfigure}
    
    \vspace{0.5em}
    
    \begin{subfigure}[b]{0.32\textwidth}
        \centering
        \includegraphics[width=\textwidth,trim=0 44 0 55,clip]{images/cpl-dvl/PathBB/success_50-50.pdf}
    \end{subfigure}
    \hfill
    \begin{subfigure}[b]{0.32\textwidth}
        \centering
        \includegraphics[width=\textwidth,trim=0 44 0 55,clip]{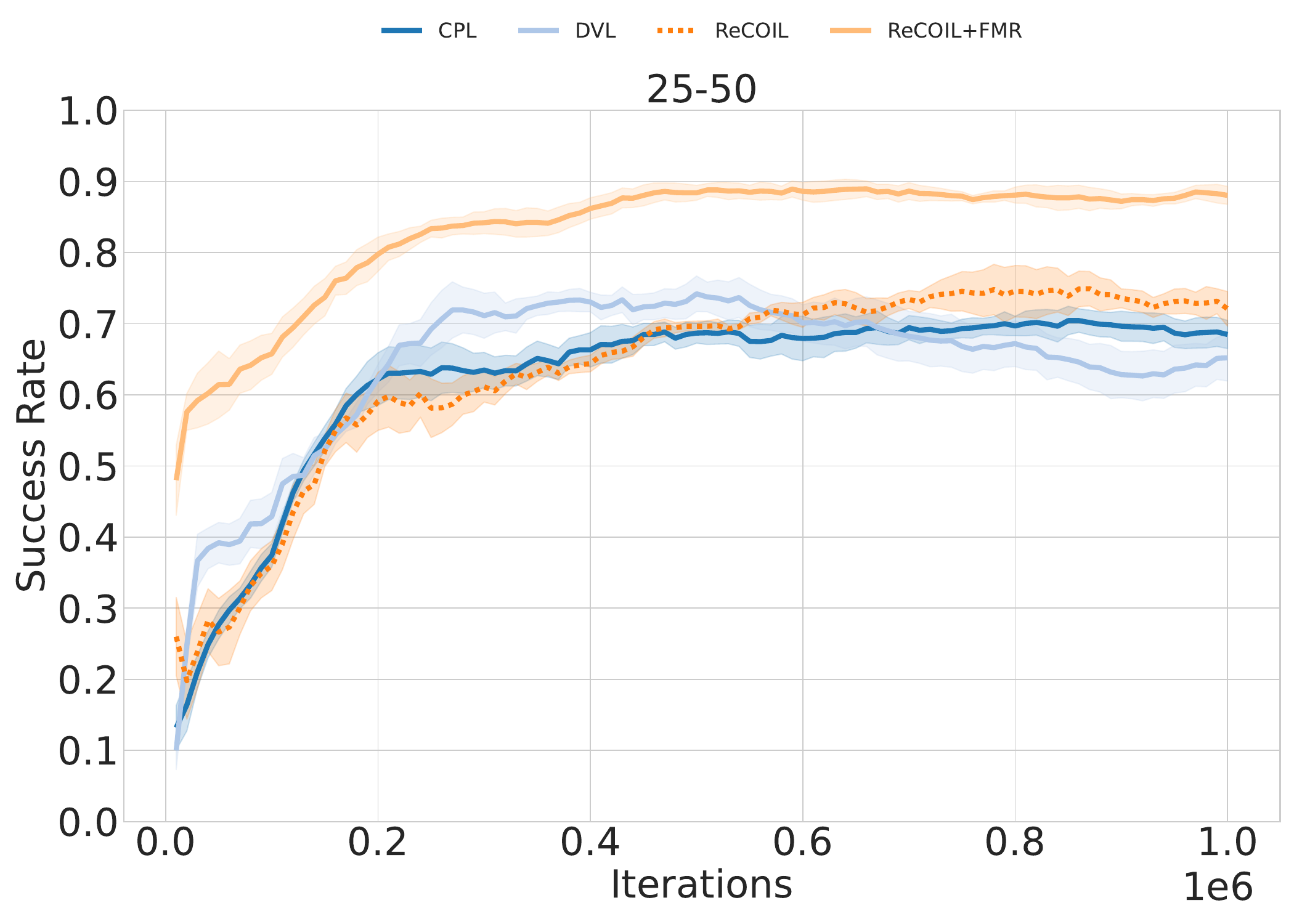}
    \end{subfigure}
    \hfill
    \begin{subfigure}[b]{0.32\textwidth}
        \centering
        \includegraphics[width=\textwidth,trim=0 44 0 55,clip]{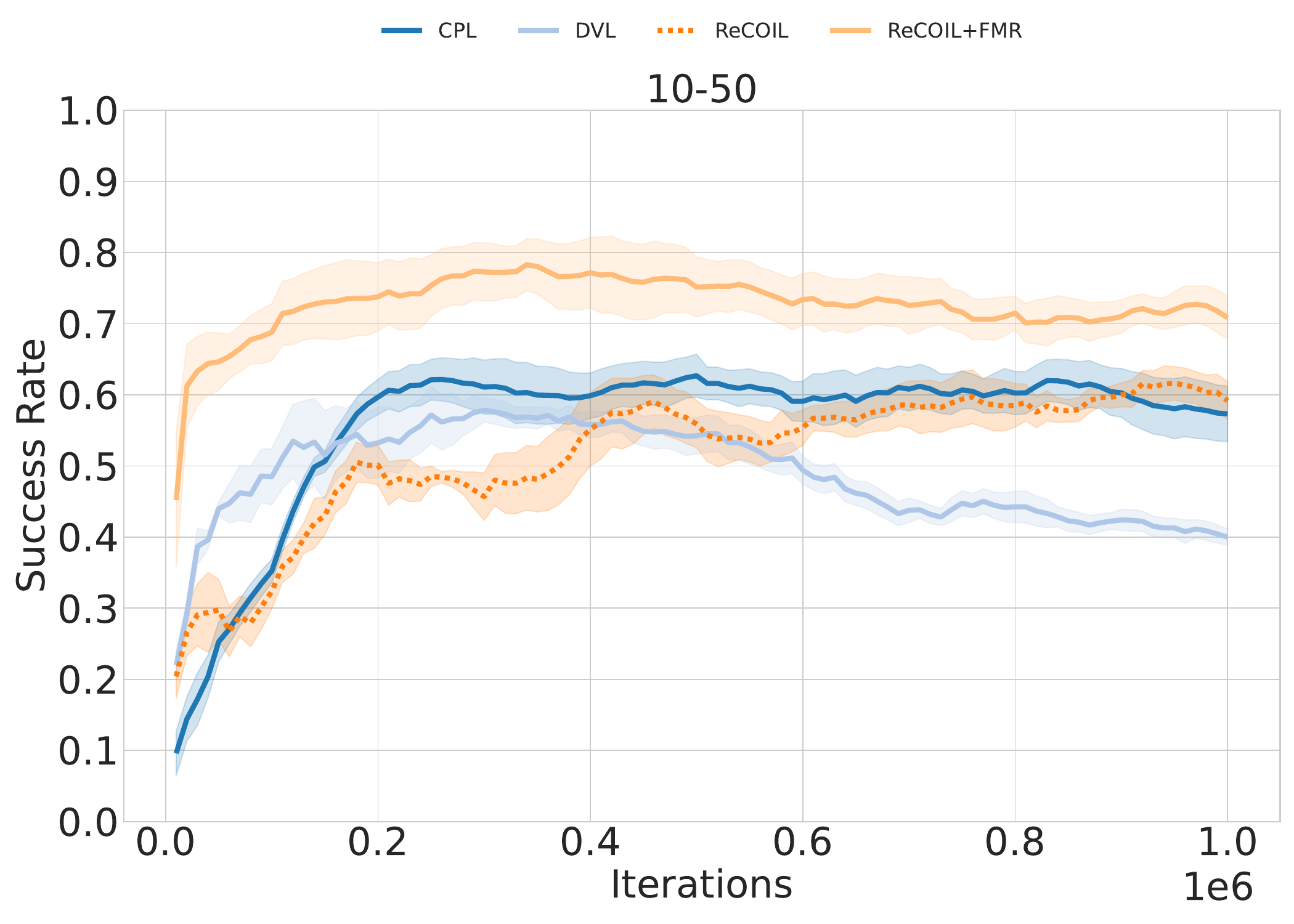}
    \end{subfigure}
    
    \begin{subfigure}[b]{0.32\textwidth}
        \centering
        \includegraphics[width=\textwidth,trim=0 0 0 81,clip]{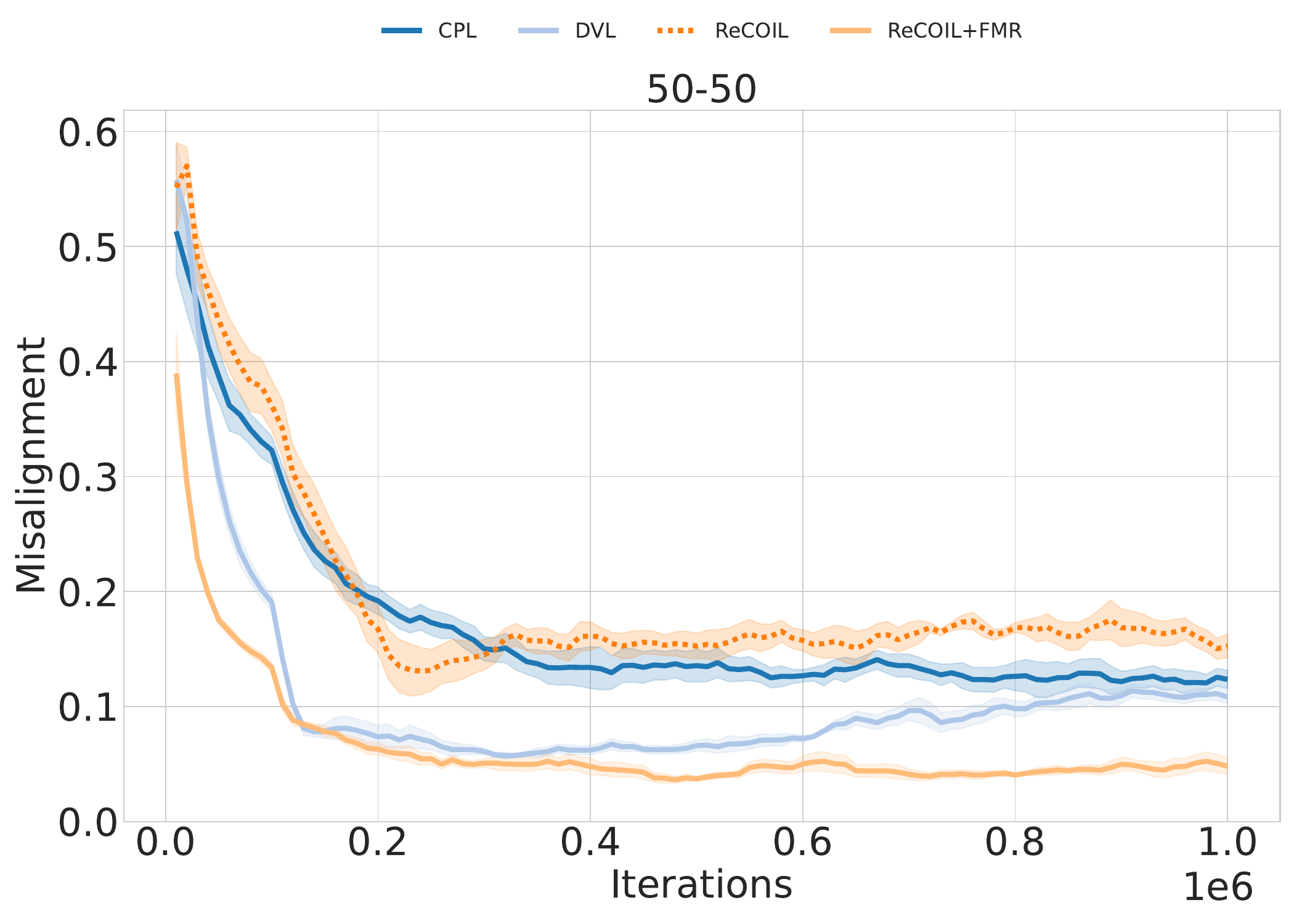}
    \end{subfigure}
    \hfill
    \begin{subfigure}[b]{0.32\textwidth}
        \centering
        \includegraphics[width=\textwidth,trim=0 0 0 81,clip]{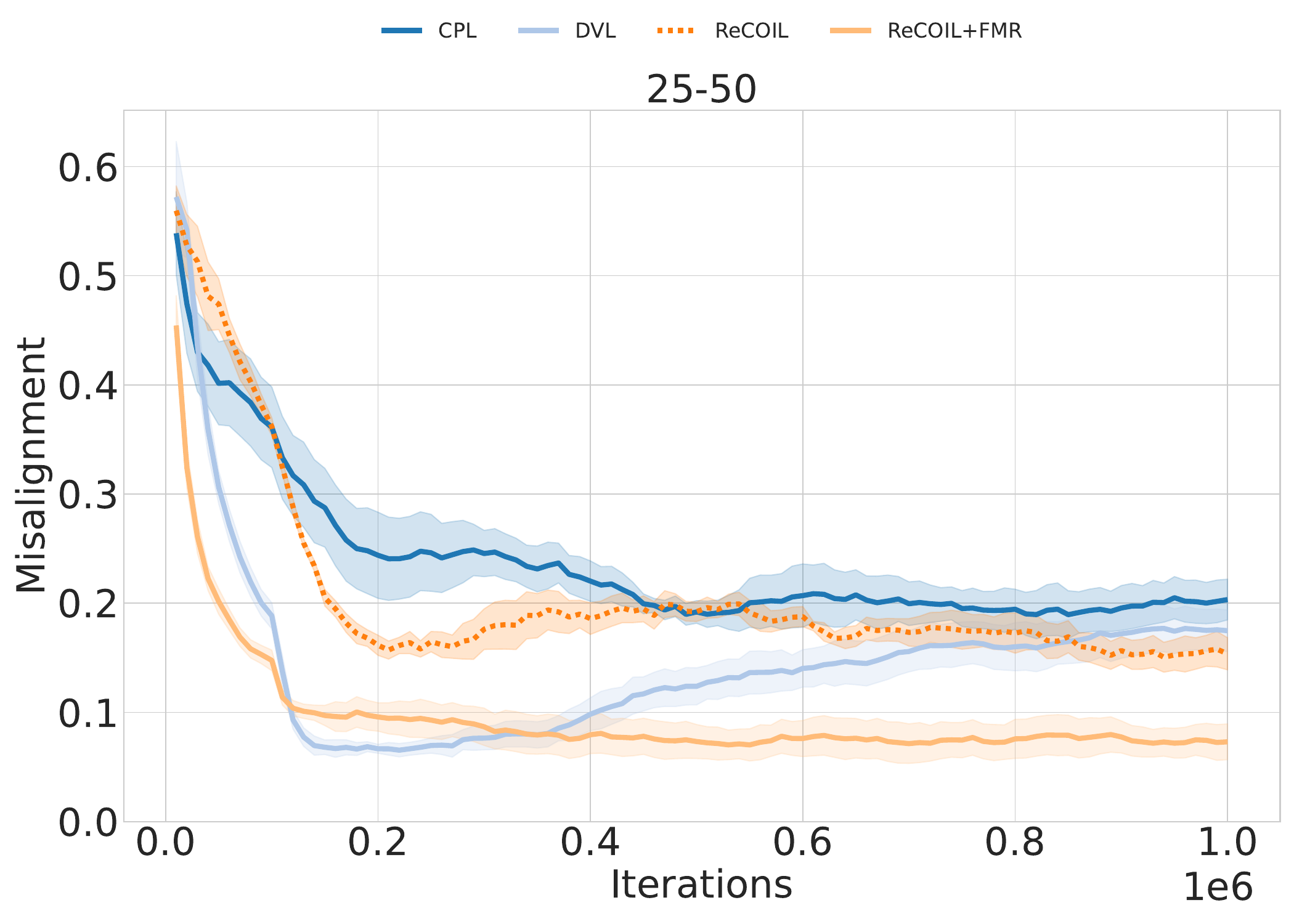}
    \end{subfigure}
    \hfill
    \begin{subfigure}[b]{0.32\textwidth}
        \centering
        \includegraphics[width=\textwidth,trim=0 0 0 81,clip]{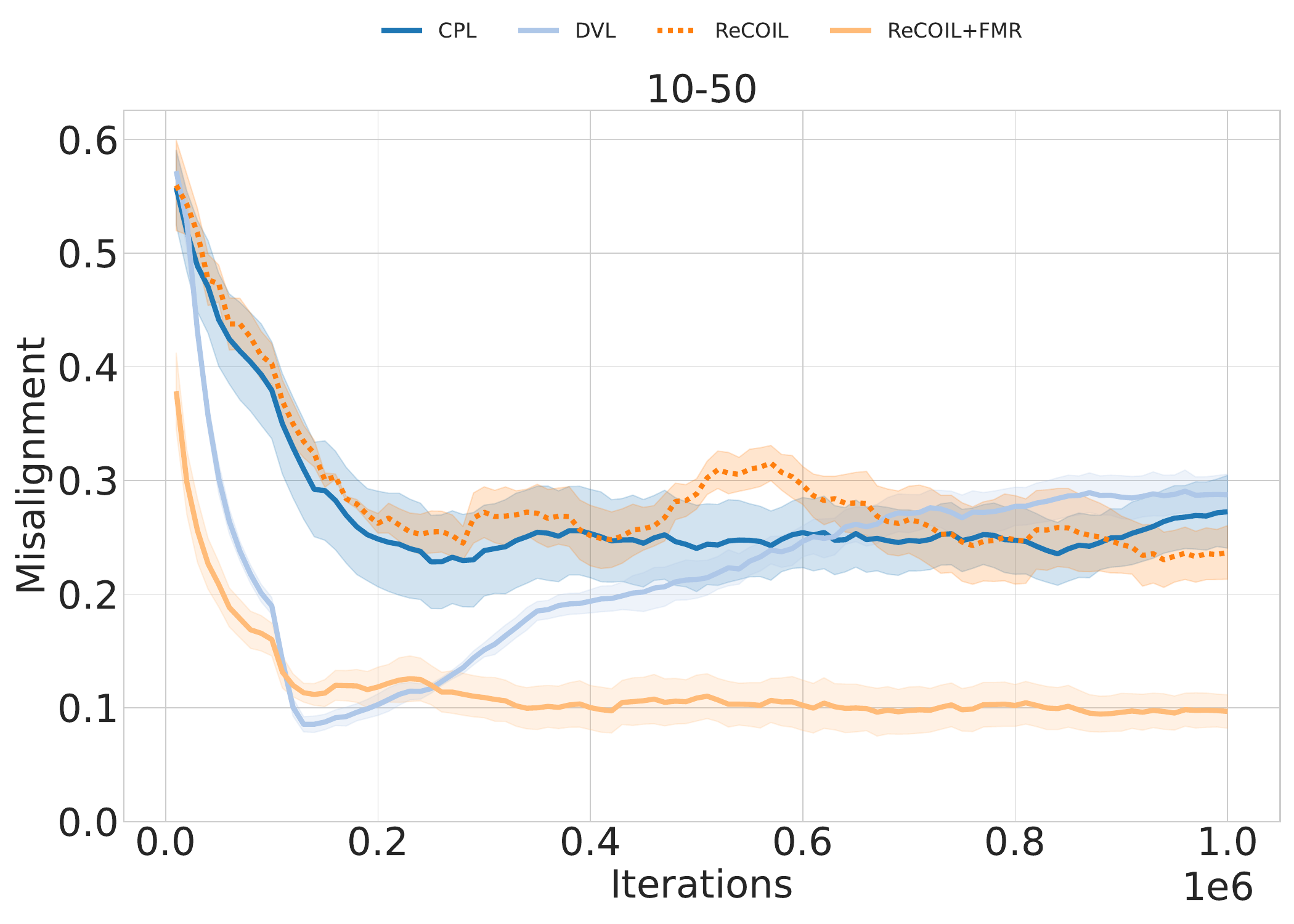}
    \end{subfigure}
    
    \caption{PathBB learning curves comparing FMR to alternative methods for utilizing evaluative feedback. The shaded region represents the standard error.}
    \label{fig:pathbb-cpl-dvl}
\end{figure}

\begin{figure}[h]
    \centering
    
    \begin{subfigure}[b]{\textwidth}
        \centering
        \includegraphics[width=0.8\textwidth,trim=0 665 0 0,clip]{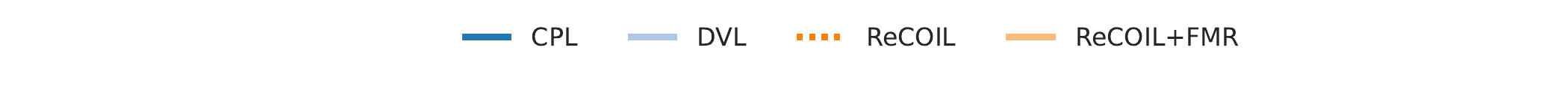}
    \end{subfigure}
    
    \vspace{0.5em}
    
    \begin{subfigure}[b]{0.32\textwidth}
        \centering
        \includegraphics[width=\textwidth,trim=0 44 0 55,clip]{images/cpl-dvl/SlowSwim/return_50-50.pdf}
    \end{subfigure}
    \hfill
    \begin{subfigure}[b]{0.32\textwidth}
        \centering
        \includegraphics[width=\textwidth,trim=0 44 0 55,clip]{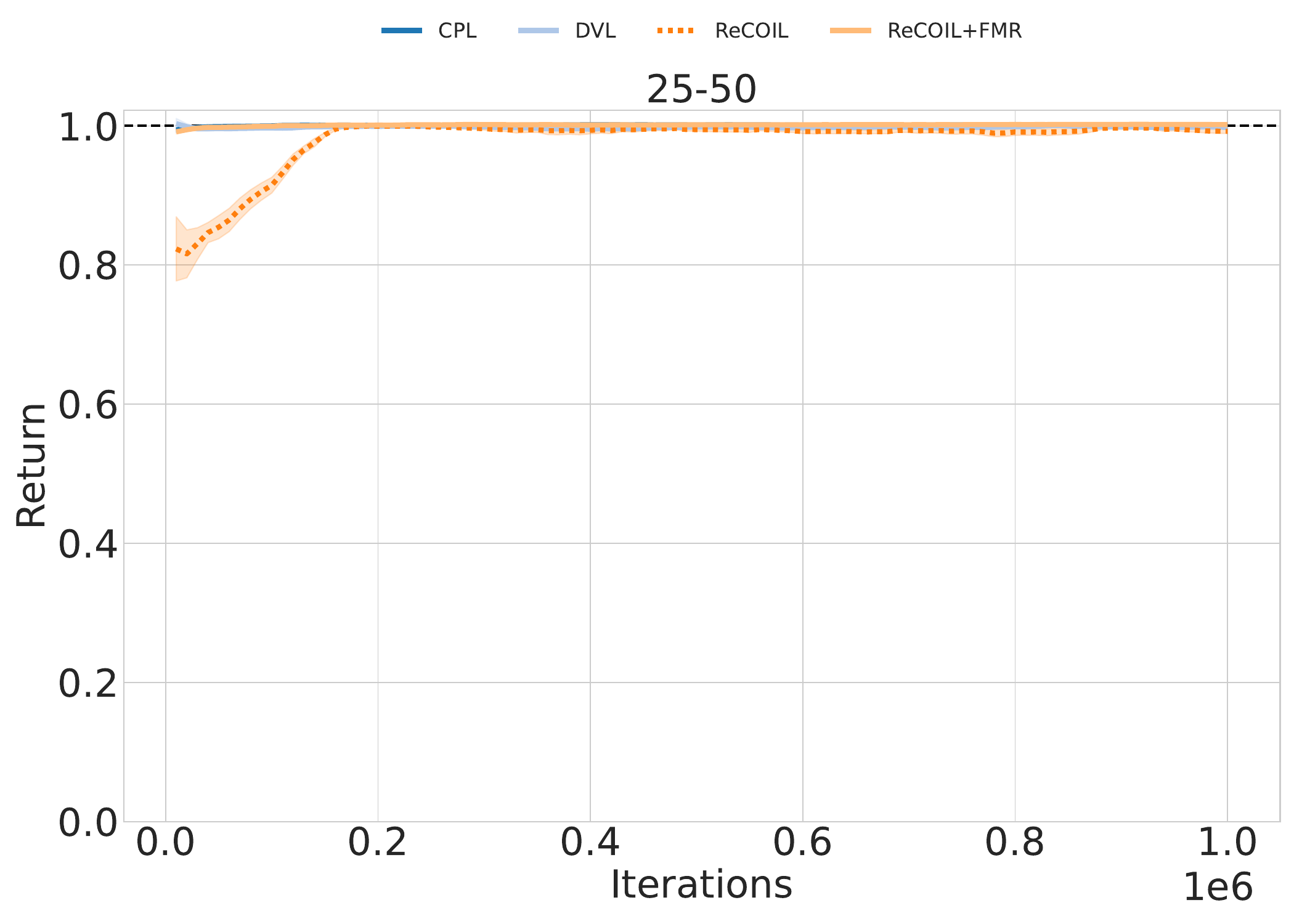}
    \end{subfigure}
    \hfill
    \begin{subfigure}[b]{0.32\textwidth}
        \centering
        \includegraphics[width=\textwidth,trim=0 44 0 55,clip]{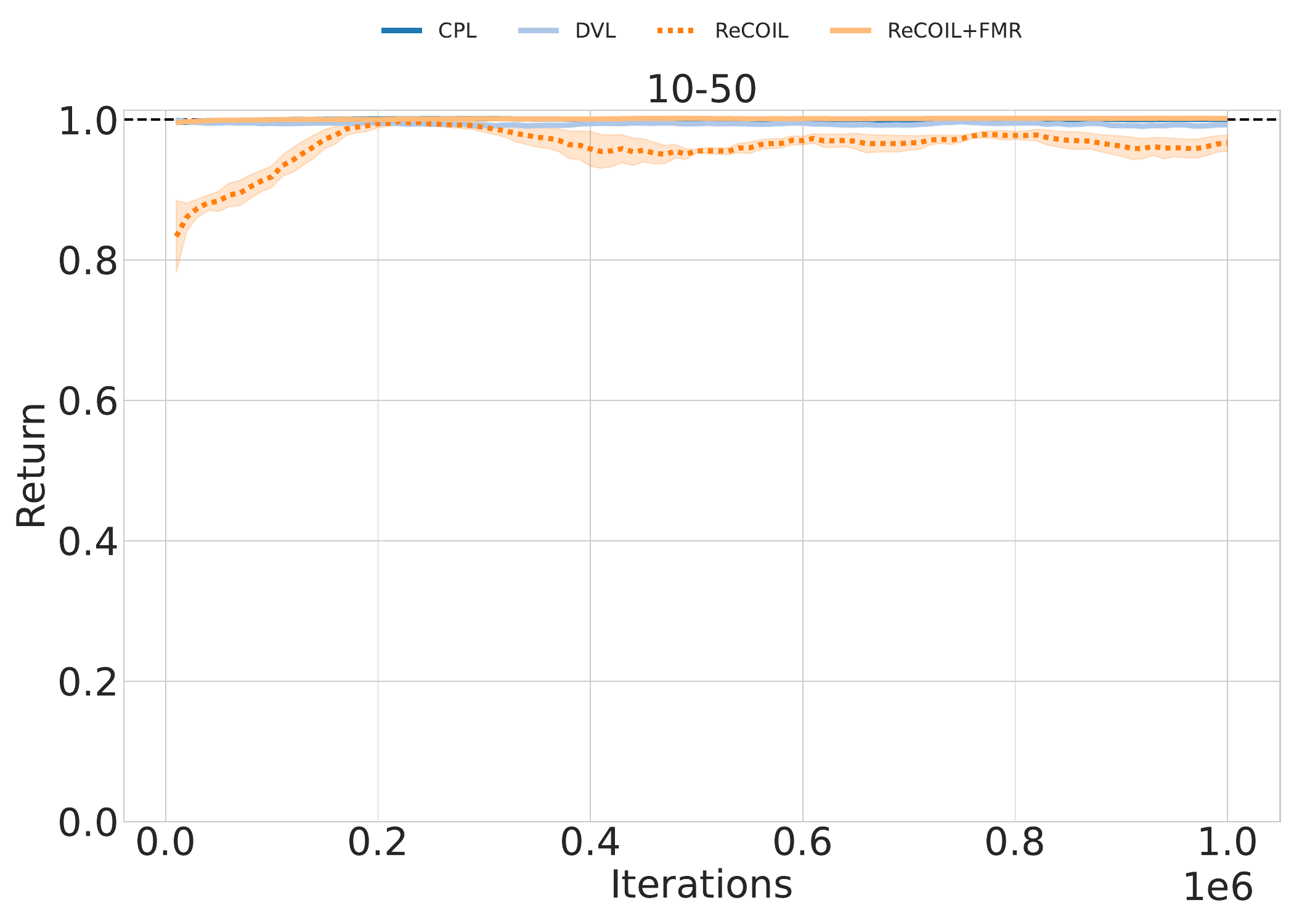}
    \end{subfigure}
    
    \begin{subfigure}[b]{0.32\textwidth}
        \centering
        \includegraphics[width=\textwidth,trim=0 0 0 81,clip]{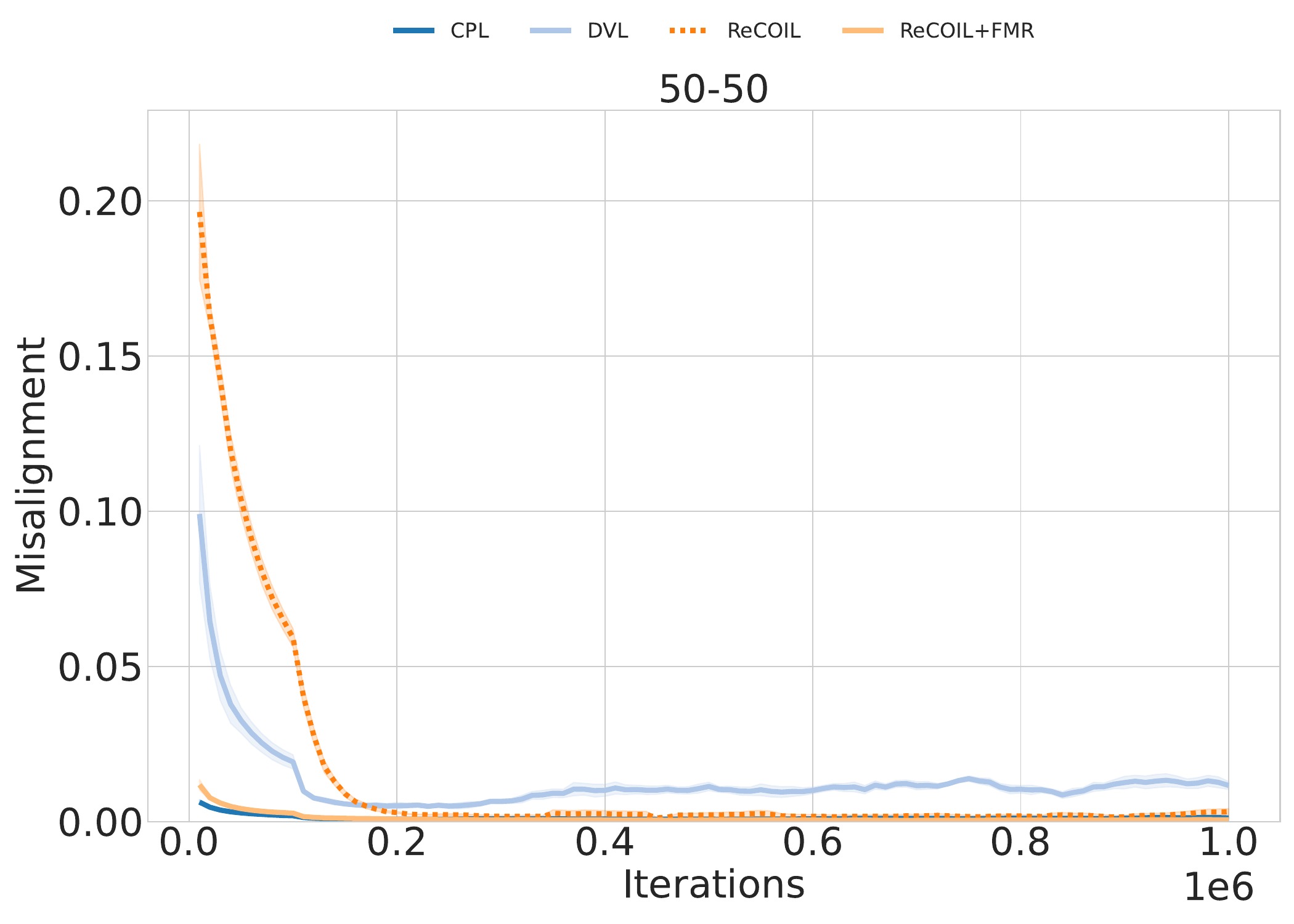}
    \end{subfigure}
    \hfill
    \begin{subfigure}[b]{0.32\textwidth}
        \centering
        \includegraphics[width=\textwidth,trim=0 0 0 81,clip]{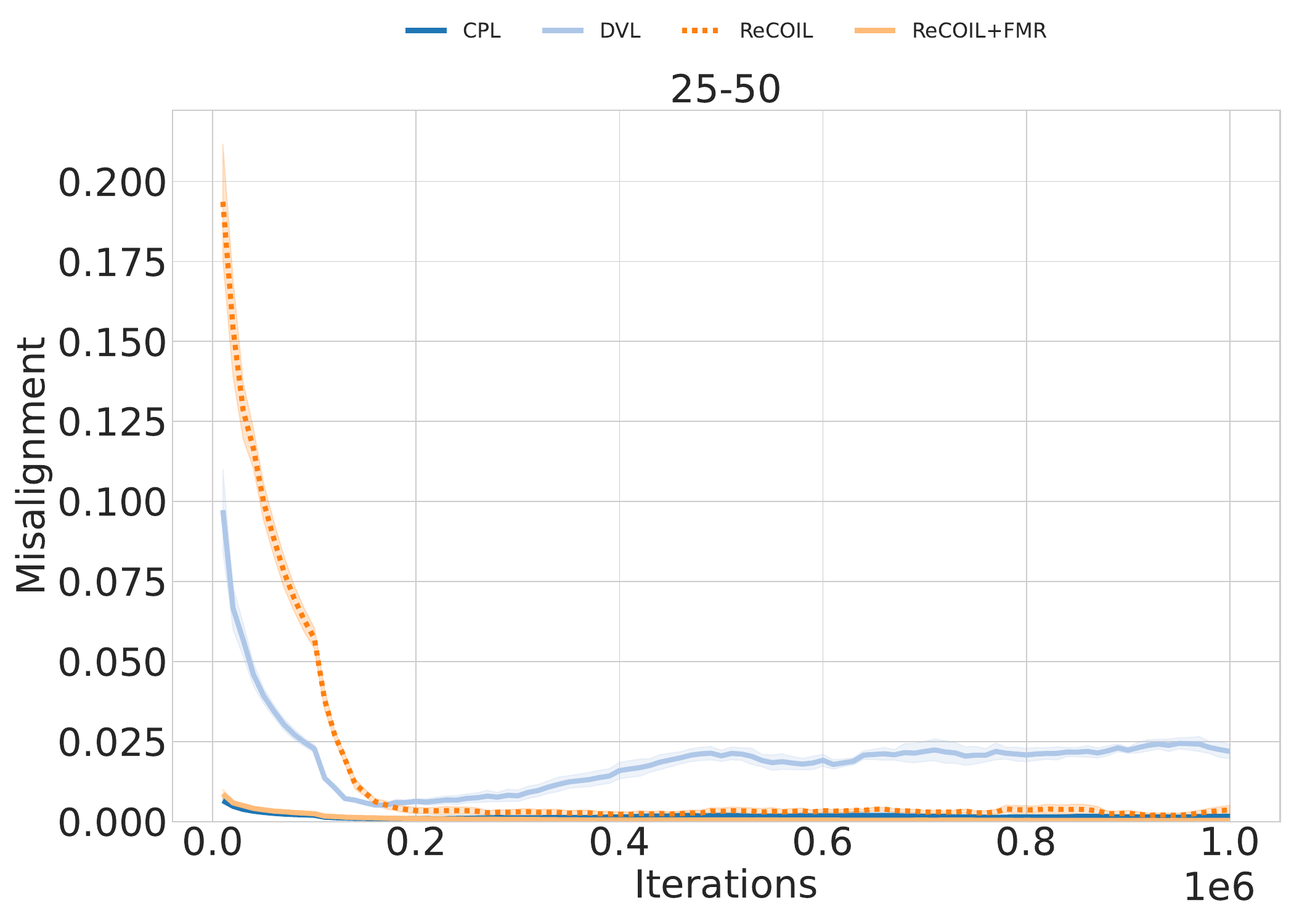}
    \end{subfigure}
    \hfill
    \begin{subfigure}[b]{0.32\textwidth}
        \centering
        \includegraphics[width=\textwidth,trim=0 0 0 81,clip]{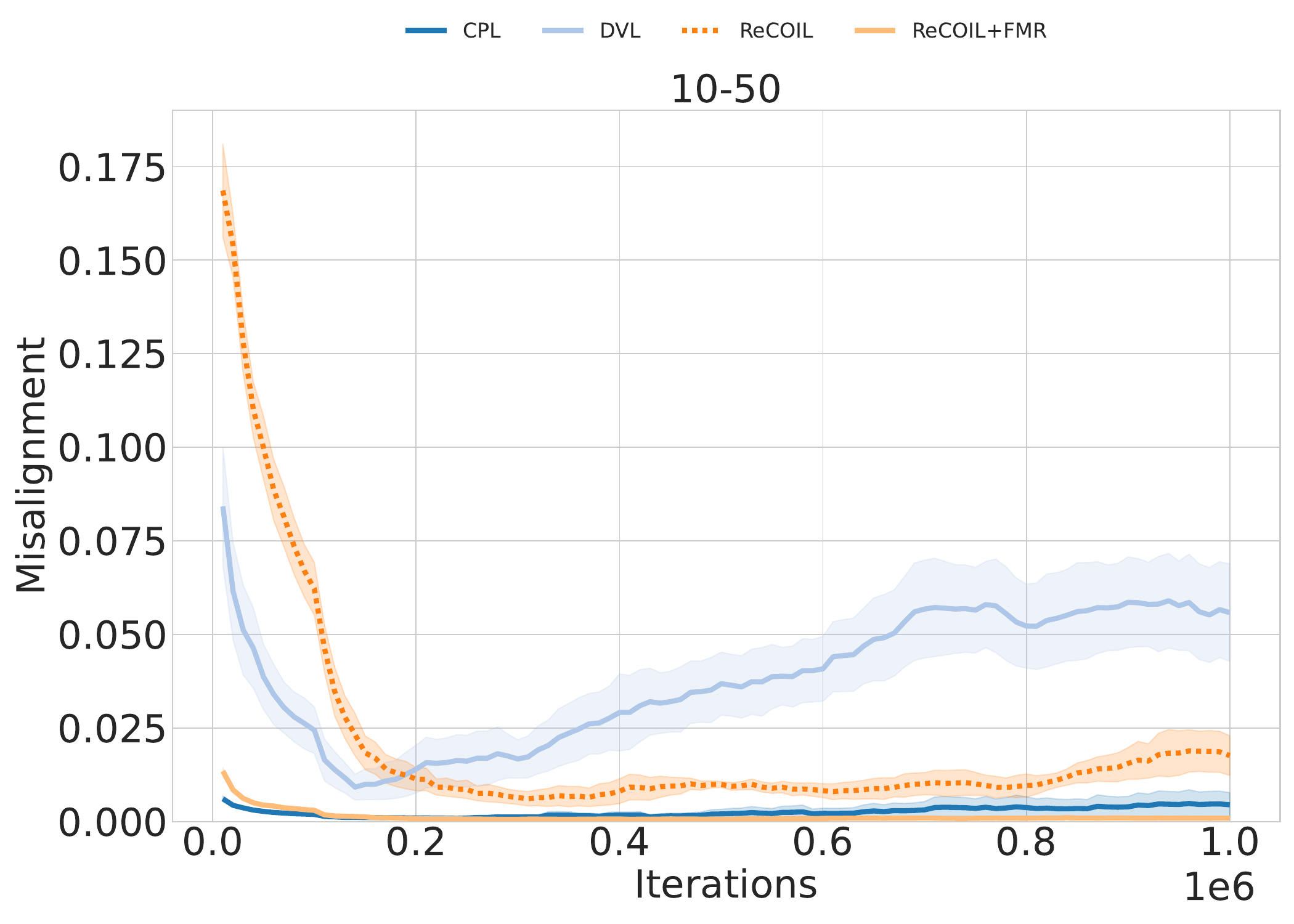}
    \end{subfigure}
    
    \caption{SlowSwim learning curves comparing FMR to alternative methods for utilizing evaluative feedback. The shaded region represents the standard error.}
    \label{fig:swim-cpl-dvl}
\end{figure}

\begin{figure}[h]
    \centering
    
    \begin{subfigure}[b]{\textwidth}
        \centering
        \includegraphics[width=0.8\textwidth,trim=0 665 0 0,clip]{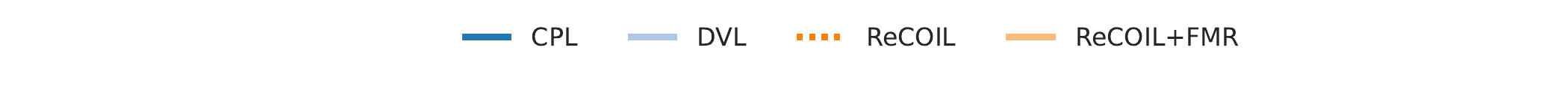}
    \end{subfigure}
    
    \vspace{0.5em}
    
    \begin{subfigure}[b]{0.32\textwidth}
        \centering
        \includegraphics[width=\textwidth,trim=0 44 0 55,clip]{images/cpl-dvl/SlowHop/return_50-50.pdf}
    \end{subfigure}
    \hfill
    \begin{subfigure}[b]{0.32\textwidth}
        \centering
        \includegraphics[width=\textwidth,trim=0 44 0 55,clip]{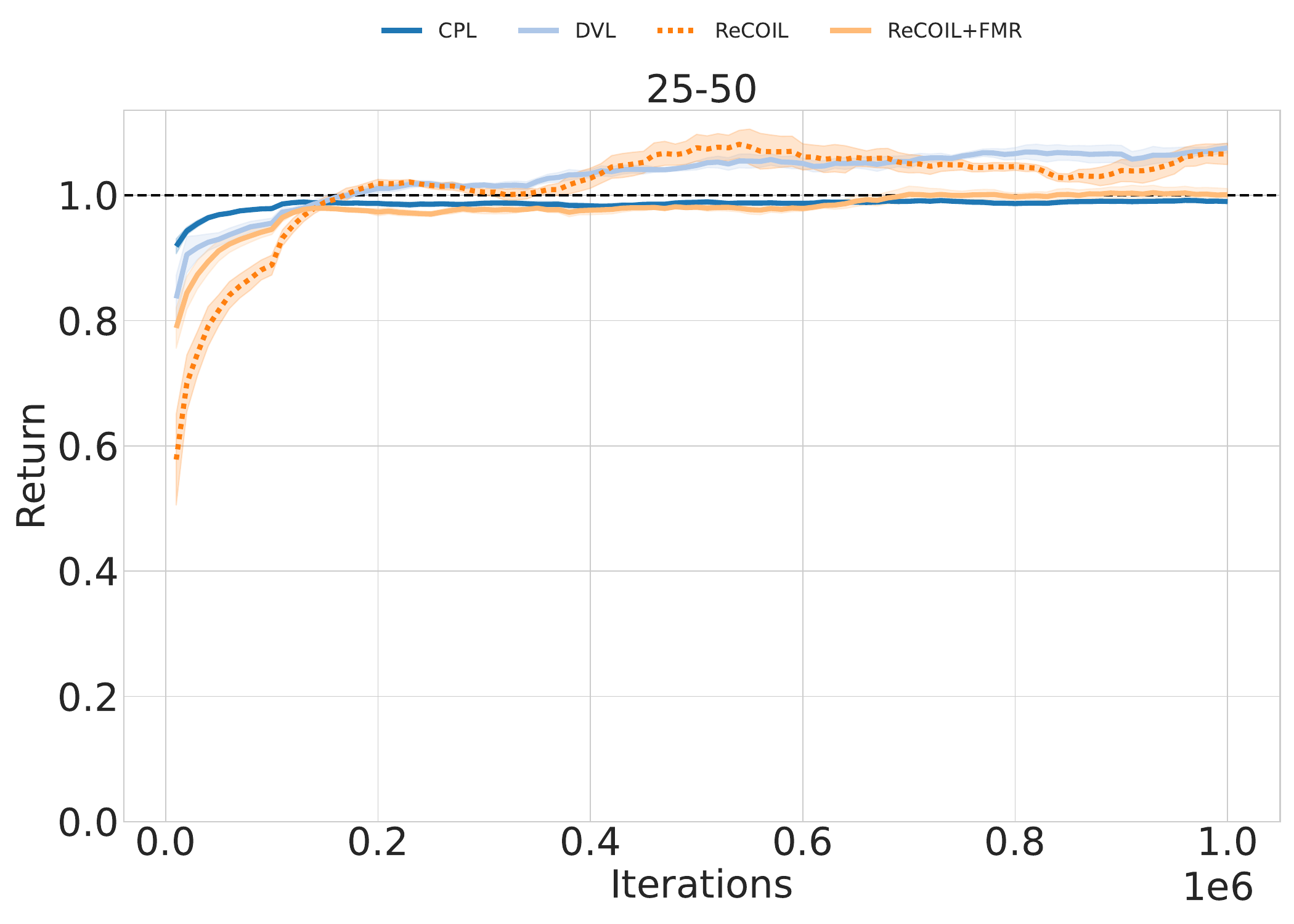}
    \end{subfigure}
    \hfill
    \begin{subfigure}[b]{0.32\textwidth}
        \centering
        \includegraphics[width=\textwidth,trim=0 44 0 55,clip]{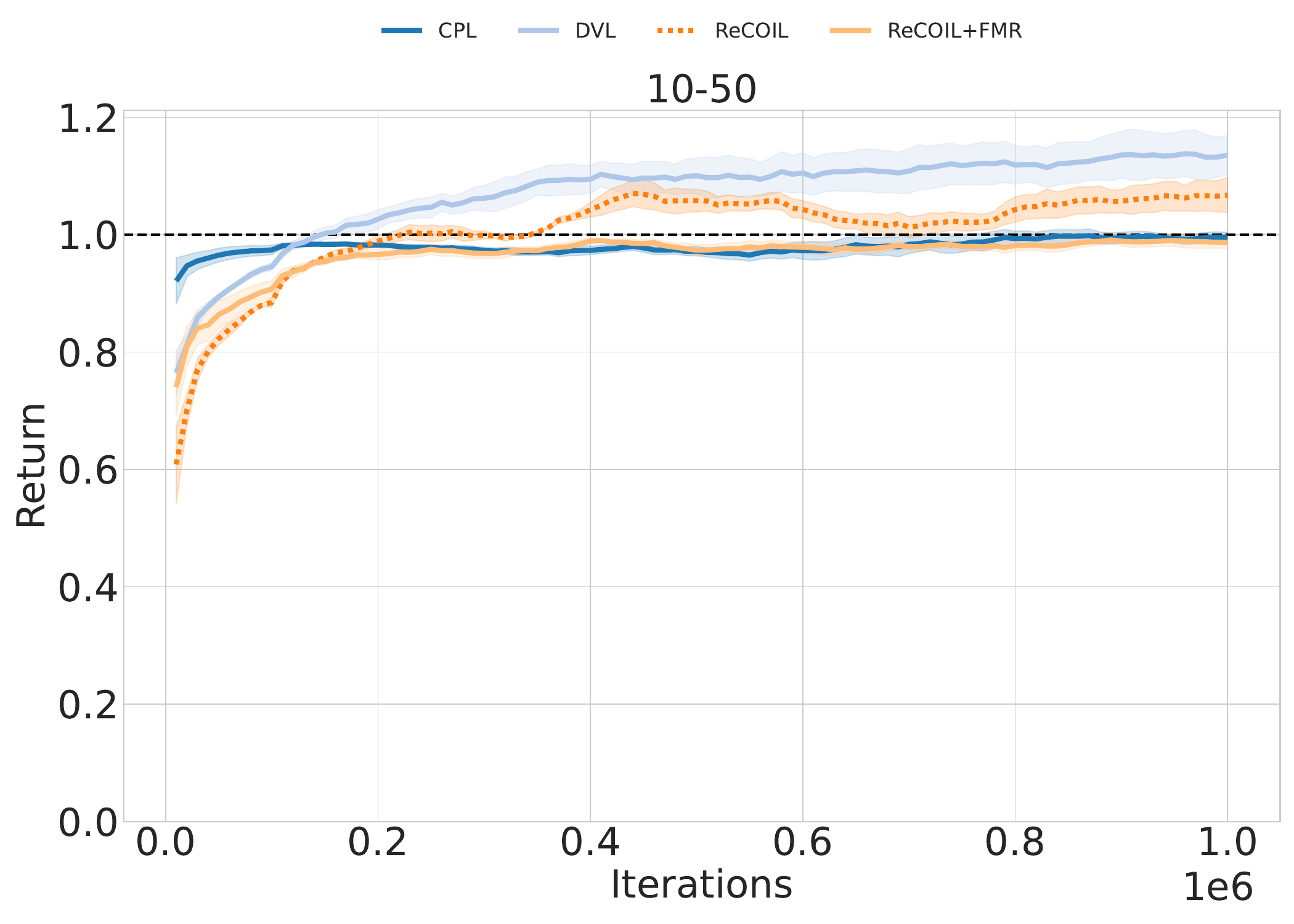}
    \end{subfigure}
    
    \vspace{-1em}
    \begin{subfigure}[b]{0.32\textwidth}
        \centering
        \includegraphics[width=\textwidth,trim=0 0 0 81,clip]{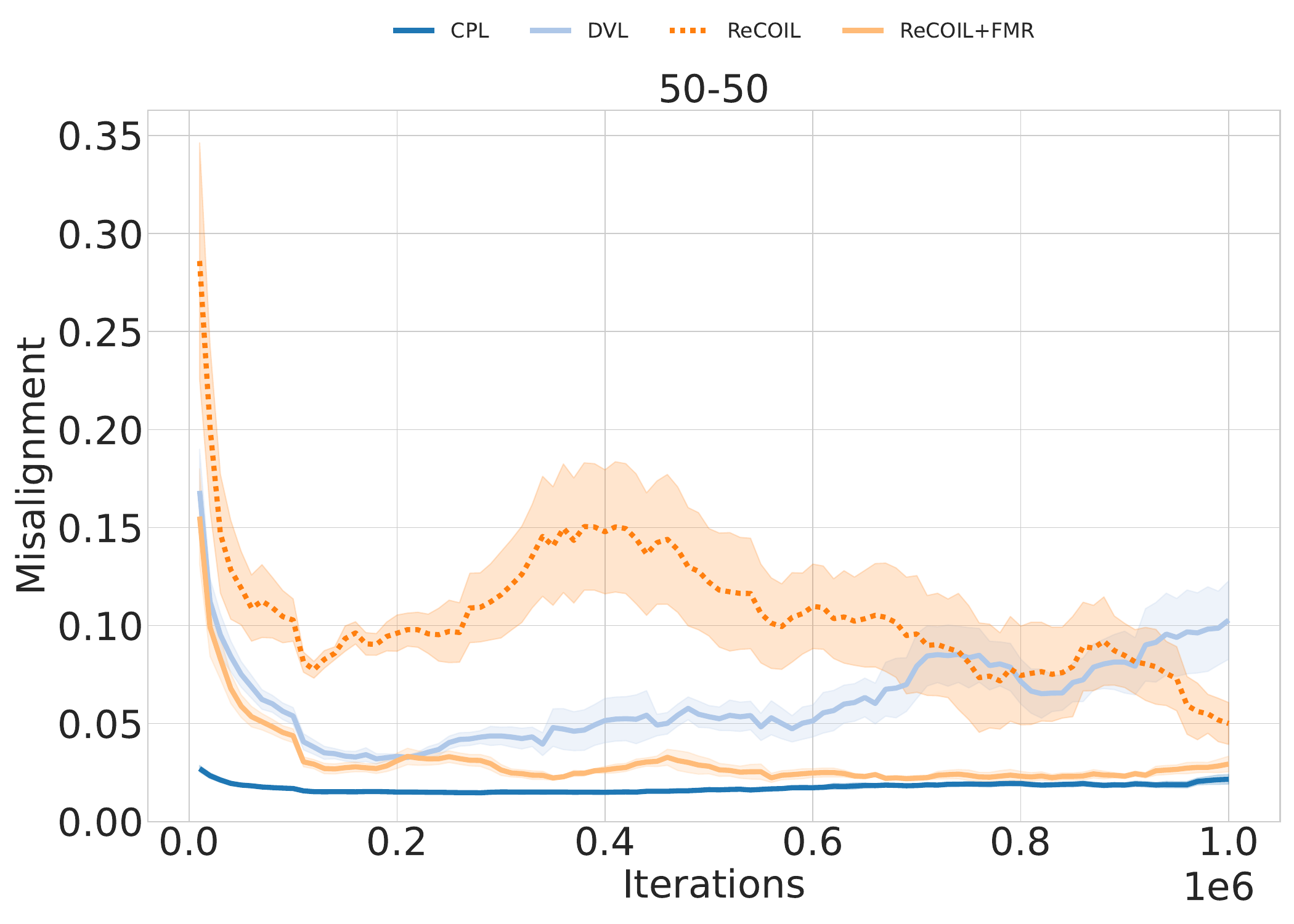}
    \end{subfigure}
    \hfill
    \begin{subfigure}[b]{0.32\textwidth}
        \centering
        \includegraphics[width=\textwidth,trim=0 0 0 81,clip]{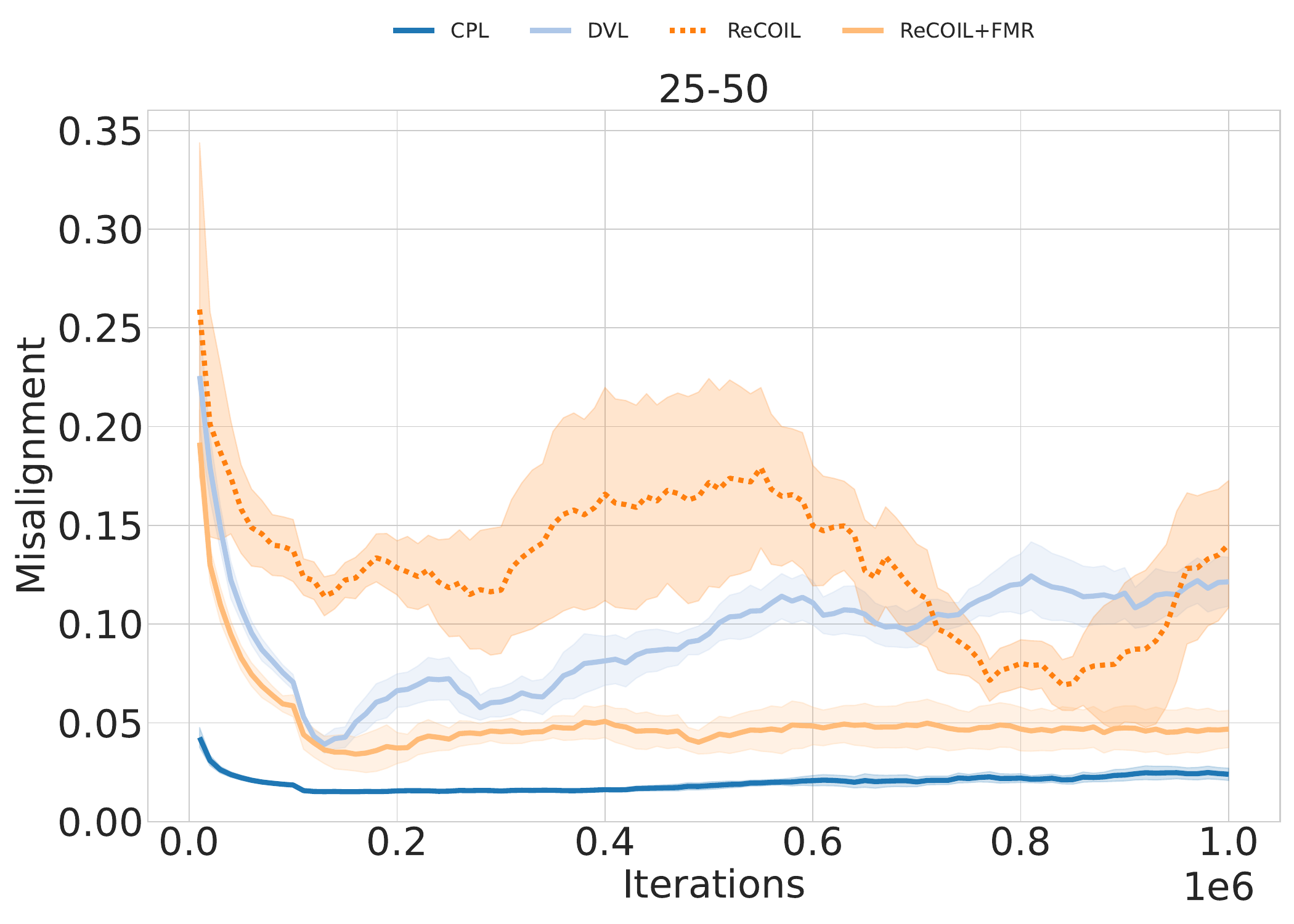}
    \end{subfigure}
    \hfill
    \begin{subfigure}[b]{0.32\textwidth}
        \centering
        \includegraphics[width=\textwidth,trim=0 0 0 81,clip]{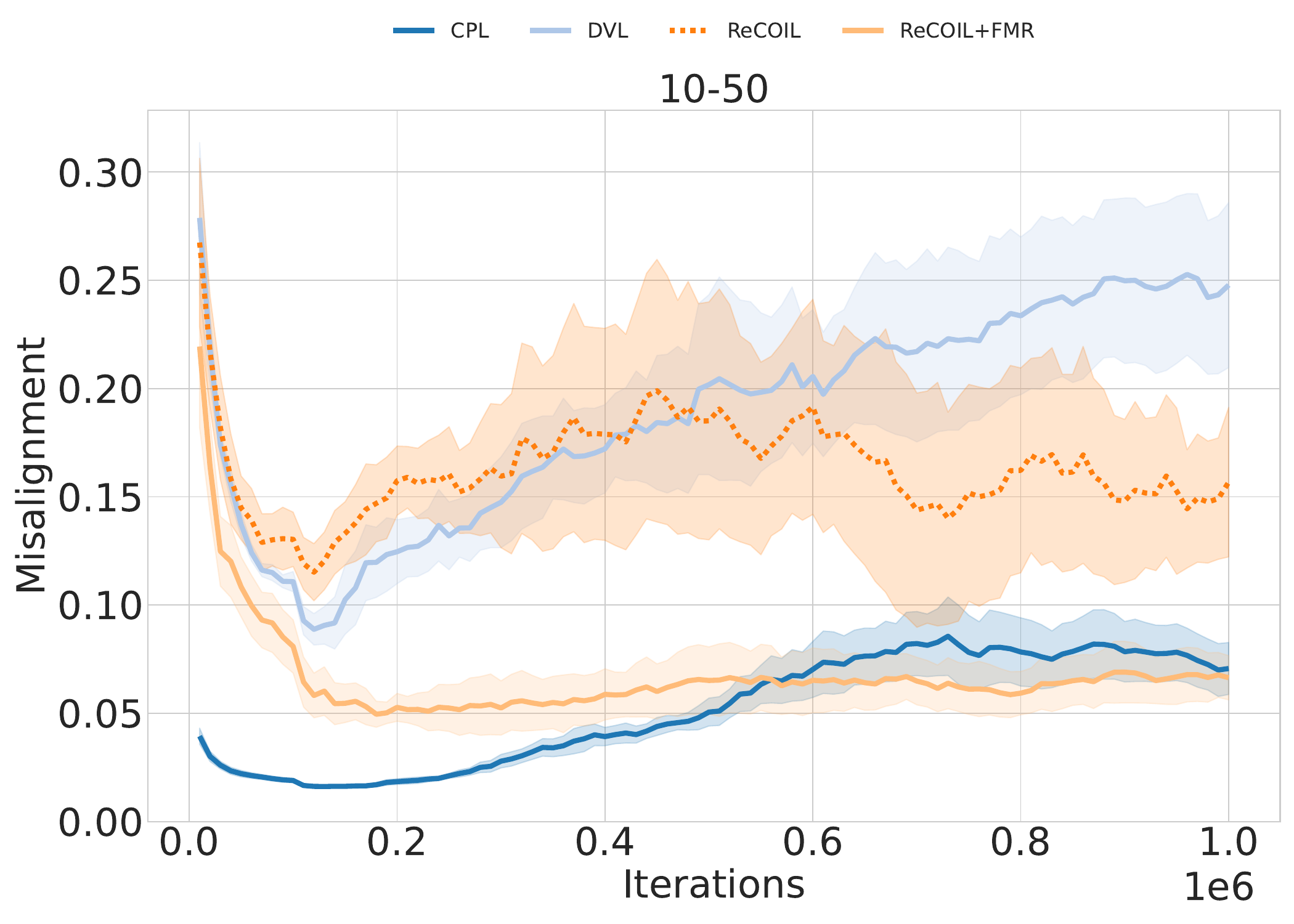}
    \end{subfigure}
    
    \caption{SlowHop learning curves comparing FMR to alternative methods for utilizing evaluative feedback. The shaded region represents the standard error.}
    \label{fig:hop-cpl-dvl}
\end{figure}

\begin{figure}[h]
    \centering
    
    \begin{subfigure}[b]{\textwidth}
        \centering
        \includegraphics[width=0.8\textwidth,trim=0 665 0 0,clip]{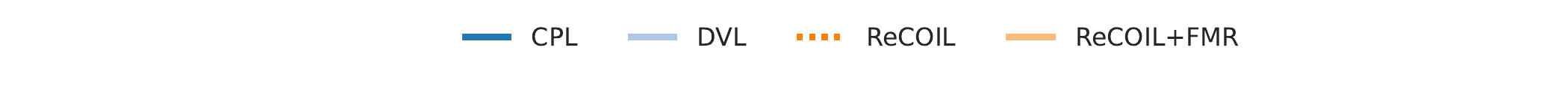}
    \end{subfigure}
    
    \vspace{0.5em}
    
    \begin{subfigure}[b]{0.32\textwidth}
        \centering
        \includegraphics[width=\textwidth,trim=0 44 0 55,clip]{images/cpl-dvl/SlowWalk/return_50-50.pdf}
    \end{subfigure}
    \hfill
    \begin{subfigure}[b]{0.32\textwidth}
        \centering
        \includegraphics[width=\textwidth,trim=0 44 0 55,clip]{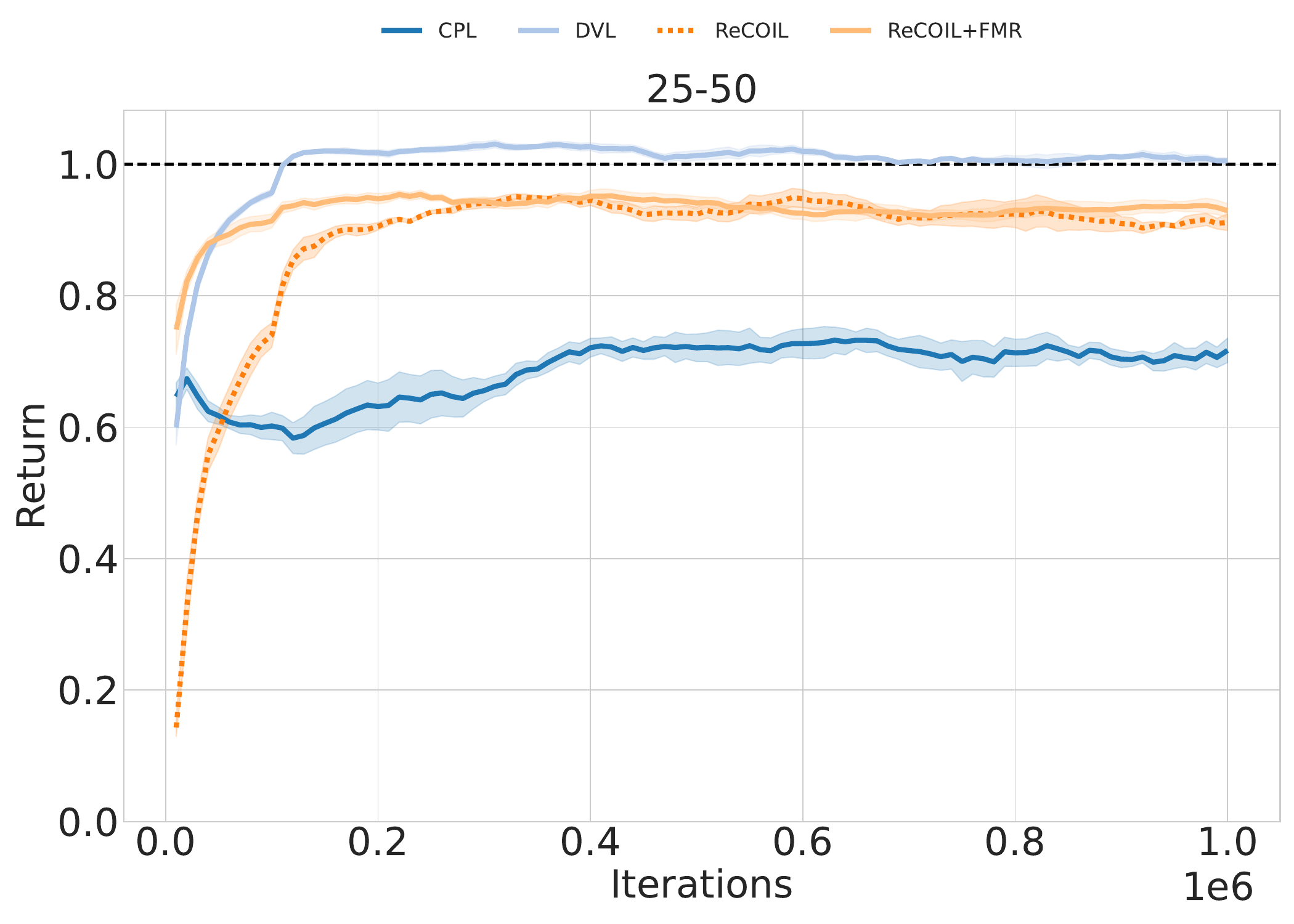}
    \end{subfigure}
    \hfill
    \begin{subfigure}[b]{0.32\textwidth}
        \centering
        \includegraphics[width=\textwidth,trim=0 44 0 55,clip]{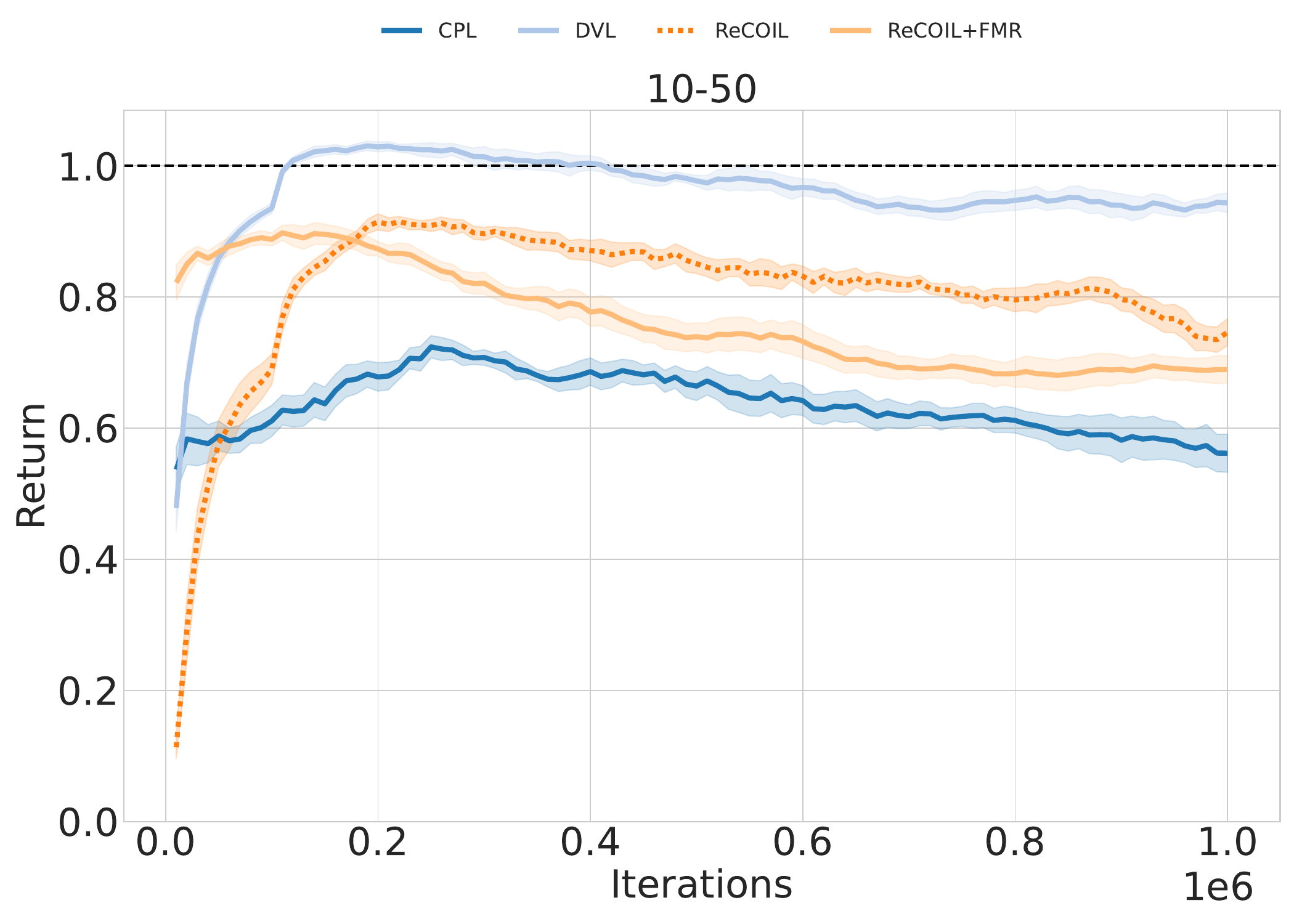}
    \end{subfigure}
    
    \vspace{-1em}
    \begin{subfigure}[b]{0.32\textwidth}
        \centering
        \includegraphics[width=\textwidth,trim=0 0 0 81,clip]{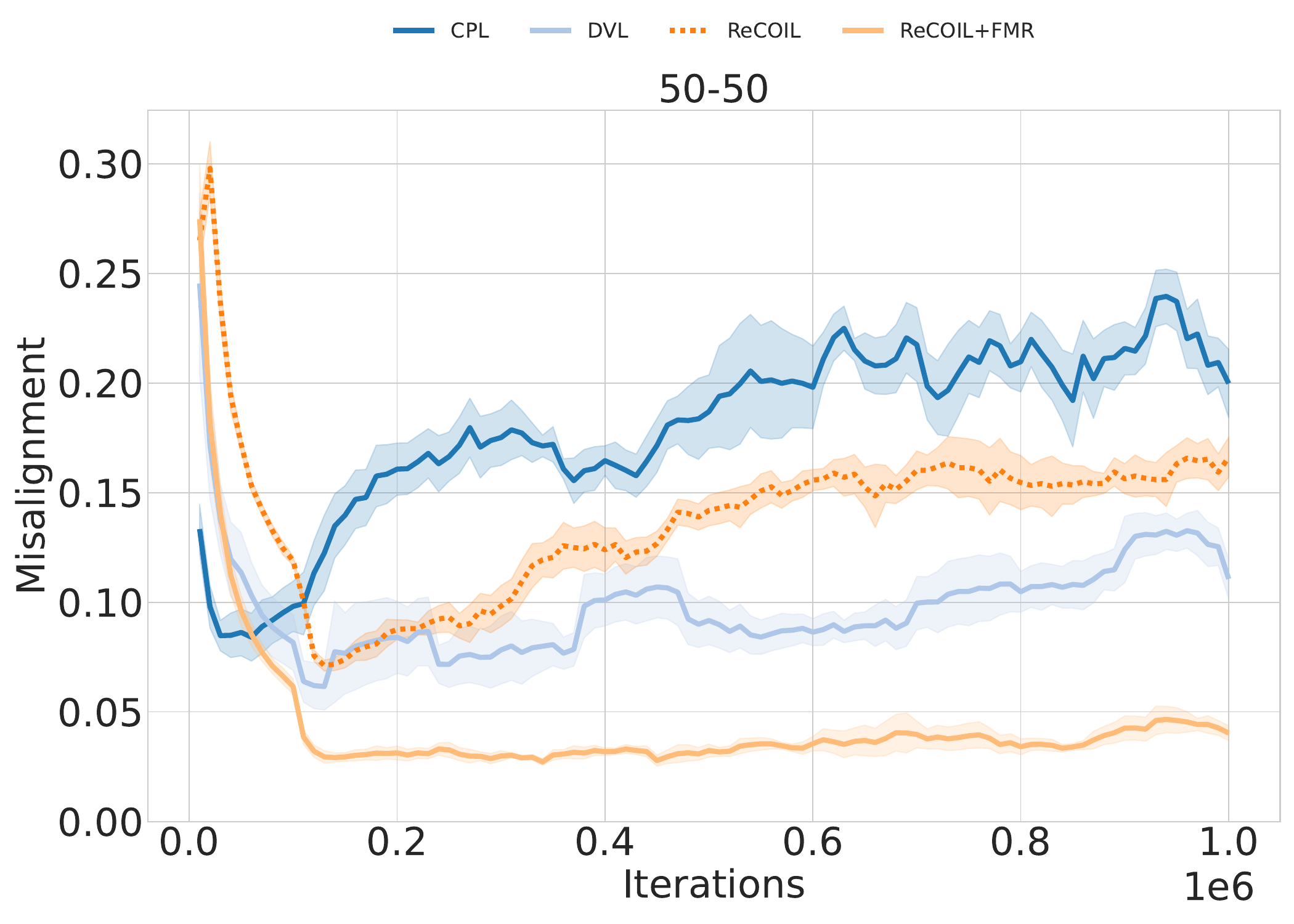}
    \end{subfigure}
    \hfill
    \begin{subfigure}[b]{0.32\textwidth}
        \centering
        \includegraphics[width=\textwidth,trim=0 0 0 81,clip]{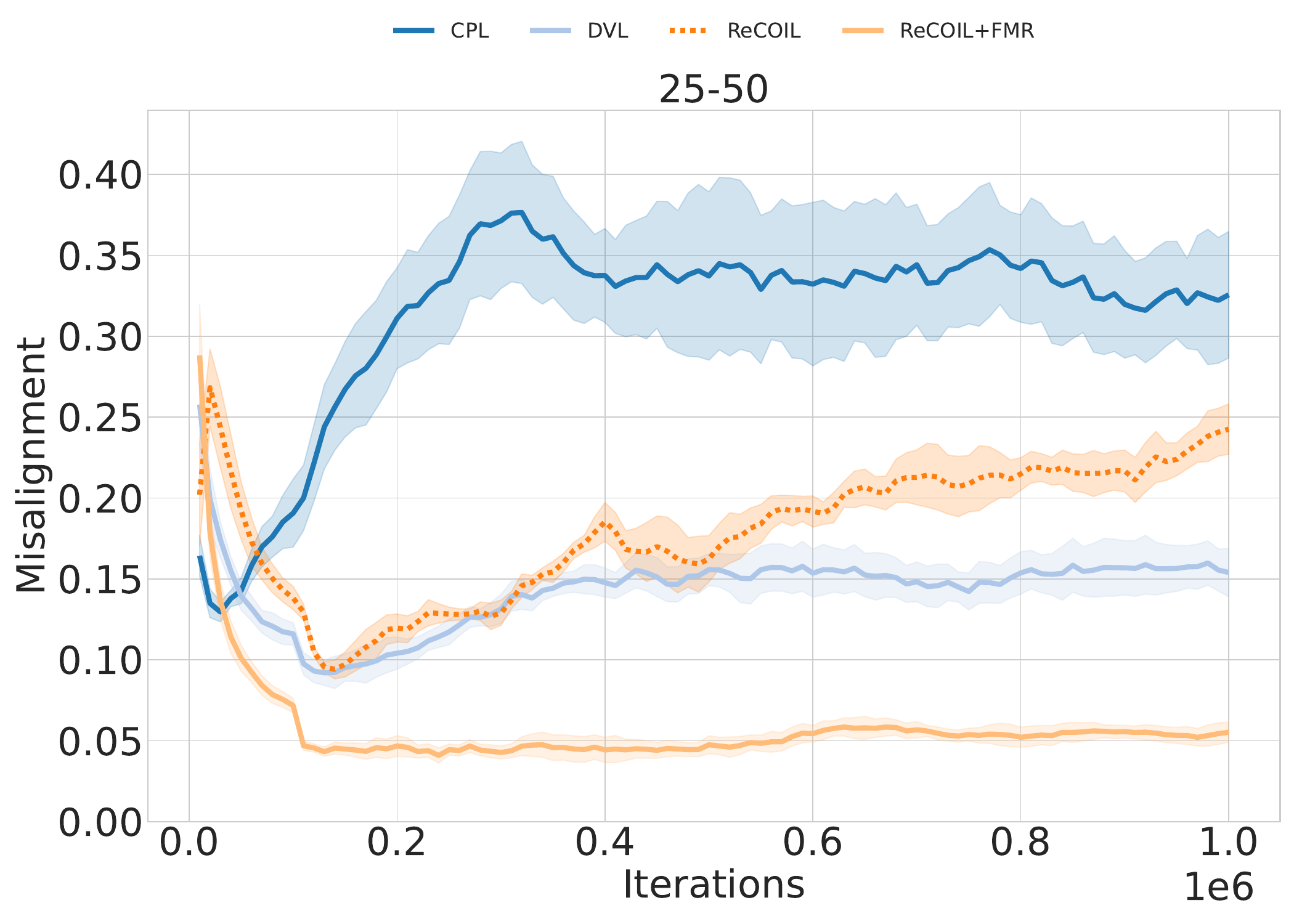}
    \end{subfigure}
    \hfill
    \begin{subfigure}[b]{0.32\textwidth}
        \centering
        \includegraphics[width=\textwidth,trim=0 0 0 81,clip]{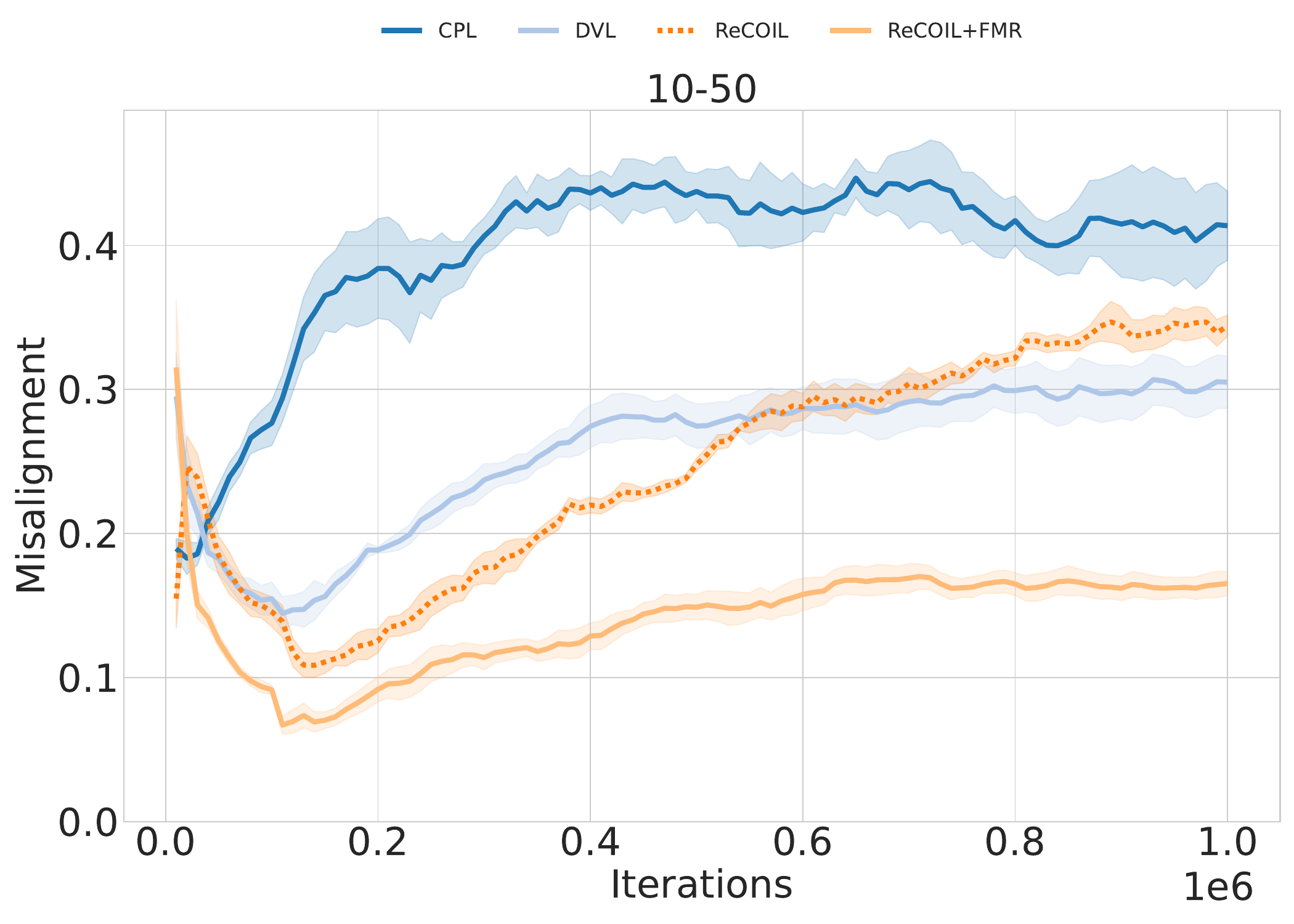}
    \end{subfigure}
    
    \caption{SlowWalk learning curves comparing FMR to alternative methods for utilizing evaluative feedback. The shaded region represents the standard error.}
    \label{fig:walk-cpl-dvl}
\end{figure}

%% file: appendix/no-pos.tex
Below are the results for the navigation aligned policies, comparing with and without positive feedback for $D^I$ using ReCOIL+FMR.  

\begin{table}[h]
\centering
\caption{PathM performance comparison for ReCOIL+FMR with and without positive feedback. Results show mean ± std for the last 10 evaluations, over 5 seeds.}
\label{tab:pathm-no-pos}
\begin{tabular}{llrr}
\toprule
Algorithm & Ratio & Suc. & Mis. \\
\midrule
\multirow[t]{3}{*}{w/ pos} 
& 50-50 & 0.988 ± 0.11 & 0.015 ± 0.08 \\
& 25-50 & 0.974 ± 0.16 & 0.014 ± 0.08 \\
& 10-50 & 0.947 ± 0.22 & 0.020 ± 0.09 \\
\cline{1-4}
\multirow[t]{3}{*}{w/o pos} 
& 50-50 & 0.993 ± 0.08 & 0.008 ± 0.06 \\
& 25-50 & 0.982 ± 0.13 & 0.009 ± 0.06 \\
& 10-50 & 0.933 ± 0.25 & 0.007 ± 0.05 \\
\bottomrule
\end{tabular}
\end{table}

\begin{table}[h]
\centering
\caption{PathBB performance comparison for ReCOIL+FMR with and without positive feedback. Results show mean ± std for the last 10 evaluations, over 5 seeds.}
\label{tab:pathbb-no-pos}
\begin{tabular}{llrr}
\toprule
Algorithm & Ratio & Suc. & Mis. \\
\midrule
\multirow[t]{3}{*}{w/ pos} 
& 50-50 & 0.919 ± 0.27 & 0.048 ± 0.17 \\
& 25-50 & 0.880 ± 0.32 & 0.073 ± 0.19 \\
& 10-50 & 0.708 ± 0.45 & 0.097 ± 0.21 \\
\cline{1-4}
\multirow[t]{3}{*}{w/o pos} 
& 50-50 & 0.938 ± 0.24 & 0.036 ± 0.17 \\
& 25-50 & 0.892 ± 0.31 & 0.055 ± 0.20 \\
& 10-50 & 0.746 ± 0.44 & 0.134 ± 0.29 \\
\bottomrule
\end{tabular}
\end{table}

\begin{figure}[h]
    \centering
    
    \begin{subfigure}[b]{\textwidth}
        \centering
        \includegraphics[width=0.8\textwidth,trim=0 665 0 0,clip]{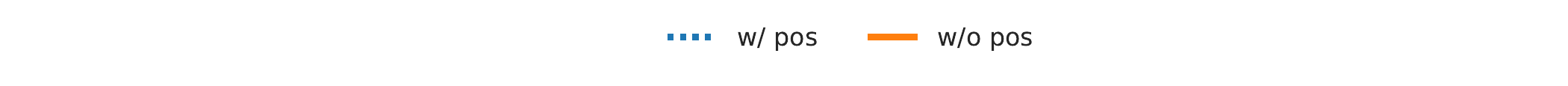}
    \end{subfigure}
    
    \vspace{0.5em}
    
    \begin{subfigure}[b]{0.32\textwidth}
        \centering
        \includegraphics[width=\textwidth,trim=0 44 0 55,clip]{images/no-pos/PathM/success_50-50.pdf}
    \end{subfigure}
    \hfill
    \begin{subfigure}[b]{0.32\textwidth}
        \centering
        \includegraphics[width=\textwidth,trim=0 44 0 55,clip]{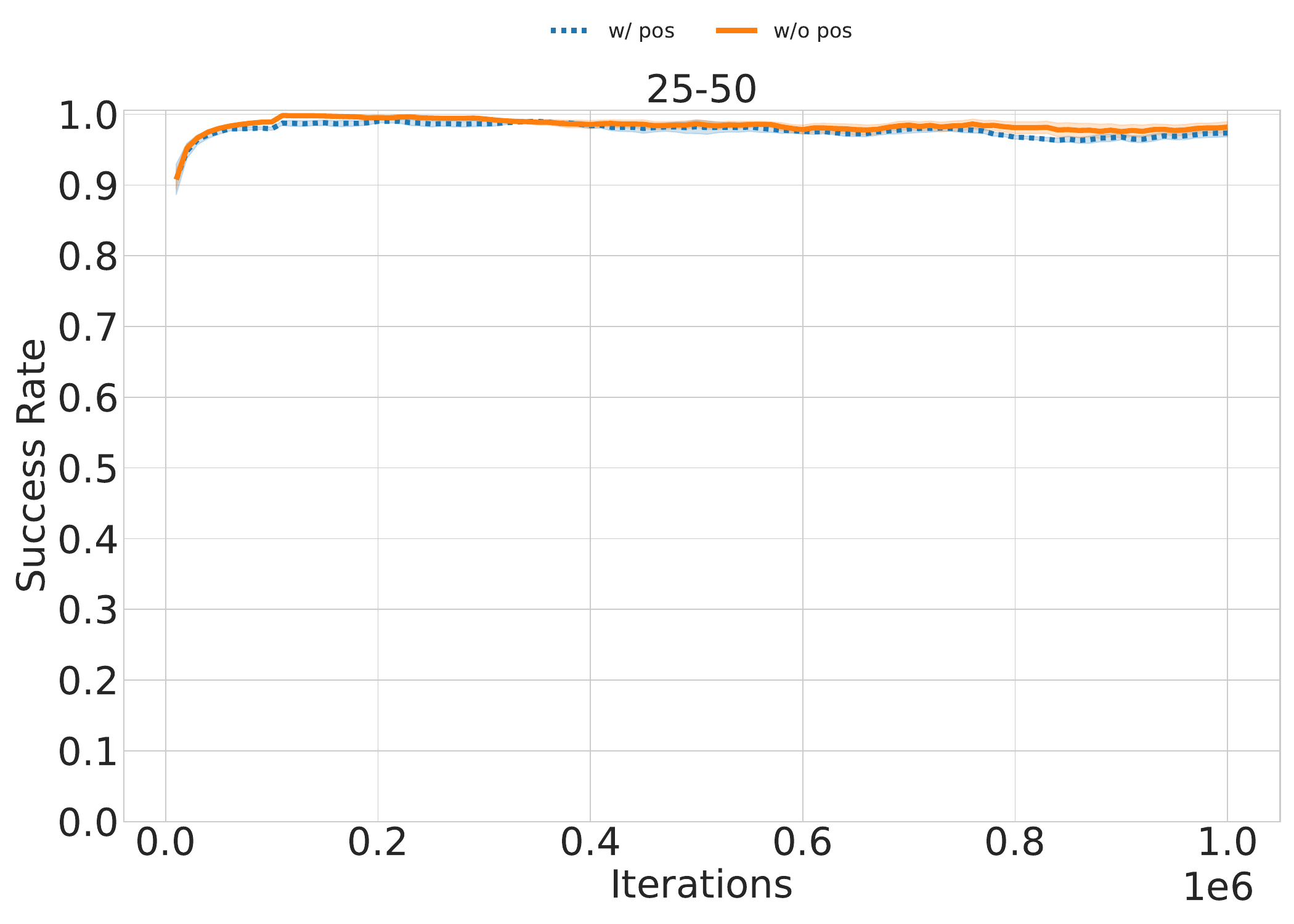}
    \end{subfigure}
    \hfill
    \begin{subfigure}[b]{0.32\textwidth}
        \centering
        \includegraphics[width=\textwidth,trim=0 44 0 55,clip]{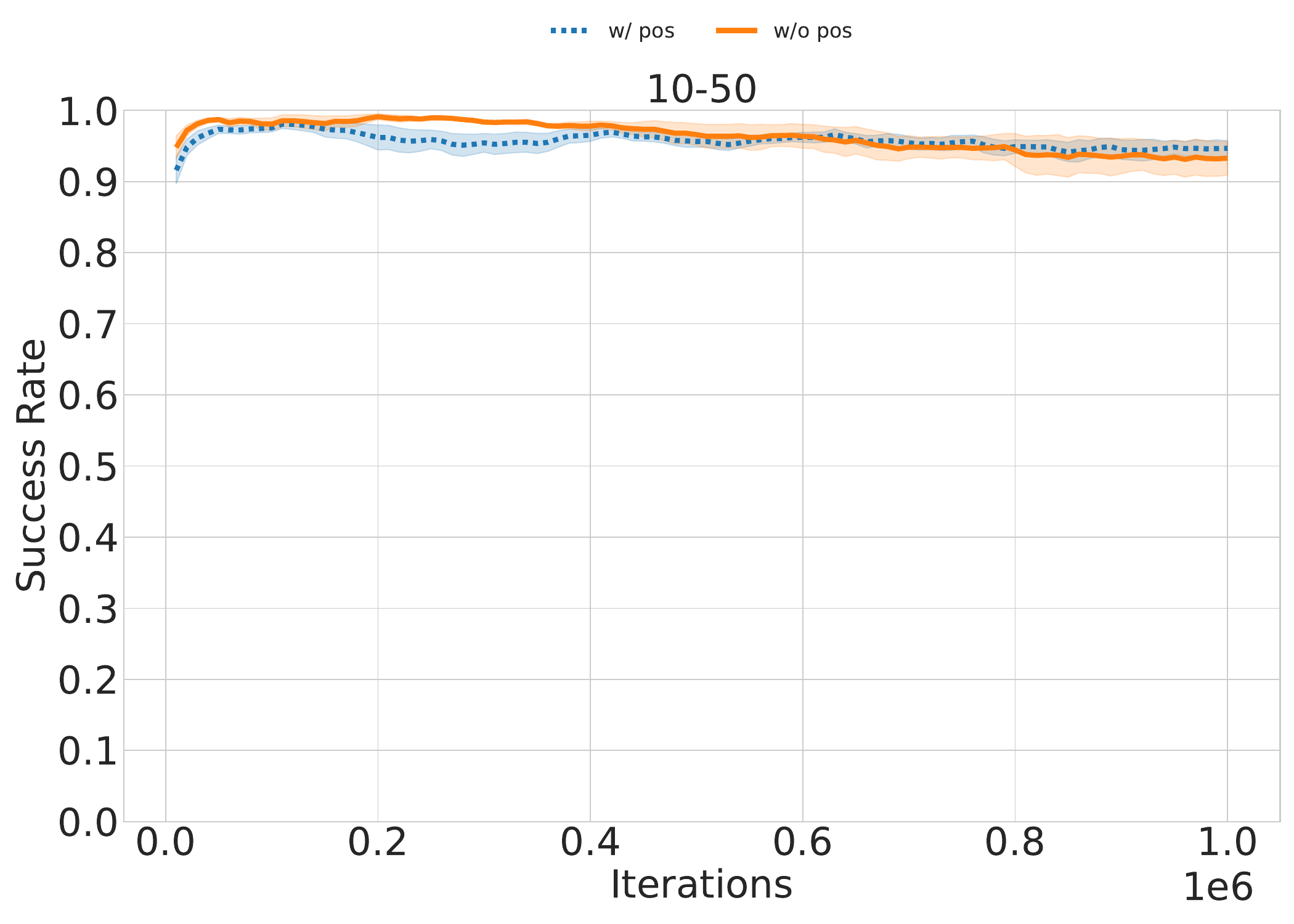}
    \end{subfigure}
    
    \begin{subfigure}[b]{0.32\textwidth}
        \centering
        \includegraphics[width=\textwidth,trim=0 0 0 81,clip]{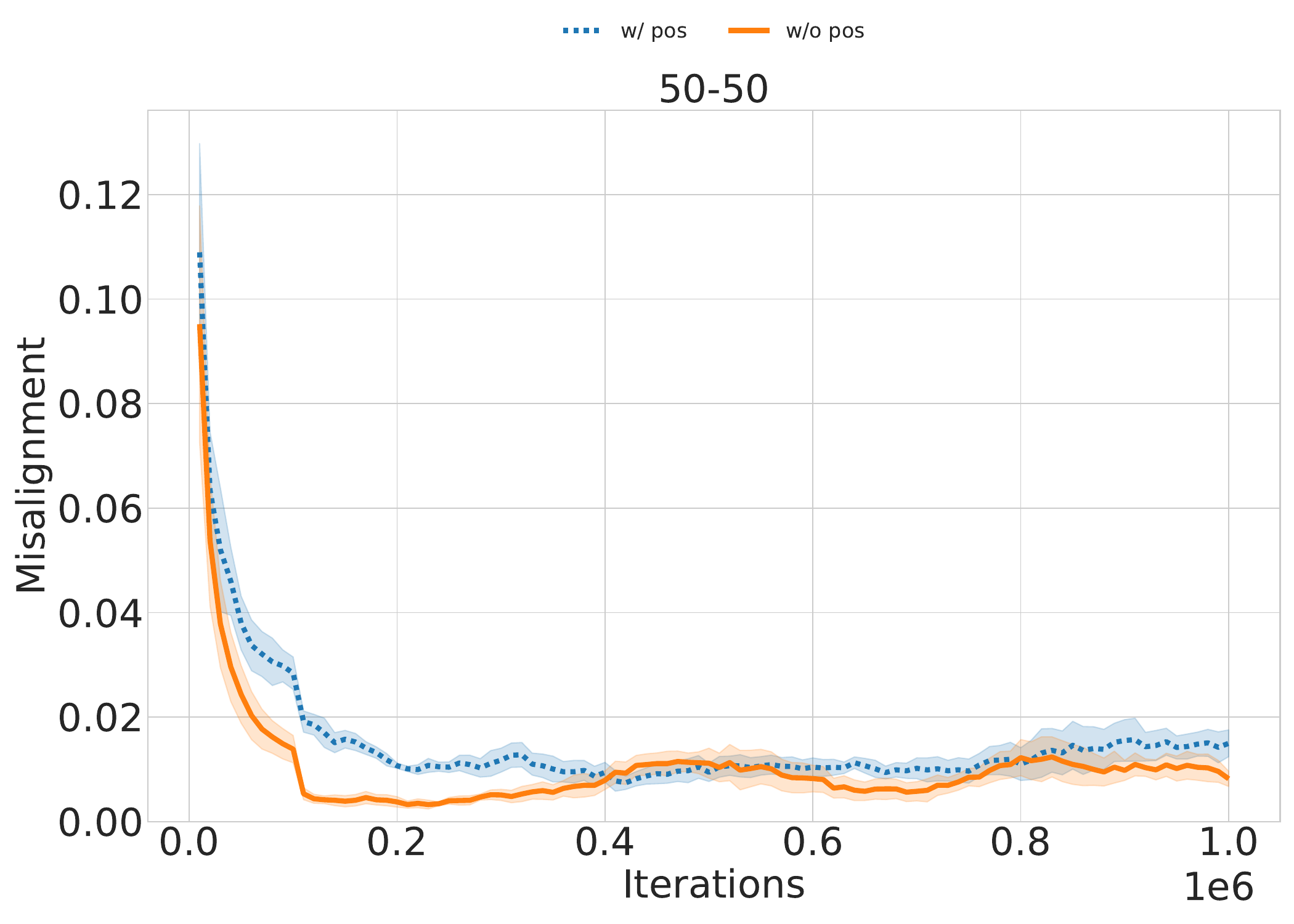}
    \end{subfigure}
    \hfill
    \begin{subfigure}[b]{0.32\textwidth}
        \centering
        \includegraphics[width=\textwidth,trim=0 0 0 81,clip]{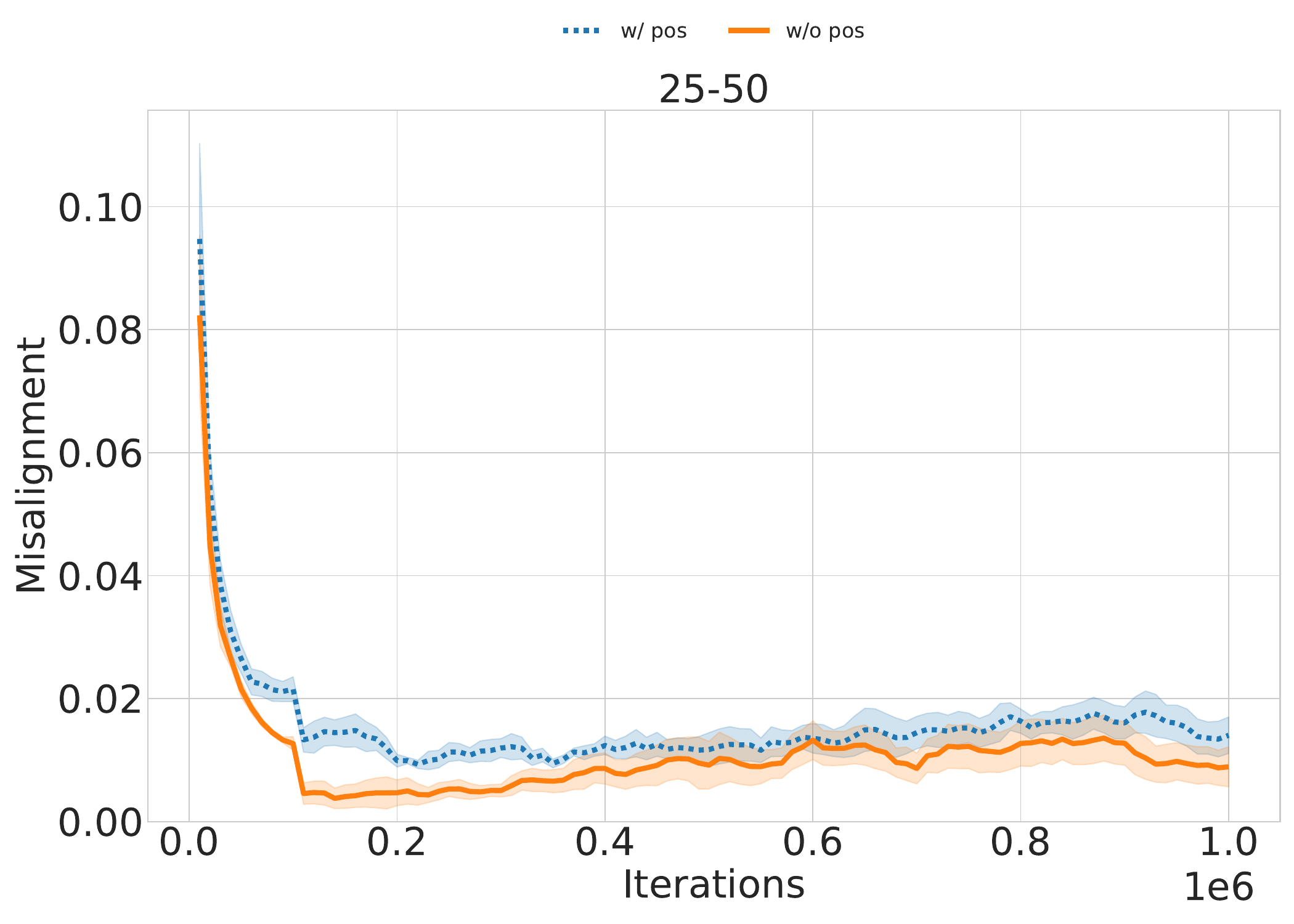}
    \end{subfigure}
    \hfill
    \begin{subfigure}[b]{0.32\textwidth}
        \centering
        \includegraphics[width=\textwidth,trim=0 0 0 81,clip]{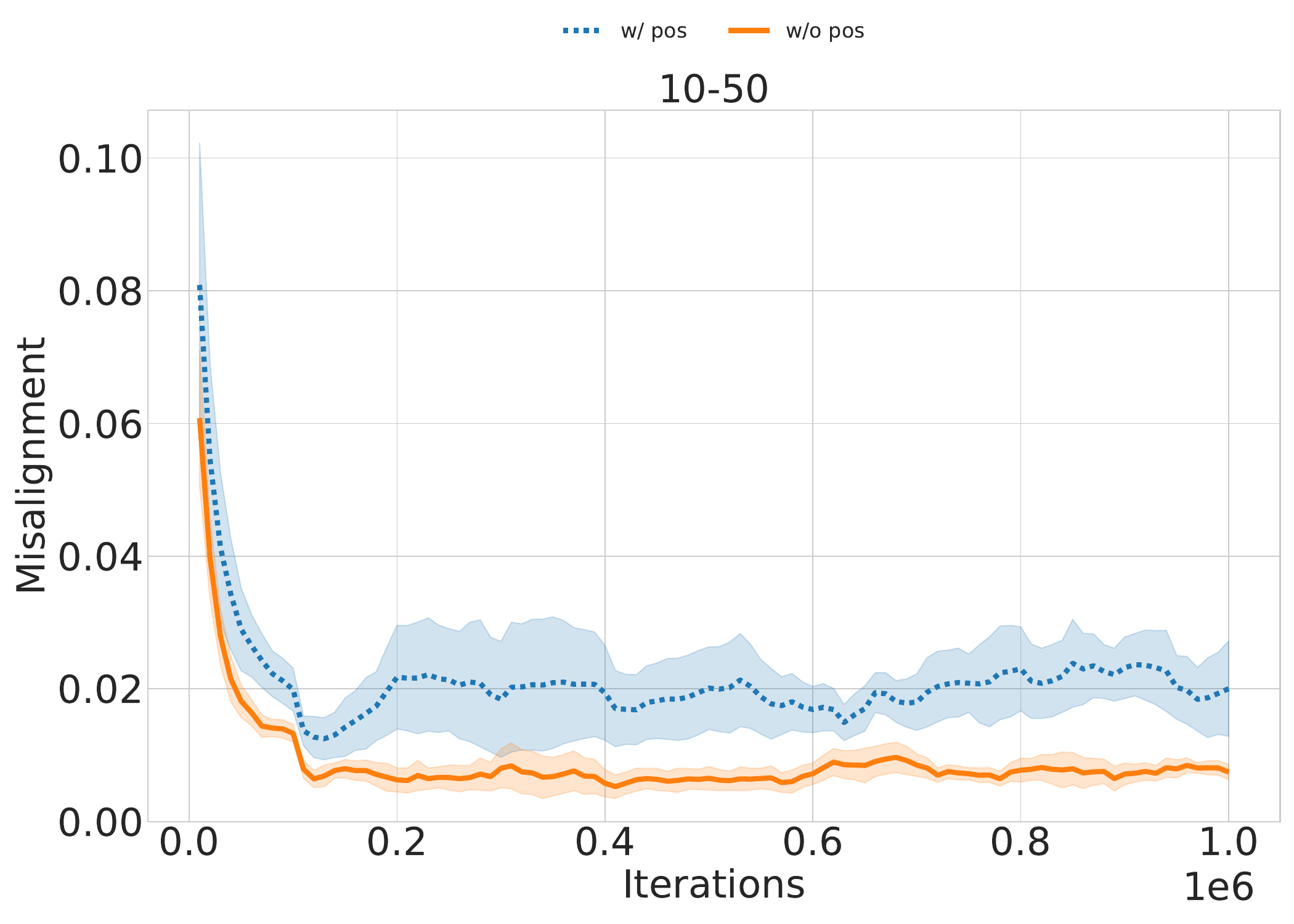}
    \end{subfigure}
    
    \caption{PathM learning curves for ReCOIL+FMR with and without positive feedback. The shaded region represents the standard error.}
    \label{fig:pathm-no-pos}
\end{figure}

\begin{figure}[h]
    \centering
    
    \begin{subfigure}[b]{\textwidth}
        \centering
        \includegraphics[width=0.8\textwidth,trim=0 665 0 0,clip]{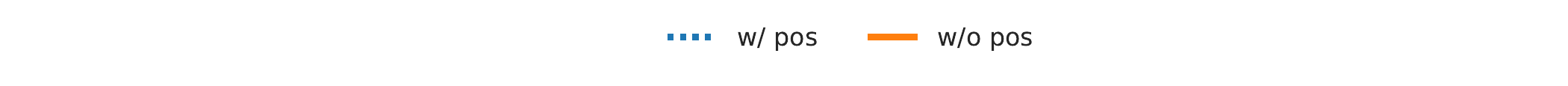}
    \end{subfigure}
    
    \vspace{0.5em}
    
    \begin{subfigure}[b]{0.32\textwidth}
        \centering
        \includegraphics[width=\textwidth,trim=0 44 0 55,clip]{images/no-pos/PathBB/success_50-50.pdf}
    \end{subfigure}
    \hfill
    \begin{subfigure}[b]{0.32\textwidth}
        \centering
        \includegraphics[width=\textwidth,trim=0 44 0 55,clip]{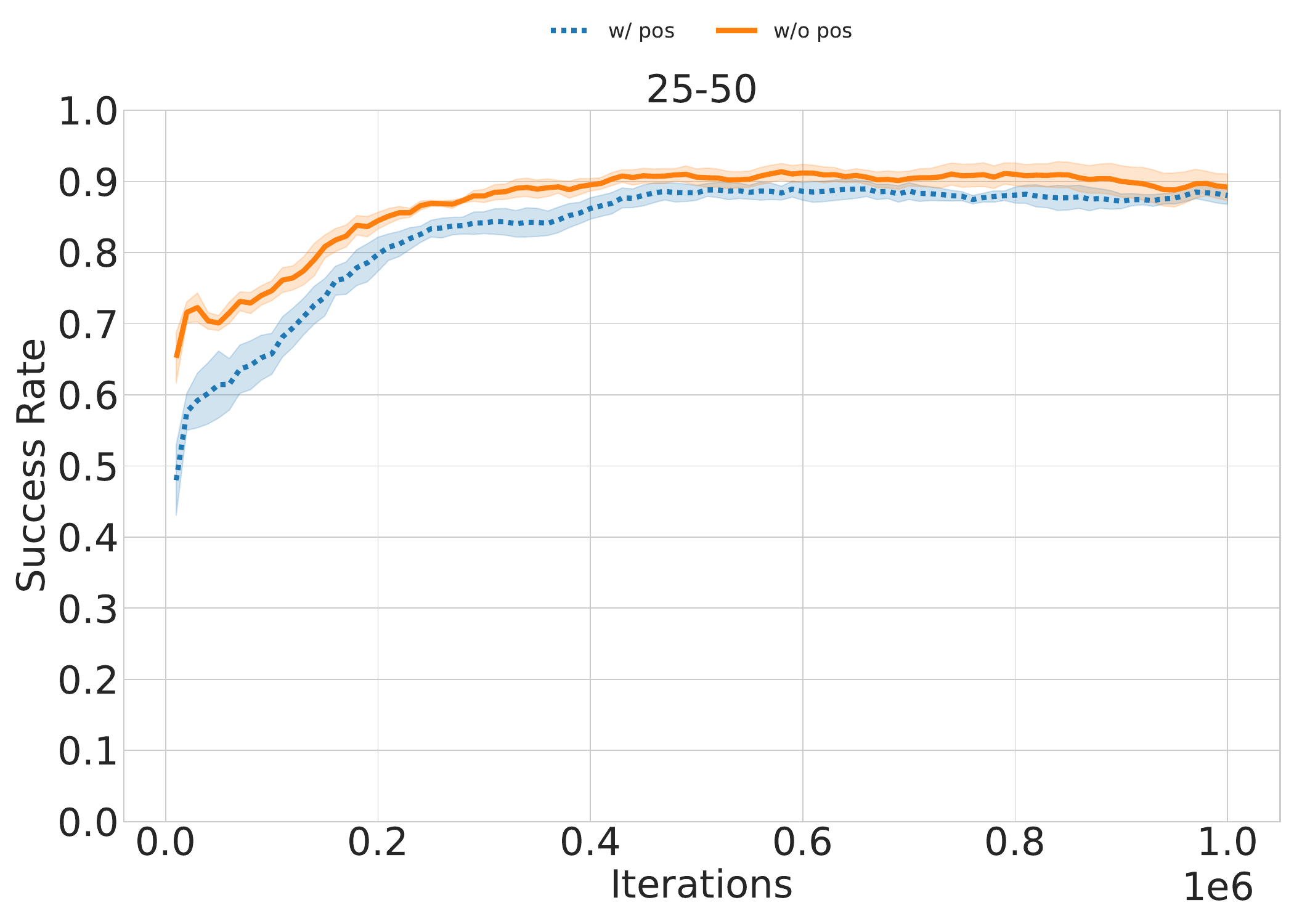}
    \end{subfigure}
    \hfill
    \begin{subfigure}[b]{0.32\textwidth}
        \centering
        \includegraphics[width=\textwidth,trim=0 44 0 55,clip]{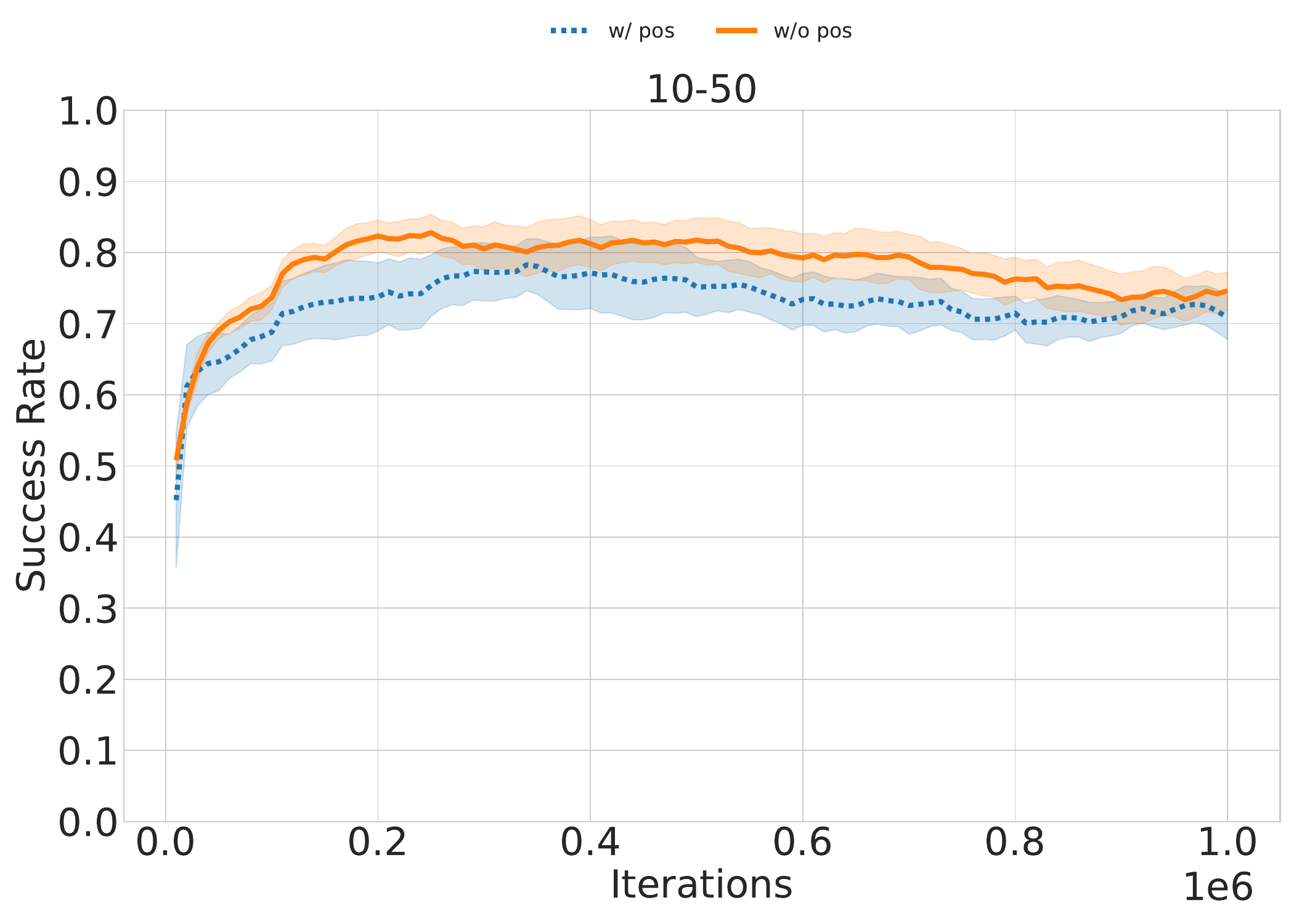}
    \end{subfigure}
    
    \begin{subfigure}[b]{0.32\textwidth}
        \centering
        \includegraphics[width=\textwidth,trim=0 0 0 81,clip]{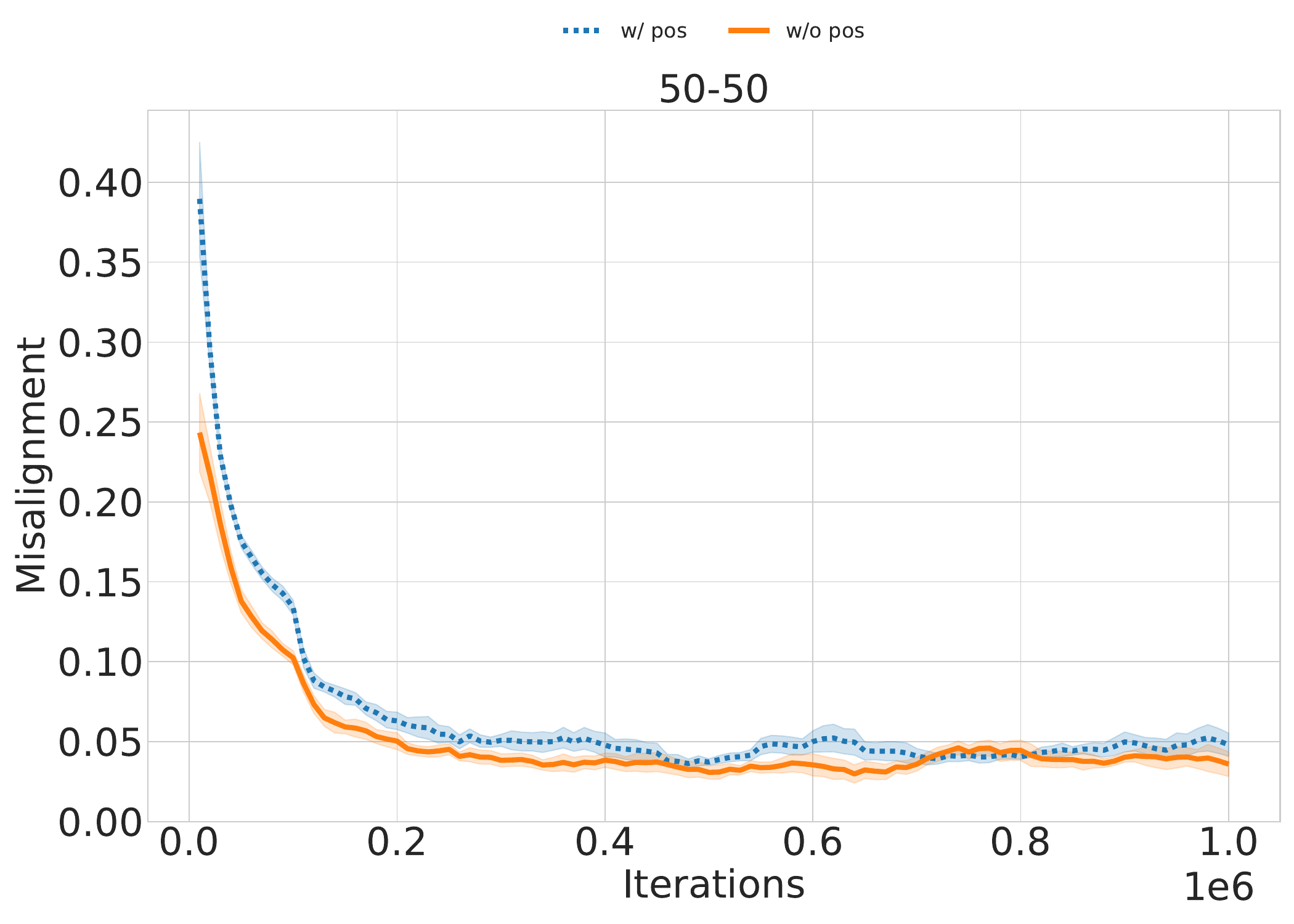}
    \end{subfigure}
    \hfill
    \begin{subfigure}[b]{0.32\textwidth}
        \centering
        \includegraphics[width=\textwidth,trim=0 0 0 81,clip]{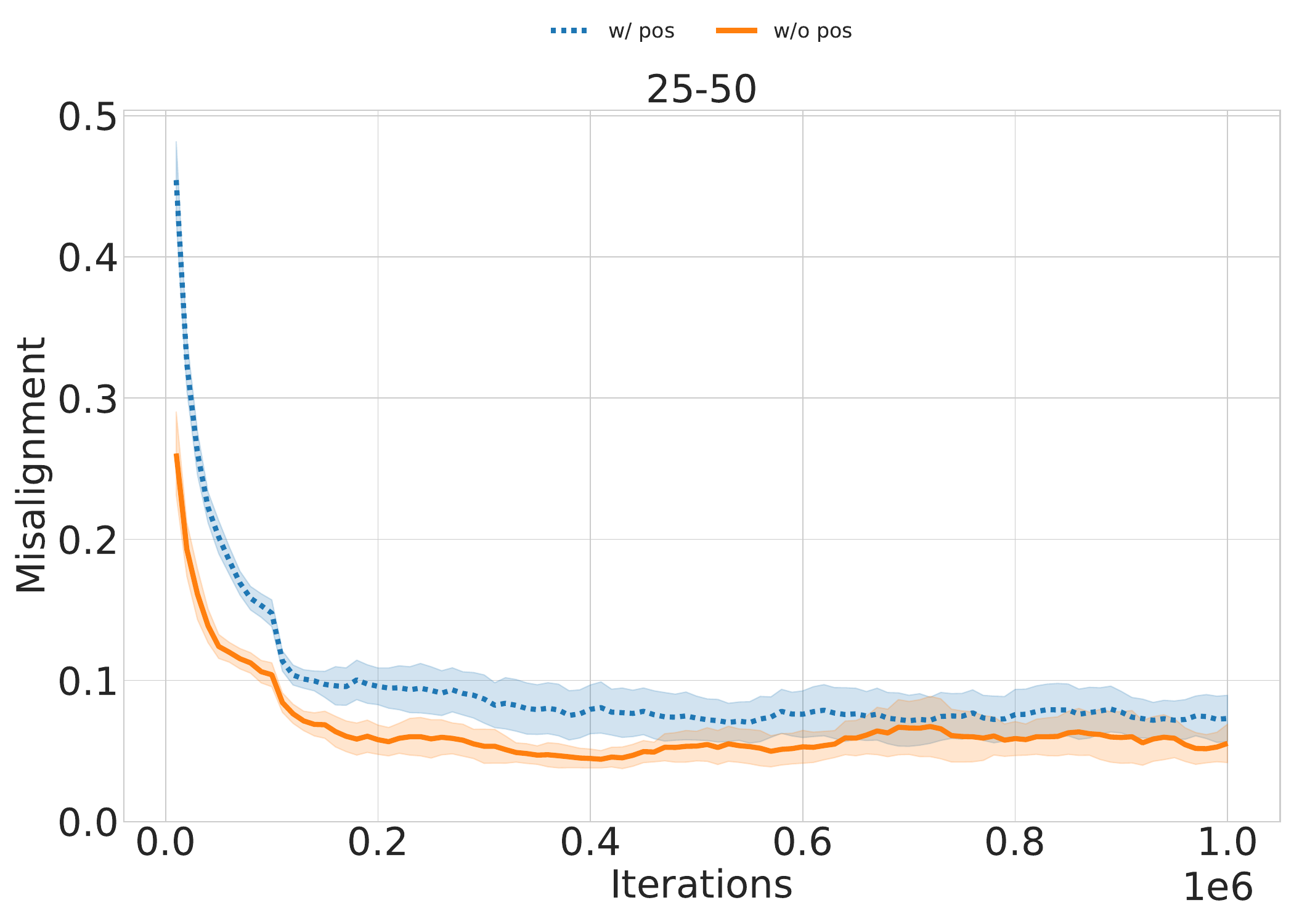}
    \end{subfigure}
    \hfill
    \begin{subfigure}[b]{0.32\textwidth}
        \centering
        \includegraphics[width=\textwidth,trim=0 0 0 81,clip]{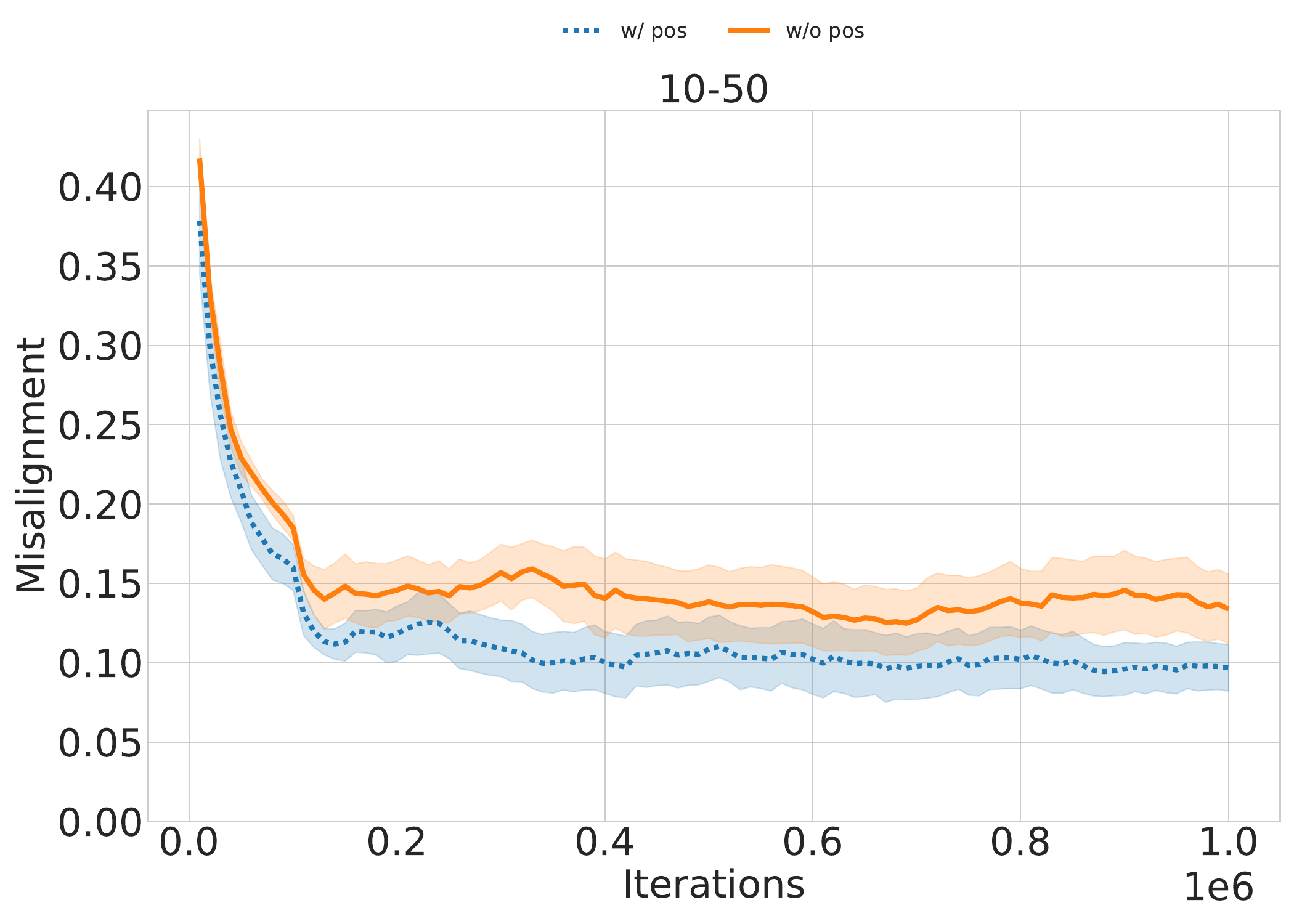}
    \end{subfigure}
    
    \caption{PathBB learning curves for ReCOIL+FMR with and without positive feedback. The shaded region represents the standard error.}
    \label{fig:pathbb-no-pos}
\end{figure}

%% file: appendix/imp-demos.tex
Subsequent results are presented for testing the usefulness of imperfect demonstrations across all tasks. Results compare ReCOIL+FMR trained using a 10–50 data ratio against BC and IQL when trained on 10 randomly selected expert demonstrations with oversampling.

\begin{table}[h]
\centering
\caption{PathM performance comparison between ReCOIL+FMR and BC/IQL trained using only expert demonstrations. Results show mean ± std for the last 10 evaluations, over 5 seeds.}
\label{tab:pathm-imp-demos}
\begin{tabular}{c|cc}
\toprule
Algorithm & Suc. & Mis. \\
\midrule
\multirow[t]{1}{*}{BC Expert} & 0.710 ± 0.45 & 0.167 ± 0.31 \\
\multirow[t]{1}{*}{IQL Expert} & 0.801 ± 0.40 & 0.088 ± 0.24 \\
\multirow[t]{1}{*}{ReCOIL+FMR} & 0.947 ± 0.22 & 0.020 ± 0.09 \\
\bottomrule
\end{tabular}
\end{table}

\begin{table}[h]
\centering
\caption{PathBB performance comparison between ReCOIL+FMR and BC/IQL trained using only expert demonstrations. Results show mean ± std for the last 10 evaluations, over 5 seeds.}
\label{tab:pathbb-imp-demos}
\begin{tabular}{c|cc}
\toprule
Algorithm & Suc. & Mis. \\
\midrule
\multirow[t]{1}{*}{BC Expert} & 0.854 ± 0.35 & 0.068 ± 0.20 \\
\multirow[t]{1}{*}{IQL Expert} & 0.883 ± 0.32 & 0.076 ± 0.22 \\
\multirow[t]{1}{*}{ReCOIL+FMR} & 0.708 ± 0.45 & 0.097 ± 0.21 \\
\bottomrule
\end{tabular}
\end{table}

\begin{table}[h]
\centering
\caption{SlowSwim performance comparison between ReCOIL+FMR and BC/IQL trained using only expert demonstrations. Results show mean ± std for the last 10 evaluations, over 5 seeds.}
\label{tab:swim-imp-demos}
\begin{tabular}{ccc}
\toprule
Algorithm & Return & Mis. \\
\midrule
\multirow[t]{1}{*}{BC Expert} &  0.9991 ± 0.034 & 0.0004 ± 0.001 \\
\multirow[t]{1}{*}{IQL Expert} &  1.0017 ± 0.008 & 0.0003 ± 0.001 \\
\multirow[t]{1}{*}{ReCOIL+FMR} &  1.0018 ± 0.008 & 0.0009 ± 0.003 \\
\bottomrule
\end{tabular}
\end{table}
\begin{table}[h]
\centering
\caption{SlowHop performance comparison between ReCOIL+FMR and BC/IQL trained using only expert demonstrations. Results show mean ± std for the last 10 evaluations, over 5 seeds.}
\label{tab:hop-imp-demos}
\begin{tabular}{ccc}
\toprule
Algorithm & Return & Mis. \\
\midrule
\multirow[t]{1}{*}{BC Expert} & 0.974 ± 0.125 & 0.014 ± 0.006 \\
\multirow[t]{1}{*}{IQL Expert} & 0.973 ± 0.130 & 0.016 ± 0.007 \\
\multirow[t]{1}{*}{ReCOIL+FMR} & 0.986 ± 0.225 & 0.066 ± 0.168 \\
\bottomrule
\end{tabular}
\end{table}

\begin{table}[h]
\centering
\caption{SlowWalk performance comparison between ReCOIL+FMR and BC/IQL trained using only expert demonstrations. Results show mean ± std for the last 10 evaluations, over 5 seeds.}
\label{tab:walk-imp-demos}
\begin{tabular}{ccc}
\toprule
Algorithm  & Return & Mis. \\
\midrule
\multirow[t]{1}{*}{BC Expert} & 0.944 ± 0.182 & 0.071 ± 0.086 \\
\multirow[t]{1}{*}{IQL Expert} & 0.769 ± 0.315 & 0.136 ± 0.110 \\
\multirow[t]{1}{*}{ReCOIL+FMR} & 0.690 ± 0.323 & 0.165 ± 0.162 \\
\bottomrule
\end{tabular}
\end{table}

\begin{figure}[h]
    \centering
    
    \begin{subfigure}[b]{\textwidth}
        \centering
        \includegraphics[width=0.8\textwidth,trim=0 680 0 0,clip]{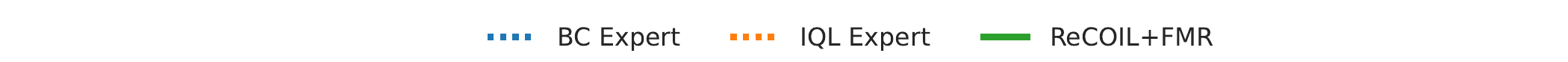}
    \end{subfigure}
    
    \vspace{0.5em}
    
    \begin{subfigure}[b]{0.48\textwidth}
        \centering
        \includegraphics[width=\textwidth,trim=0 0 0 55,clip]{images/imp-demos/PathM/success.pdf}
    \end{subfigure}
    \hfill
    \begin{subfigure}[b]{0.48\textwidth}
        \centering
        \includegraphics[width=\textwidth,trim=0 0 0 55,clip]{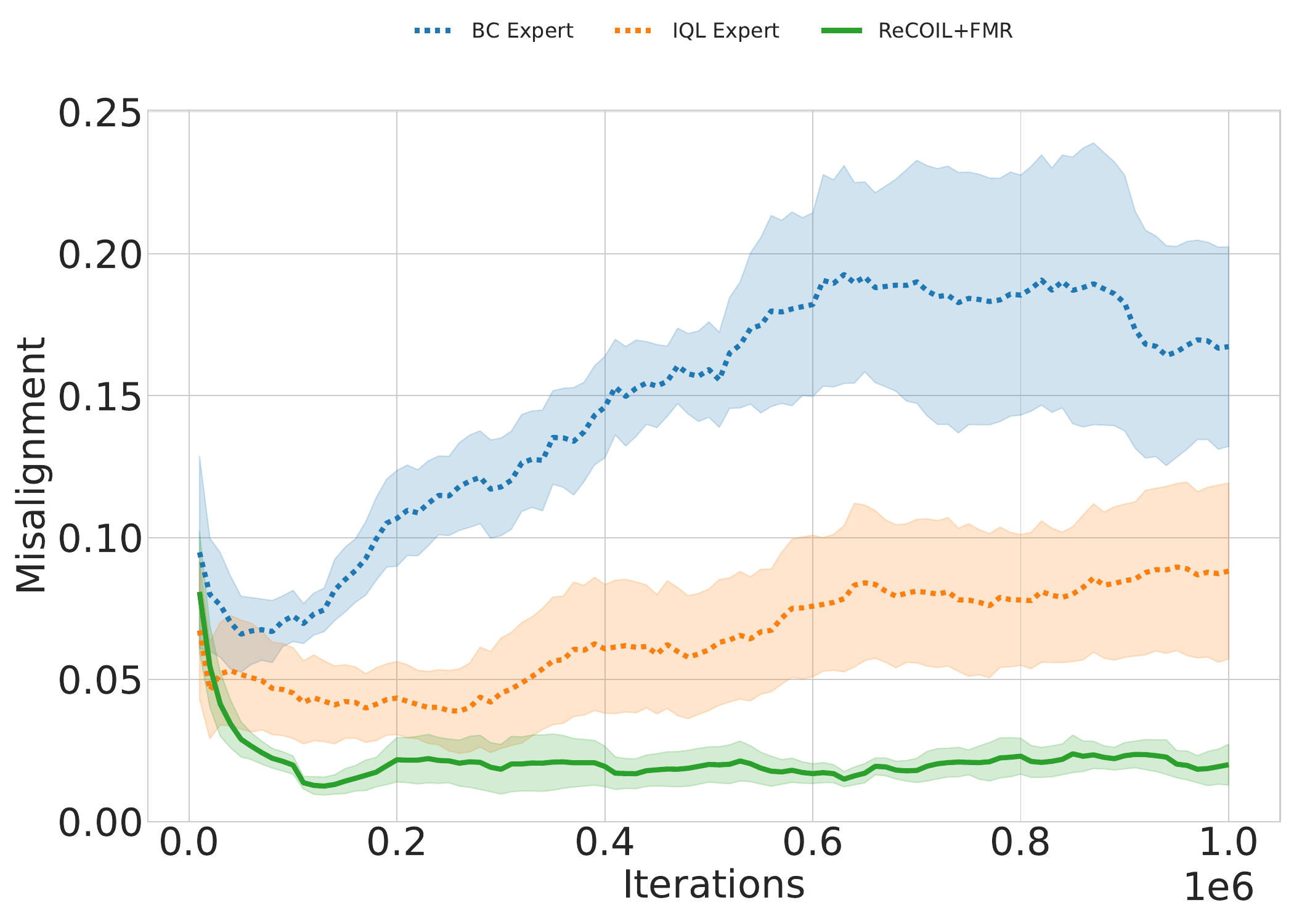}
    \end{subfigure}
    
    \caption{PathM learning curves for ReCOIL+FMR and BC/IQL trained using only expert demonstrations. The shaded region represents the standard error.}
    \label{fig:pathm-imp-demos}
\end{figure}

\begin{figure}[h]
    \centering
    
    \begin{subfigure}[b]{\textwidth}
        \centering
        \includegraphics[width=0.8\textwidth,trim=0 680 0 0,clip]{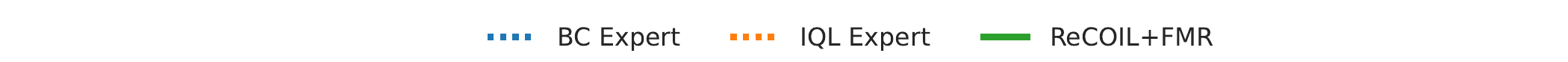}
    \end{subfigure}
    
    \vspace{0.5em}
    
    \begin{subfigure}[b]{0.48\textwidth}
        \centering
        \includegraphics[width=\textwidth,trim=0 0 0 55,clip]{images/imp-demos/PathBB/success.pdf}
    \end{subfigure}
    \hfill
    \begin{subfigure}[b]{0.48\textwidth}
        \centering
        \includegraphics[width=\textwidth,trim=0 0 0 55,clip]{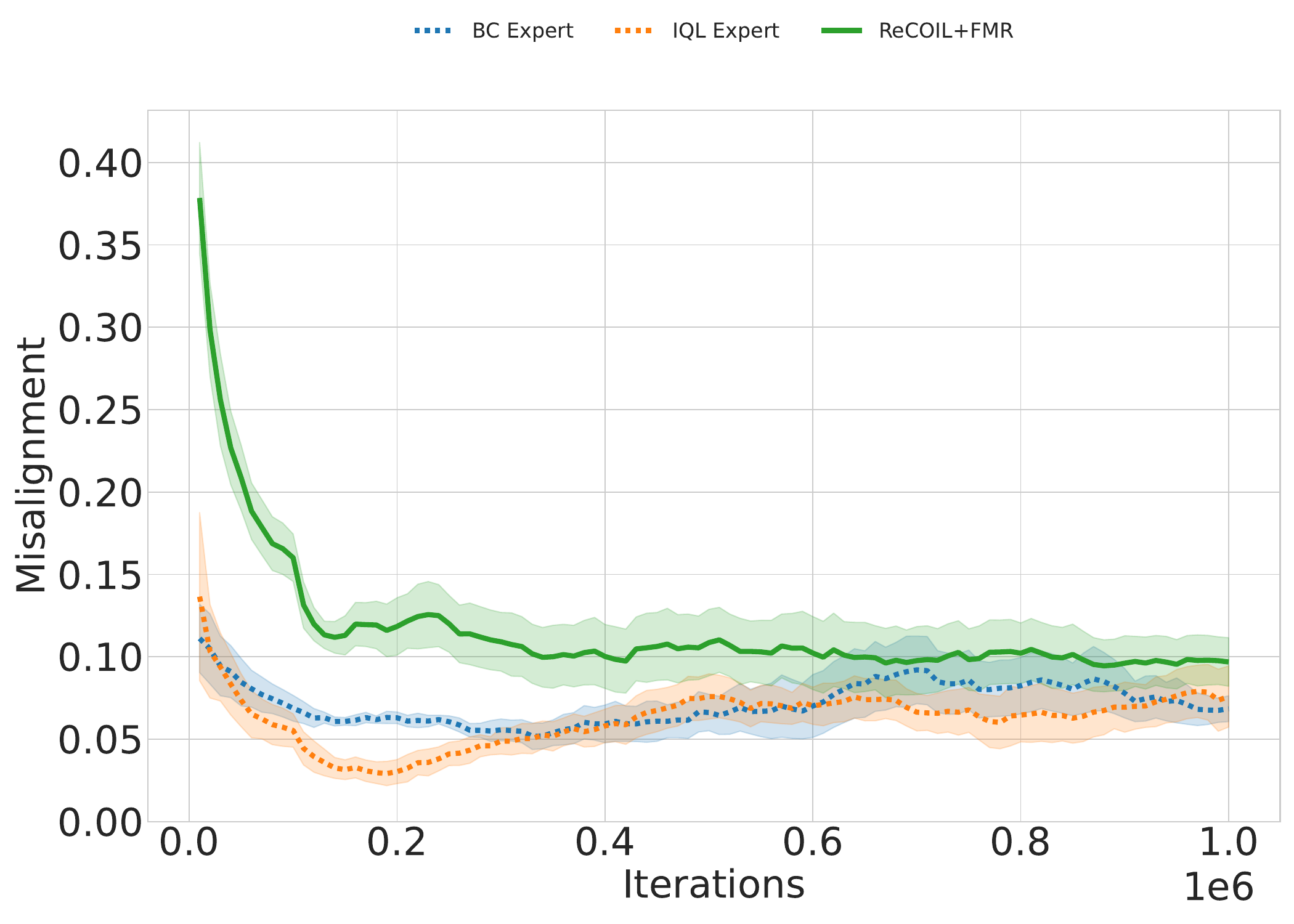}
    \end{subfigure}
    
    \caption{PathBB learning curves for ReCOIL+FMR and BC/IQL trained using only expert demonstrations. The shaded region represents the standard error.}
    \label{fig:pathbb-imp-demos}
\end{figure}

\begin{figure}[h]
    \centering
    
    \begin{subfigure}[b]{\textwidth}
        \centering
        \includegraphics[width=0.8\textwidth,trim=0 680 0 0,clip]{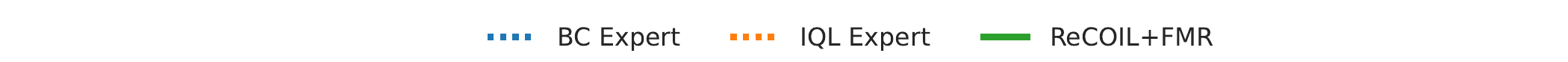}
    \end{subfigure}
    
    \vspace{0.5em}
    
    \begin{subfigure}[b]{0.48\textwidth}
        \centering
        \includegraphics[width=\textwidth,trim=0 0 0 55,clip]{images/imp-demos/SlowSwim/return.pdf}
    \end{subfigure}
    \hfill
    \begin{subfigure}[b]{0.48\textwidth}
        \centering
        \includegraphics[width=\textwidth,trim=0 0 0 55,clip]{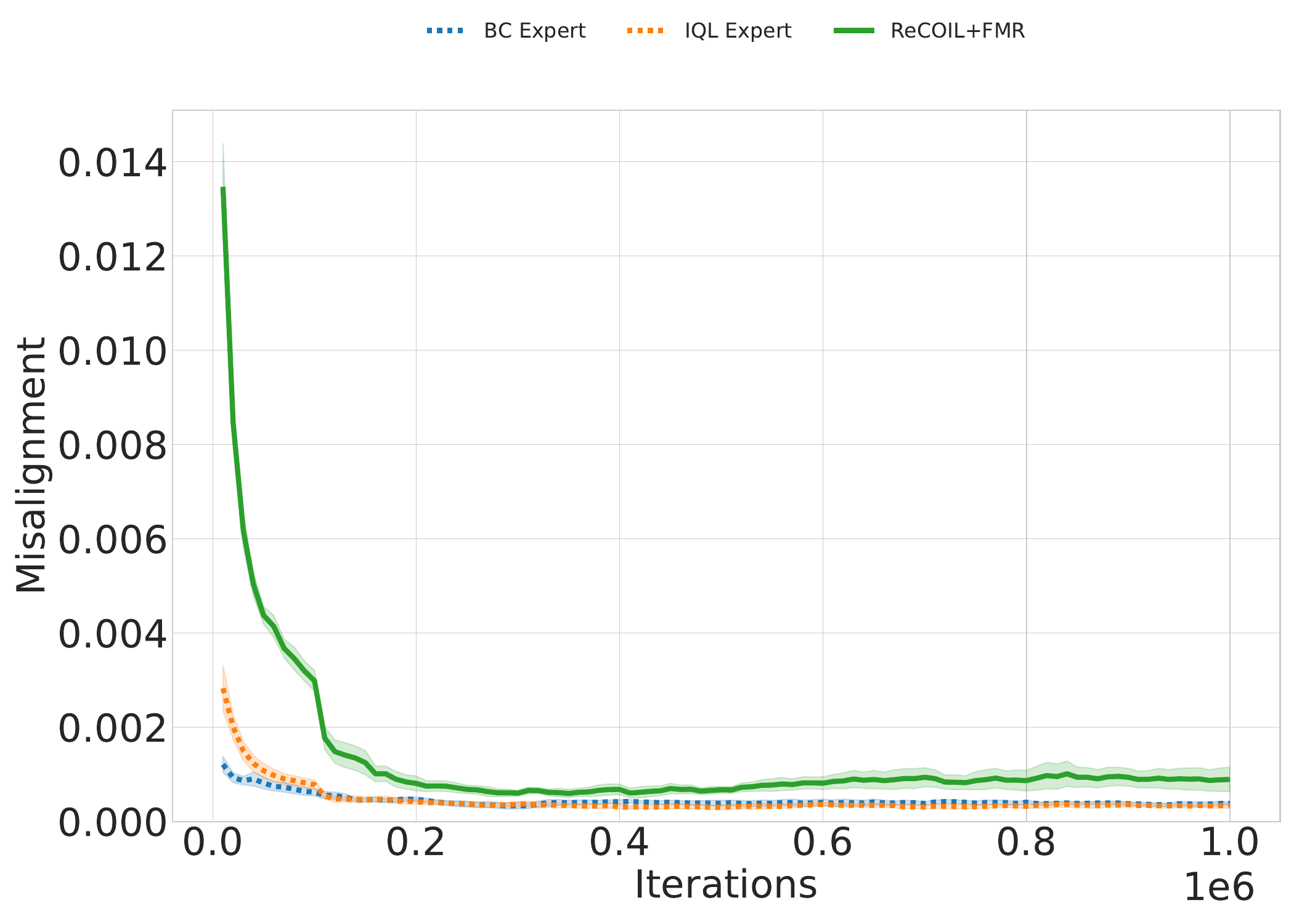}
    \end{subfigure}
    
    \caption{SlowSwim learning curves for ReCOIL+FMR and BC/IQL trained using only expert demonstrations. The shaded region represents the standard error.}
    \label{fig:swim-imp-demos}
\end{figure}

\begin{figure}[h]
    \centering
    
    \begin{subfigure}[b]{\textwidth}
        \centering
        \includegraphics[width=0.8\textwidth,trim=0 680 0 0,clip]{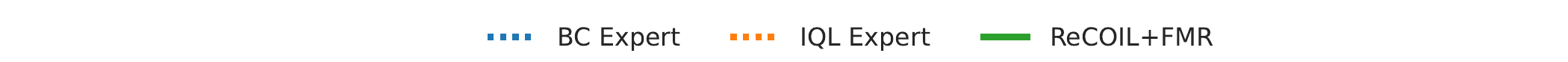}
    \end{subfigure}
    
    \vspace{0.5em}
    
    \begin{subfigure}[b]{0.48\textwidth}
        \centering
        \includegraphics[width=\textwidth,trim=0 0 0 55,clip]{images/imp-demos/SlowHop/return.pdf}
    \end{subfigure}
    \hfill
    \begin{subfigure}[b]{0.48\textwidth}
        \centering
        \includegraphics[width=\textwidth,trim=0 0 0 55,clip]{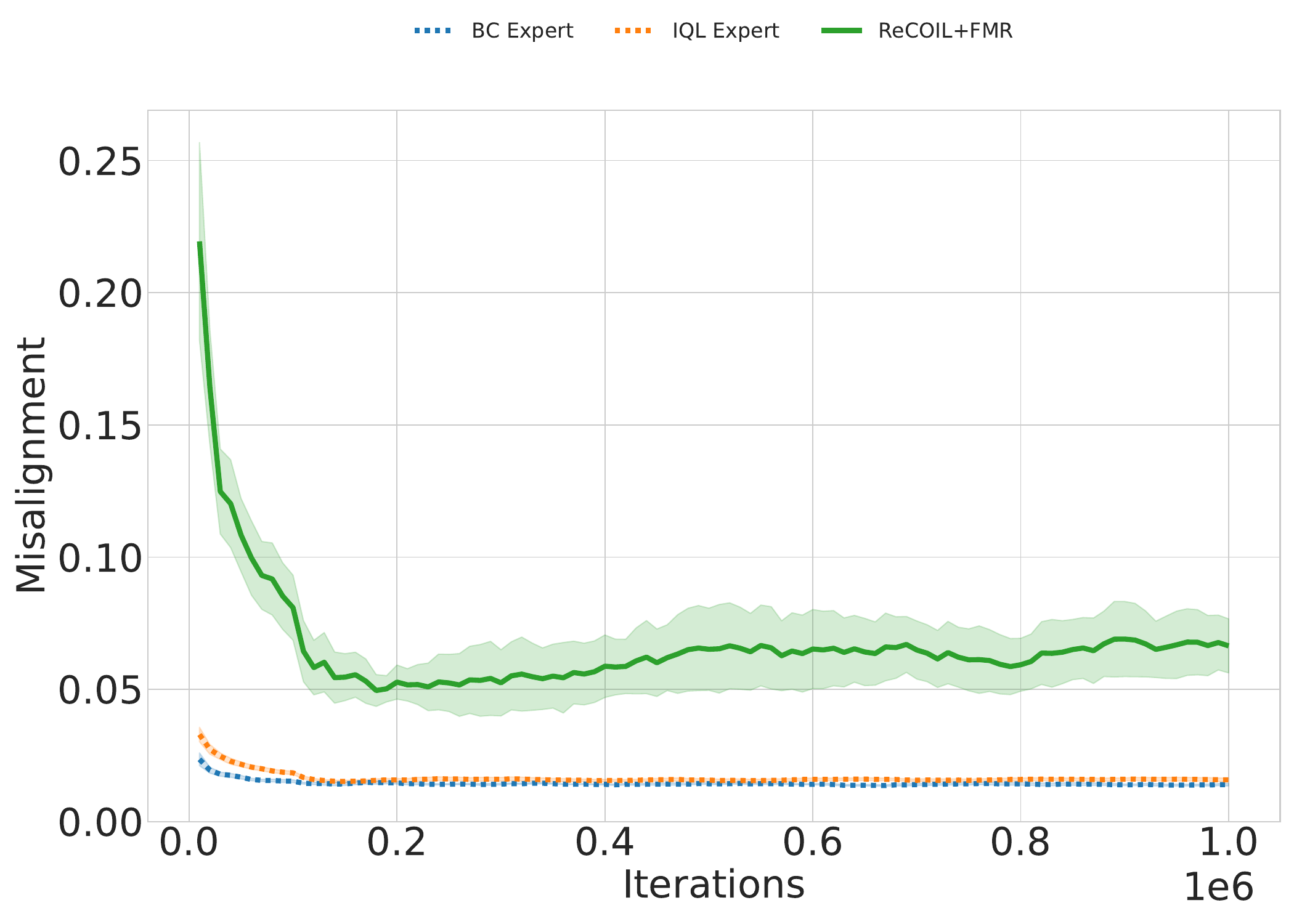}
    \end{subfigure}
    
    \caption{SlowHop learning curves for ReCOIL+FMR and BC/IQL trained using only expert demonstrations. The shaded region represents the standard error.}
    \label{fig:hop-imp-demos}
\end{figure}

\begin{figure}[h]
    \centering
    
    \begin{subfigure}[b]{\textwidth}
        \centering
        \includegraphics[width=0.8\textwidth,trim=0 680 0 0,clip]{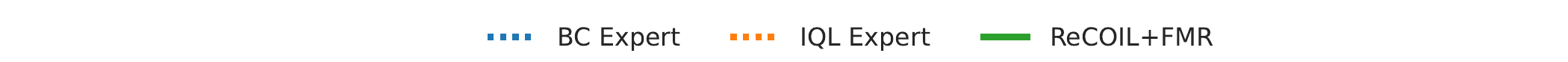}
    \end{subfigure}
    
    \vspace{0.5em}
    
    \begin{subfigure}[b]{0.48\textwidth}
        \centering
        \includegraphics[width=\textwidth,trim=0 0 0 55,clip]{images/imp-demos/SlowWalk/return.pdf}
    \end{subfigure}
    \hfill
    \begin{subfigure}[b]{0.48\textwidth}
        \centering
        \includegraphics[width=\textwidth,trim=0 0 0 55,clip]{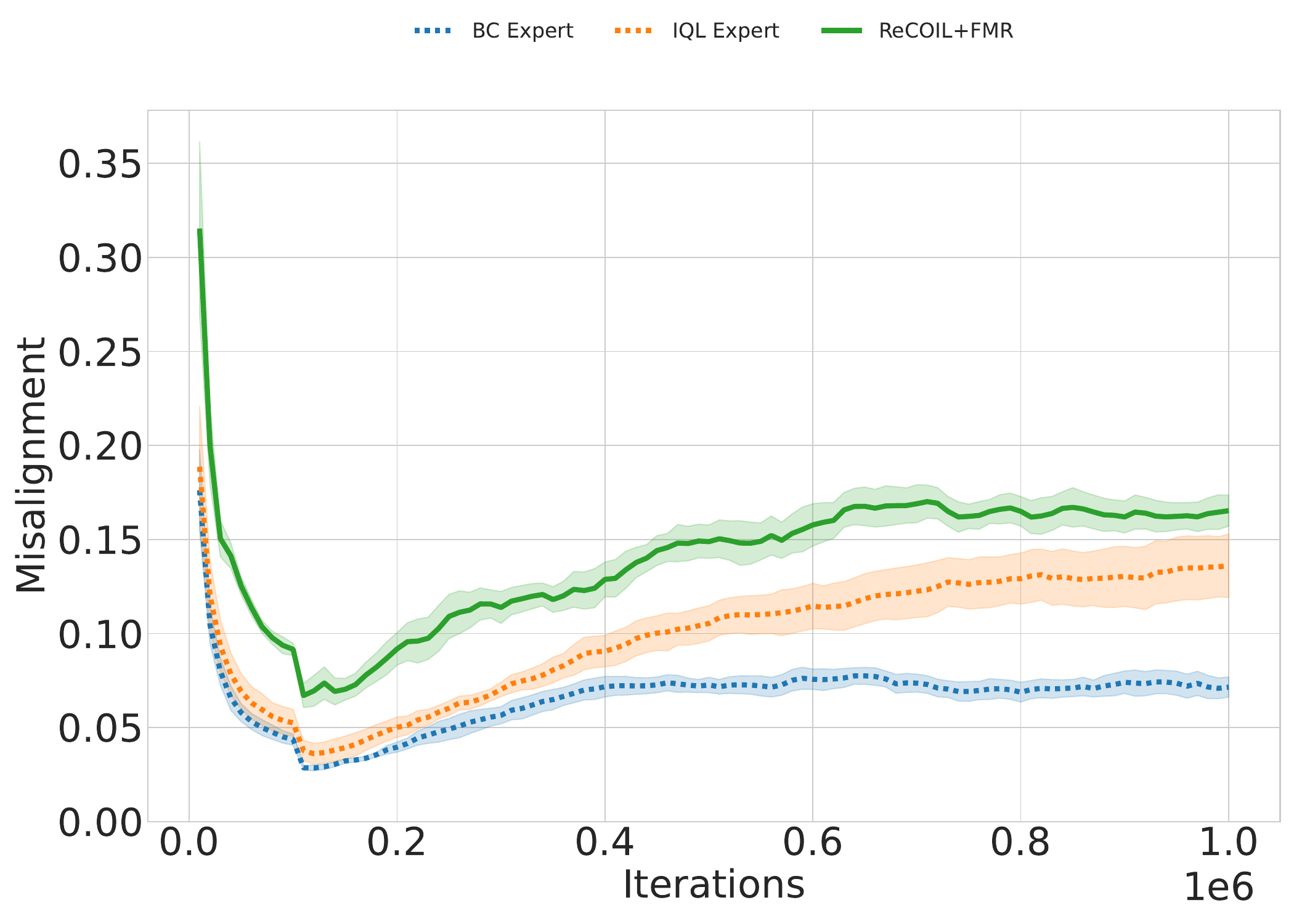}
    \end{subfigure}
    
    \caption{SlowWalk learning curves for ReCOIL+FMR and BC/IQL trained using only expert demonstrations. The shaded region represents the standard error.}
    \label{fig:walk-imp-demos}
\end{figure}

%% file: appendix/fb-scale.tex
The following results test the feedback scale using a 50-50 data ratio. Results are reported for ReCOIL+FMR where 0\% corresponds to ReCOIL, 20\% corresponds to 10 demonstrations with feedback, 50\% corresponds to 25 demonstrations with feedback, and 100\% corresponds to all demonstrations having feedback. Demonstrations are randomly selected to receive feedback. When a demonstration is selected, feedback for the entire demonstration is used. All remaining demonstrations do not have feedback.

\begin{figure}[h]
    \centering
    
    \begin{subfigure}[b]{\textwidth}
        \centering
        \includegraphics[width=0.8\textwidth,trim=0 680 0 0,clip]{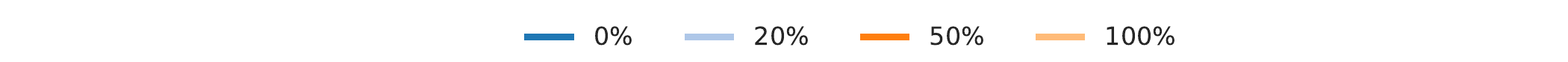}
    \end{subfigure}
    
    \vspace{0.5em}
    
    \begin{subfigure}[b]{0.48\textwidth}
        \centering
        \includegraphics[width=\textwidth,trim=0 0 0 55,clip]{images/fb-reduce/PathM/success_50-50.pdf}
    \end{subfigure}
    \hfill
    \begin{subfigure}[b]{0.48\textwidth}
        \centering
        \includegraphics[width=\textwidth,trim=0 0 0 55,clip]{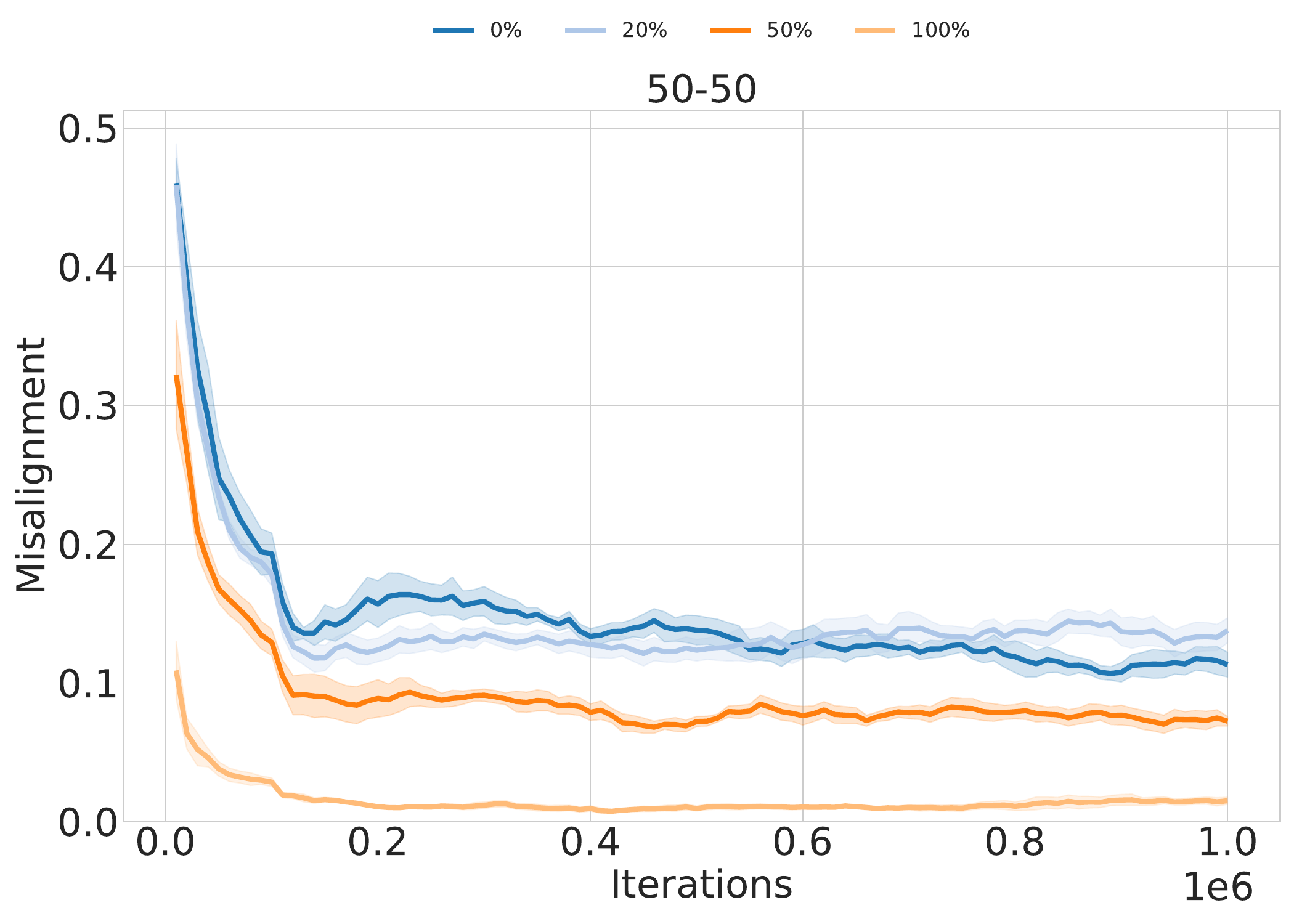}
    \end{subfigure}
    
    \caption{PathM learning curves for ReCOIL+FMR with reduced feedback. The shaded region represents the standard error.}
    \label{fig:pathm-scale}
\end{figure}

\begin{figure}[h]
    \centering
    
    \begin{subfigure}[b]{\textwidth}
        \centering
        \includegraphics[width=0.8\textwidth,trim=0 680 0 0,clip]{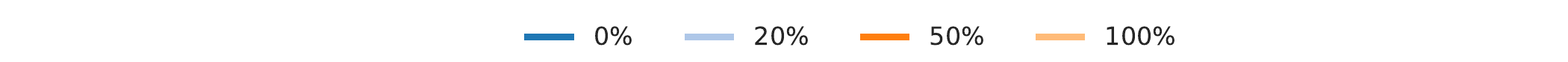}
    \end{subfigure}
    
    \vspace{0.5em}
    
    \begin{subfigure}[b]{0.48\textwidth}
        \centering
        \includegraphics[width=\textwidth,trim=0 0 0 55,clip]{images/fb-reduce/PathBB/success_50-50.pdf}
    \end{subfigure}
    \hfill
    \begin{subfigure}[b]{0.48\textwidth}
        \centering
        \includegraphics[width=\textwidth,trim=0 0 0 55,clip]{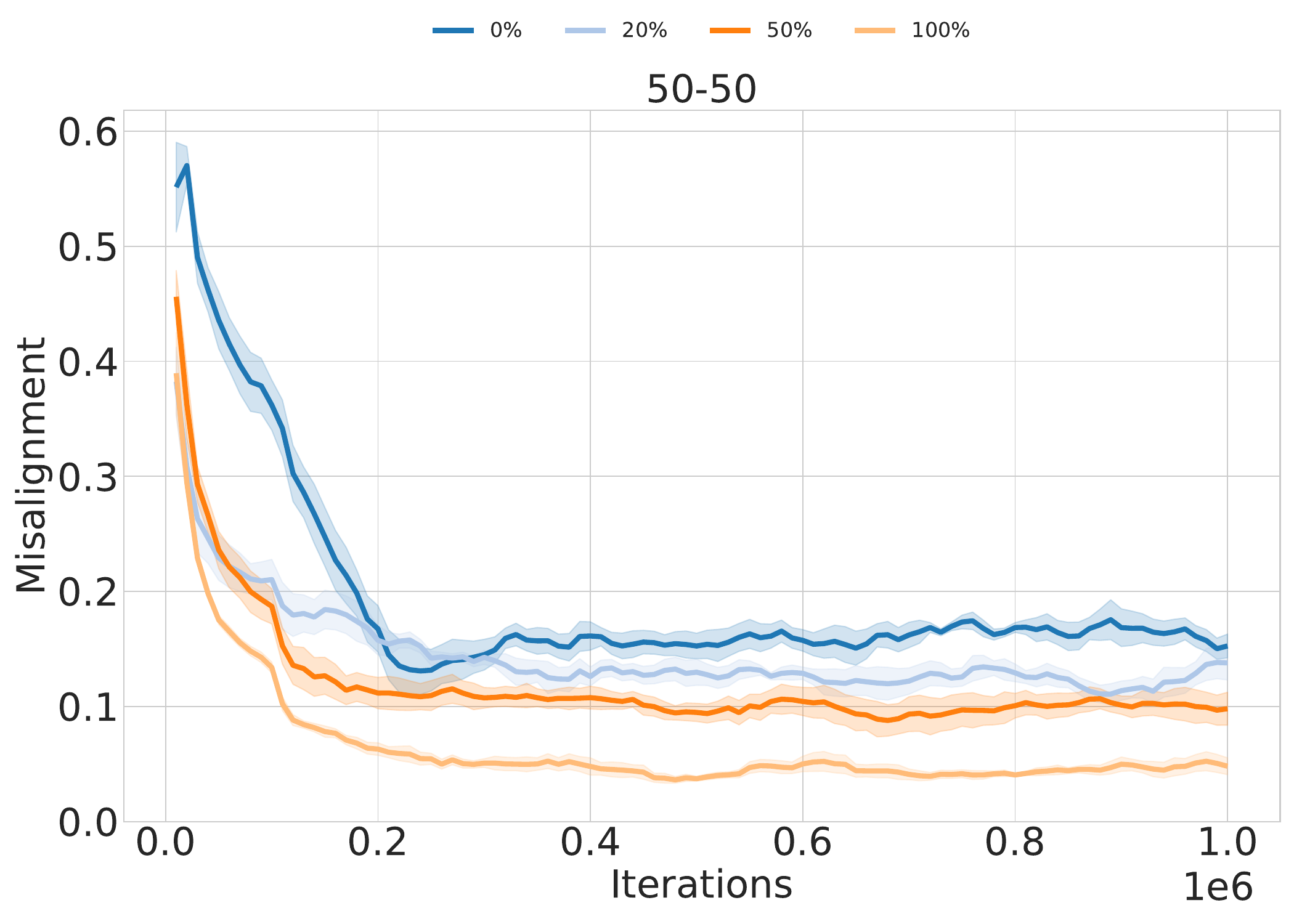}
    \end{subfigure}
    
    \caption{PathBB learning curves for ReCOIL+FMR with reduced feedback.The shaded region represents the standard error.}
    \label{fig:pathbb-scale}
\end{figure}

\begin{figure}[h]
    \centering
    
    \begin{subfigure}[b]{\textwidth}
        \centering
        \includegraphics[width=0.8\textwidth,trim=0 680 0 0,clip]{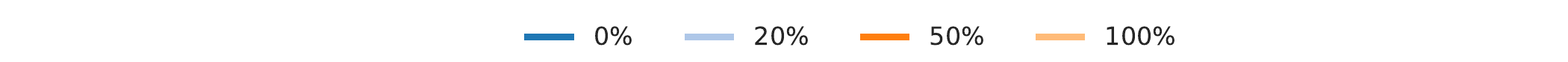}
    \end{subfigure}
    
    \vspace{0.5em}
    
    \begin{subfigure}[b]{0.48\textwidth}
        \centering
        \includegraphics[width=\textwidth,trim=0 0 0 55,clip]{images/fb-reduce/SlowSwim/return_50-50.pdf}
    \end{subfigure}
    \hfill
    \begin{subfigure}[b]{0.48\textwidth}
        \centering
        \includegraphics[width=\textwidth,trim=0 0 0 55,clip]{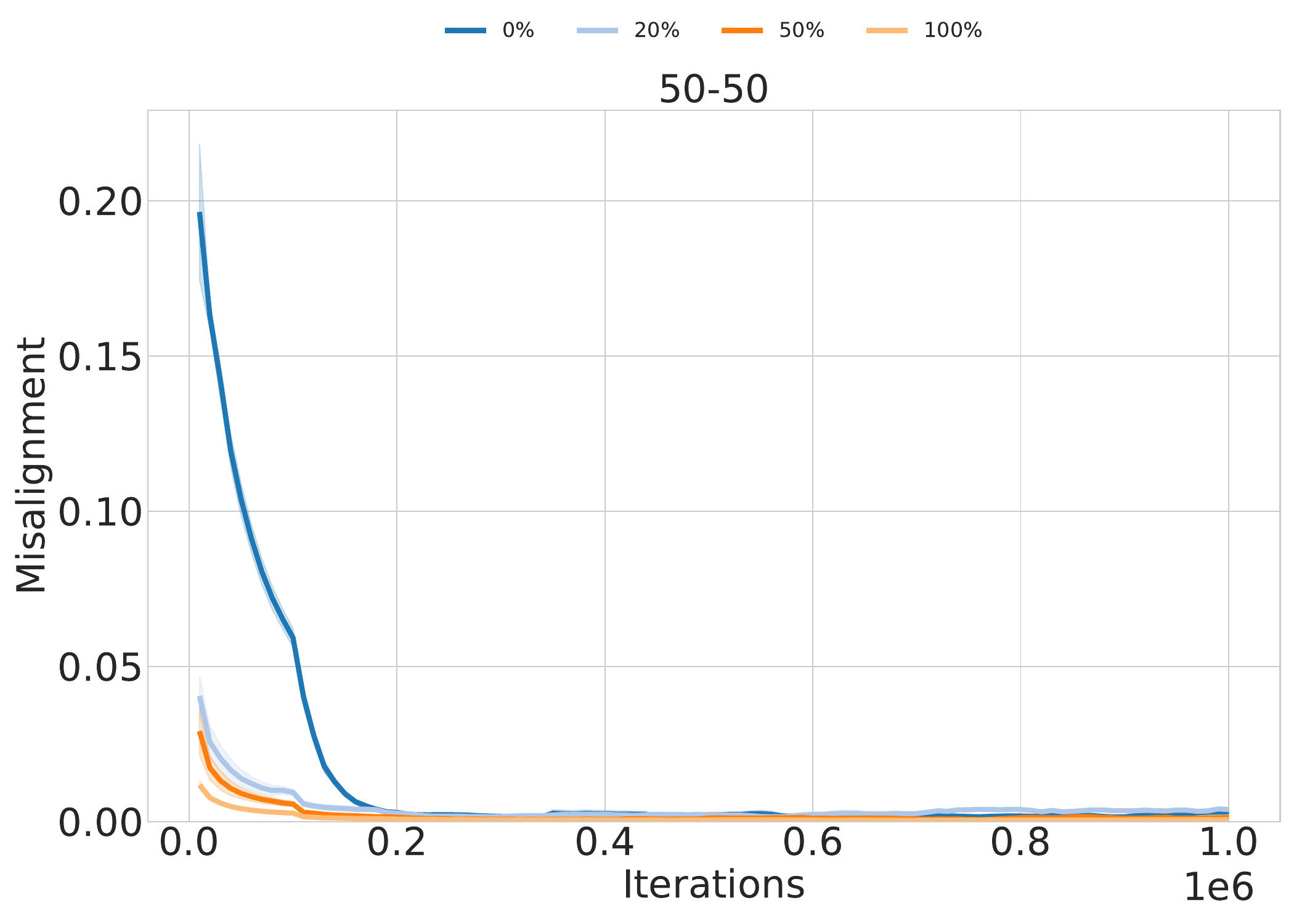}
    \end{subfigure}
    
    \caption{SlowSwim learning curves for ReCOIL+FMR with reduced feedback. The shaded region represents the standard error.}
    \label{fig:swim-scale}
\end{figure}

\begin{figure}[h]
    \centering
    
    \begin{subfigure}[b]{\textwidth}
        \centering
        \includegraphics[width=0.8\textwidth,trim=0 680 0 0,clip]{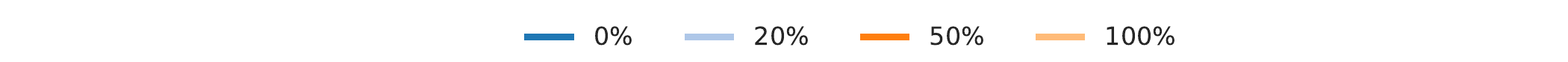}
    \end{subfigure}
    
    \vspace{0.5em}
    
    \begin{subfigure}[b]{0.48\textwidth}
        \centering
        \includegraphics[width=\textwidth,trim=0 0 0 55,clip]{images/fb-reduce/SlowHop/return_50-50.pdf}
    \end{subfigure}
    \hfill
    \begin{subfigure}[b]{0.48\textwidth}
        \centering
        \includegraphics[width=\textwidth,trim=0 0 0 55,clip]{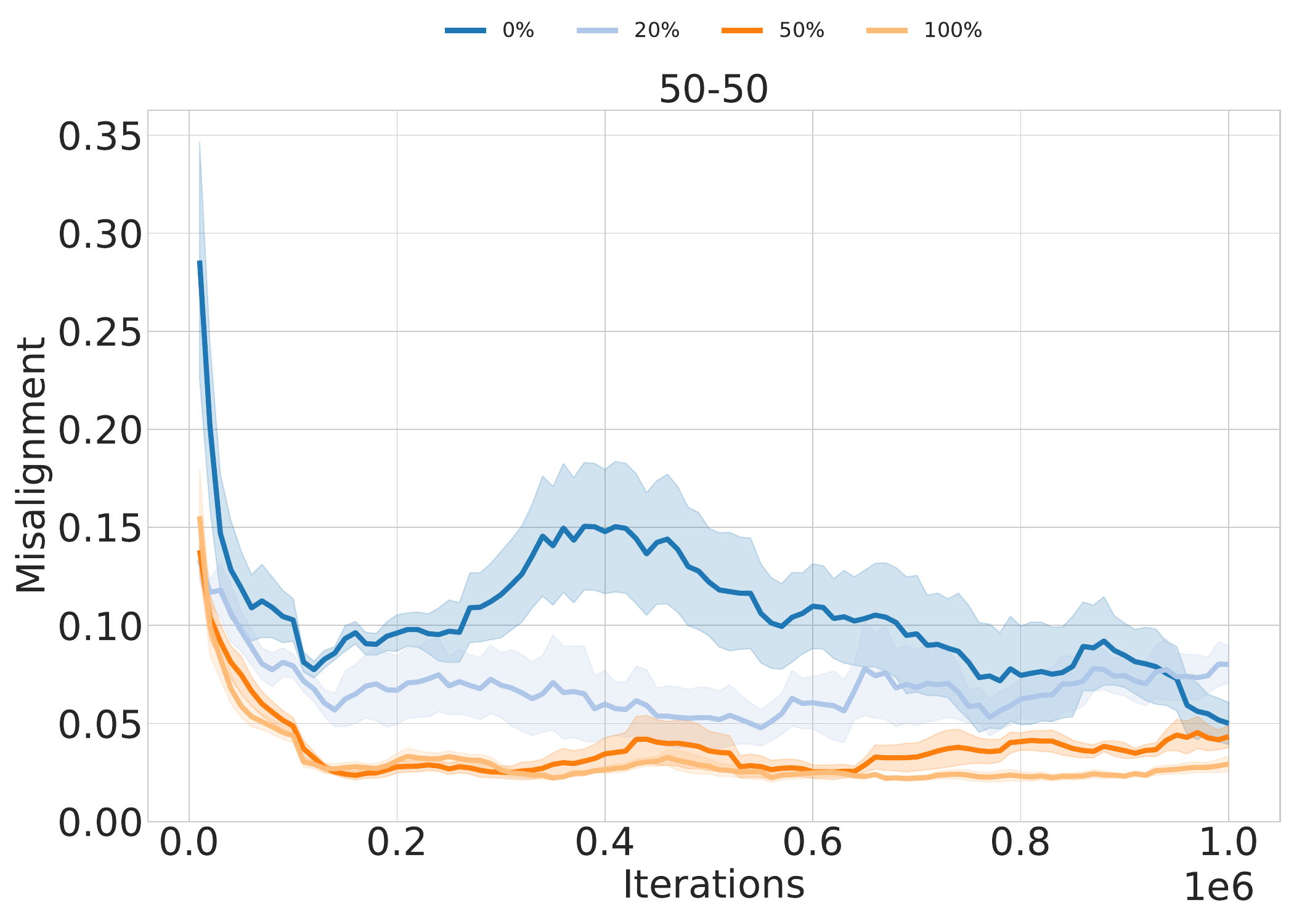}
    \end{subfigure}
    
    \caption{SlowHop learning curves for ReCOIL+FMR with reduced feedback. The shaded region represents the standard error.}
    \label{fig:hop-scale}
\end{figure}

\begin{figure}[h]
    \centering
    
    \begin{subfigure}[b]{\textwidth}
        \centering
        \includegraphics[width=0.8\textwidth,trim=0 680 0 0,clip]{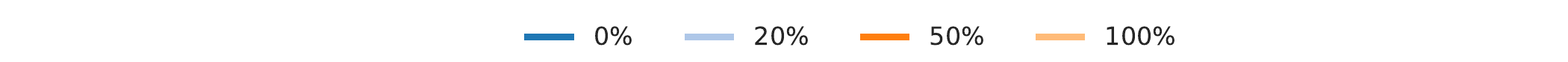}
    \end{subfigure}
    
    \vspace{0.5em}
    
    \begin{subfigure}[b]{0.48\textwidth}
        \centering
        \includegraphics[width=\textwidth,trim=0 0 0 55,clip]{images/fb-reduce/SlowWalk/return_50-50.pdf}
    \end{subfigure}
    \hfill
    \begin{subfigure}[b]{0.48\textwidth}
        \centering
        \includegraphics[width=\textwidth,trim=0 0 0 55,clip]{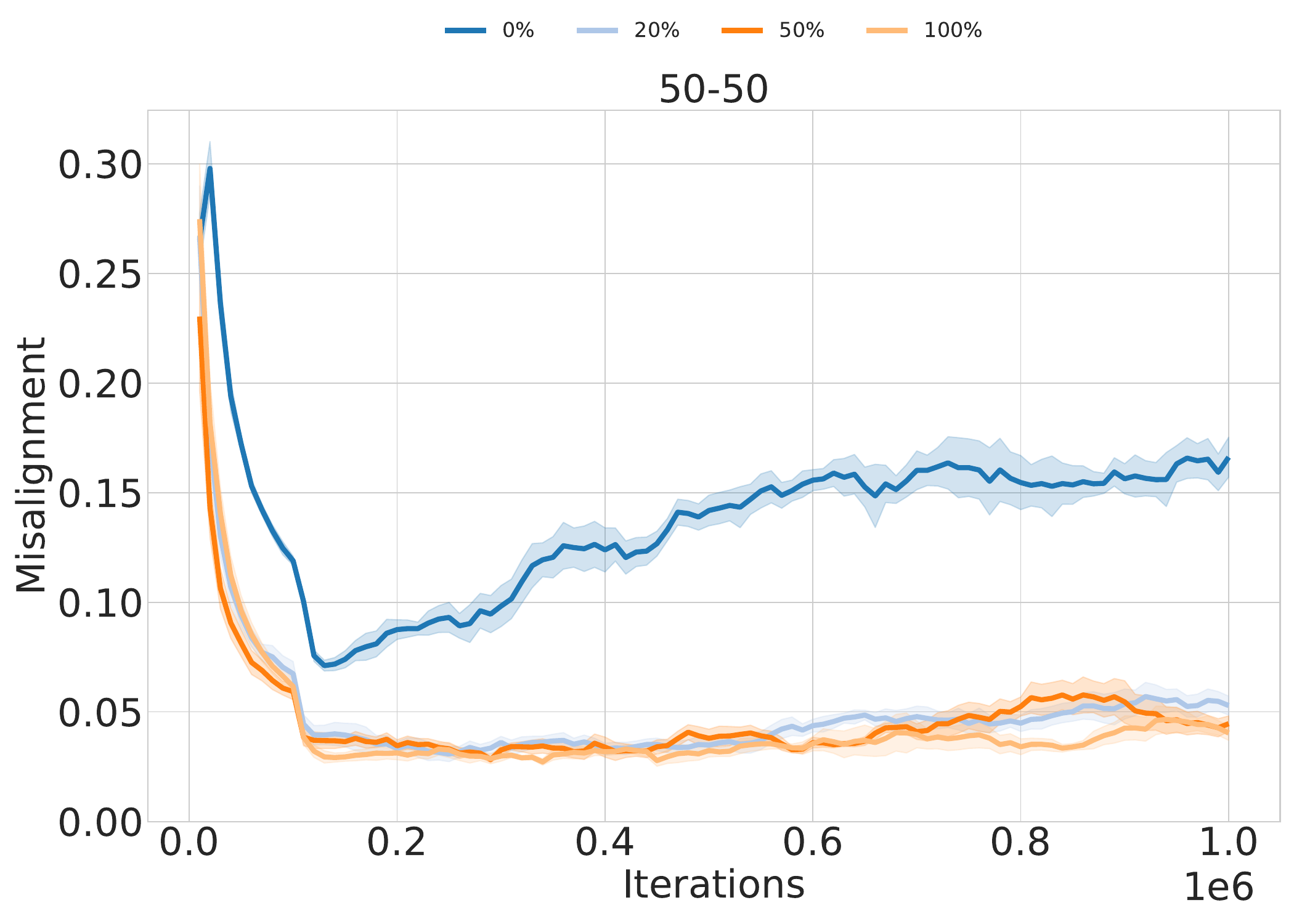}
    \end{subfigure}
    
    \caption{SlowWalk learning curves for ReCOIL+FMR with reduced feedback. The shaded region represents the standard error.}
    \label{fig:walk-scale}
\end{figure}

%% file: appendix/gen-tau.tex
The following are the results when comparing feedback-adaptive temperature $\tau_j$ with \textit{generalized} feedback-adaptive temperature $\tau^\text{gen}$. The comparison is performed across all tasks using ReCOIL+FMR in environments with a discrete action space, demonstrating negligible performance differences. This indicates that either formulation is valid, where $\tau^\text{gen}$ generalizes $\tau_j$ beyond discrete action spaces to support continuous action spaces.

\begin{table}[h]
\centering
\caption{PathM performance comparison for ReCOIL+FMR when using $\tau_j$ and $\tau^\text{gen}$. Results show mean ± std for the last 10 evaluations, over 5 seeds.}
\label{tab:pathm-taus}
\begin{tabular}{llrr}
\toprule
$\tau$ Type & Ratio & Suc. & Mis. \\
\midrule
\multirow[t]{3}{*}{$\tau^\text{gen}$} 
& 50-50 & 0.986 ± 0.12 & 0.014 ± 0.08 \\
& 25-50 & 0.974 ± 0.16 & 0.019 ± 0.09 \\
& 10-50 & 0.941 ± 0.24 & 0.020 ± 0.08 \\
\cline{1-4}
\multirow[t]{3}{*}{$\tau_j$} 
& 50-50 & 0.988 ± 0.11 & 0.015 ± 0.08 \\
& 25-50 & 0.974 ± 0.16 & 0.014 ± 0.08 \\
& 10-50 & 0.947 ± 0.22 & 0.020 ± 0.09 \\
\bottomrule
\end{tabular}
\end{table}

\begin{table}[h]
\centering
\caption{PathBB performance comparison for ReCOIL+FMR when using $\tau_j$ and $\tau^\text{gen}$. Results show mean ± std for the last 10 evaluations, over 5 seeds.}
\label{tab:pathbb-ataus}
\begin{tabular}{llrr}
\toprule
$\tau$ Type & Ratio & Suc. & Mis. \\
\midrule
\multirow[t]{3}{*}{$\tau^\text{gen}$} 
& 50-50 & 0.919 ± 0.27 & 0.058 ± 0.19 \\
& 25-50 & 0.880 ± 0.33 & 0.053 ± 0.15 \\
& 10-50 & 0.706 ± 0.46 & 0.109 ± 0.24 \\
\cline{1-4}
\multirow[t]{3}{*}{$\tau_j$} 
& 50-50 & 0.919 ± 0.27 & 0.048 ± 0.17 \\
& 25-50 & 0.880 ± 0.32 & 0.073 ± 0.19 \\
& 10-50 & 0.708 ± 0.45 & 0.097 ± 0.21 \\
\bottomrule
\end{tabular}
\end{table}

\begin{table}[h]
\centering
\caption{SlowSwim performance comparison for ReCOIL+FMR when using $\tau_j$ and $\tau^\text{gen}$. Results show mean ± std for the last 10 evaluations, over 5 seeds.}
\label{tab:swim-ataus}
\begin{tabular}{llrr}
\toprule
$\tau$ Type & Ratio & Return & Mis. \\
\midrule
\multirow[t]{3}{*}{$\tau^\text{gen}$} 
& 50-50 & 1.0014 ± 0.014 & 0.0007 ± 0.005 \\
& 25-50 & 1.0019 ± 0.008 & 0.0003 ± 0.001 \\
& 10-50 & 1.0021 ± 0.015 & 0.0007 ± 0.003 \\
\cline{1-4}
\multirow[t]{3}{*}{$\tau_j$}
& 50-50 & 1.0011 ± 0.008 & 0.0005 ± 0.003 \\
& 25-50 & 1.0017 ± 0.008 & 0.0004 ± 0.002 \\
& 10-50 & 1.0018 ± 0.008 & 0.0009 ± 0.003 \\
\bottomrule
\end{tabular}
\end{table}

\begin{table}[h]
\centering
\caption{SlowHop performance comparison for ReCOIL+FMR when using $\tau_j$ and $\tau^\text{gen}$. Results show mean ± std for the last 10 evaluations, over 5 seeds.}
\label{tab:gop-taus}
\begin{tabular}{llrr}
\toprule
$\tau$ Type & Ratio & Return & Mis. \\
\midrule
\multirow[t]{3}{*}{$\tau^\text{gen}$}
& 50-50 & 0.988 ± 0.124 & 0.033 ± 0.114 \\
& 25-50 & 0.993 ± 0.166 & 0.042 ± 0.133 \\
& 10-50 & 0.984 ± 0.217 & 0.061 ± 0.157 \\
\cline{1-4}
\multirow[t]{3}{*}{$\tau_j$} 
& 50-50 & 0.990 ± 0.127 & 0.029 ± 0.099 \\
& 25-50 & 1.001 ± 0.179 & 0.047 ± 0.135 \\
& 10-50 & 0.986 ± 0.225 & 0.066 ± 0.168 \\
\bottomrule
\end{tabular}
\end{table}

\begin{table}[h]
\centering
\caption{SlowWalk performance comparison for ReCOIL+FMR when using $\tau_j$ and $\tau^\text{gen}$. Results show mean ± std for the last 10 evaluations, over 5 seeds.}
\label{tab:walk-ataus}
\begin{tabular}{llrr}
\toprule
$\tau$ Type & Ratio & Return & Mis. \\
\midrule
\multirow[t]{3}{*}{$\tau^\text{gen}$} 
& 50-50 & 0.952 ± 0.179 & 0.035 ± 0.105 \\
& 25-50 & 0.936 ± 0.197 & 0.046 ± 0.107 \\
& 10-50 & 0.679 ± 0.328 & 0.152 ± 0.150 \\
\cline{1-4}
\multirow[t]{3}{*}{$\tau_j$} 
& 50-50 & 0.954 ± 0.169 & 0.040 ± 0.119 \\
& 25-50 & 0.929 ± 0.206 & 0.055 ± 0.122 \\
& 10-50 & 0.690 ± 0.323 & 0.165 ± 0.162 \\
\bottomrule
\end{tabular}
\end{table}

\begin{figure}[h]
    \centering
    
    \begin{subfigure}[b]{\textwidth}
        \centering
        \includegraphics[width=0.8\textwidth,trim=0 665 0 0,clip]{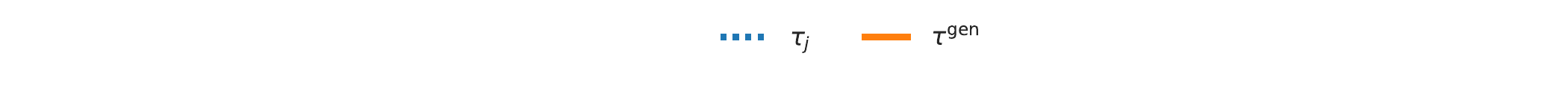}
    \end{subfigure}
    
    \vspace{0.5em}
    
    \begin{subfigure}[b]{0.32\textwidth}
        \centering
        \includegraphics[width=\textwidth,trim=0 44 0 55,clip]{images/gen-tau/PathM/success_50-50.pdf}
    \end{subfigure}
    \hfill
    \begin{subfigure}[b]{0.32\textwidth}
        \centering
        \includegraphics[width=\textwidth,trim=0 44 0 55,clip]{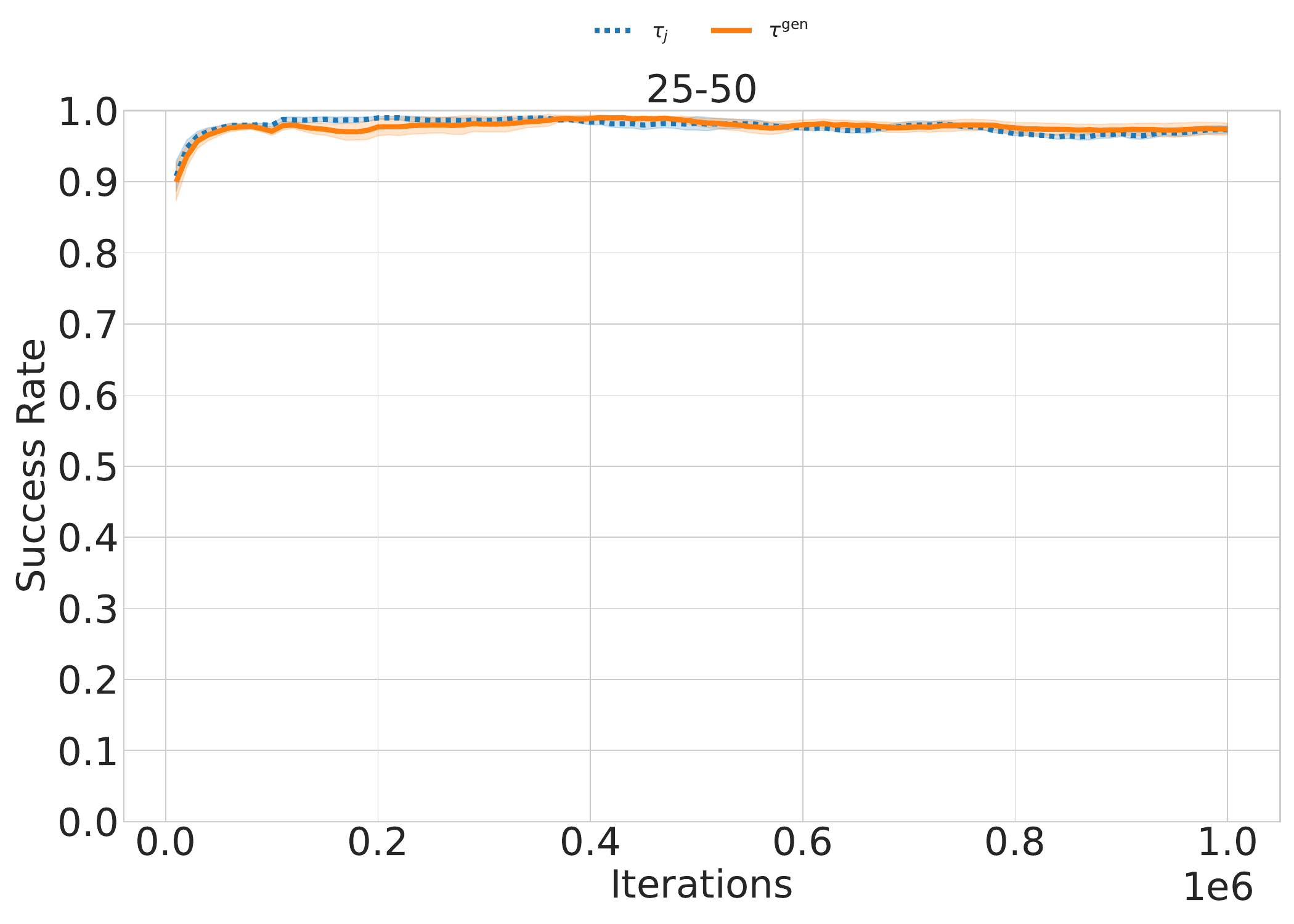}
    \end{subfigure}
    \hfill
    \begin{subfigure}[b]{0.32\textwidth}
        \centering
        \includegraphics[width=\textwidth,trim=0 44 0 55,clip]{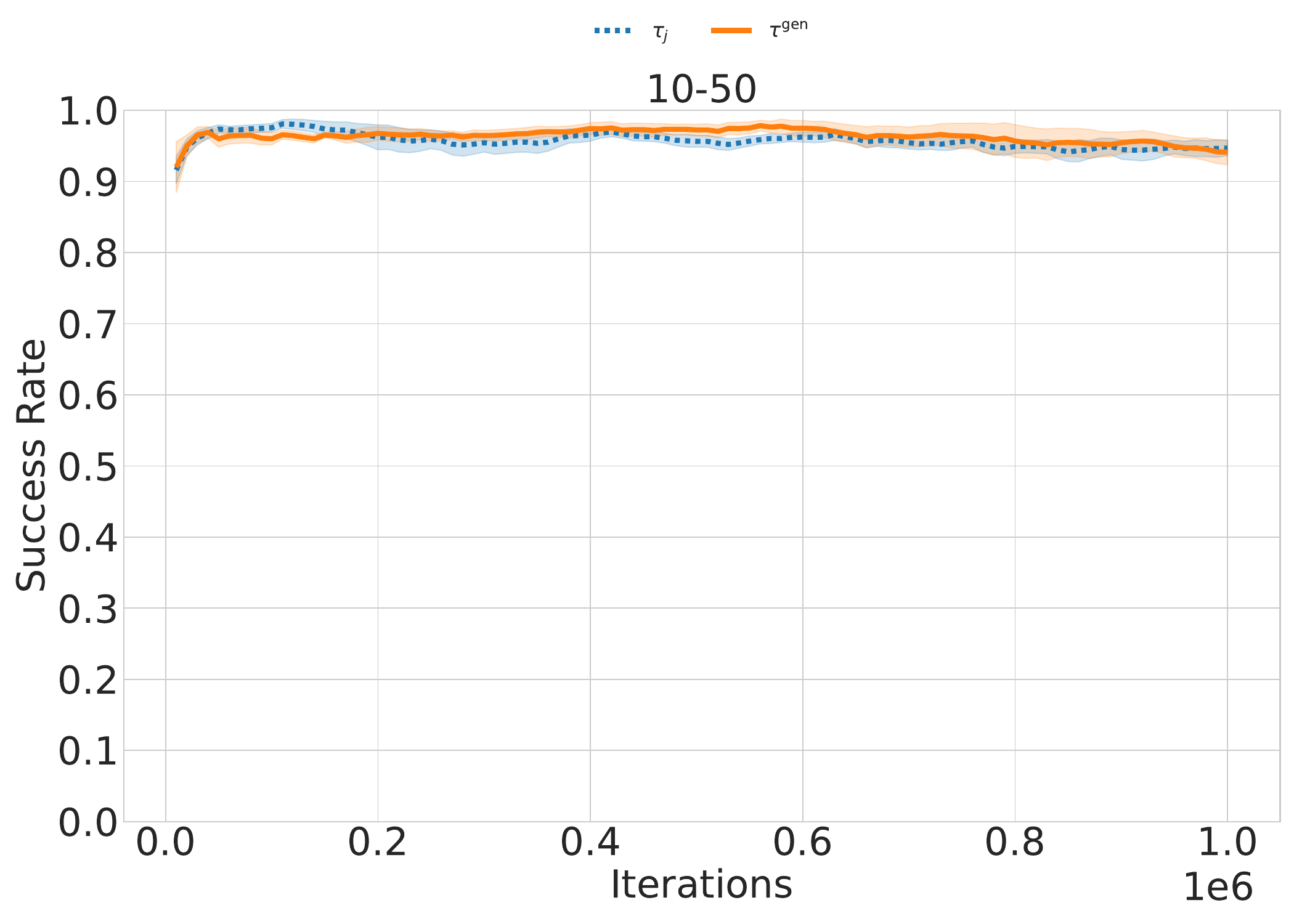}
    \end{subfigure}
    
    \begin{subfigure}[b]{0.32\textwidth}
        \centering
        \includegraphics[width=\textwidth,trim=0 0 0 81,clip]{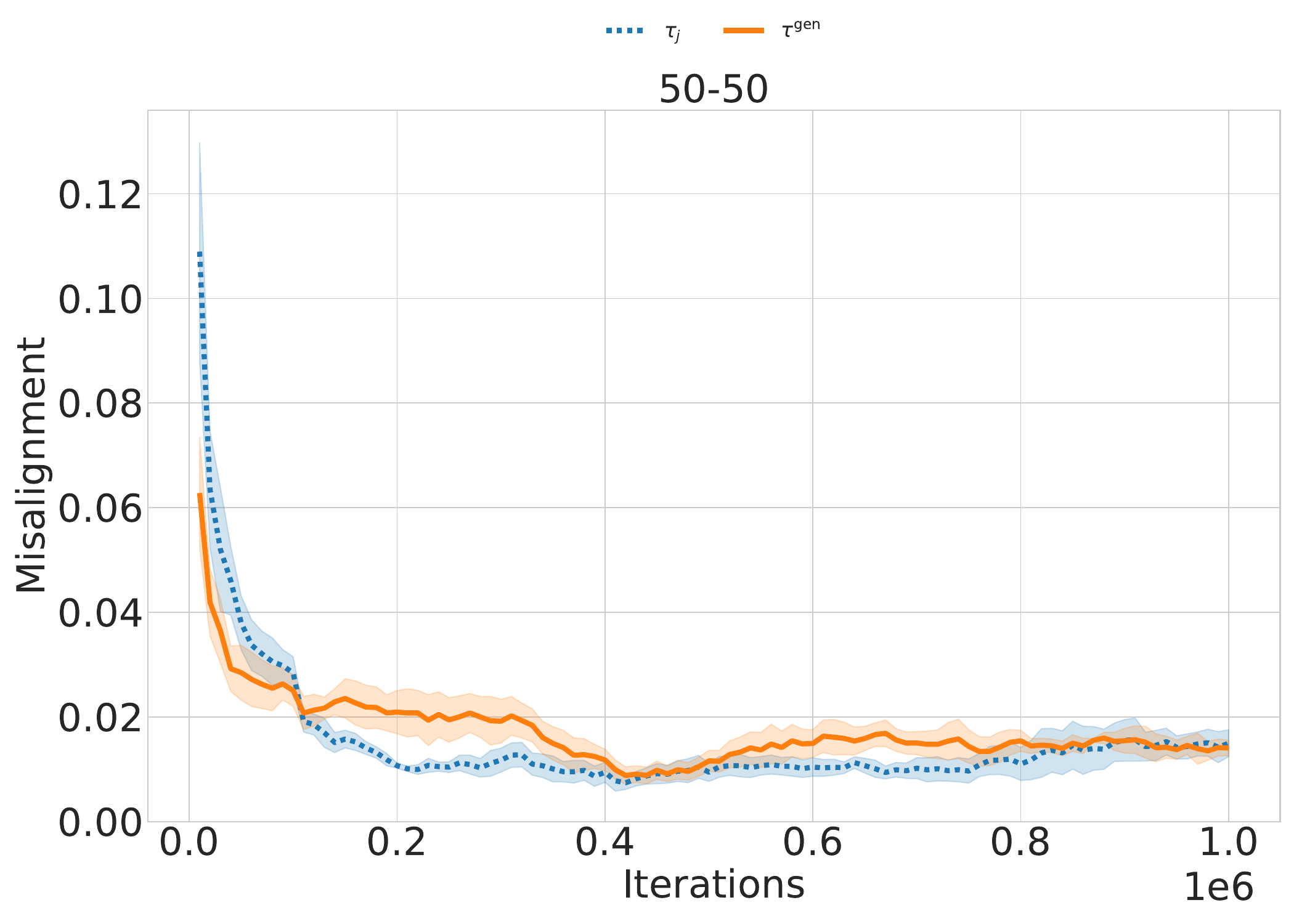}
    \end{subfigure}
    \hfill
    \begin{subfigure}[b]{0.32\textwidth}
        \centering
        \includegraphics[width=\textwidth,trim=0 0 0 81,clip]{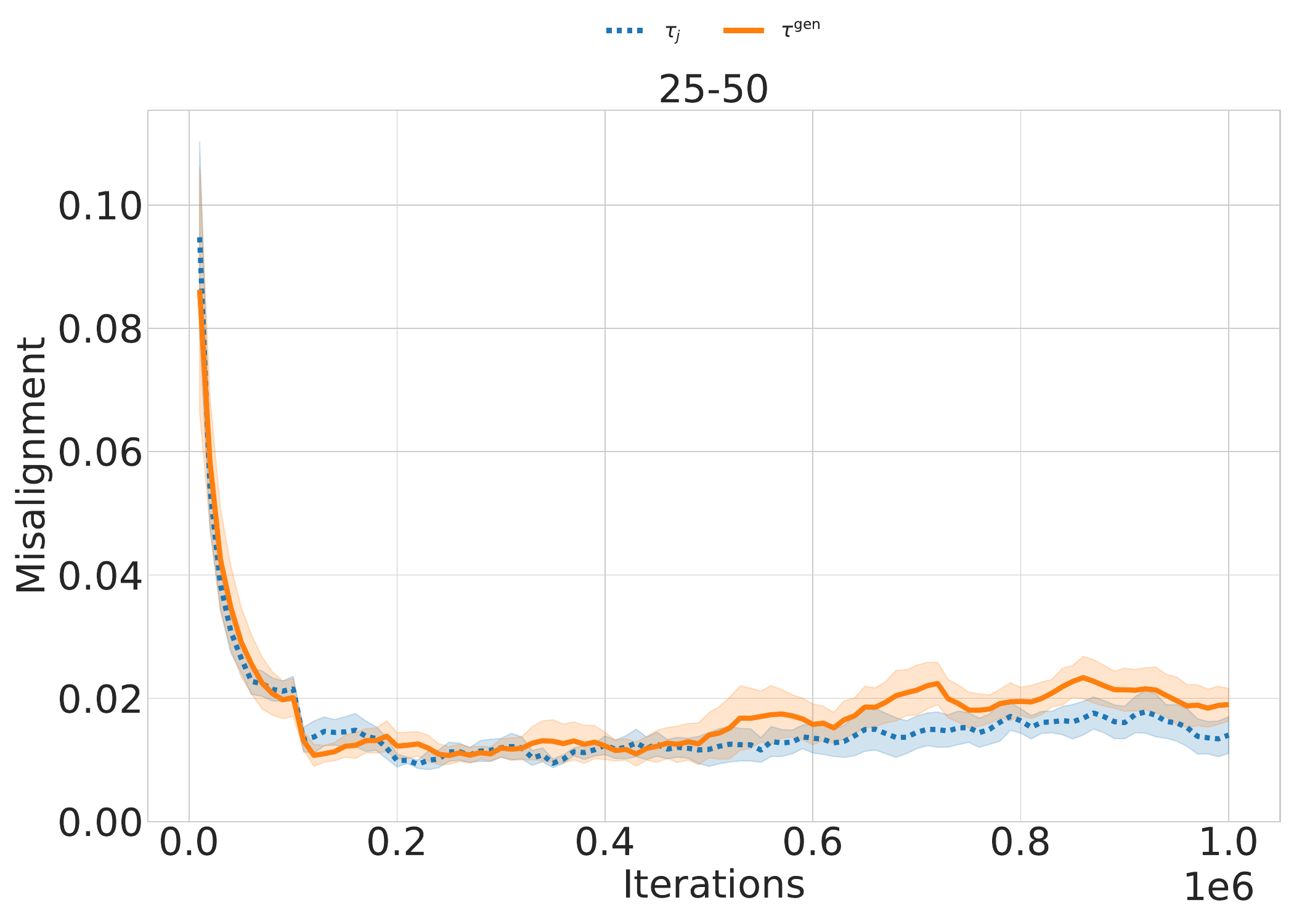}
    \end{subfigure}
    \hfill
    \begin{subfigure}[b]{0.32\textwidth}
        \centering
        \includegraphics[width=\textwidth,trim=0 0 0 81,clip]{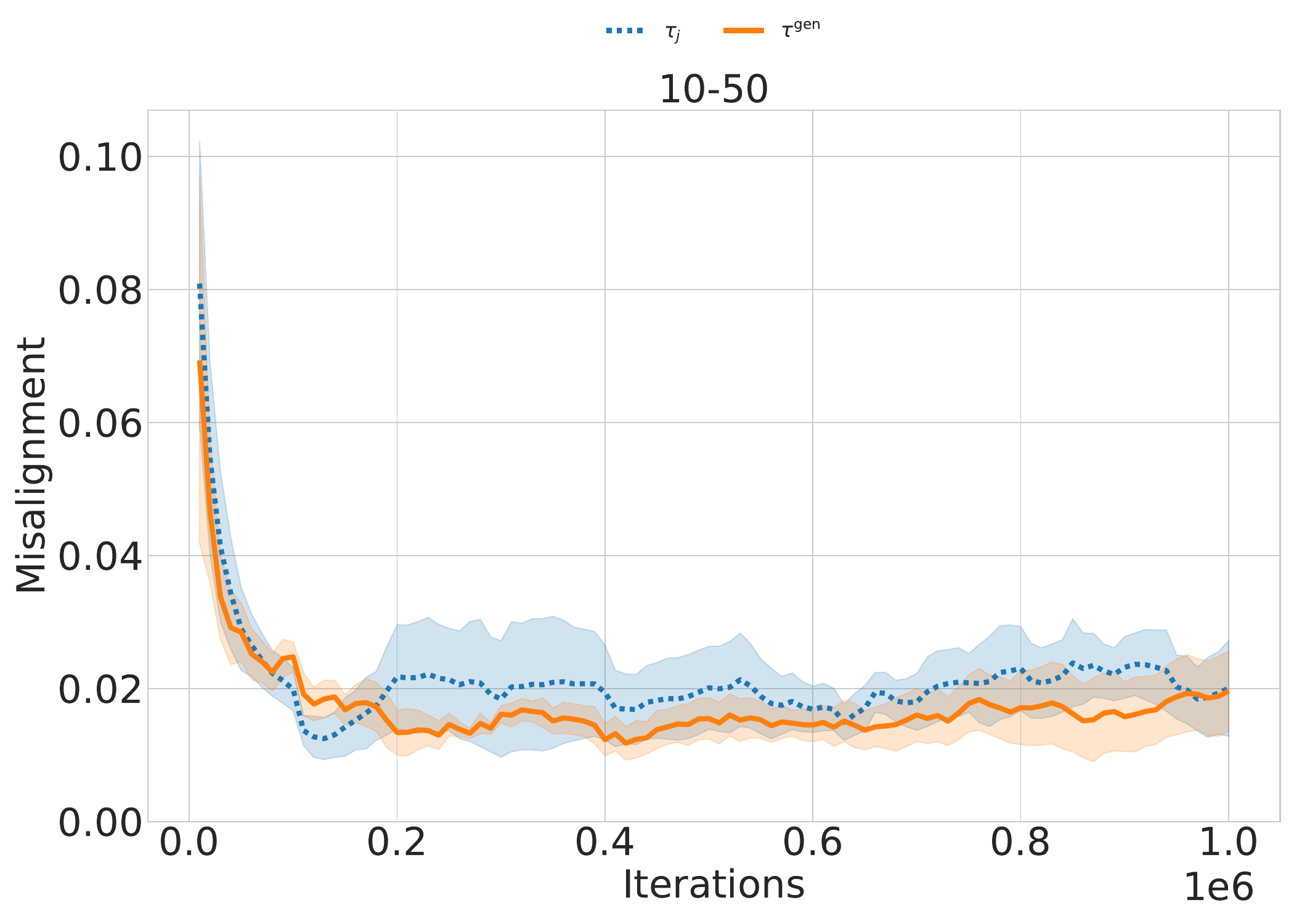}
    \end{subfigure}
    
    \caption{PathM learning curves for ReCOIL+FMR when using $\tau_j$ and $\tau^\text{gen}$. The shaded region represents the standard error.}
    \label{fig:pathm-ataus}
\end{figure}

\begin{figure}[h]
    \centering
    
    \begin{subfigure}[b]{\textwidth}
        \centering
        \includegraphics[width=0.8\textwidth,trim=0 665 0 0,clip]{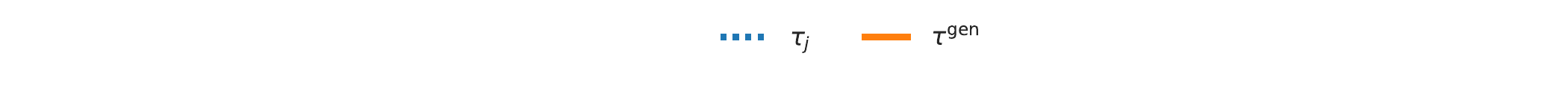}
    \end{subfigure}
    
    \vspace{0.5em}
    
    \begin{subfigure}[b]{0.32\textwidth}
        \centering
        \includegraphics[width=\textwidth,trim=0 44 0 55,clip]{images/gen-tau/PathBB/success_50-50.pdf}
    \end{subfigure}
    \hfill
    \begin{subfigure}[b]{0.32\textwidth}
        \centering
        \includegraphics[width=\textwidth,trim=0 44 0 55,clip]{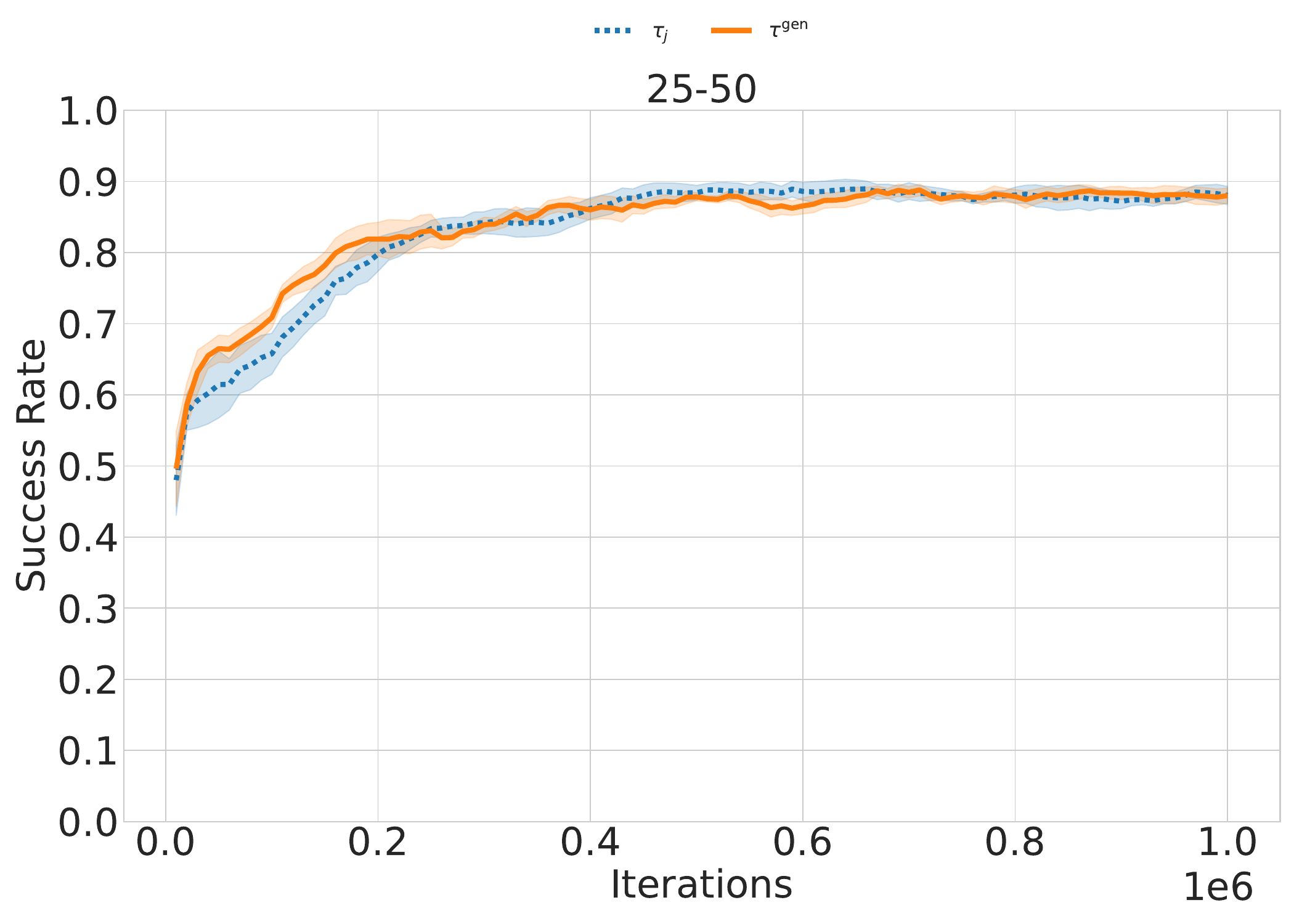}
    \end{subfigure}
    \hfill
    \begin{subfigure}[b]{0.32\textwidth}
        \centering
        \includegraphics[width=\textwidth,trim=0 44 0 55,clip]{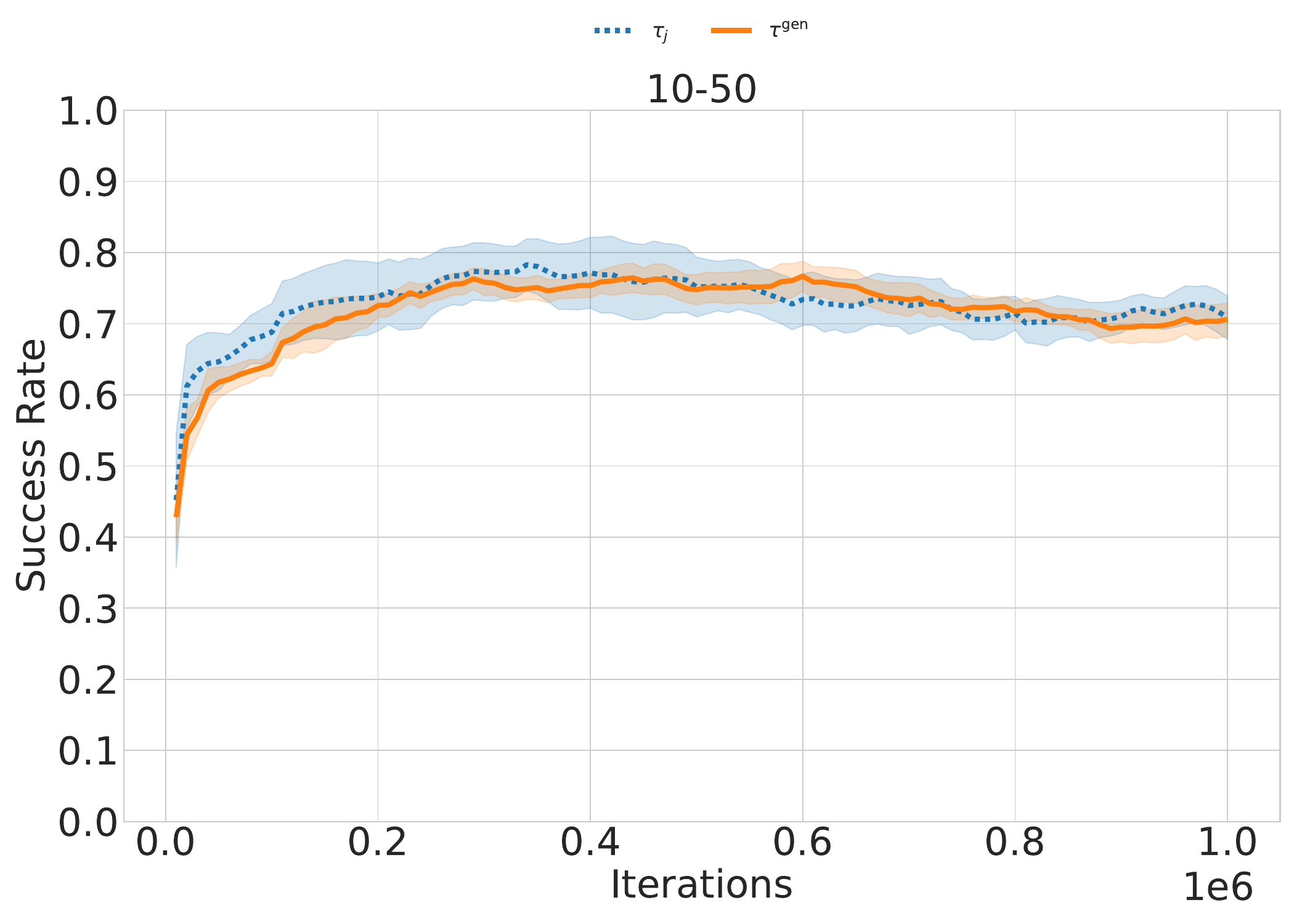}
    \end{subfigure}
    
    \begin{subfigure}[b]{0.32\textwidth}
        \centering
        \includegraphics[width=\textwidth,trim=0 0 0 81,clip]{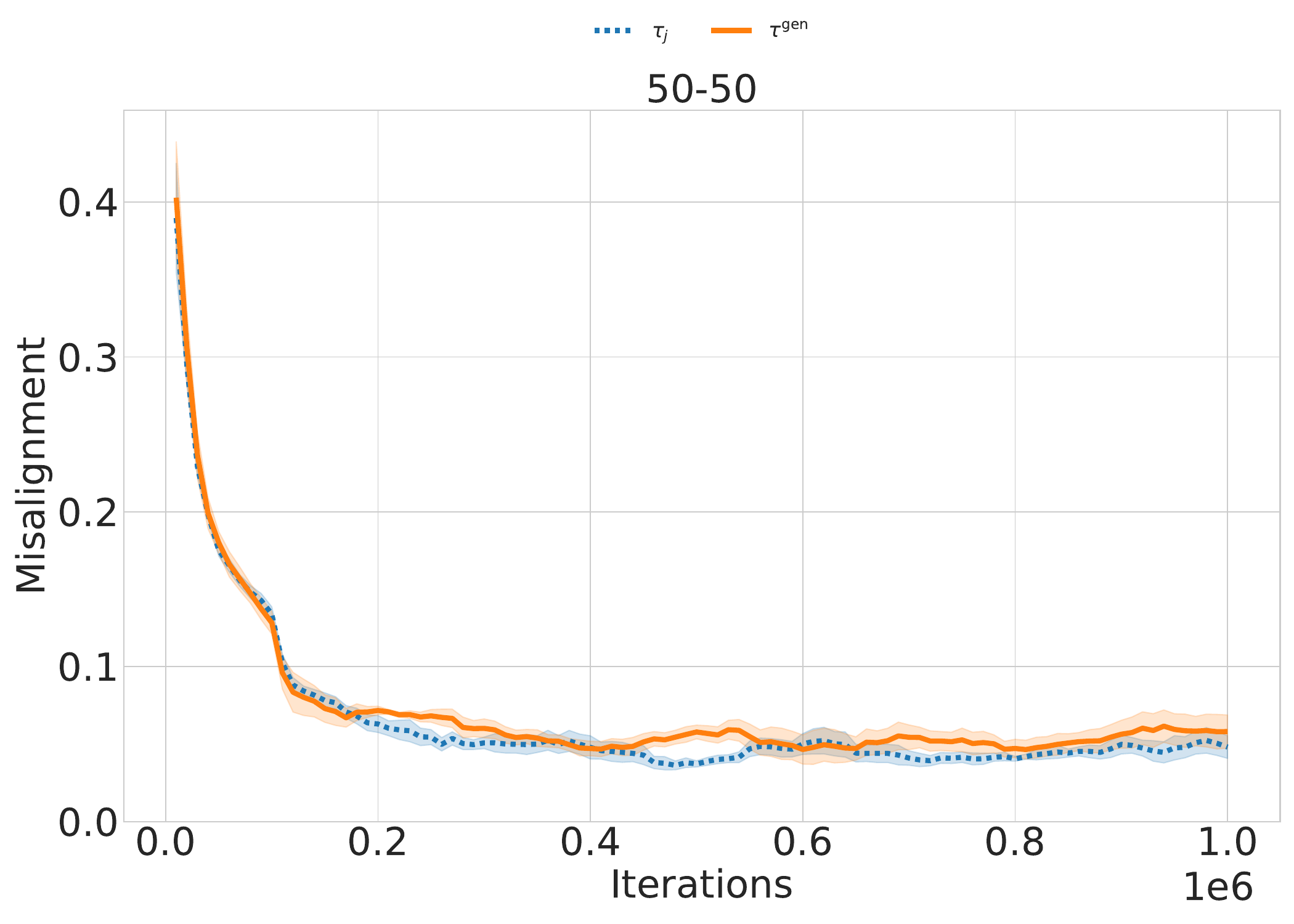}
    \end{subfigure}
    \hfill
    \begin{subfigure}[b]{0.32\textwidth}
        \centering
        \includegraphics[width=\textwidth,trim=0 0 0 81,clip]{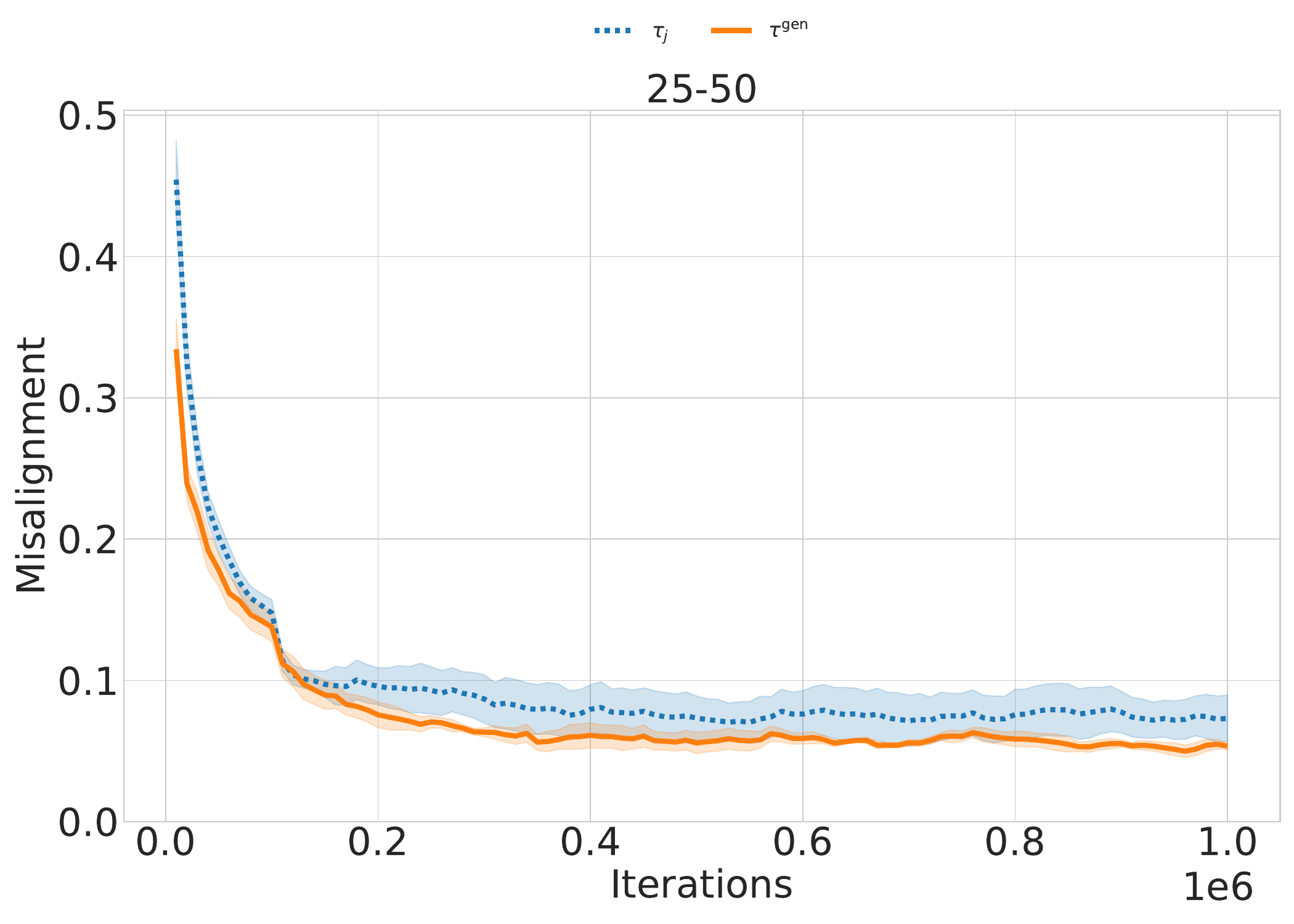}
    \end{subfigure}
    \hfill
    \begin{subfigure}[b]{0.32\textwidth}
        \centering
        \includegraphics[width=\textwidth,trim=0 0 0 81,clip]{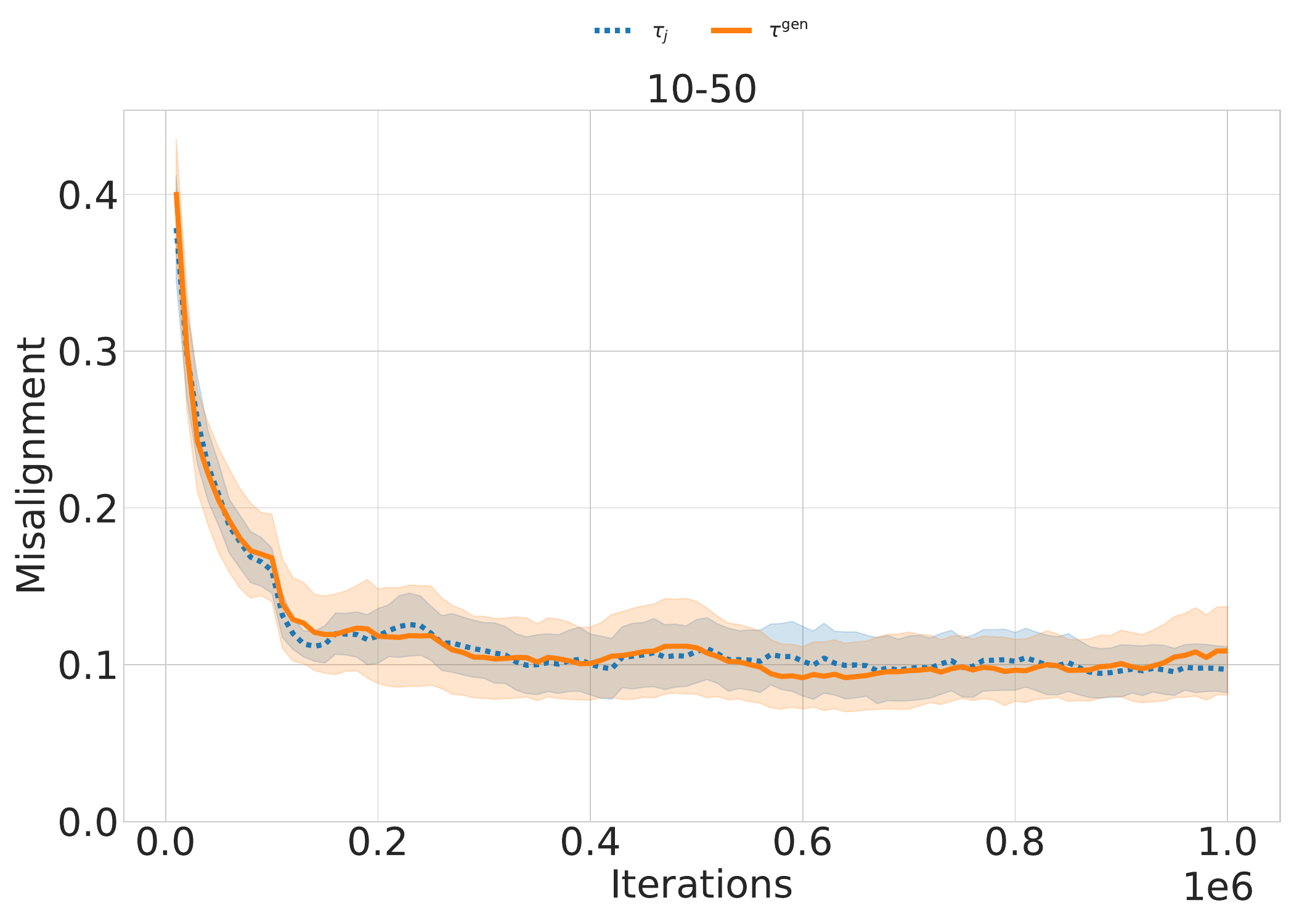}
    \end{subfigure}
    
    \caption{PathBB learning curves for ReCOIL+FMR when using $\tau_j$ and $\tau^\text{gen}$. The shaded region represents the standard error.}
    \label{fig:pathbb-ataus}
\end{figure}

\begin{figure}[h]
    \centering
    
    \begin{subfigure}[b]{\textwidth}
        \centering
        \includegraphics[width=0.8\textwidth,trim=0 665 0 0,clip]{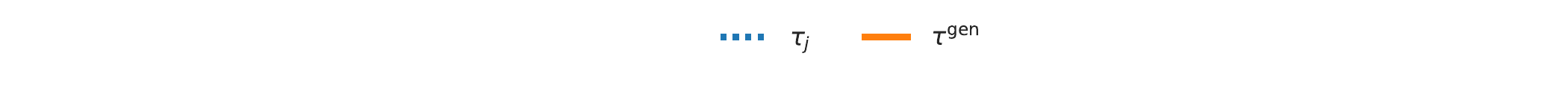}
    \end{subfigure}
    
    \vspace{0.5em}
    
    \begin{subfigure}[b]{0.32\textwidth}
        \centering
        \includegraphics[width=\textwidth,trim=0 44 0 55,clip]{images/gen-tau/SlowSwim/return_50-50.pdf}
    \end{subfigure}
    \hfill
    \begin{subfigure}[b]{0.32\textwidth}
        \centering
        \includegraphics[width=\textwidth,trim=0 44 0 55,clip]{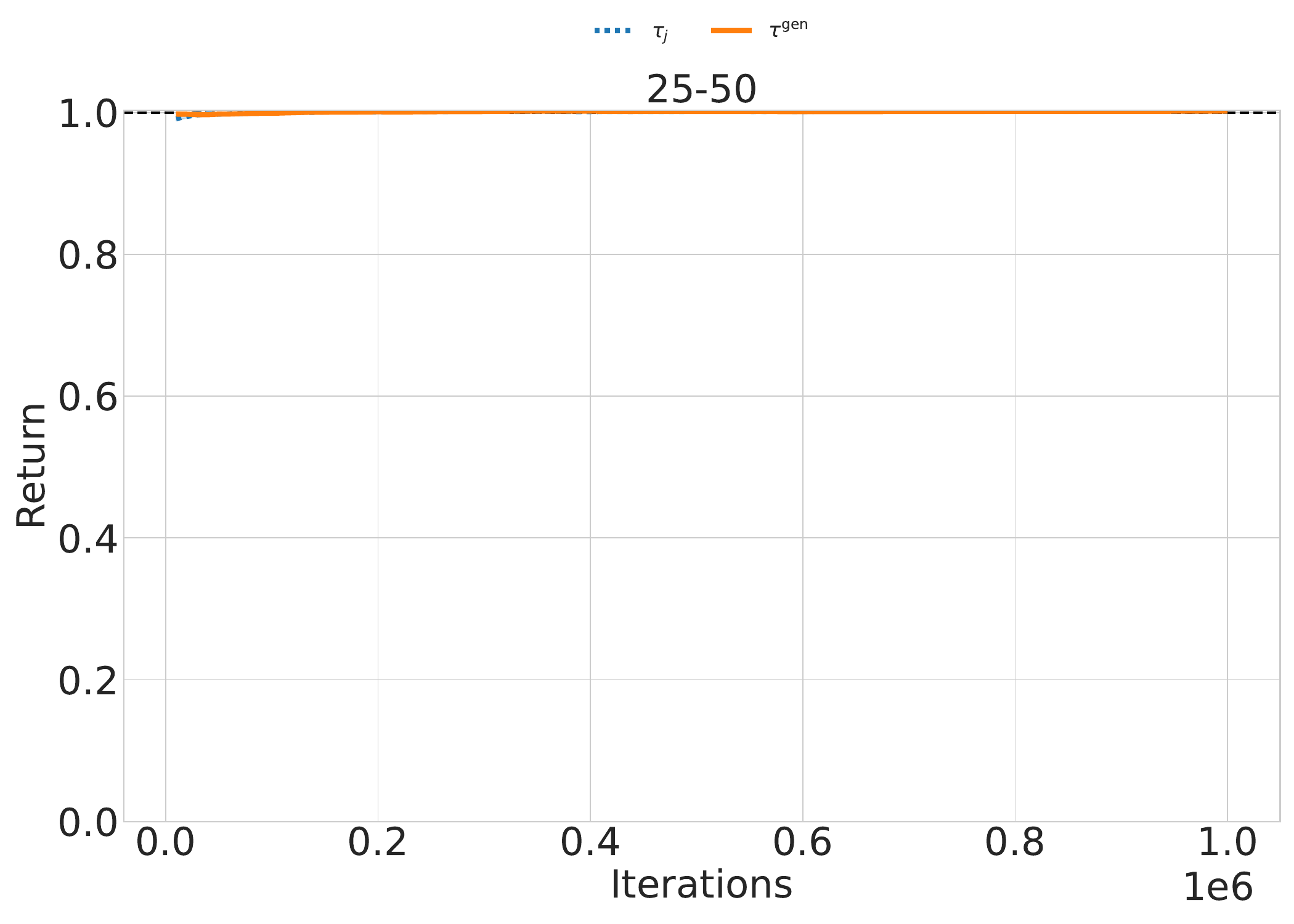}
    \end{subfigure}
    \hfill
    \begin{subfigure}[b]{0.32\textwidth}
        \centering
        \includegraphics[width=\textwidth,trim=0 44 0 55,clip]{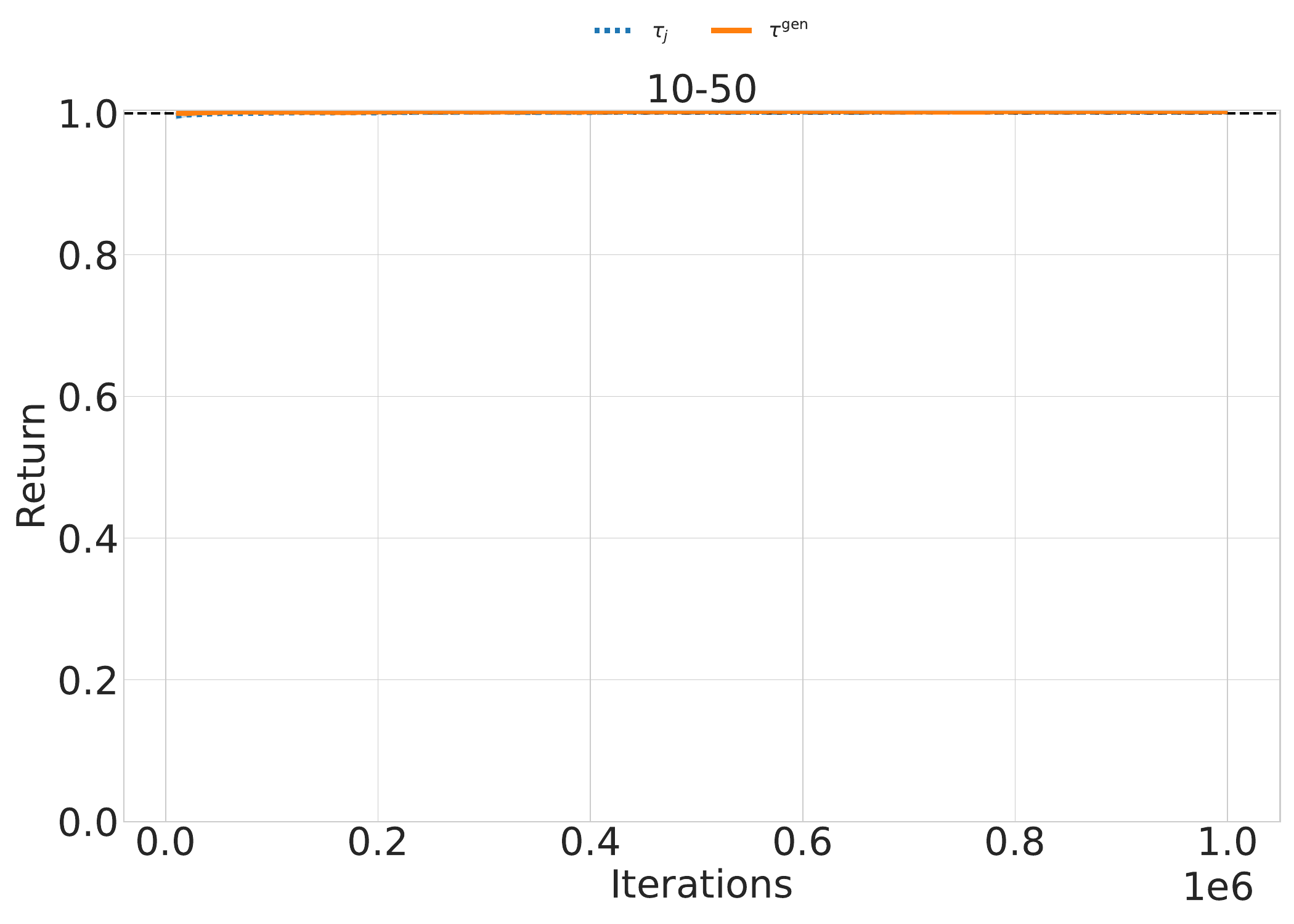}
    \end{subfigure}
    
    \begin{subfigure}[b]{0.32\textwidth}
        \centering
        \includegraphics[width=\textwidth,trim=0 0 0 81,clip]{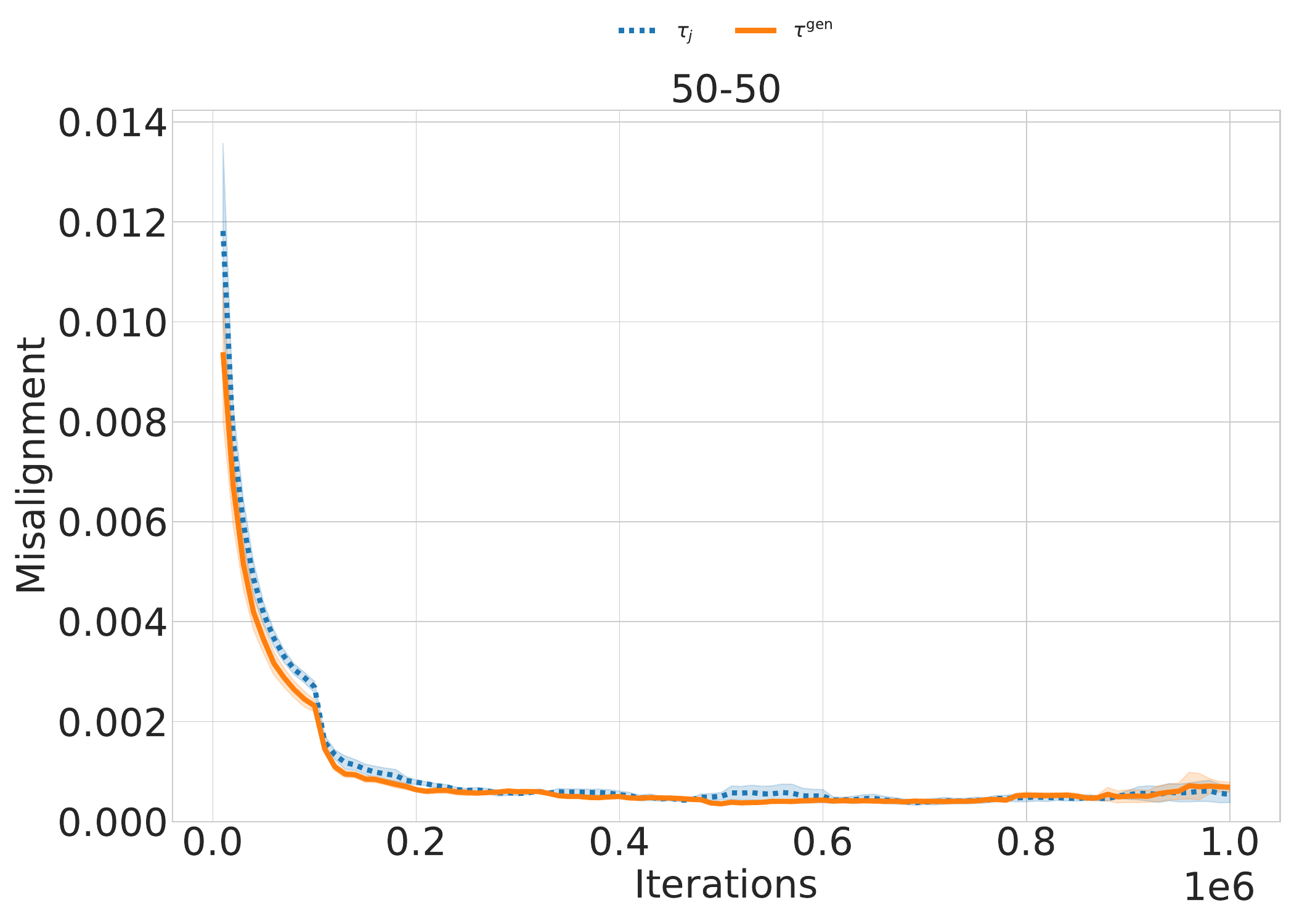}
    \end{subfigure}
    \hfill
    \begin{subfigure}[b]{0.32\textwidth}
        \centering
        \includegraphics[width=\textwidth,trim=0 0 0 81,clip]{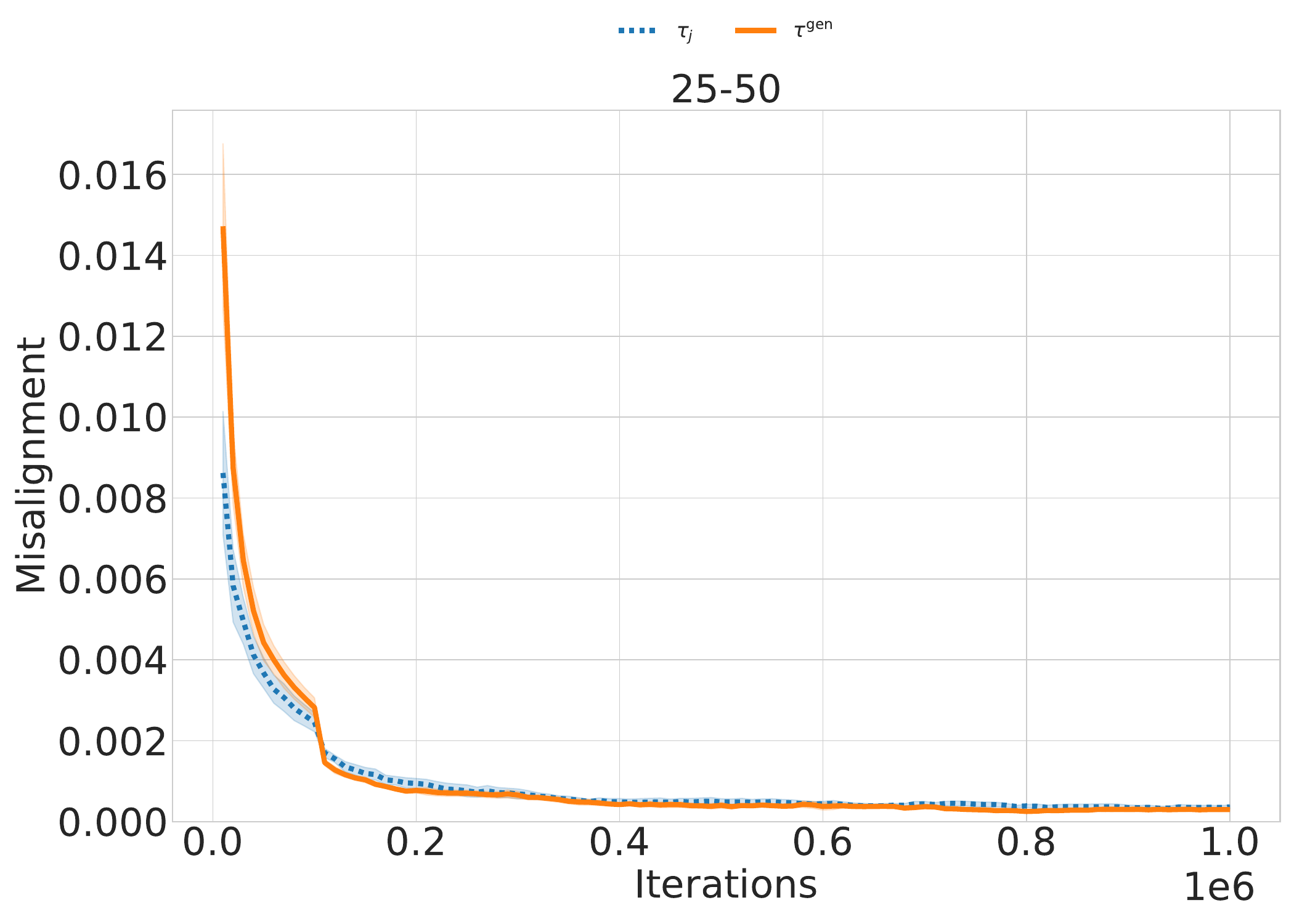}
    \end{subfigure}
    \hfill
    \begin{subfigure}[b]{0.32\textwidth}
        \centering
        \includegraphics[width=\textwidth,trim=0 0 0 81,clip]{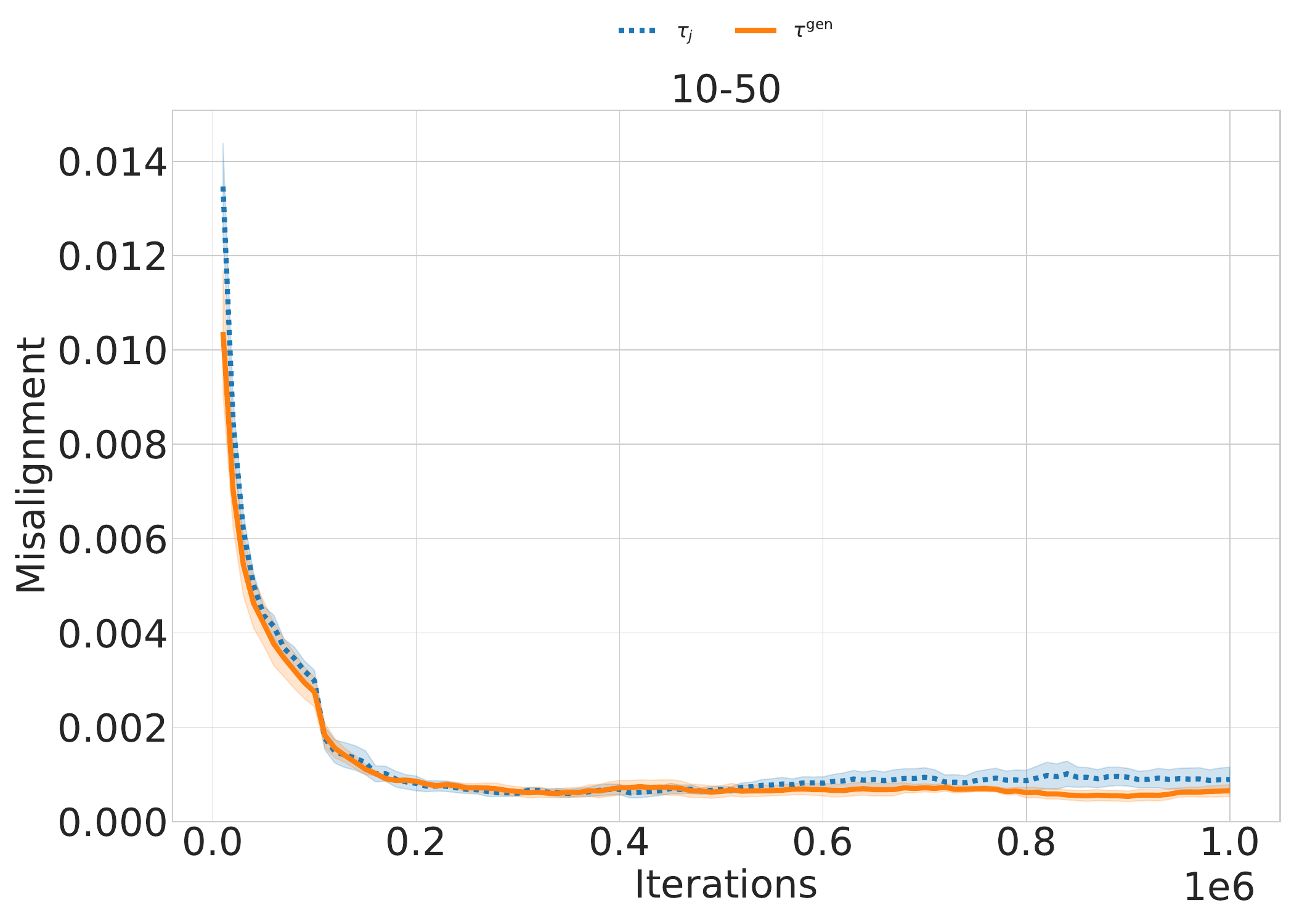}
    \end{subfigure}
    
    \caption{SlowSwim learning curves for ReCOIL+FMR when using $\tau_j$ and $\tau^\text{gen}$. The shaded region represents the standard error.}
    \label{fig:swim-ataus}
\end{figure}

\begin{figure}[h]
    \centering
    
    \begin{subfigure}[b]{\textwidth}
        \centering
        \includegraphics[width=0.8\textwidth,trim=0 665 0 0,clip]{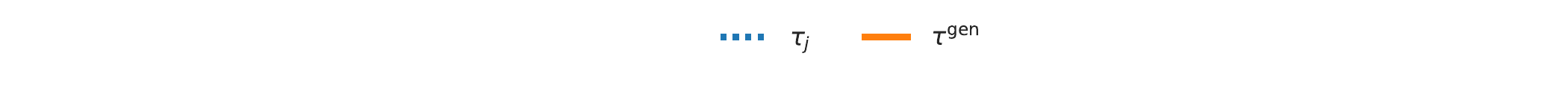}
    \end{subfigure}
    
    \vspace{0.5em}
    
    \begin{subfigure}[b]{0.32\textwidth}
        \centering
        \includegraphics[width=\textwidth,trim=0 44 0 55,clip]{images/gen-tau/SlowHop/return_50-50.pdf}
    \end{subfigure}
    \hfill
    \begin{subfigure}[b]{0.32\textwidth}
        \centering
        \includegraphics[width=\textwidth,trim=0 44 0 55,clip]{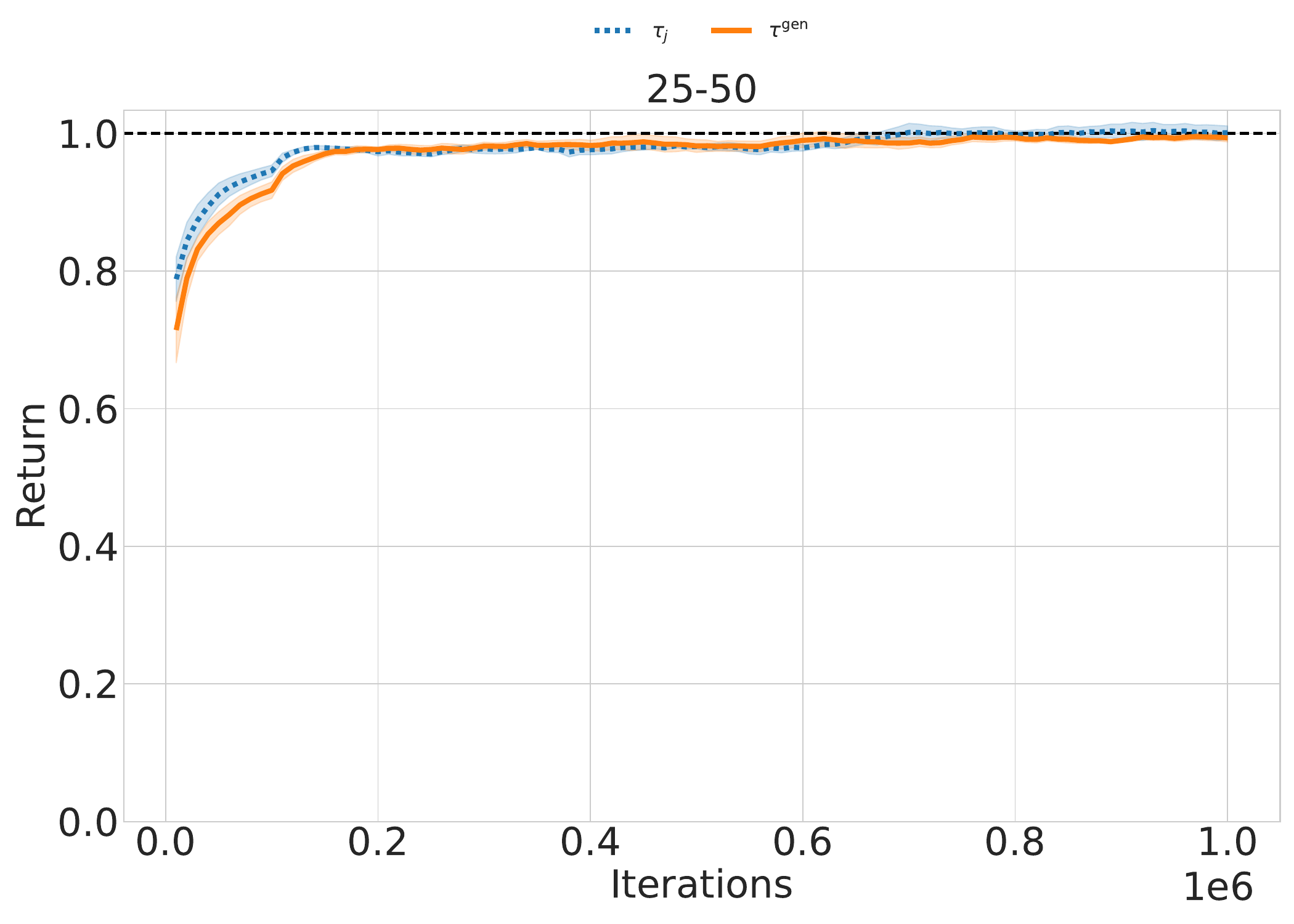}
    \end{subfigure}
    \hfill
    \begin{subfigure}[b]{0.32\textwidth}
        \centering
        \includegraphics[width=\textwidth,trim=0 44 0 55,clip]{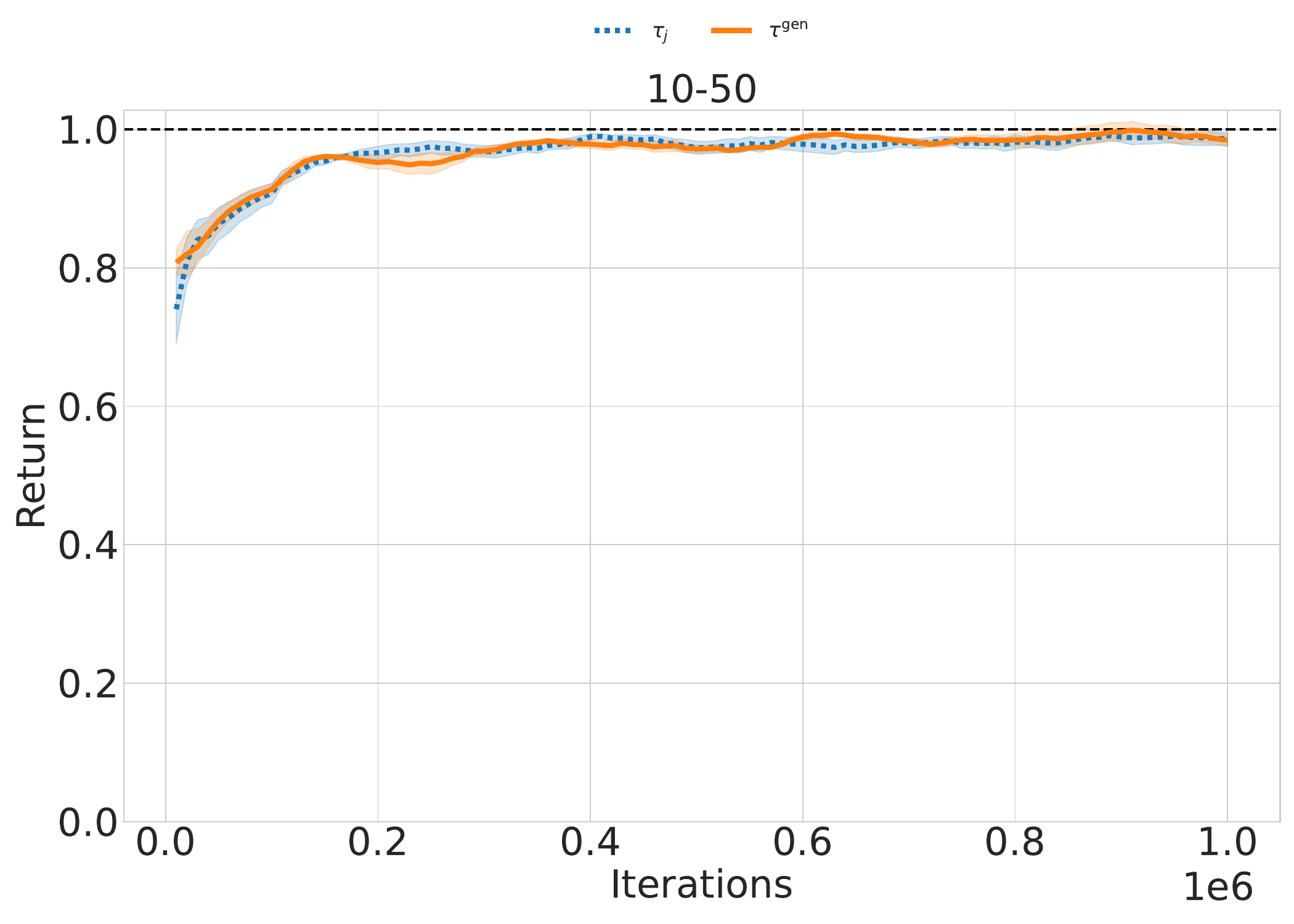}
    \end{subfigure}
    
    \begin{subfigure}[b]{0.32\textwidth}
        \centering
        \includegraphics[width=\textwidth,trim=0 0 0 81,clip]{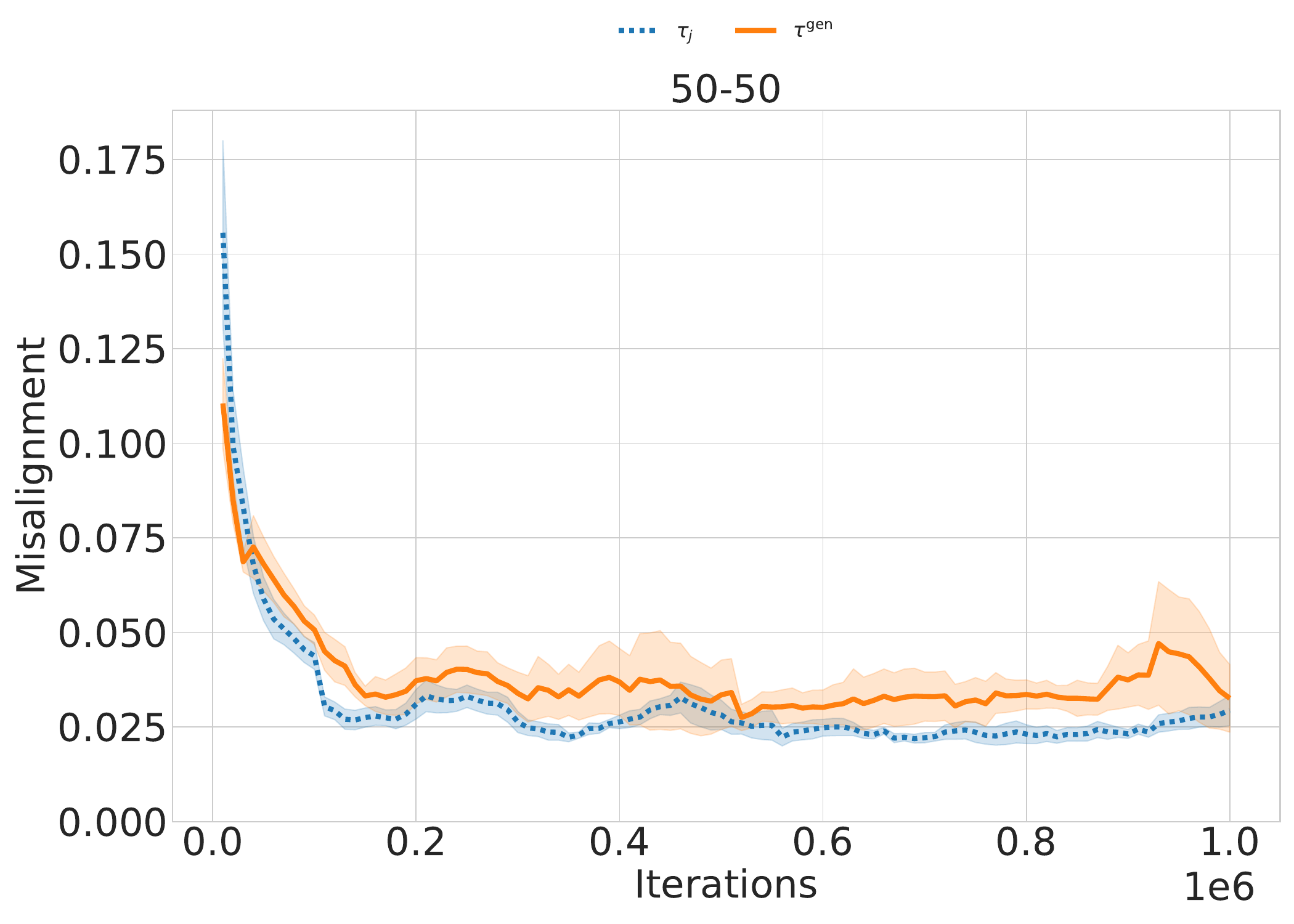}
    \end{subfigure}
    \hfill
    \begin{subfigure}[b]{0.32\textwidth}
        \centering
        \includegraphics[width=\textwidth,trim=0 0 0 81,clip]{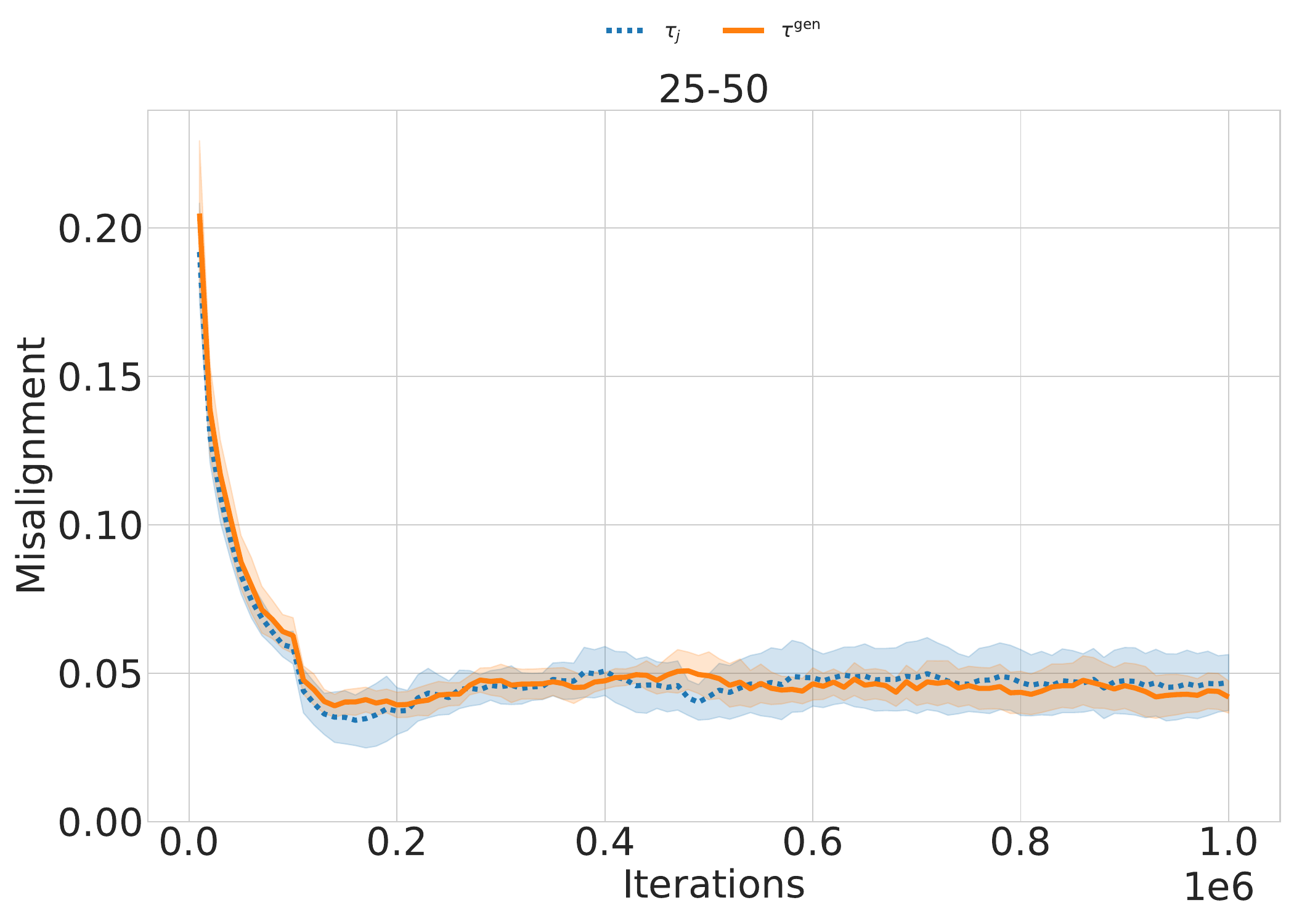}
    \end{subfigure}
    \hfill
    \begin{subfigure}[b]{0.32\textwidth}
        \centering
        \includegraphics[width=\textwidth,trim=0 0 0 81,clip]{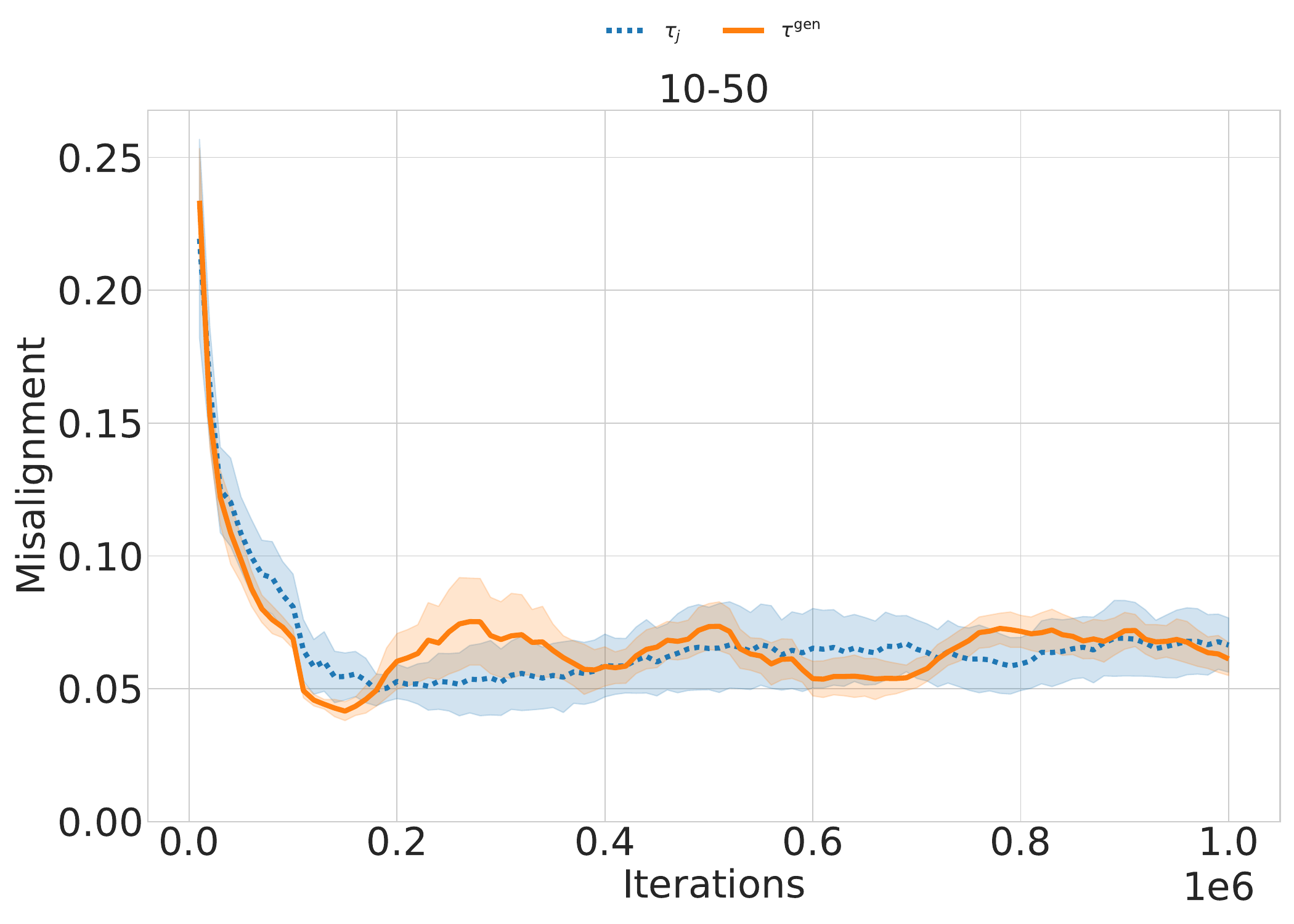}
    \end{subfigure}
    
    \caption{ SlowHop learning curves for ReCOIL+FMR when using $\tau_j$ and $\tau^\text{gen}$. The shaded region represents the standard error.}
    \label{fig:hop-ataus}
\end{figure}

\begin{figure}[h]
    \centering
    
    \begin{subfigure}[b]{\textwidth}
        \centering
        \includegraphics[width=0.8\textwidth,trim=0 665 0 0,clip]{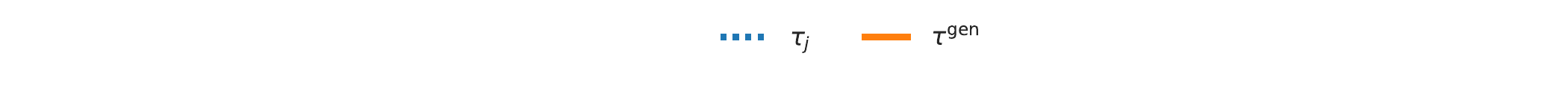}
    \end{subfigure}
    
    \vspace{0.5em}
    
    \begin{subfigure}[b]{0.32\textwidth}
        \centering
        \includegraphics[width=\textwidth,trim=0 44 0 55,clip]{images/gen-tau/SlowWalk/return_50-50.pdf}
    \end{subfigure}
    \hfill
    \begin{subfigure}[b]{0.32\textwidth}
        \centering
        \includegraphics[width=\textwidth,trim=0 44 0 55,clip]{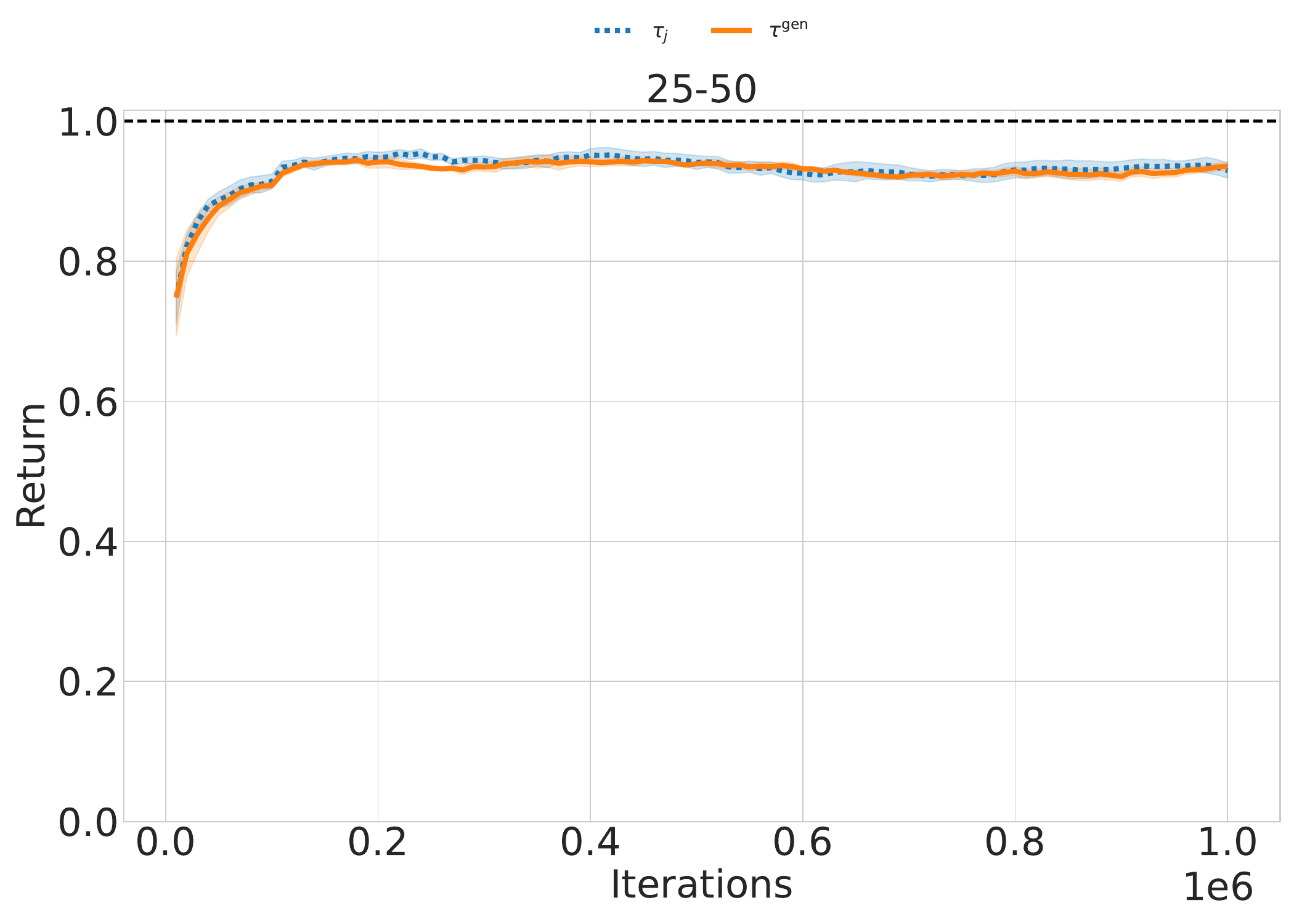}
    \end{subfigure}
    \hfill
    \begin{subfigure}[b]{0.32\textwidth}
        \centering
        \includegraphics[width=\textwidth,trim=0 44 0 55,clip]{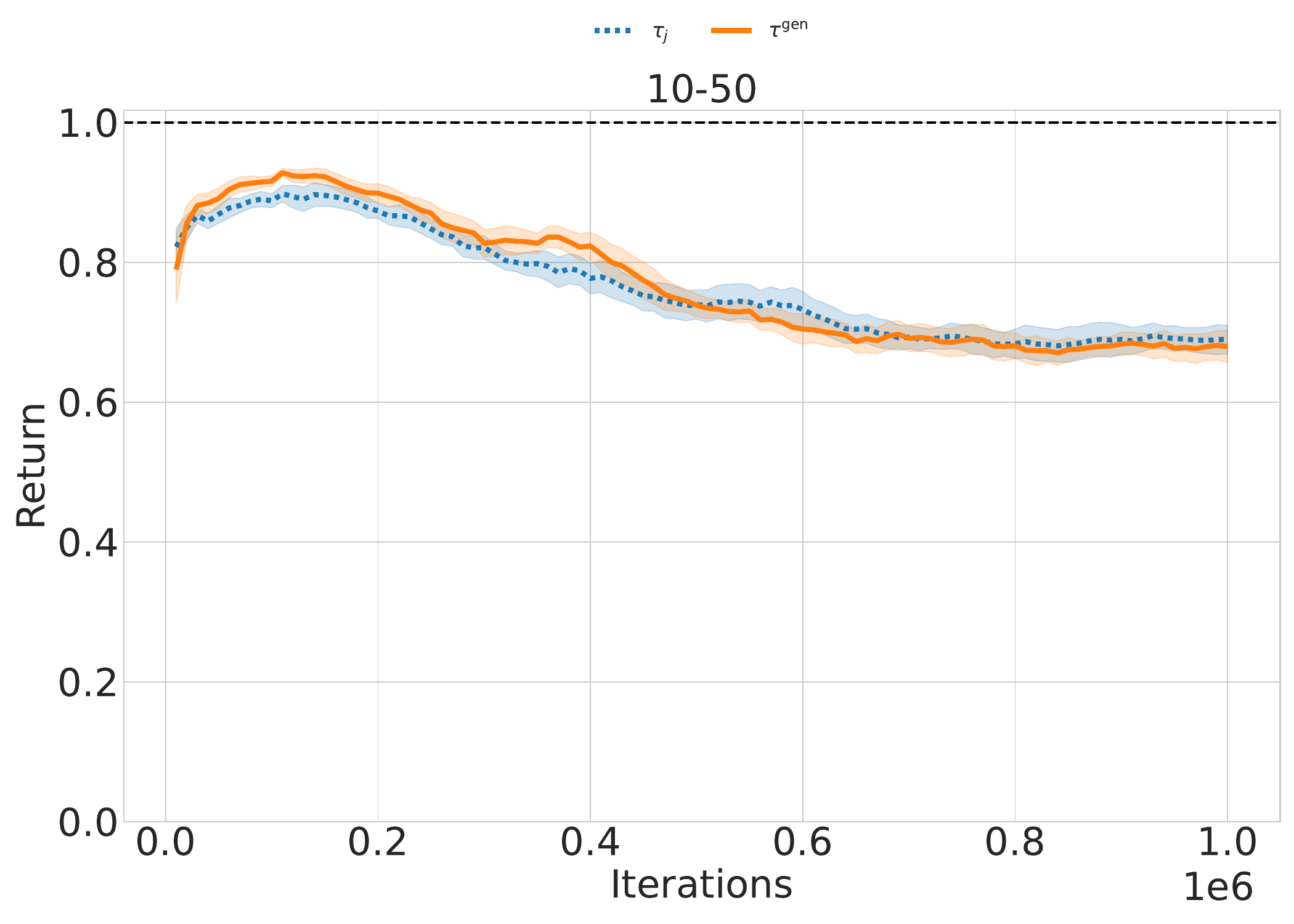}
    \end{subfigure}
    
    \begin{subfigure}[b]{0.32\textwidth}
        \centering
        \includegraphics[width=\textwidth,trim=0 0 0 81,clip]{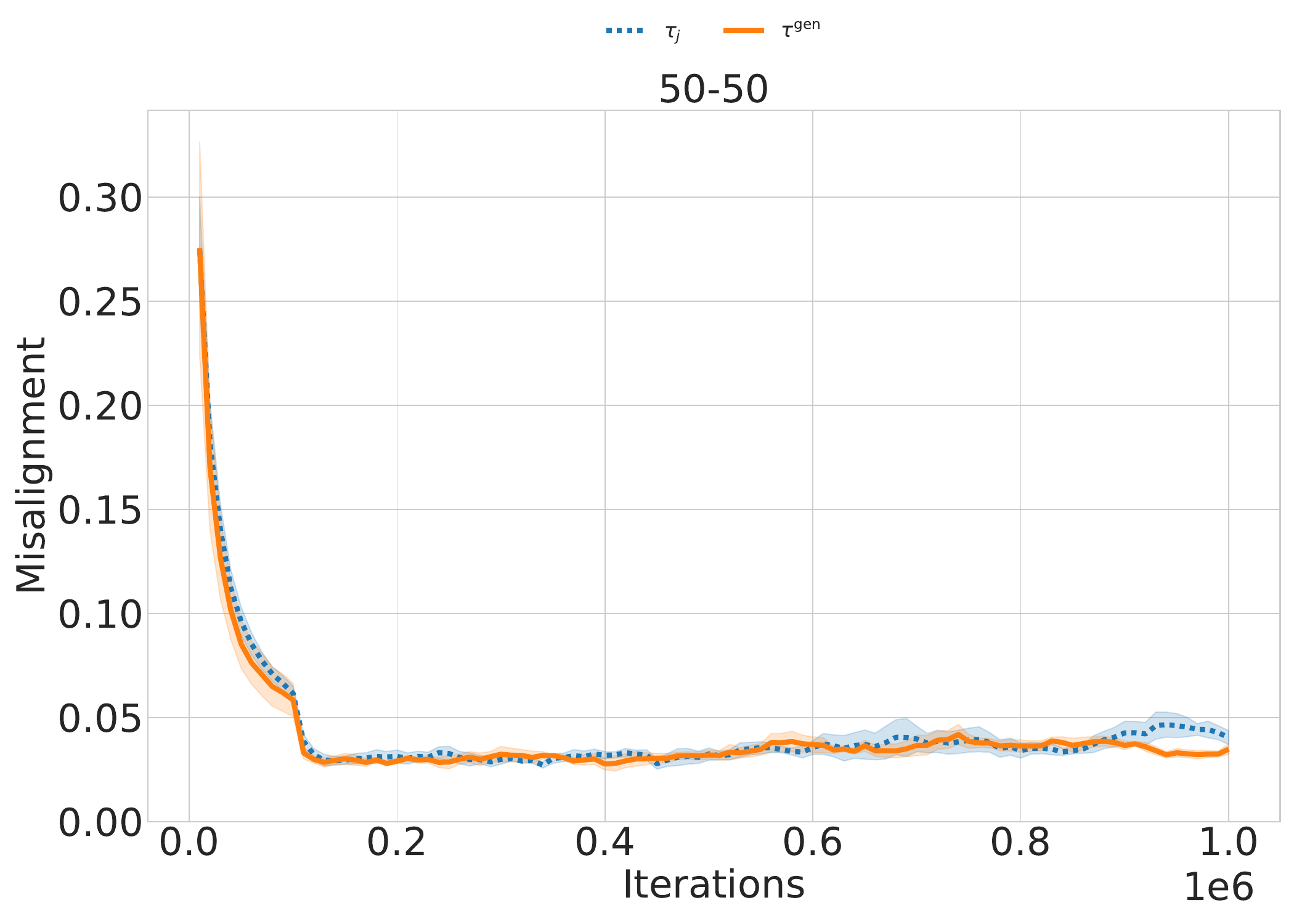}
    \end{subfigure}
    \hfill
    \begin{subfigure}[b]{0.32\textwidth}
        \centering
        \includegraphics[width=\textwidth,trim=0 0 0 81,clip]{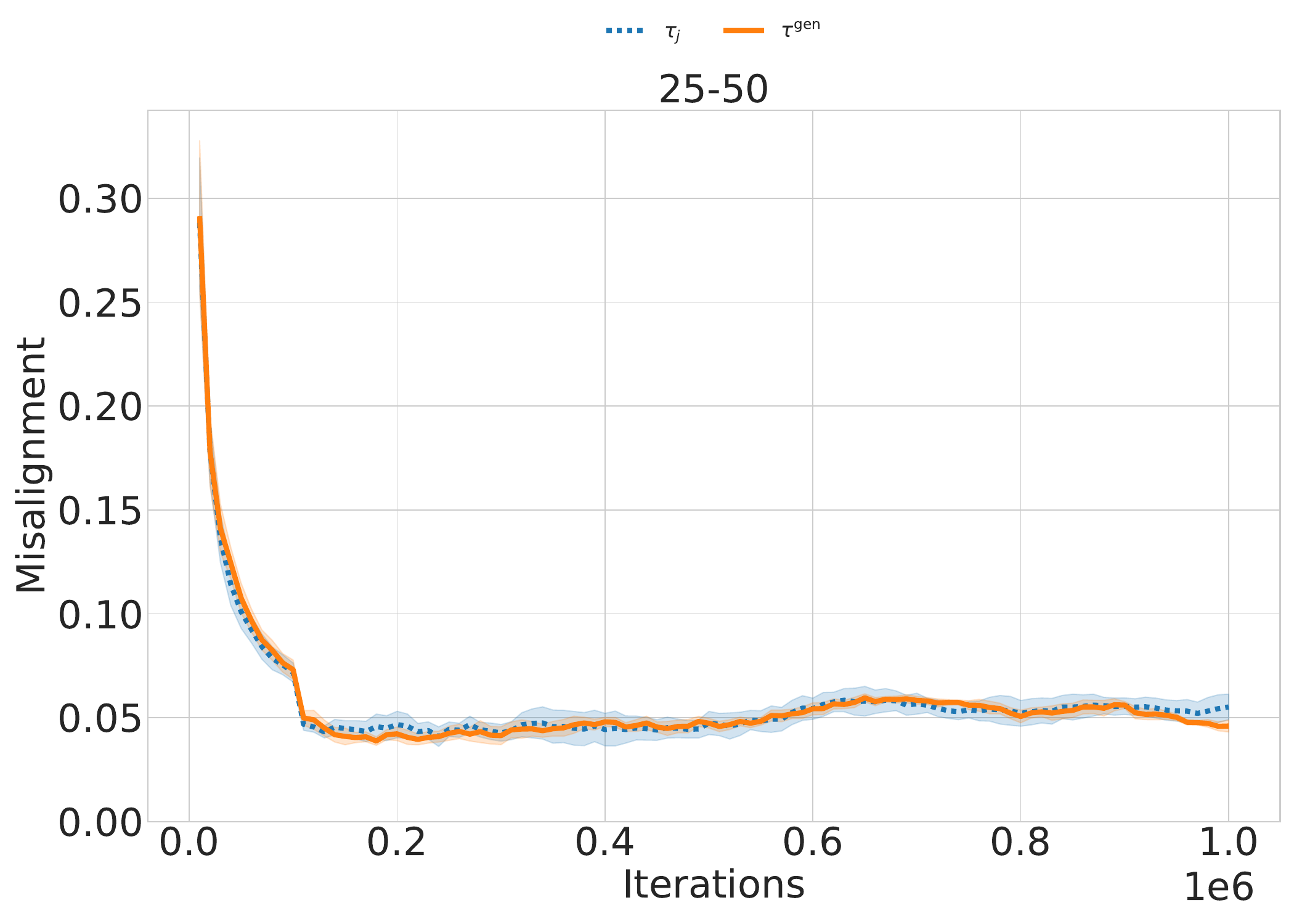}
    \end{subfigure}
    \hfill
    \begin{subfigure}[b]{0.32\textwidth}
        \centering
        \includegraphics[width=\textwidth,trim=0 0 0 81,clip]{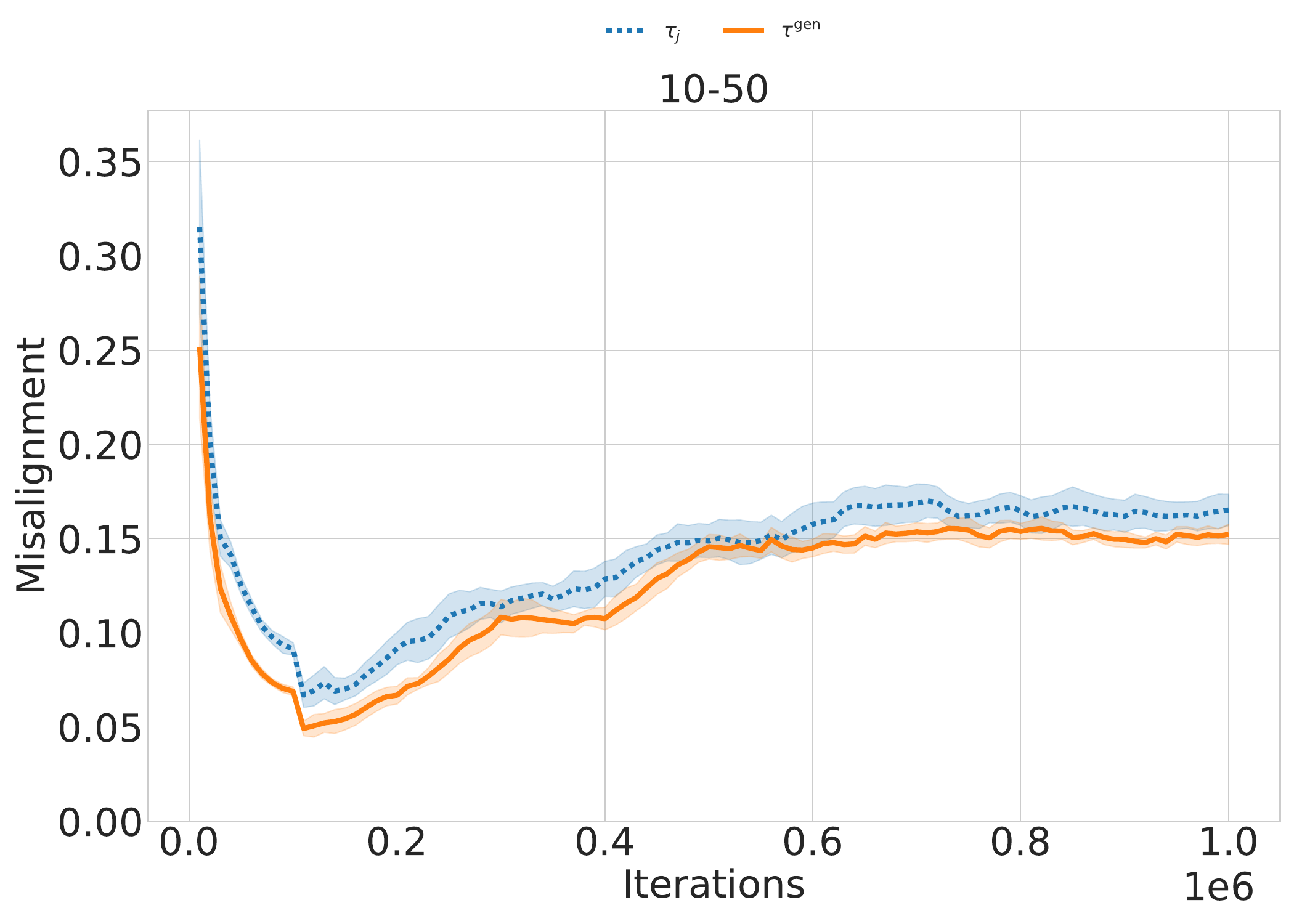}
    \end{subfigure}
    
    \caption{SlowWalk learning curves for ReCOIL+FMR when using $\tau_j$ and $\tau^\text{gen}$. The shaded region represents the standard error.}
    \label{fig:walk-ataus}
\end{figure}

%% file: appendix/hyperparams.tex
\subsection{$\beta$ Values}
Below are the results for various $\beta$ values when computing the temperature $\tau_j$ based on feedback for ReCOIL+FM.

\begin{table}[h]
\centering
\caption{PathBB comparison of different $\beta$ values for ReCOIL+FMR with a 10-50 data ratio. Results show mean ± std for the last 10 evaluations, over 5 seeds.}
\label{tab:pathbb-beta}
\begin{tabular}{ccc}
\toprule
$\beta$ & Suc. & Mis. \\
\midrule
\multirow[t]{1}{*}{1.5} & 0.592 ± 0.49 & 0.134 ± 0.24 \\
\multirow[t]{1}{*}{10} & 0.708 ± 0.45 & 0.097 ± 0.21 \\
\multirow[t]{1}{*}{100} & 0.702 ± 0.46 & 0.191 ± 0.32 \\
\multirow[t]{1}{*}{1000} & 0.709 ± 0.45 & 0.098 ± 0.24 \\
\bottomrule
\end{tabular}
\end{table}

\begin{figure}[h]
    \centering
    
    \begin{subfigure}[b]{\textwidth}
        \centering
        \includegraphics[width=0.8\textwidth,trim=0 665 0 0,clip]{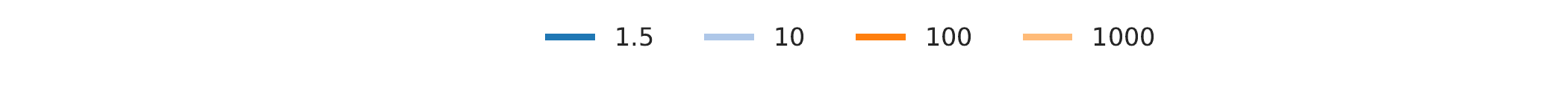}
    \end{subfigure}
    
    \vspace{0.5em}
    
    \begin{subfigure}[b]{0.48\textwidth}
        \centering
        \includegraphics[width=\textwidth,trim=0 0 0 55,clip]{images/beta/PathBB/success_10-50.pdf}
    \end{subfigure}
    \hfill
    \begin{subfigure}[b]{0.48\textwidth}
        \centering
        \includegraphics[width=\textwidth,trim=0 0 0 55,clip]{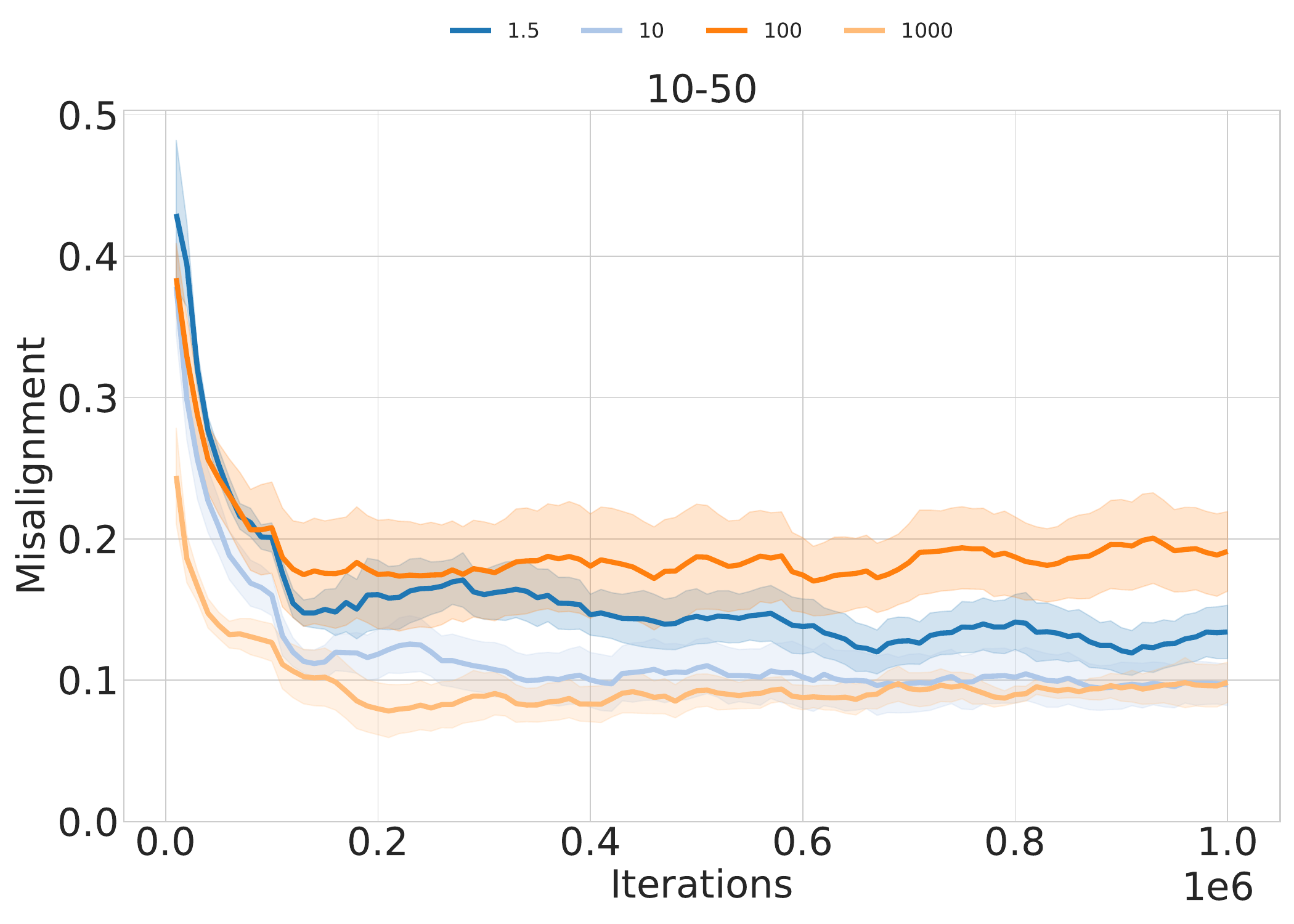}
    \end{subfigure}
    
    \caption{PathBB learning curves of $\beta$ values for ReCOIL+FMR with a 10-50 data ratio. The shaded region represents the standard error.}
    \label{fig:pathbb-beta}
\end{figure}

\clearpage
\subsection{$\alpha$ Values}
Below are the results for various $\alpha$ values for controlling the regularization or alignment strength of ReCOIL+FM. 

\begin{table}[h]
\centering
\caption{PathBB performance comparison of different $\alpha$ values for ReCOIL+FMR with a 10-50 data ratio. Results show mean ± std for the last 10 evaluations, over 5 seeds.}
\label{tab:pathbb-alpha}
\begin{tabular}{ccc}
\toprule
$\alpha$ & Suc. & Mis. \\
\midrule
\multirow[t]{1}{*}{0.1} & 0.506 ± 0.50 & 0.209 ± 0.31 \\
\multirow[t]{1}{*}{1}  & 0.708 ± 0.45 & 0.097 ± 0.21 \\
\multirow[t]{1}{*}{10} & 0.672 ± 0.47 & 0.092 ± 0.21 \\
\multirow[t]{1}{*}{100} & 0.510 ± 0.50 & 0.074 ± 0.17 \\
\bottomrule
\end{tabular}
\end{table}

\begin{figure}[h]
    \centering
    
    \begin{subfigure}[b]{\textwidth}
        \centering
        \includegraphics[width=0.8\textwidth,trim=0 665 0 0,clip]{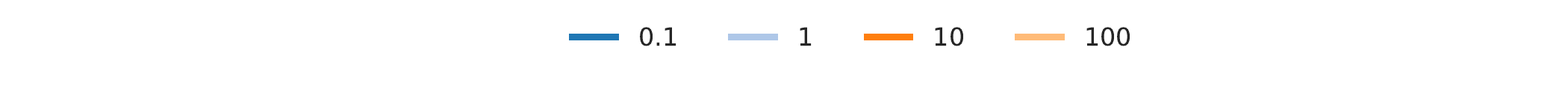}
    \end{subfigure}
    
    \vspace{0.5em}
    
    \begin{subfigure}[b]{0.48\textwidth}
        \centering
        \includegraphics[width=\textwidth,trim=0 0 0 55,clip]{images/kld/PathBB/success_10-50.pdf}
    \end{subfigure}
    \hfill
    \begin{subfigure}[b]{0.48\textwidth}
        \centering
        \includegraphics[width=\textwidth,trim=0 0 0 55,clip]{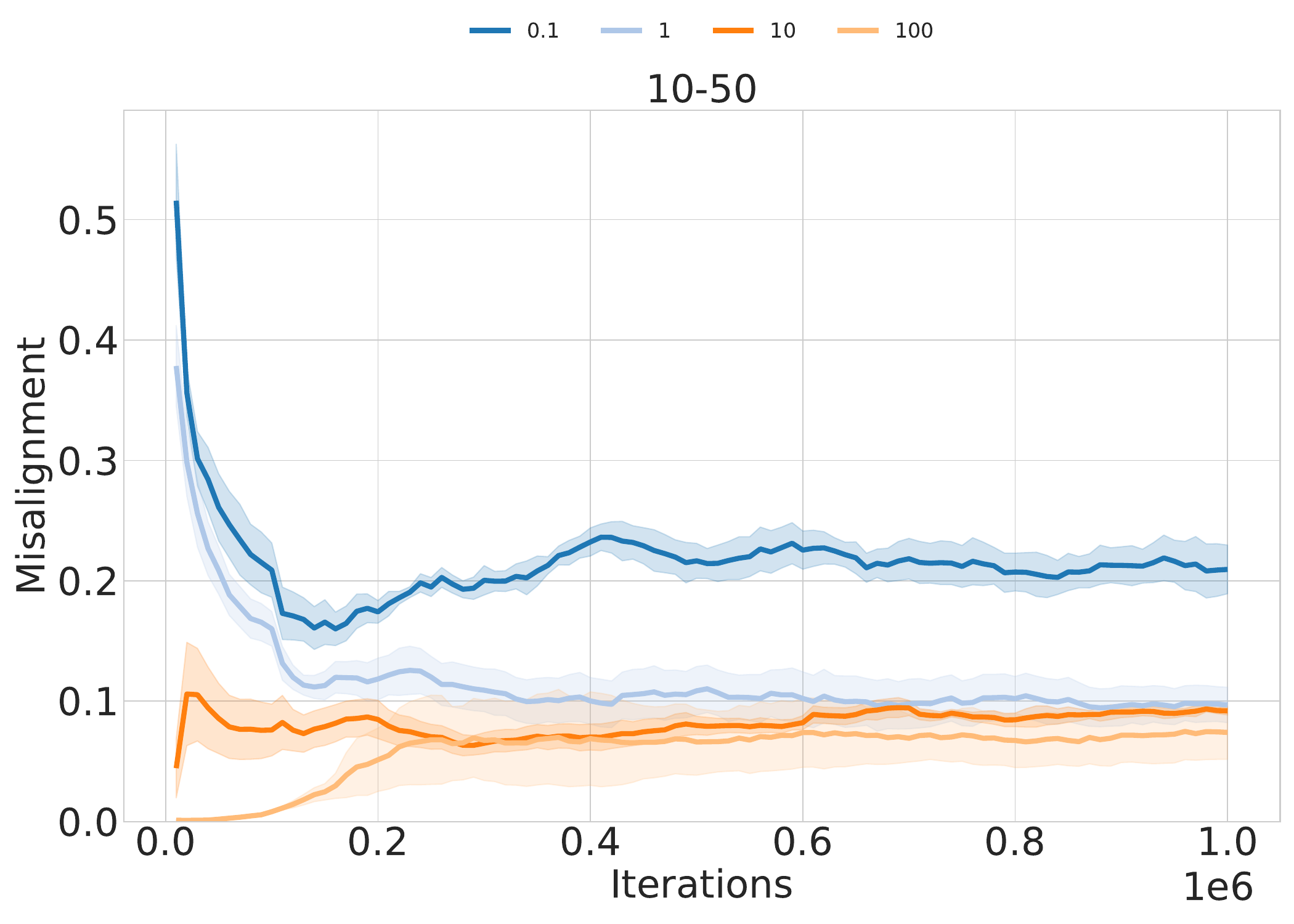}
    \end{subfigure}
    
    \caption{PathBB learning curves of $\alpha$ values for ReCOIL+FMR using PathBB with a 10-50 data ratio. The shaded region represents the standard error.}
    \label{fig:pathbb-alpha}
\end{figure}

\clearpage
\subsection{Credit Assignment Values}
Below are the results for time windows when applying credit assignment to feedback for ReCOIL+FMR. 

\begin{table}[h]
\centering
\caption{PathBB performance comparison of different credit assignment times for ReCOIL+FMR with a 10-50 data ratio. Results show mean ± std for the last 10 evaluations, over 5 seeds.}
\label{tab:pathbb-ca}
\begin{tabular}{ccc}
\toprule
Time & Suc. & Mis. \\
\midrule
\multirow[t]{1}{*}{300ms} & 0.688 ± 0.46 & 0.124 ± 0.25 \\
\multirow[t]{1}{*}{600ms} & 0.708 ± 0.45 & 0.097 ± 0.21 \\
\multirow[t]{1}{*}{1000ms} & 0.758 ± 0.43 & 0.119 ± 0.30 \\
\bottomrule
\end{tabular}
\end{table}

\begin{figure}[h]
    \centering
    
    \begin{subfigure}[b]{\textwidth}
        \centering
        \includegraphics[width=0.8\textwidth,trim=0 665 0 0,clip]{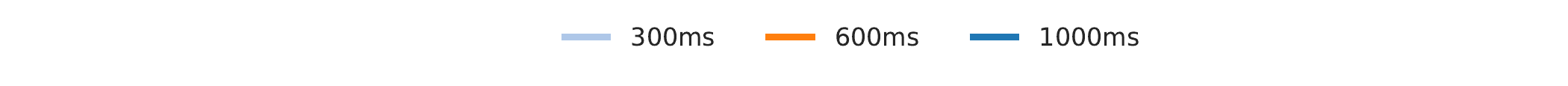}
    \end{subfigure}
    
    \vspace{0.5em}
    
    \begin{subfigure}[b]{0.48\textwidth}
        \centering
        \includegraphics[width=\textwidth,trim=0 0 0 55,clip]{images/ca/PathBB/success_10-50.pdf}
    \end{subfigure}
    \hfill
    \begin{subfigure}[b]{0.48\textwidth}
        \centering
        \includegraphics[width=\textwidth,trim=0 0 0 55,clip]{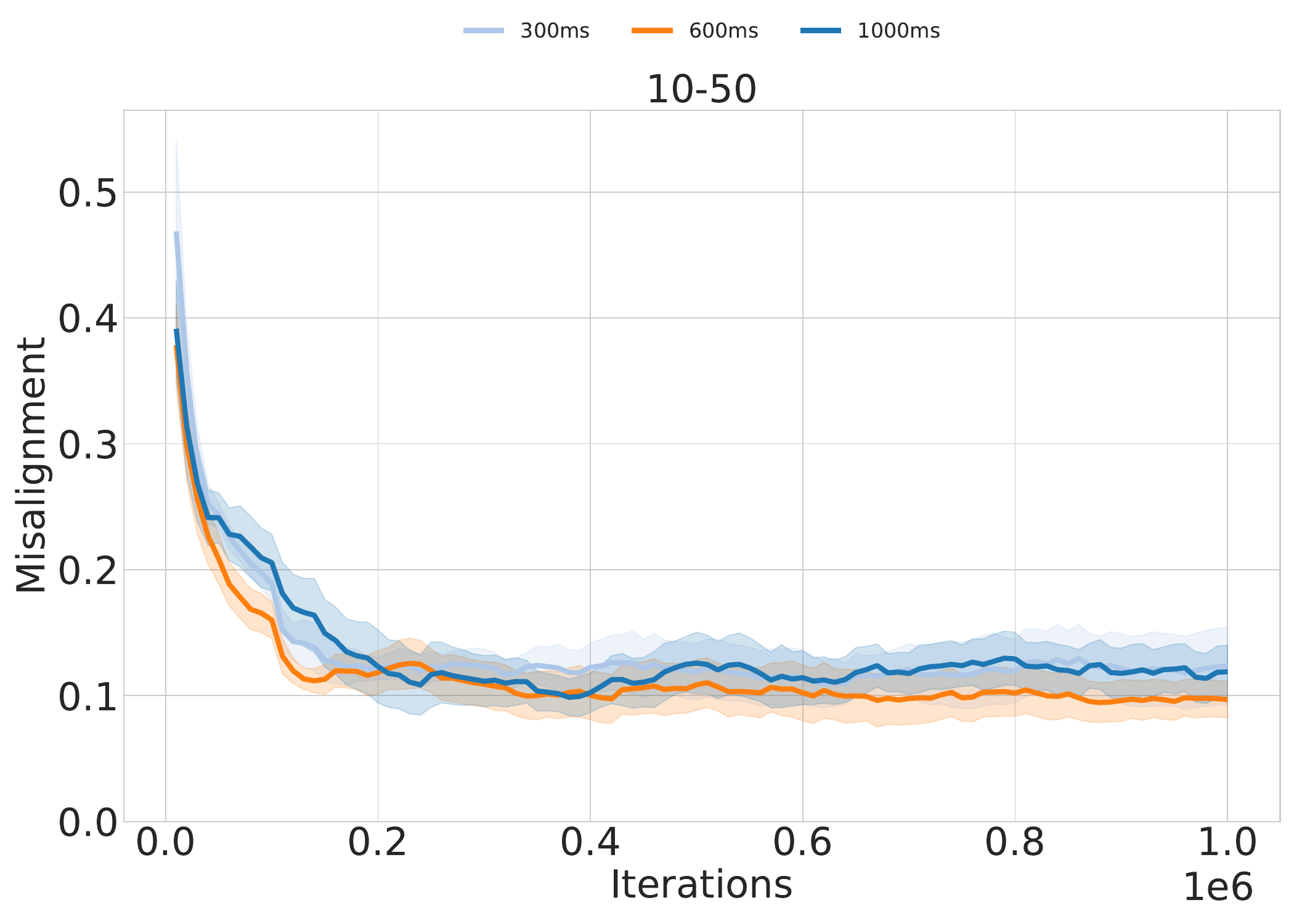}
    \end{subfigure}
    
    \caption{PathBB Learning curves of various credit assignment times for ReCOIL+FMR with a 10-50 data ratio. The shaded region represents the standard error.}
    \label{fig:pathbb-ca}
\end{figure}

%% file: appendix/entropy.tex
\begin{figure}[h]
    \centering
    \begin{subfigure}[b]{0.48\textwidth}
        \centering
        \includegraphics[width=\textwidth]{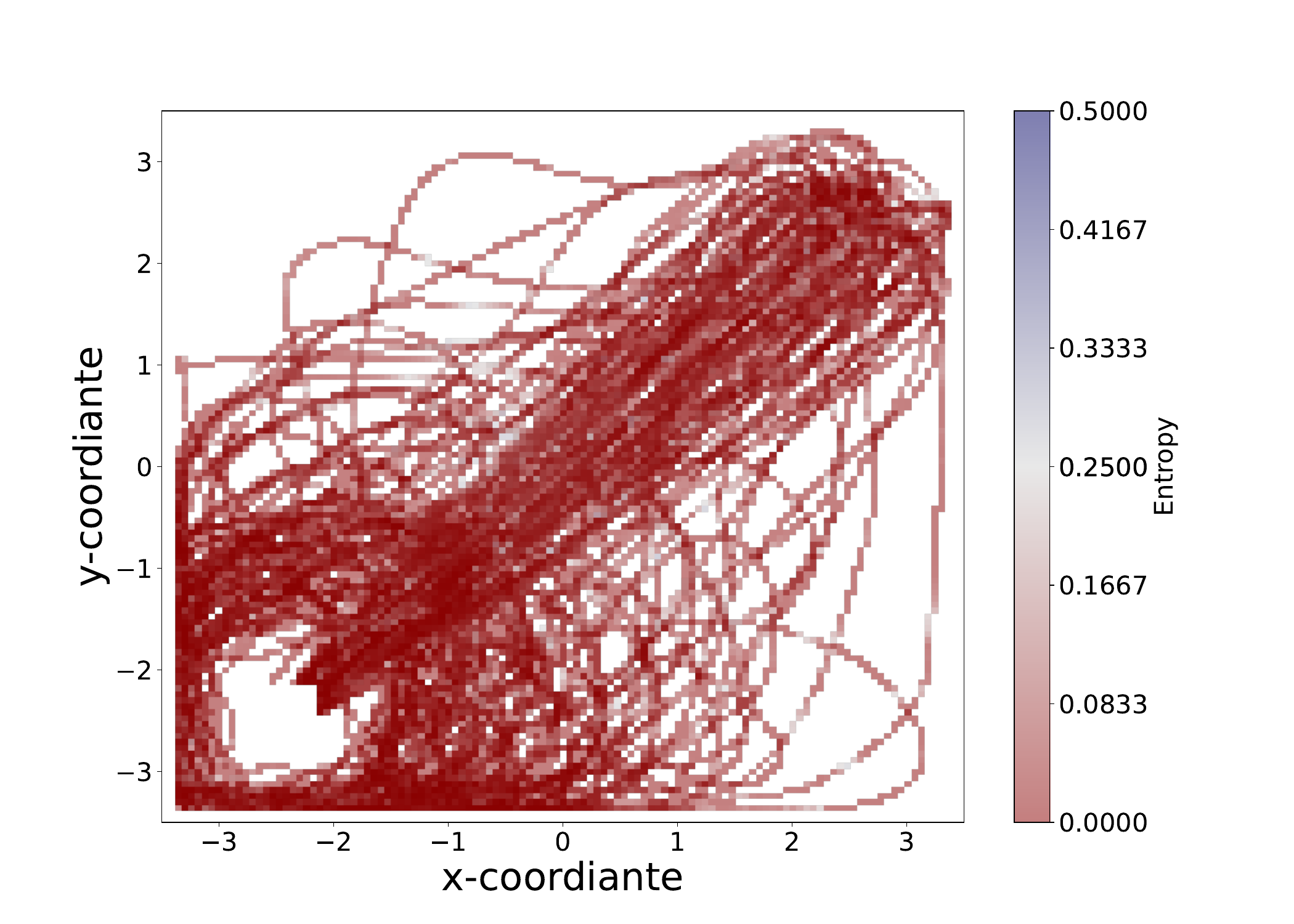}
        \caption{ReCOIL}
    \end{subfigure}
    \hfill
    \begin{subfigure}[b]{0.48\textwidth}
        \centering
        \includegraphics[width=\textwidth]{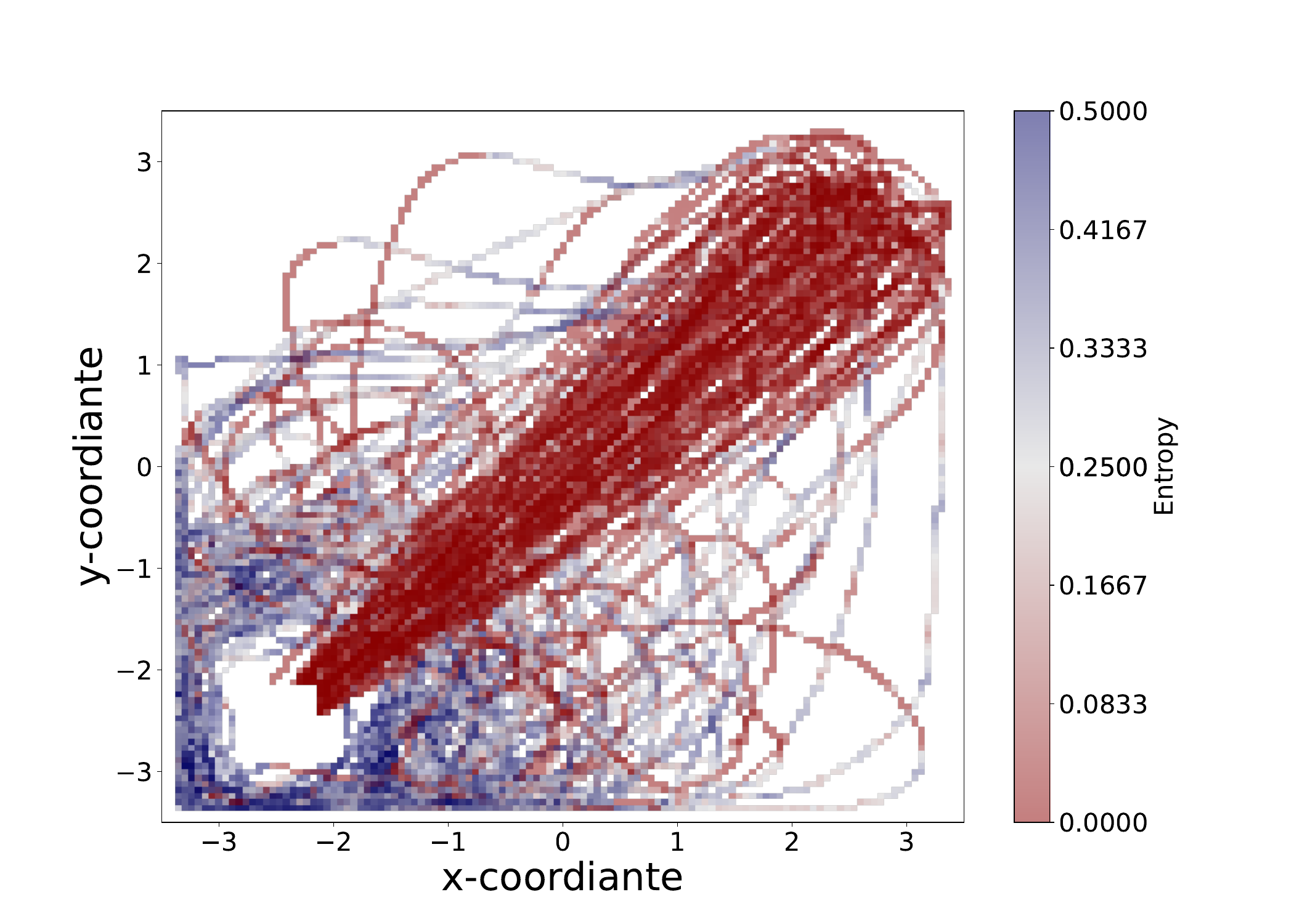}
        \caption{ReCOIL+FMR}
    \end{subfigure}
    \caption{PathM heatmap of normalized entropy for ReCOIL and ReCOIL+FMR.}
    \label{fig:pathm-entropy}
\end{figure}

\begin{figure}[h]
    \centering
    \begin{subfigure}[b]{0.48\textwidth}
        \centering
        \includegraphics[width=\textwidth]{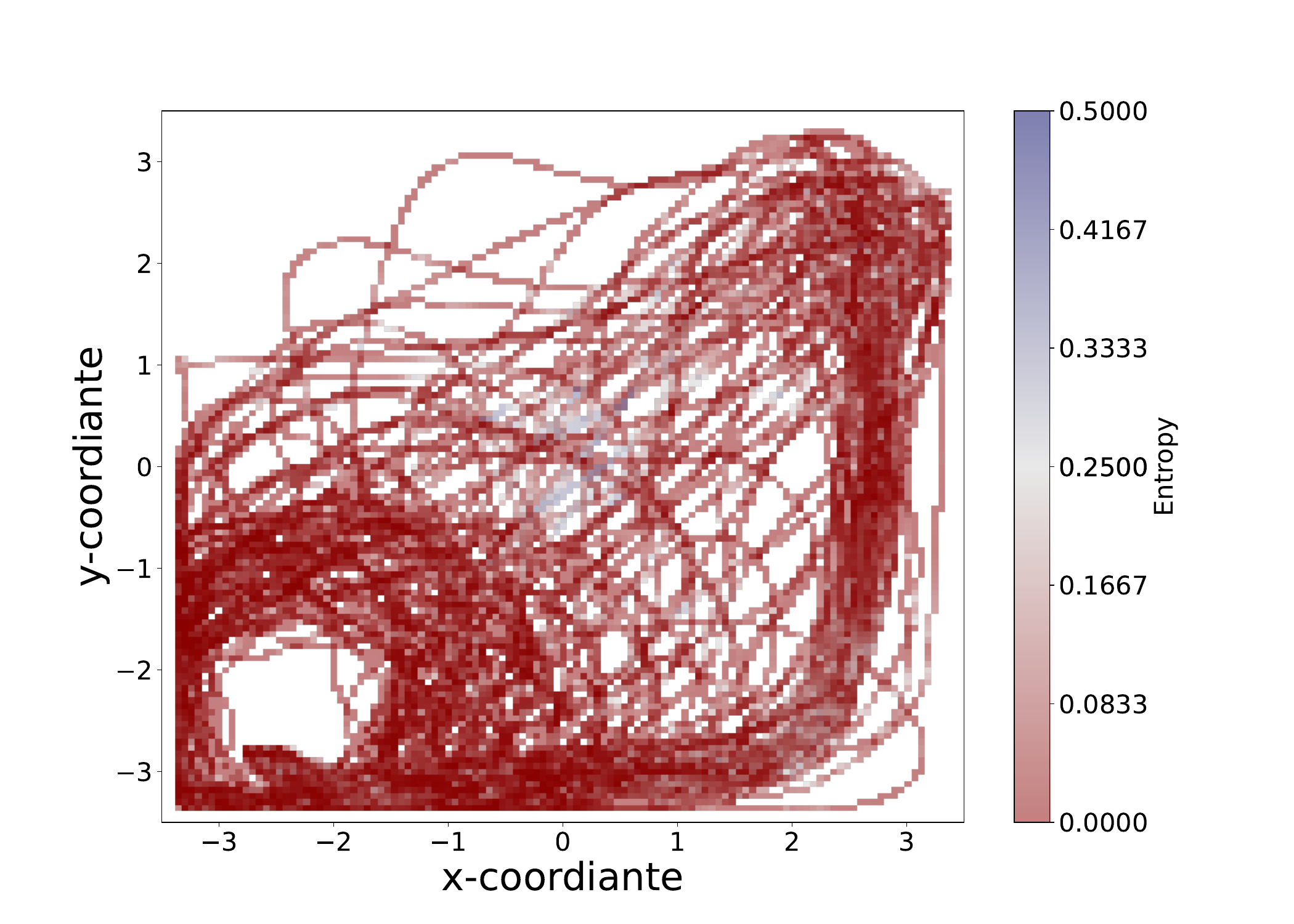}
        \caption{ReCOIL}
    \end{subfigure}
    \hfill
    \begin{subfigure}[b]{0.48\textwidth}
        \centering
        \includegraphics[width=\textwidth]{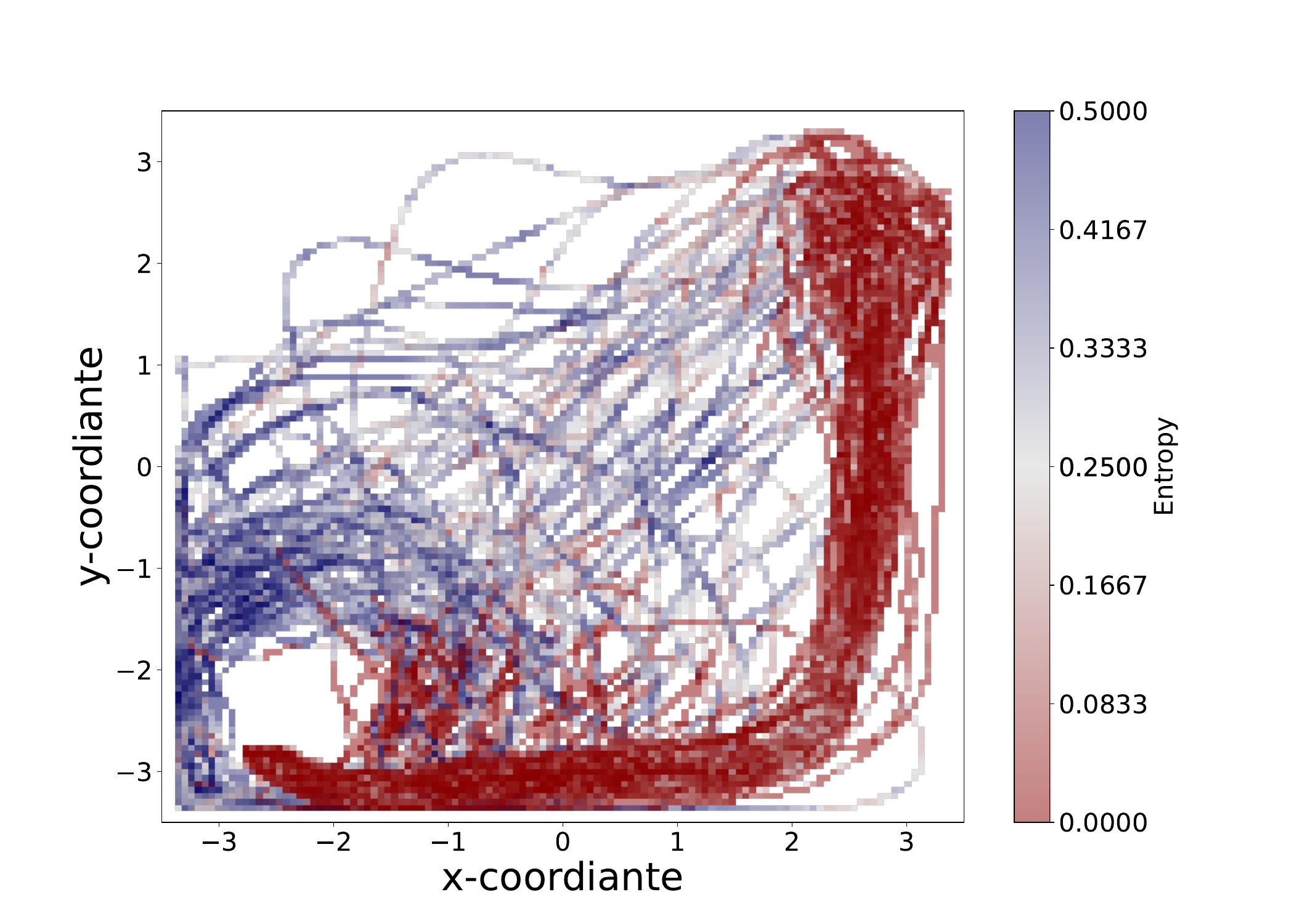}
        \caption{ReCOIL+FMR}
    \end{subfigure}
    \caption{PathBB heatmap of normalized entropy for ReCOIL and ReCOIL+FMR.}
    \label{fig:pathbb-entropy}
\end{figure}